%% file: main.tex
\PassOptionsToPackage{table}{xcolor}
\documentclass[11pt]{article}
\usepackage{enumerate}
\usepackage[OT1]{fontenc}
\usepackage[usenames]{color}
\usepackage{smile}
\usepackage{booktabs}
\usepackage[colorlinks,
linkcolor=red,
anchorcolor=blue,
citecolor=blue
]{hyperref}
\usepackage{mathrsfs}
\usepackage{fullpage}
\usepackage{hyperref}
\usepackage[protrusion=true, expansion=true]{microtype}
\usepackage{float}
\usepackage{subfigure}
\usepackage{amsfonts,amsmath,amssymb,amsthm,url,xspace,mathtools,authblk}
\usepackage{tikz}
\usepackage{verbatim}
\usetikzlibrary{arrows,shapes}
\usepackage[bottom]{footmisc}
\usepackage{enumitem}
\usepackage{lscape,soul,diagbox}
\usepackage{caption}
\usepackage{algorithmicx,algpseudocode}
\usepackage{math}
\usepackage{extarrows}

\ifx\counterwithout\undefined\usepackage{chngcntr}\fi
\counterwithout{equation}{section}

\allowdisplaybreaks[1]
\mathtoolsset{showonlyrefs=true}

\usepackage{xargs}
\usepackage[colorinlistoftodos,prependcaption,textsize=tiny]{todonotes}
\newcommandx{\unsure}[2][1=]{\todo[linecolor=red,backgroundcolor=red!25,bordercolor=red,#1]{#2}}
\newcommandx{\change}[2][1=]{\todo[linecolor=blue,backgroundcolor=blue!25,bordercolor=blue,#1]{#2}}
\newcommandx{\info}[2][1=]{\todo[linecolor=OliveGreen,backgroundcolor=OliveGreen!25,bordercolor=OliveGreen,#1]{#2}}
\newcommandx{\improvement}[2][1=]{\todo[linecolor=Plum,backgroundcolor=Plum!25,bordercolor=Plum,#1]{#2}}

\usepackage{alphalph}
\definecolor{AveRSGreen}{RGB}{232,245,233}
\newcommand{\gc}[1]{\cellcolor{AveRSGreen}#1}
\newcommand{\gmc}[1]{\multicolumn{1}{|c|}{\cellcolor{AveRSGreen}#1}}
\newcommand{\Gthree}[3]{\gc{#1} & \gmc{#2} & \gc{#3}}

\newcommand{\gaXi}{\Xi^\star_{\text{gasn, ave}}}
\newcommand{\aaXi}{\Xi^\star_{\text{asn, ave}}}
\newcommand{\uaXi}{\Xi^\star_{\text{usn, ave}}}

\newcommand{\alXi}{\Xi^\star_{\text{asn, last}}}

\newcommand{\answerTODO}[1][]{\textcolor{red}{\bf [TODO]}}
\newcommand{\justificationTODO}[1][]{\textcolor{red}{\bf [TODO]}}

\newcommand{\papertitle}{Inference for Newton Methods with Accelerated Sketch-and-Project via Random Scaling}

\begin{document}
\date{}

\title{\papertitle}

\author[1]{Xinchen Du}
\author[2]{Elizaveta Rebrova}
\author[3]{Micha{\l} Derezi\'{n}ski}
\author[1]{Sen Na}

\affil[1]{School of Industrial and Systems Engineering, Georgia Institute of Technology}
\affil[2]{Department of Operations Research and Financial Engineering, Princeton University}
\affil[3]{Computer Science and Engineering, University of Michigan}

\maketitle

\begin{abstract}

We study an online sketched Newton method that approximates the Newton direction at each step via a state-of-the-art sketching solver, called the \textit{generalized accelerated sketch-and-project solver} (\texttt{GAS}), thereby mitigating the computational bottleneck of classical second-order methods. The \texttt{GAS} solver improves upon vanilla, unaccelerated sketch-and-project solvers by achieving accelerated convergence through Nesterov momentum updates, and accommodates a flexible projection metric whose proper choice further reduces computational cost. 
Building on this design, we establish asymptotic normality of the \textit{averaged} sketched Newton iterates and characterize their limiting covariance matrix. The resulting covariance recovers that of the unaccelerated sketched Newton method under a specific choice of acceleration parameters, converges more rapidly (in the number of sketching steps) to the minimax-optimal covariance in general, and is smaller than that of the last iterate produced by the accelerated method.
Finally, we strengthen these results by establishing a \textit{functional central limit theorem} for the Newton iterates, which allows us to bypass explicit covariance estimation and develop an online inference procedure based on random scaling. Specifically, we construct a pivotal test statistic by appropriately rescaling the averaged iterates, so that its limiting distribution is free of any unknown parameters, enabling asymptotically valid online inference. Numerical experiments demonstrate superior performance of the proposed inference procedure.

\end{abstract}

\section{Introduction}\label{sec:1}

We consider the following stochastic optimization problem:
\begin{equation}\label{prob}
\min_{\bx\in\mathbb{R}^d} f(\bx)=\mE_{\xi\sim \P}[F(\bx;\xi)],
\end{equation}
where $f: \mathbb{R}^d \rightarrow \mathbb{R}$ is a stochastic, strongly convex objective function, $F(\cdot;\xi)$ is its noisy realization, and $\xi \sim \P$ is a random variable. Despite its simplicity, Problem \eqref{prob} serves as a fundamental building block and arises in a wide range of online optimization tasks in data science, operations research, and machine learning \citep{ShalevShwartz2009Stochastic, Li2010Contextual, McMahan2013Ad, Du2022High}. This problem can be naturally cast as a parameter estimation problem: $\bx$ denotes the model parameter, $\xi$ corresponds to a data sample drawn from $\mathcal{P}$, and the minimizer $\bx^{\star} \coloneqq \argmin_{\bx \in \mR^d} f(\bx)$ represents the underlying true model parameter.

In recent years, a growing body of literature has focused on developing \textit{online} methods for stochastic optimization problems motivated by their computational and memory efficiency over offline methods. Stochastic gradient descent (SGD), introduced by \cite{Robbins1951Stochastic, Kiefer1952Stochastic}, is one of the most widely used online methods for Problem \eqref{prob}. Given a fresh sample $\xi_t$ at each step, SGD updates
\begin{equation}\label{equ:SGD_update}
\bx_{t+1} = \bx_t - \varphi_t \nabla F(\bx_t; \xi_t),
\end{equation}
with a decaying step size $\varphi_t$, and converges to $\bx^{\star}$ both almost surely and in mean square under standard conditions \citep{Robbins1971convergence, Bertsekas2000Gradient, Moulines2011Non}. 
However, modern applications increasingly involve high-stakes decision-making problems in economics \citep{Kleinberg2017Human, Athey2021Policy}, scientific machine learning \citep{Karniadakis2021Physics, Psaros2023Uncertainty}, and reinforcement learning \citep{Garcia2015Comprehensive, Gottesman2019Guidelines, Brunke2022Safe}. For such problems, reliable decisions require more than point estimates; it is essential to quantify uncertainty of online stochastic methods. To this end, one must understand the limiting distribution of the iterates $\bx_t$ and construct confidence sets around the true solution $\bx^{\star}$.
This need has inspired a series of work on online inference for SGD. In particular,  \cite{Ruppert1988Efficient, Polyak1992Acceleration} proposed averaging the SGD iterates: $\bar{\bx}_t = \sum_{i = 0}^{t-1} \bx_i /t$, and proved a central limit theorem of the form 
\begin{equation}\label{nsequ:sgd_normal}
\sqrt{t} (\bar{\bx}_t - \bx^{\star}) \stackrel{d}{\to} \mathcal{N}(0, \Omega^{\star})\quad \text{ with }\quad \Omega^{\star} = (\nabla^2 f(\bx^\star))^{-1}\text{cov}(\nabla F(\bx^{\star}; \xi)) (\nabla^2 f(\bx^\star))^{-1},
\end{equation}
where the limiting covariance $\tOmega$ also achieves \textit{asymptotic minimax optimality} in the sense of H{\'a}jek and Le Cam \citep{Hajek1972Local, LeCam1972Limits}. Subsequent work has extended this result to a broad class of gradient-based algorithms, including implicit, nonsmooth, and moment-adjusted variants \citep{Toulis2017Asymptotic, Liang2019Statistical, Duchi2021Asymptotic, Davis2024Asymptotic, Jiang2025Online}.

To facilitate online inference via asymptotic normality, various covariance matrix estimators have been proposed, including the plug-in estimator \citep{Chen2020Statistical} and the batch-means estimator \citep{Zhu2021Online, Roy2023Online, Jiang2025Online}. However, the former is intrusive, as it requires estimating the inverse Hessian that is not needed for SGD itself; while the latter, although computationally more efficient, converges slowly both in theory and practice.
In fact, covariance estimation is not necessary for inference; instead, one can construct pivotal statistics whose limiting distributions are free of unknown quantities, thereby bypassing covariance estimation altogether. One prominent approach is the \emph{random-scaling method} \citep{Lee2022Fast}, which self-normalizes the error of averaged SGD iterates via a random-scaling matrix. This framework has since been extended to ROOT-SGD \citep{Luo2022Covariance}, stochastic approximation under Markovian data \citep{Li2023Online}, weighted-averaged SGD \citep{Wei2023Weighted}, the Kiefer-Wolfowitz method \citep{Chen2024Online}, and sequential quadratic programming~\citep{Du2025Online}.

While first-order methods such as \eqref{equ:SGD_update} are efficient per step, their associated inference procedures still require updating proper covariance or random-scaling matrices, leading to an $O(d^2)$ time complexity. Moreover, stochastic gradient noise can induce high variance in the iterates, and first-order methods are sensitive to ill-conditioning \citep{Ruder2016Overview, Bottou2018Optimization}. These limitations motivate the use of Newton-type methods, which leverage curvature information to stabilize the iterates and address the issue of ill-conditioning.
Specifically, second-order (quasi-)Newton methods update the iterates~as:
\begin{equation}\label{equ:second_order_method}
\bx_{t+1} = \bx_t + \varphi_t \Delta \bx_t \quad \quad \text{with} \quad  \quad  B_t \Delta \bx_t = - \nabla F(\bx_t; \xi_t),
\end{equation}
where $\varphi_t$ is a decaying step size, and $B_t \approx \nabla^2 f(\bx_t)$ is a conditioning matrix approximating the~objective Hessian.
Recent work has increasingly focused on online inference for \eqref{equ:second_order_method}; see, e.g., \cite{Leluc2023Asymptotic, GodichonBaggioni2025Adaptive, Cenac2025efficient, Gao2025Online}. However, inverting the Hessian is computationally expensive, typically requiring $O(d^3)$ time per~step.

To address this challenge, \cite{Gower2015Randomized} introduced a sketch-and-project solver for solving linear systems with $O(d^2)$ time complexity. Leveraging this solver, \cite{Na2025Statistical} developed a sketched Newton method (for constrained problems), referred to as the unaccelerated sketched Newton (USN) method. The authors established the asymptotic normality of the last USN iterate $\bx_t$, while the subsequent work of \cite{Du2025Online} extended this result to the averaged USN iterate $\bar{\bx}_t = \sum_{i = 0}^{t-1} \bx_i /t$:
\begin{equation}\label{nsequ:normal}
1/\sqrt{\varphi_t}(\bx_t-\tx) \stackrel{d}{\to} \mathcal{N}(0, \Xi^\star_{\text{usn, last}})\quad\quad\quad \text{and} \quad\quad \quad \sqrt{t}(\bar{\bx}_t-\tx) \stackrel{d}{\to} \mathcal{N}(0, \Xi^\star_{\text{usn, ave}}).
\end{equation}
Moreover, \cite{Du2025Online} also showed that under the same scaling (i.e., $\varphi_t=1/t$), we have $\Xi^\star_{\text{usn, ave}}\preceq \Xi^\star_{\text{usn, last}}$, indicating that \textit{iterate averaging improves statistical efficiency}.

In parallel, recent advances in randomized numerical linear algebra have accelerated the convergence of the sketch-and-project solver in \cite{Gower2015Randomized}. In particular, \cite{Gower2018Accelerated, Derezinski2025Fine} incorporated Nesterov momentum~acceleration into the solver without increasing its computational cost. They showed that the expected squared error decays exponentially at rate $1 - \mu_t$ for vanilla unaccelerated solver, but at a faster~rate $1 - \sqrt{\mu_t / \nu_t}$ for accelerated solver. Here, $(\mu_t,\nu_t)$ are iteration-dependent parameters satisfying $1 \le \nu_t \le 1 / \mu_t$ (implying $1-\sqrt{\mu_t / \nu_t}\leq 1-\mu_t$; see \eqref{prop_mu_nu}). Inspired by this improved rate, \cite{Wang2026Inference} developed an online Newton method with accelerated sketch-and-project (ASN), and established asymptotic normality for the \textbf{last} ASN iterate $\bx_t$, analogous to the left result in \eqref{nsequ:normal} but with a different covariance~$\Xi^\star_{\text{asn, last}}$.

Despite these advances in the inference of USN and ASN methods \citep{Na2025Statistical, Kuang2025Online, Du2025Online, Wang2026Inference}, existing studies face \textbf{three key limitations}. \textbf{(a)} The sketch-and-project solvers used in these works, whether unaccelerated or accelerated, are less practical and general than those originally developed in \cite{Gower2018Accelerated, Derezinski2025Fine}. In particular, state-of-the-art solvers adopt a general projection metric, referred to as \textit{generalized accelerated sketch-and-project solver} (\texttt{GAS}), while the inference literature considers only the identity projection metric. With alternative choices of the projection metric, \texttt{GAS} can significantly reduce computational cost compared to the identity-metric setting. 
\textbf{(b)} While iterate averaging stabilizes the Newton iterates (cf. the right result of \eqref{nsequ:normal}), its integration with \texttt{GAS} remains open, and no online inference procedure has been developed in this setting to bypass covariance estimation.
\textbf{(c)} Multiple limiting covariance matrices arise in the analysis, including $\Xi^\star_{\text{usn, last}}$, $\Xi^\star_{\text{usn, ave}}$, $\Xi^\star_{\text{asn, last}}$, $\Xi^\star_{\text{asn, ave}}$, and the minimax-optimal~covariance $\tOmega$. While it is known that $\tOmega \preceq \Xi^\star_{\text{usn, ave}} \preceq \Xi^\star_{\text{usn, last}}$, the relationships between the covariances induced by accelerated sketching, $\Xi^\star_{\text{asn, last}}$ and $\Xi^\star_{\text{asn, ave}}$, and all others remain unclear. Understanding the ordering is crucial for characterizing \textit{computational-statistical trade-offs} of sketched~Newton~methods.

\vskip4pt
\noindent $\bullet$ \textbf{Our contributions.} In this paper, we propose an online Newton method with generalized accelerated sketch-and-project (GASN) by embedding \texttt{GAS} into \eqref{equ:second_order_method}. We quantify how \texttt{GAS} affects the uncertainty and statistical efficiency of the stochastic Newton iterates. In particular, we show both almost-sure~and $L^2$ convergence of GASN, and establish asymptotic normality for the \textbf{averaged} GASN iterate $\bar{\bx}_t$:
\begin{equation}\label{nsequ:gasn_normal}
\sqrt{t} (\bar{\bx}_t - \bx^{\star}) \stackrel{d}{\to} \mathcal{N}(\boldsymbol{0}, \Xi^{\star}_{\text{gasn, ave}}).
\end{equation}
Furthermore, we prove that $\Xi^\star_{\text{gasn, ave}}$ reveals a computational-statistical trade-off induced by the \texttt{GAS} solver, with three regimes that connect to existing methods.

\textbf{(a) Comparison with Exact Newton.} If the linear system in \eqref{equ:second_order_method} is solved exactly, then $\gaXi = \Omega^{\star}$ and we match the statistical efficiency of the averaged SGD estimator \eqref{nsequ:sgd_normal}, which is also minimax optimal \citep{Polyak1992Acceleration, Duchi2021Asymptotic, Davis2024Asymptotic}. 
When \texttt{GAS} is deployed, $\gaXi \succeq \Omega^{\star}$, but the gap is negligible as $\|\gaXi - \Omega^{\star}\| = O((1-\sqrt{\mu^\star/\nu^\star})^\tau)$, which decays exponentially in the number of sketching steps $\tau$. Here, $(\mu^\star, \nu^\star)$ are sketching-related quantities evaluated at $\tx$ (Proposition \ref{prop:comp_stats_tradeoff}).

\textbf{(b) Comparison with USN.} If we set $\mu_t \nu_t = 1$, i.e., the regime where provable acceleration of \texttt{GAS} may not be achieved, then $\gaXi = \uaXi$ and we match the efficiency of averaged USN iterates in \eqref{nsequ:normal}. For \texttt{GAS} with $\mu_t \nu_t < 1$, $\gaXi$ converges to $\tOmega$ with a provably \textbf{faster} rate than $\uaXi$, since $\|\uaXi - \Omega^{\star}\| = O((1-\mu^\star)^\tau)$ and $1-\mu^\star \geq 1-\sqrt{\mu^\star/\nu^\star}$. This suggests that~GASN improves computational efficiency \textbf{without} sacrificing statistical efficiency.

\textbf{(c) Comparison with ASN.} GASN reduces to ASN when using the identity projection metric. Compared with \cite{Wang2026Inference}, which studied the normality of the last ASN iterate, we can show $\aaXi \preceq \alXi$, indicating that iterate averaging also stabilizes the variability of the ASN method.

With the above complete understanding of uncertainty quantification for sketched Newton methods, we then turn these normality results into practical inference. Inspired by \cite{Lee2022Fast, Li2023Online, Chen2024Online, Du2025Online}, we avoid estimating the covariance $\gaXi$ but studentize $\bar{\bx}_t$ by a \textit{random scaling} matrix $V_t$ (see \eqref{def:V_t}). We then show that the resulting test statistic is \textit{asymptotically pivotal}: for any direction $\boldsymbol{w} \in \mR^{d}$,
\begin{equation*}
\frac{\sqrt{t} \cdot \boldsymbol{w}^{\top} (\bar{\bx}_t - \bx^{\star})}{\sqrt{\boldsymbol{w}^{\top} V_t \boldsymbol{w}}} \stackrel{d}{\longrightarrow} \frac{W_1(1)}{\sqrt{\int_{0}^1  \left(W_1(r) - rW_1(1) \right)^2 dr}},
\end{equation*}
where $W_1(\cdot)$ denotes standard one-dimensional Brownian motion, and the limiting distribution on the right-hand side is free of unknown parameters. This result requires us to strengthen \eqref{nsequ:gasn_normal} to a functional central limit theorem, and also allows us to perform inference under weaker moment conditions on the gradient noise and sketching distribution than those required for covariance estimation \citep{Kuang2025Online, Wang2026Inference}. We demonstrate the promising empirical performance of GASN and its associated inference procedure through experiments on linear and logistic regression problems.

\vskip4pt

\noindent$\bullet$ \textbf{Notation.} Throughout the paper, $\|\cdot\|$ denotes the $\ell_2$ norm for vectors and the spectral norm for matrices, $\|\cdot\|_{F}$ denotes the Frobenius norm, and $\mathrm{Tr}(\cdot)$ denotes the trace operator. For a positive definite matrix $B$, the matrix-induced norm is defined as $\|\bx\|_B \coloneqq \sqrt{\bx^{\top} B \bx}$. For sequences $\{a_t, b_t\}$, we write $a_t = O(b_t)$, or equivalently $a_t \lesssim b_t$, if $|a_t| \leq c\,|b_t|$ for some constant $c > 0$ and all sufficiently large $t$, and $a_t = o(b_t)$ if $|a_t/b_t| \to 0$ as $t \to \infty$. We denote $O_p, o_p$ the order in probability. We denote $I$ as the identity matrix, $\boldsymbol{0}$ as the zero vector or matrix, $\boldsymbol{e}_i$ as the $i$-th standard basis vector, and $\boldsymbol{1}$ as the all-ones vector; their dimensions are clear from the context. For a compatible sequence of matrices $\{A_i\}$, we let $\prod_{k=i}^{j} A_k = A_j A_{j-1} \cdots A_i$ when $j \geq i$ and $I$ when $j < i$. The Loewner ordering is written as $A \succeq B$ (resp.\ $A \preceq B$) to mean $A - B$ (resp.\ $B - A$) is positive semidefinite, and $\lambda_{\min}(A)$, $\lambda_{\max}(A)$ denote the smallest and largest eigenvalues of $A$. We write $f_t = f(\bx_t)$ and $f^\star = f(\bx^\star)$, with analogous shorthands for $\nabla f_t$, $\nabla^2 f_t$, and so on.

\section{Online Newton with Generalized Accelerated Sketch-and-Project}\label{sec:2}

In this section, we introduce an online Newton method based on a state-of-the-art sketching solver, the \textit{generalized accelerated sketch-and-project solver} \texttt{GAS} from \cite{Gower2018Accelerated, Derezinski2025Fine}. The method is a direct~extension of online Newton methods \citep{Na2022Hessian, Na2025Statistical, Kuang2025Online, Wang2026Inference, Gao2025Online} by replacing their linear system solvers with \texttt{GAS}.

At each outer iteration $t \geq 0$, we draw a sample $\xi_t \sim \mathcal{P}$, and compute the stochastic objective~gradient and Hessian estimates as $\bg_t = \nabla F(\bx_t; \xi_t)$ and $H_t = \nabla^2 F(\bx_t; \xi_t)$. We then compute the averaged Hessian $B_t \coloneqq \frac{1}{t}\sum_{i=0}^{t-1} H_i$ based on samples $\{\xi_i\}_{i = 0}^{t-1}$ (let $B_0=I$). Next, we apply the \texttt{GAS} solver and update as:
\begin{equation}\label{equ:update}
\bx_{t+1} = \bx_t + \varphi_t \cdot \texttt{GAS}(B_t,- \bg_t; E_t,\alpha_t,\beta_t,\gamma_t,\tau),
\end{equation}
where $\varphi_t = C_{\varphi} \cdot (t+1)^{-\varphi}$ is a decreasing step size sequence with $C_{\varphi} > 0$ and $\varphi \in (0.5, 1)$.  Here, \texttt{GAS}($\cdot$) returns an approximation to the Newton direction $\Delta \bx_t$:
\begin{equation}\label{equ:Newton_system} 
B_t \Delta \bx_t = - \bg_t,
\end{equation}
which can be performed in $O(d^2)$ time. To keep notation light, we drop the outer iteration index $t$ and summarize $\texttt{GAS}(B, -\bg; E, \alpha, \beta, \gamma, \tau)$ in Algorithm \ref{alg:generalized}. For this method, $(B, - \bg)$ are the left- and right-hand side matrix and vector; $E$ is the projection metric; $(\alpha, \beta,\gamma)$ are acceleration parameters;~and $\tau$ is the number of sketching steps. We refer to \cite{Gower2018Accelerated, Derezinski2025Fine} for further~details.

\begin{algorithm}[H]
\caption{\textbf{G}eneralized \textbf{A}ccelerated \textbf{S}ketch-and-Project Solver: $\texttt{GAS}(B, -\bg; E, \alpha, \beta, \gamma, \tau)$}\label{alg:generalized}
\begin{algorithmic}[1]
\State \textbf{Target:} $B\Delta\bx = - \bg$;
\State {\bfseries Initialize:} set initial values $\bz_{0} = \boldsymbol{v}_{0} = \boldsymbol{0}$, a positive definite matrix $E \in \mR^{d \times d}$;
\For{$j = 0, 1, \ldots \tau-1$}
\State $\by_{j} = \alpha \boldsymbol{v}_{j} + (1-\alpha) \bz_{j}$;
\State Generate a sketching vector/matrix $S_{j}\sim S\in\mathbb{R}^{d\times s}$;
\State $\boldsymbol{\omega}_{j} = E^{-1} B S_{j} (S_{j}^{\top} B E^{-1} B S_{j})^{\dagger} S_{j}^{\top}(B \by_{j} + \bg)$, where $\dagger$ denotes pseudoinverse;
\State $\bz_{j+1} = \by_{j} - \boldsymbol{\omega}_{j}$;
\State $\boldsymbol{v}_{j+1} = \beta \boldsymbol{v}_{j} + (1-\beta) \by_{j} - \gamma \boldsymbol{\omega}_{j}$;
\EndFor
\State \textbf{Output:} $\bz_\tau$.
\end{algorithmic}
\end{algorithm}

\vspace{-0.3cm}

\noindent $\bullet$ \textbf{Parameter setup and accelerated rate.} \cite{Gower2018Accelerated, Derezinski2025Fine} proposed defining three acceleration parameters $(\alpha,\beta,\gamma)$, governed by two parameters $(\mu, \nu)$, as
\begin{equation}\label{def:alpha_beta_gamma_t}
\alpha = 1 / (1 + \gamma \nu), \quad\quad \beta = 1 - \sqrt{\mu / \nu}, \quad\quad \gamma = 1/ \sqrt{\mu \nu},
\end{equation}
where $(\mu, \nu)$ are defined as 
\begin{equation}\label{def:mu_t_nu_t} 
\mu = \inf_{\ba\neq\0} \frac{\ba^\top Z \ba}{ \ba^\top\ba}, \quad\quad \nu = \sup_{\ba \neq\0} \frac{\ba^\top \mathbb{E}[\Tilde{Z} Z^{-1} \Tilde{Z}]\ba}{\ba^\top Z\ba}.
\end{equation}
Here, $\Tilde{Z} \coloneqq E^{-1/2} B S (S^{\top} B E^{-1} B S)^{\dagger}S^{\top}B E^{-1/2}$ is a projection matrix, and $Z \coloneqq \mathbb{E}[\Tilde{Z}]\in\mR^{d\times d}$ denotes its expectation taken over the randomness of the sketching distribution $S$ only (more precisely, in our case, the expectation is conditional on the current iterate $\bx_t$ and sample $\xi_t$). By \citet[Lemma~2]{Gower2018Accelerated}, we have the following property:
\begin{equation} \label{prop_mu_nu}
1 \leq \nu \leq 1/\mu =  \|Z^{-1}\|,
\end{equation}
implying that $\alpha\in(0,1)$, $\beta\in[0,1)$, $\gamma\geq 1$.

\begin{remark}[\textbf{Accelerated rate}]
\citet[Remark 6]{Derezinski2025Fine} shows that \texttt{GAS} achieves a linear convergence rate $\mathbb{E} \|\bz_j - \Delta\bx\|_E^2 \leq 2 (1 - \sqrt{\mu/\nu})^j \| \bz_0 - \Delta\bx \|_E^2$, where $\Delta\bx=-B^{-1}\bg$ is the exact solution. For the unaccelerated sketch-and-project solver \citep{Gower2015Randomized}, corresponding to $(\alpha, \beta,\gamma) = (0.5,0,1)$, \citet[Theorem 4.6]{Gower2015Randomized} shows the linear rate $\mathbb{E} \|\bz_j - \Delta\bx\|_E^2 \leq (1 - \mu)^j \| \bz_0 - \Delta\bx \|_E^2$. 
Comparing these two results, we observe that the provable acceleration disappears \textbf{if and only if} $\mu \nu = 1$, i.e., $\gamma = 1$. This condition is weaker than the parameter setting of the unaccelerated solver, as momentum can still be employed, although it no longer yields an improved convergence rate.

In addition, the accelerated rate indeed implies substantial and practically observable speedups. For coordinate sketches, i.e., sampling $S_j$ from canonical basis $\{\boldsymbol{e}_i\}_{i=1}^d$, also called randomized Kaczmarz method \citep{Strohmer2008Randomized}, \citet[Corollary 4]{Gower2018Accelerated} gives $\mu\nu = \lambda_{\min}(B)/\min_i B_{ii}$. Thus, the accelerated factor is strictly smaller, and the improvement becomes more pronounced when $\min_i B_{ii}$ is much larger~than $\lambda_{\min}(B)$. This includes ill-conditioned regimes where eigenvalues of $B$ concentrate near the~largest eigenvalue.

\end{remark}

\begin{remark}[\textbf{Comparison to prior works}]

The GASN method \eqref{equ:update} is a direct application of the sketch-and-project solver in \cite{Gower2018Accelerated, Derezinski2025Fine} to online Newton methods; however, our analysis of uncertainty quantification and inference for this method significantly extends prior works in \cite{Na2025Statistical, Du2025Online, Wang2026Inference}.

\textbf{First}, all prior works set the projection metric matrix of the sketch-and-project solver $E = I$, while \texttt{GAS} allows a general positive definite choice of $E$. This includes an important alternative $E = B$, which is particularly advantageous for positive definite systems. In this case, the sketched direction reduces to $\boldsymbol{\omega}_j = S_j (S_j^{\top} B S_j)^{\dagger} S_j^{\top}(B\boldsymbol{y}_j+ \bg)$. This form can be much computationally cheaper than the choice $E = I$, which involves the matrix $S_j^{\top} B^2 S_j$. Indeed, for coordinate sketches, $S_j^{\top} B S_j$ corresponds to a sampled entry (or, more generally, a small submatrix) of $B$. In contrast, forming $S_j^{\top} B^2 S_j = (B S_j)^{\top} (B S_j)$ requires accessing the full column $B S_j$, which is significantly more~expensive, especially for dense matrices.

\textbf{Second}, \cite{Na2025Statistical, Du2025Online} analyzed the asymptotic normality of the last and averaged USN iterates, while \cite{Wang2026Inference} analyzed the last ASN iterates. To our knowledge, no existing work has analyzed the averaged GASN iterates. More importantly, as discussed in the contributions in Section \ref{sec:1}, we compare our~resulting covariance with all existing methods and show that generalized accelerated sketch-and-project not only improves computational efficiency, but also achieves improved statistical efficiency. In other words, it provides a ``free" gain in statistical efficiency, without incurring more computations. 

\end{remark}

\section{Assumptions and Convergence Guarantees}\label{sec:3}

In this section, we introduce assumptions and present global convergence results for GASN. We first set up the filtration to capture the randomness in GASN (including random samples, random sketches, and random metrics). Define $\mF_{t-1} = \sigma(\{\xi_i,  \psi_i , \{S_{i, j}\}_{j=0}^{\tau-1}\}_{i=0}^{t-1})$ and $\mF_{t-0.5} =  \sigma(\{\xi_i, \psi_i, \{S_{i, j}\}_{j=0}^{\tau-1}\}_{i=0}^{t-1} \\ \cup \xi_t)$, and let $\mF_{-1}$ denote the trivial $\sigma$-algebra for notation consistency. Here, $\{\psi_i\}_{i=0}^{t-1}$ is an auxiliary sequence of random variables introduced into the filtration to provide extra flexibility in the construction of $E_t$. We assume $E_t$ is $\mF_{t-1}$-measurable; in practice, $E_t$ is simply $\sigma(\{\xi_i, \{S_{i, j}\}_{j=0}^{\tau-1}\}_{i=0}^{t-1})$-measurable. By definitions, $\mF_{t-1}$ includes all randomness of $\bx_{0:t}$, while $\mF_{t-0.5}$ additionally includes the sample~$\xi_t$.

We now present assumptions, all of which are very standard in optimization and online inference literature \citep{Chen2024Online, Cenac2025efficient, Na2025Statistical, Kuang2025Online, Du2025Online, Wang2026Inference}. The first assumption requires the objective to be strongly convex~with a Lipschitz-smooth gradient. The boundedness of Hessian can be readily relaxed when incorporating proper Hessian truncation \citep{Bercu2020Efficient,Leluc2023Asymptotic}.

\begin{assumption}\label{ass:1}
We assume that $f(\cdot)$ and $F(\cdot;\xi)$ are twice continuously differentiable. For $\forall \bx, \gamma_H \leq \lambda_{\min} \left(\nabla^2 F(\bx;\xi)\right) \leq \lambda_{\max}\left(\nabla^2 F(\bx;\xi)\right) \leq \Upsilon_H$ for some constants $0 < \gamma_H < \Upsilon_H$. We assume $\nabla^2f$ is $\Upsilon_L$-Lipschitz continuous, i.e., $\|\nabla^2 f(\bx) - \nabla^2 f(\bx^{\prime})\|\leq \Upsilon_{L}\|\bx-\bx^{\prime}\|$ for $\forall \bx,\bx'$.
\end{assumption}

The next assumption imposes a growth condition on the moment of the gradient and Hessian noise, which aligns with existing literature in online inference \citep{Chen2024Online, Du2025Online, Li2023Online}.

\vspace{-0.05cm}
\begin{assumption}\label{ass:2}
	
For any $t\geq 0$, we assume the gradient and Hessian estimates are unbiased: $\mathbb{E}[\bg_t\mid \mathcal{F}_{t-1}]=\nabla f_t$ and $\mathbb{E}[H_t\mid \mathcal{F}_{t-1}]=\nabla^2 f_t$; and we assume the growth conditions on their moments:
\begin{subequations}\label{equ:ass:2}
\begin{align}
\mE \! \left[\|\bg_t - \nabla f_t\|^{2+\epsilon} \mid \mathcal{F}_{t-1}\right] & \leq C_{g,1}\|\bx_t-\bx^\star\|^{2+\epsilon} + C_{g,2}, \label{ass:2:4th} \\
\mE \! \left[\|H_t-\nabla^2 f_t\|^2 \mid \mF_{t-1}\right] & \leq C_{H,1}\|\bx_t-\bx^\star\|^2 + C_{H,2}, \label{ass:3:4th}
\end{align}
\end{subequations}
for some constants $\epsilon, C_{g,1}, C_{g,2}, C_{H,1}, C_{H,2} > 0$.	
\end{assumption}

We would like to highlight that the above growth condition is weaker than those required by existing online inference procedures based on covariance matrix estimation, which typically impose at least a fourth-moment growth condition (i.e., $\epsilon = 2$) \citep{Chen2020Statistical, Zhu2021Online, Davis2024Asymptotic, Na2025Statistical, Kuang2025Online, Wang2026Inference}. It is also weaker than the bounded-moment condition $\mE[\|\bg_t - \nabla f_t\|^{2+\epsilon} \mid \bx_t] \leq C$ assumed in \cite{Du2025Online}. We finally state the following assumptions on the sketching distribution.

\vspace{-0.15cm}
\begin{assumption}\label{ass:4}
For any $t\geq 0$, the sketching matrices $S_{t,j} \stackrel{iid}{\sim} S$ are independent of $\xi_t$ and satisfy $\mE[B_tS(S^{\top}B_t^2S)^{\dagger}S^{\top}B_t \mid \mF_{t-1}] \succeq \gamma_S I$ and $\mE[\|S\|\|S^{\dagger}\|]\leq \Upsilon_S$ for some constants \mbox{$\gamma_S, \Upsilon_S>0$}.~Furthermore, the projection metric $E_t$ satisfies $\gamma_E I \preceq E_t \preceq \Upsilon_E I$ for some constants $\gamma_E, \Upsilon_E > 0$.
\end{assumption}

The two expectations above are taken over the randomness of the sketching distribution $S$, and it is worth noting that our moment condition on $\|S\| \, \|S^\dagger\|$ is weaker than the second-moment condition $\mE[(\|S\| \, \|S^\dagger\|)^2] \leq C$ imposed in \cite{Kuang2025Online, Wang2026Inference}. The lower bound on $B_t S (S^\top B_t^2 S)^\dagger S^\top B_t$ is a standard requirement for sketch-and-project solvers to ensure convergence \citep{Gower2015Randomized, Gower2018Accelerated, Derezinski2025Fine}, and it holds for common sketching distributions, including Gaussian sketches $S \sim \mN(\mathbf{0}, \Sigma)$ and coordinate sketches $S \sim \text{Unif}(\{\be_i\}_{i=1}^d)$ \citep{Strohmer2008Randomized, Na2025Statistical}. The boundedness of $E_t$ holds automatically for the canonical choices $E_t = I$ and $E_t = B_t$, and ensures the invertibility of $Z_t$ (see Lemma \ref{lem:4}), as required by $\nu_t$ in \eqref{def:mu_t_nu_t} (recall that we omit the outer iteration index $t$ in that definition).

\begin{lemma}\label{lem:4}
We recall from \eqref{def:mu_t_nu_t} that $Z_t = \mE[E_t^{-1/2}B_t S (S^{\top}B_t E_t^{-1} B_tS)^{\dagger}S^{\top}B_t E_t^{-1/2} \mid \mF_{t-1}]$. Under Assumption \ref{ass:4}, we have $Z_t \succeq (\gamma_E \gamma_S / \Upsilon_E) I$ for all $ t\geq 0$.
\end{lemma}

In the following theorem, we establish global convergence guarantees for the GASN method.

\begin{theorem}[\textbf{Almost-sure and $L^2$-convergence}]\label{thm:as_convergence}
Let $\{\bx_t\}$ be the iterates generated by \eqref{equ:update}. Under Assumptions \ref{ass:1}, \ref{ass:2}, \ref{ass:4}, and supposing the number of sketching steps satisfies $ \tau \cdot (1 - \sqrt{\mu_t/\nu_t})^{\tau - 2} \leq \gamma_H\sqrt{\gamma_E}/(4 \Upsilon_{H}\sqrt{\Upsilon_E})$ for all $t\geq 0$, then we have $\bx_t \rightarrow \bx^{\star}$ as $t \rightarrow \infty$ almost surely. In addition,
\begin{equation}\label{equ:xt_rate}
\mE[\|\bx_t - \bx^{\star}\|^2] \lesssim \varphi_t \quad \text{for any}\;  t\geq 0,
\end{equation}
where $\varphi_t = C_{\varphi} \cdot (t+1)^{-\varphi}$ is the step size sequence with $C_{\varphi} > 0$ and $\varphi \in (0.5, 1)$.
\end{theorem}

The condition on $\tau$ can be satisfied uniformly. By \eqref{prop_mu_nu} and Lemma \ref{lem:4}, we have $(1 - \sqrt{\mu_t / \nu_t})^{\tau - 2} \leq (1 - \gamma_E \gamma_S / \Upsilon_E)^{\tau - 2}$ for all $t \geq 0$. Since $0 \leq 1 - \gamma_E \gamma_S / \Upsilon_E < 1$, we can indeed choose a uniform $\tau$ to satisfy the condition. This condition is also standard: in the special case $E_t = I$ used by \cite{Kuang2025Online, Wang2026Inference}, we have $\gamma_E = \Upsilon_E = 1$, and our condition reduces to that of \citet[Theorem 3.8]{Wang2026Inference}. Overall, the condition on $\tau$ ensures that the approximate Newton direction returned by the \texttt{GAS} solver remains close (on average) to the exact Newton direction.

\section{Uncertainty Quantification and Online Inference for GASN}\label{sec:4}

In this section, we quantify the uncertainty of GASN, arising from two distinct sources: random sampling of the data and randomized computation in the \texttt{GAS} solver. To this end, we first establish asymptotic normality for the \textbf{averaged} GASN iterate $\bar{\bx}_t$ and characterize its limiting covariance $\Xi^{\star}_{\text{gasn,ave}}$. We then investigate the computational-statistical trade-off by quantitatively comparing $\Xi^{\star}_{\text{gasn,ave}}$ with three benchmark covariance matrices: $\Xi^{\star}_{\text{asn,last}}$, $\Xi^{\star}_{\text{usn,ave}}$, and $\Omega^{\star}$. Finally, we propose a random scaling inference procedure for practical online inference tasks.

\subsection{Asymptotic normality of averaged GASN}\label{sec:4.1}

We need a convergence assumption on $E_t$ to ensure that the projection metric of the \texttt{GAS} solver~stabilizes. 

\begin{assumption}\label{ass:5}  
We assume that $E_t \stackrel{a.s.}{\to} E^{\star}\in \mR^{d \times d}$ and $\mE[\|E_t - E^{\star}\|^2] \lesssim \varphi_t$ for all $t \geq 0$.
\end{assumption}

Assumption \ref{ass:5} is mild enough to accommodate a broad range of choices for $E_t$. The two choices suggested in \cite{Gower2015Randomized, Gower2018Accelerated, Derezinski2025Fine}, $E_t = I$ and $E_t = B_t$, satisfy the assumption without any additional conditions (see Lemma \ref{lem:converge_rate_and_Bt_converge} for the proof of $E_t = B_t$).

We now introduce additional notation. Let $B^\star = \nabla^2 f(\tx)$ and define the projection matrix as~in~\eqref{def:mu_t_nu_t} by $\tilde{Z}^{\star} \coloneqq (E^{\star})^{-1/2}B^{\star} S (S^{\top}B^{\star} (E^{\star})^{-1} B^{\star}S)^{\dagger}S^{\top}B^{\star} (E^{\star})^{-1/2}$. Let $Z^{\star} \coloneqq \mE [\tilde{Z}^{\star}]$, where the expectation is taken over the randomness of $S$. Moreover, we define $(\alpha^{\star}, \beta^{\star}, \gamma^{\star})$ as in \eqref{def:alpha_beta_gamma_t} with $(\mu^{\star}, \nu^{\star})$ in \eqref{def:mu_t_nu_t} evaluated at $\tilde{Z}^\star$ and $Z^{\star}$. For  $S_0, S_1, \cdots, S_{\tau-1} \stackrel{iid}{\sim} S$, we define 
\begin{equation}\label{equ:77}
\setlength{\arraycolsep}{3pt}
\hskip-0.3cm
\tilde{C}^{\star} \hskip-0.05cm \coloneqq \hskip-0.05cm \prod_{j = 0}^{\tau-1}\begin{pmatrix}
(1 - \alpha^{\star})(I - \tilde{Z}_j^{\star}) & \alpha^{\star} (I - \tilde{Z}_j^{\star}) \\
(1 - \alpha^{\star})(1 - \beta^{\star})I - (1 - \alpha^{\star}) \gamma^{\star} \tilde{Z}_j^{\star} & (\alpha^{\star} + \beta^{\star} - \alpha^{\star} \beta^{\star})I - \alpha^{\star} \gamma^{\star} \tilde{Z}_j^{\star}
\end{pmatrix}\hskip-0.05cm \in \hskip-0.05cm \mR^{2d \times 2d}
\end{equation}
where $\tilde{Z}_j^{\star} \coloneqq (E^{\star})^{-1/2}B^{\star} S_j (S_j^{\top}B^{\star} (E^{\star})^{-1} B^{\star}S_j)^{\dagger}S_j^{\top}B^{\star} (E^{\star})^{-1/2} \in \mR^{d \times d}$. We also define 
\begin{equation}\label{def:tilde_H_star}
\tilde{K}^{\star} \coloneqq (E^{\star})^{-1/2} ([\tilde{C}^{\star}]_{1, 1} + [\tilde{C}^{\star}]_{1, 2} ) (E^{\star})^{1/2} = (E^{\star})^{-1/2} \begin{pmatrix}
I & \boldsymbol{0}
\end{pmatrix} \tilde{C} ^{\star} \begin{pmatrix}
I \\ I
\end{pmatrix} (E^{\star})^{1/2} \in \mR^{d \times d},
\end{equation}
where $I \in \mR^{d \times d}$ is the identity matrix, and $[\tilde{C}^\star]_{1,1}, [\tilde{C}^\star]_{1,2} \in \mR^{d \times d}$ denote the upper-left and upper-right $d \times d$ blocks of $\tilde{C}^\star$, respectively. Correspondingly, we define $C^{\star} \coloneqq \mE[\tilde{C}^{\star}]$ and $K^{\star} \coloneqq \mE[\tilde{K}^{\star}]$. 

With these definitions in place, we now present the asymptotic normality of the averaged GASN~method. The result holds for any step size sequence $\varphi_t = C_{\varphi} \cdot (t+1)^{-\varphi}$ with $C_{\varphi} > 0$ and $\varphi \in (0.5, 1)$.

\begin{theorem}[\textbf{Normality of the averaged GASN}]\label{thm:asymptotic_normality}

Under the conditions of Theorem \ref{thm:as_convergence} and~Assumption \ref{ass:5}, we let $\bar{\bx}_t = \sum_{i = 0}^{t-1} \bx_i/t$ and have
\begin{equation}\label{equ:averaged_CLT}
\sqrt{t} \cdot (\bar{\bx}_t - \bx^{\star}) \stackrel{d}{\rightarrow} \mathcal{N} \left(\boldsymbol{0}, \Xi^{\star}_{\text{gasn, ave}} \right),	
\end{equation}
where $\Xi^{\star}_{\text{gasn, ave}} = (I - K^{\star})^{-1} \mE[(I - \tilde{K}^{\star}) \Omega^{\star} (I - \tilde{K}^{\star})^{\top}] (I - K^{\star})^{-\top}$.
\end{theorem}

\noindent $\bullet$ \textbf{Comparison with Exact Newton.} In the following proposition, we characterize the computational-statistical trade-off between GASN and Exact Newton methods.
\vskip2pt

\begin{proposition}[\textbf{GASN vs. EN}]\label{prop:comp_stats_tradeoff}
Under the conditions of Theorem \ref{thm:asymptotic_normality}, we have

\noindent \textbf{(a):} Without the \texttt{GAS} solver (i.e., solving $B_t \Delta \bx_t = - \bg_t$ exactly), we have $\Xi^{\star}_{\text{gasn,ave}} = \Omega^{\star}$.
	
\noindent \textbf{(b):} With the \texttt{GAS} solver, $\Xi^{\star}_{\text{gasn,ave}} \succeq \Omega^{\star}$. Furthermore, their relative difference can be bounded by
\begin{equation}\label{nsequ:1}
\frac{\|\Xi^{\star}_{\text{gasn,ave}} - \Omega^{\star}\|}{\|\Omega^{\star}\|} = O ( (1 - \sqrt{\mu^{\star}/\nu^{\star}} )^{\tau} ),
\end{equation}
where the big-$O$ notation may contain universal constants such as $\gamma_S, d, \gamma_H, \Upsilon_H, \gamma_E, \Upsilon_E$, but is independent of the sketching step $\tau$.
\end{proposition}

Proposition \ref{prop:comp_stats_tradeoff} demonstrates the computational-statistical trade-off in two parts. Part (a) shows that solving the Newton system exactly achieves minimax optimal statistical efficiency $\tOmega$ \citep{Duchi2021Asymptotic, Davis2024Asymptotic}. Part (b) shows that the \texttt{GAS} solver improves computational efficiency but leads to a loss in statistical efficiency. Fortunately, this loss is mild: the relative gap between $\gaXi$ and $\tOmega$ decays exponentially in the number of sketching steps $\tau$.

\vskip4pt

\noindent $\bullet$ \textbf{Comparison with USN.} We make connections between $\Xi^{\star}_{\text{gasn, ave}}$ and $\Xi^{\star}_{\text{usn, ave}}$ established in \cite{Du2025Online}.

\begin{proposition}[\textbf{GASN vs. USN}]\label{thm:covariance_comparison}
Under the conditions of Theorem \ref{thm:asymptotic_normality}, consider the choice of $E_t = I$. In the degenerate regime of \texttt{GAS} where $\mu_t\nu_t = 1$, we have $\Xi^\star_{\text{gasn, ave}}=\Xi^{\star}_{\text{usn, ave}}$.
\end{proposition}

Proposition \ref{thm:covariance_comparison} shows that when $\mu_t \nu_t = 1$ and $E_t=I$, the two limiting covariance matrices coincide, $\gaXi = \uaXi$; thus, the averaged GASN matches the statistical efficiency of the averaged USN \citep{Du2025Online} (see \eqref{nsequ:normal}). However, \citet[Prop. 3.6]{Du2025Online} shows that $\|\uaXi - \Omega^{\star}\|/\|\Omega^{\star}\| = O((1 - \mu^{\star})^{\tau})$. Comparing this with \eqref{nsequ:1}, we observe that when $\mu_t \nu_t < 1$, $\gaXi$ converges to the optimal covariance $\Omega^{\star}$ at a \textbf{strictly faster} rate than $\uaXi$. Thus, in the accelerated regime, GASN improves computational efficiency without sacrificing statistical efficiency relative to its unaccelerated counterpart.

\vskip4pt
\noindent $\bullet$ \textbf{Comparison with ASN.}
We now contrast GASN with the ASN method of \cite{Wang2026Inference}. Under the choice of $E_t = I$ and proper conditions on the number of sketching steps $\tau$ and the step size $\varphi_t = C_{\varphi} (t+1)^{-\varphi}$, \citet[Theorem 4.3]{Wang2026Inference} shows that the \textbf{last} ASN iterate $\bx_t$ exhibits $1/\sqrt{\varphi_{t}} \cdot (\bx_t - \bx^{\star}) \xrightarrow{d} \mathcal{N}(\boldsymbol{0}, {\Xi}^\star_{\text{asn, last}})$, where the limiting covariance $\Xi^{\star}_{\text{asn, last}}$ is given by the solution to the Lyapunov equation:
\begin{equation}\label{equ:lyp}
\left[(I - K^{\star}) -\zeta I\right]{\Xi}^\star_{\text{asn, last}} + {\Xi}^\star_{\text{asn, last}} \left[(I - K^{\star}) -\zeta I\right] = \mathbb{E}[(I - \tilde{K}^{\star}) \Omega^{\star} (I - \tilde{K}^{\star})^{\top}],
\end{equation}
with $\zeta \coloneqq \mathbf{1}_{\{\varphi = 1\}} / (2 C_{\varphi})$. We close this subsection by relating $\Xi^{\star}_{\text{gasn, ave}}$ to $\Xi^{\star}_{\text{asn, last}}$.

\begin{proposition}[\textbf{Averaged vs. Last iterate of ASN}]\label{prop:1}  
Consider $\Xi^{\star}_{\text{asn, last}}$ that solves \eqref{equ:lyp} with $C_{\varphi}=\varphi = 1$ and $E_t = I$. Suppose $\tau$ satisfies the condition in Theorem \ref{thm:as_convergence}, we have $\Xi^{\star}_{\text{gasn, ave}} \preceq \Xi^{\star}_{\text{asn, last}}$.
\end{proposition}

This proposition shows that, regardless of the sketching distribution used, the \textbf{averaged} ASN iterate is more statistically efficient than the \textbf{last} ASN iterate.

\subsection{Online inference via random scaling}\label{sec:4.2}

Building on the results of Section \ref{sec:4.1}, we introduce an online inference procedure inspired by prior work on random scaling and self-normalization \citep{Lee2022Fast, Li2023Online, Chen2024Online, Du2025Online}. We first strengthen the asymptotic normality in Theorem \ref{thm:asymptotic_normality} to a functional central limit theorem (FCLT).

\begin{theorem}[\textbf{Functional CLT}]\label{thm:FCLT}
Under the conditions of Theorem \ref{thm:as_convergence} and Assumption \ref{ass:5}, we~have
\begin{equation*}
\frac{1}{\sqrt{t}} \sum_{i = 0}^{\lfloor rt \rfloor -1} (\boldsymbol{x}_{i} - \boldsymbol{x}^{\star}) \Longrightarrow  (\Xi^{\star}_{\text{gasn, ave}})^{1/2} W_{d}(r),\quad\quad r\in[0, 1],
\end{equation*}
where $\lfloor \cdot \rfloor$ is the floor function and $W_{d}(\cdot)$ is the standard $d$-dimensional Brownian motion.	
\end{theorem}

The symbol $\Longrightarrow$ denotes weak convergence in a function space, which differs from pointwise weak convergence $\stackrel{d}{\to}$ in Theorem \ref{thm:asymptotic_normality}. Note that Theorem \ref{thm:asymptotic_normality} is a special case of Theorem \ref{thm:FCLT} with $r = 1$. By combining Theorem \ref{thm:FCLT} with continuous mapping theorem \citep[Theorem 1.11.1]{Vaart1996Weak}, we can construct a studentized, pivotal test statistic whose limiting distribution is free of any unknown parameters.

\begin{theorem}[\textbf{Inference via random scaling}]\label{thm:Random_Scaling}
Under the conditions of Theorem \ref{thm:as_convergence} and Assumption \ref{ass:5}, and suppose $\cov(\nabla F(\bx^{\star}; \xi))\succ \boldsymbol{0}$. Let us define a \emph{random scaling} matrix
\begin{equation}\label{def:V_t}
V_t \coloneqq \frac{1}{t^2} \sum_{i=1}^{t} i^2 \, (\bar{\bx}_i - \bar{\bx}_t)(\bar{\bx}_i - \bar{\bx}_t)^\top,
\end{equation}
and have for any vector $\0\neq \boldsymbol{w} \in \mR^d$,
\begin{equation}\label{eq:RS_converge_in_distribution}
\frac{\sqrt{t} \ \boldsymbol{w}^{\top} (\bar{\bx}_t - \bx^{\star})}{\sqrt{\boldsymbol{w}^{\top} V_t \boldsymbol{w}}} \stackrel{d}\longrightarrow \frac{W_1(1)}{\sqrt{\int_{0}^1  \left(W_1(r) - rW_1(1) \right)^2 dr}}.
\end{equation}
\end{theorem} 

The limiting distribution on the right-hand side of \eqref{eq:RS_converge_in_distribution} is a $t$-like distribution, which has been extensively studied in statistics literature \citep{Abadir1997Two,Kiefer2000Simple,Abadir2002Simple}. We have provided the quantiles of this distribution in Appendix \ref{app:quantile}. Finally, we can construct an asymptotically valid confidence interval for $\tx$ as: for any $p\in(0,1)$,$\quad$
\begin{equation*}
\mathbb{P} \big(\boldsymbol{w}^{\top}\bx^{\star} \in \big[ \boldsymbol{w}^{\top} \bar{\bx}_t \pm  U_{ 1 - 0.5p} \sqrt{\boldsymbol{w}^{\top}V_{t}\boldsymbol{w}/t } \big]\big) \longrightarrow 1 - p \qquad \text{as} \qquad t \to \infty,
\end{equation*}
where $U_{1 - 0.5p}$ denotes the $(1 - 0.5p)\times100\%$ quantile of the limiting distribution in \eqref{eq:RS_converge_in_distribution}. 
Note that all quantities in the test statistic in \eqref{eq:RS_converge_in_distribution}, including $\bar{\bx}_t$ and $V_t$, can be updated online (cf. Appendix \ref{app:online_update}). Thus, this confidence interval can be adjusted in real time as data are collected for the Newton~iteration.

\section{Numerical Experiments}\label{sec:numerical}

\subsection{Synthetic data experiments}\label{sec:5.1}

We evaluate our random scaling inference procedure (\texttt{AveRS}) comparing it with an online inference procedure based on the last iterate and a batch-free covariance estimator \citep{Wang2026Inference} (\texttt{LastBF}). For a fair comparison, both inference procedures are paired with three Newton methods: Exact Newton (EN), ASN (equivalently, GASN with $E_t = I$), and GASN with $E_t = B_t$ (referred to as GASN for simplicity). Following prior work \citep{Chen2020Statistical, Zhu2021Online, Na2025Statistical, Kuang2025Online, Wang2026Inference}, we consider linear and logistic regression models with varying dimension $d$, covariate covariance $\Sigma_a$, sketching distribution, and sketching steps. We report the mean absolute error (MAE), empirical coverage rate (Ave Cov), and average confidence interval length (Ave Len) over $200$ independent runs. Due to space limits, we present a subset of the results; full experimental details and complete results (e.g., convergence comparisons) are deferred to Appendix \ref{app:exp}.

\begin{figure}[H]
\centering     
\subfigure{\includegraphics[width=0.16\textwidth]{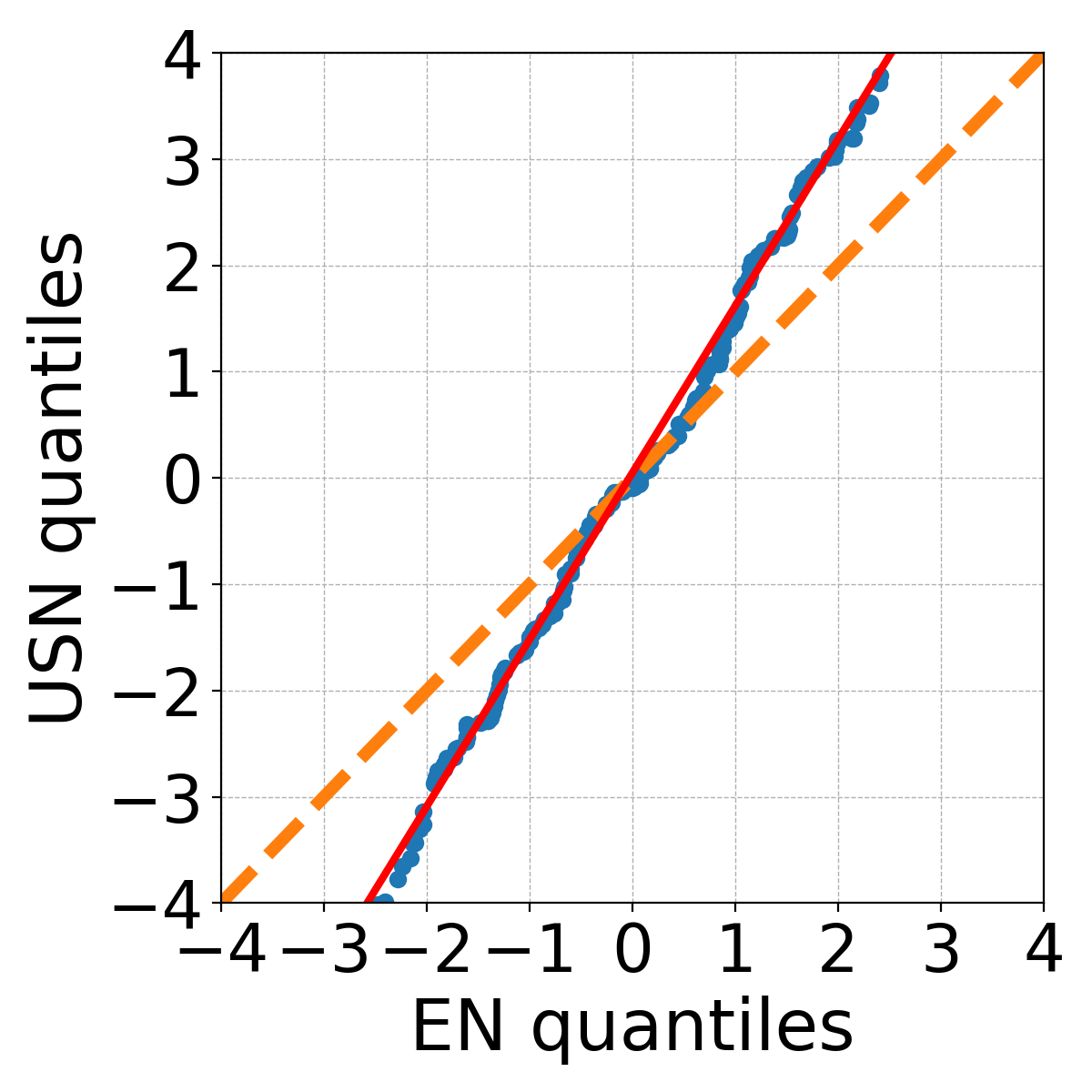}}
\subfigure{\includegraphics[width=0.16\textwidth]{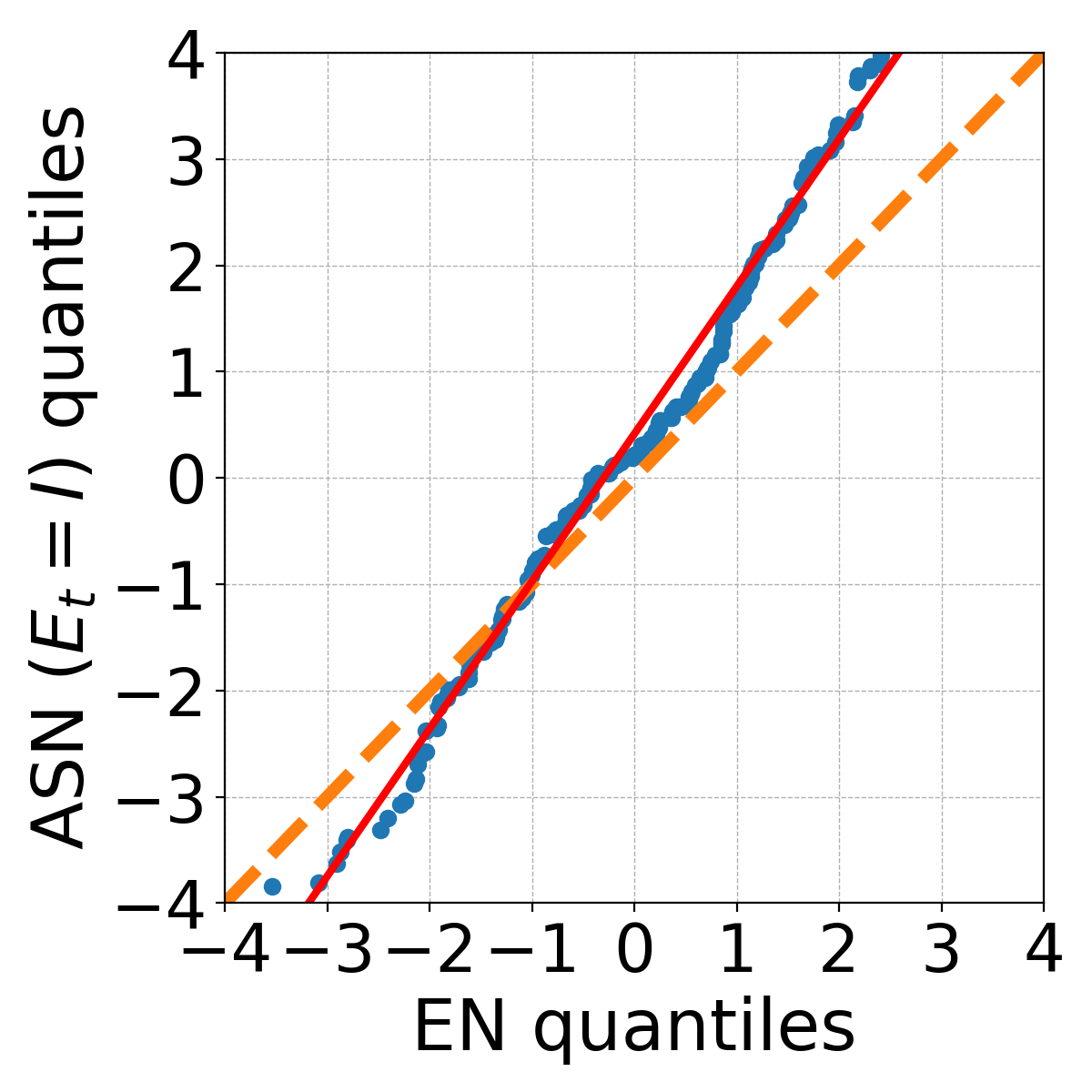}}
\subfigure{\includegraphics[width=0.16\textwidth]{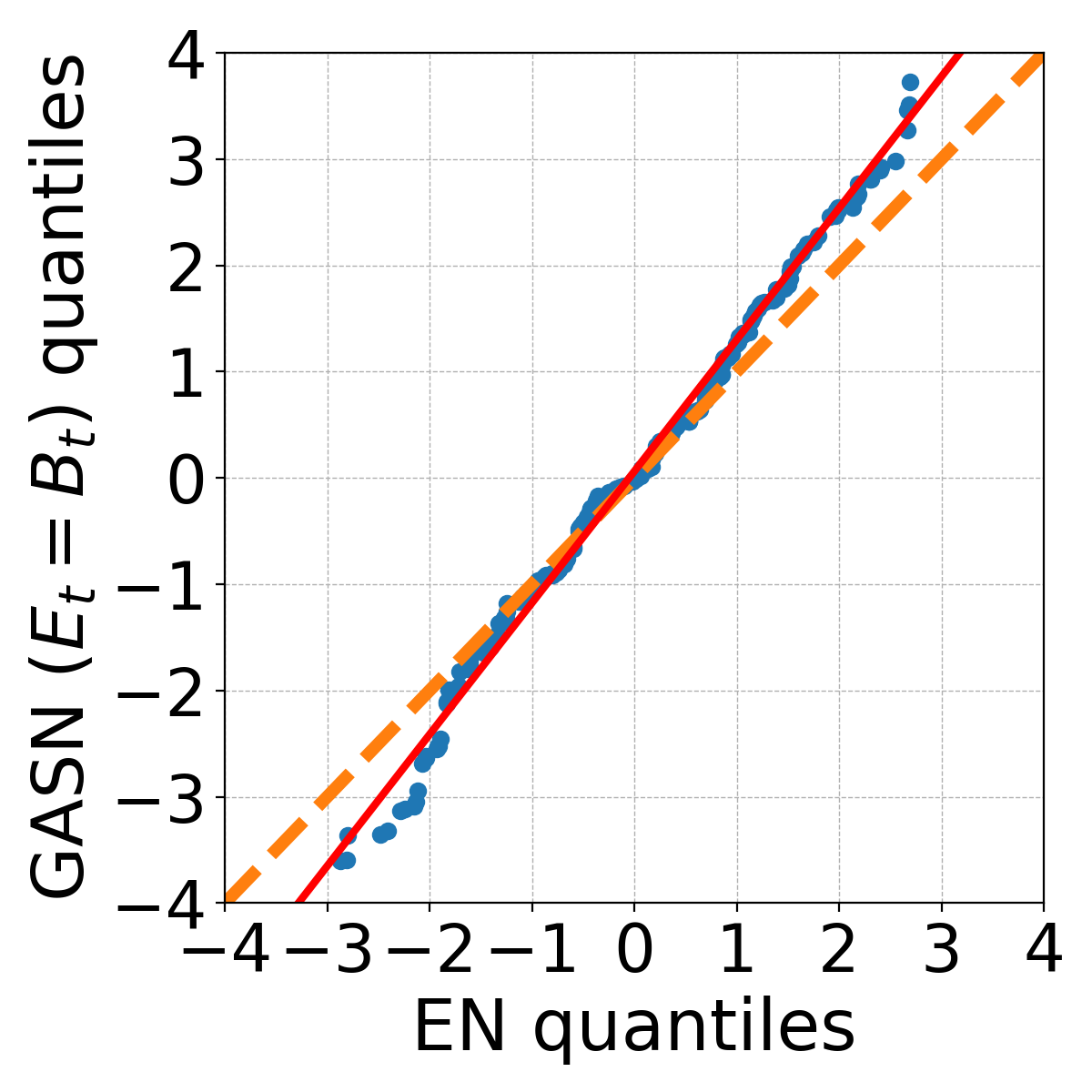}}
\subfigure{\includegraphics[width=0.16\textwidth]{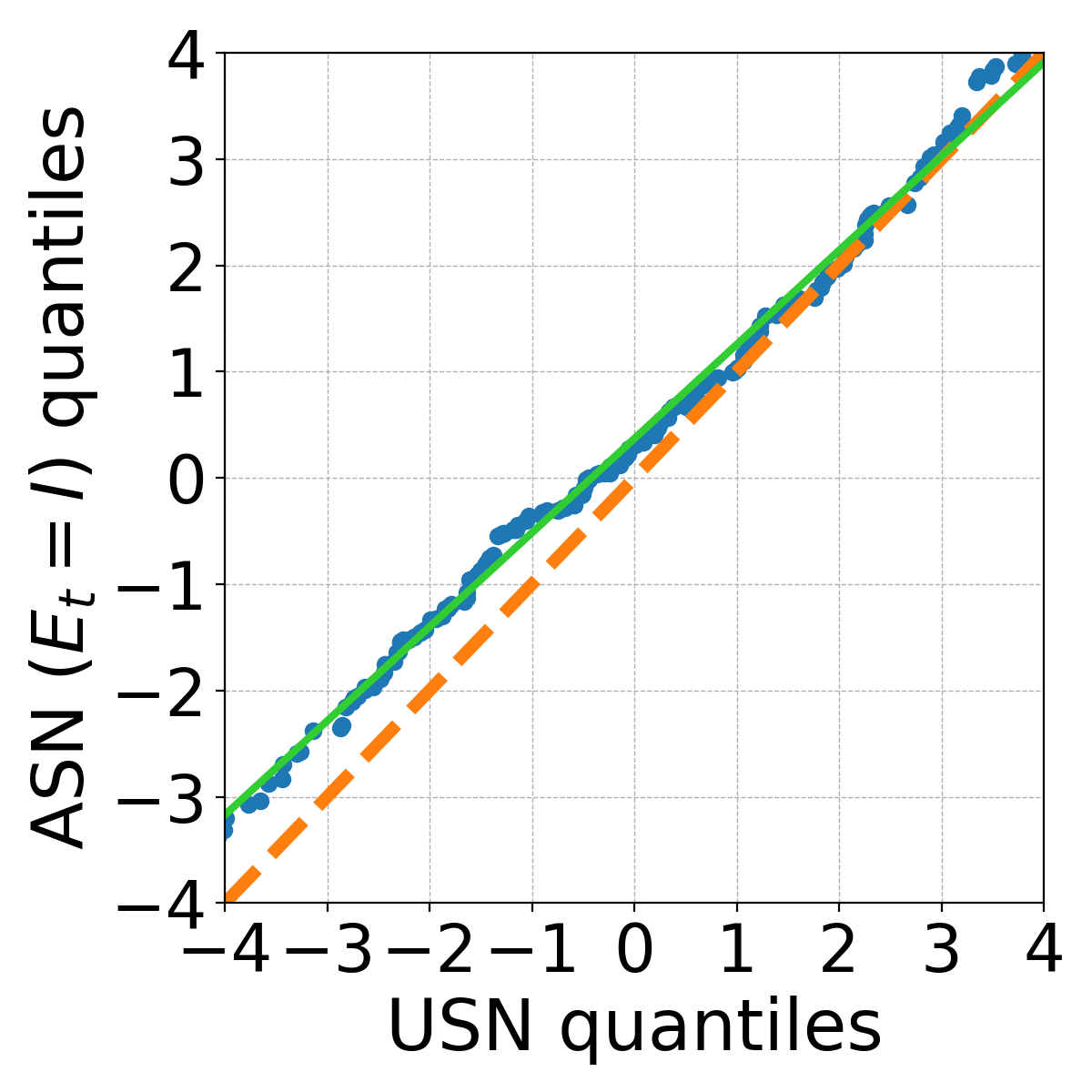}}
\subfigure{\includegraphics[width=0.16\textwidth]{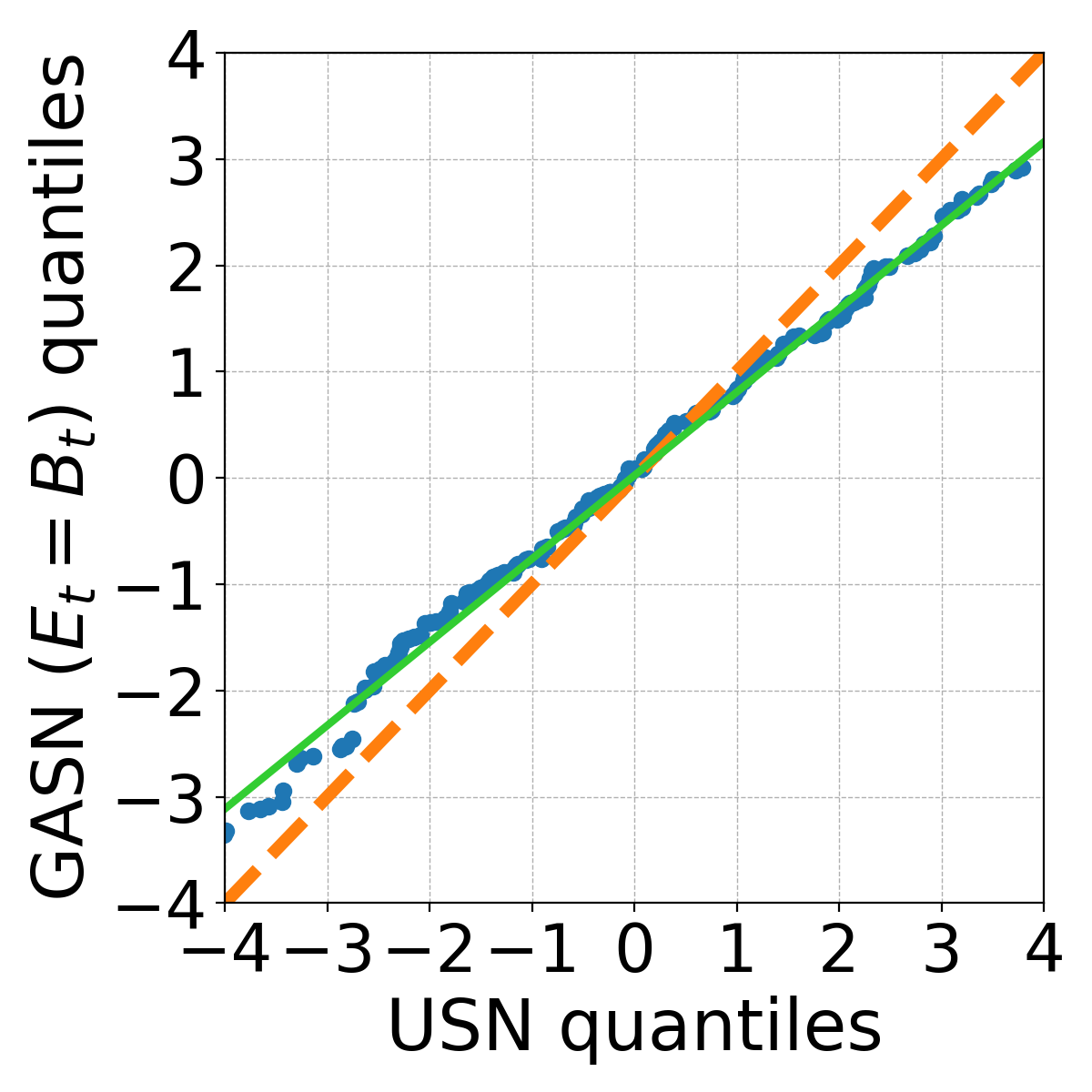}}
\subfigure{\includegraphics[width=0.16\textwidth]{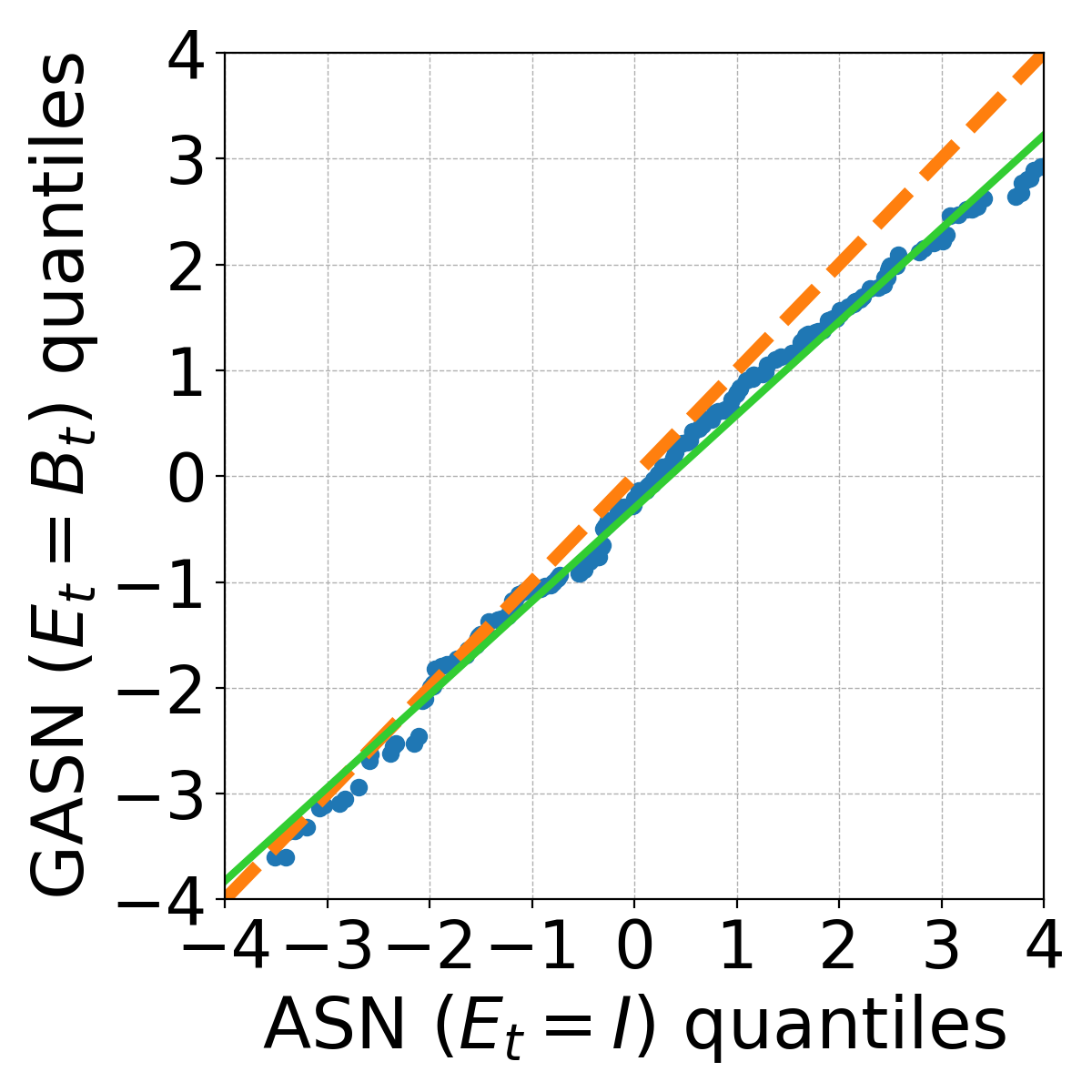}}
\vskip-0.15cm
\centering{\textbf{Linear Regression + Coordinate Sketches}}
\vspace{-0.1cm}

\subfigure{\includegraphics[width=0.16\textwidth]{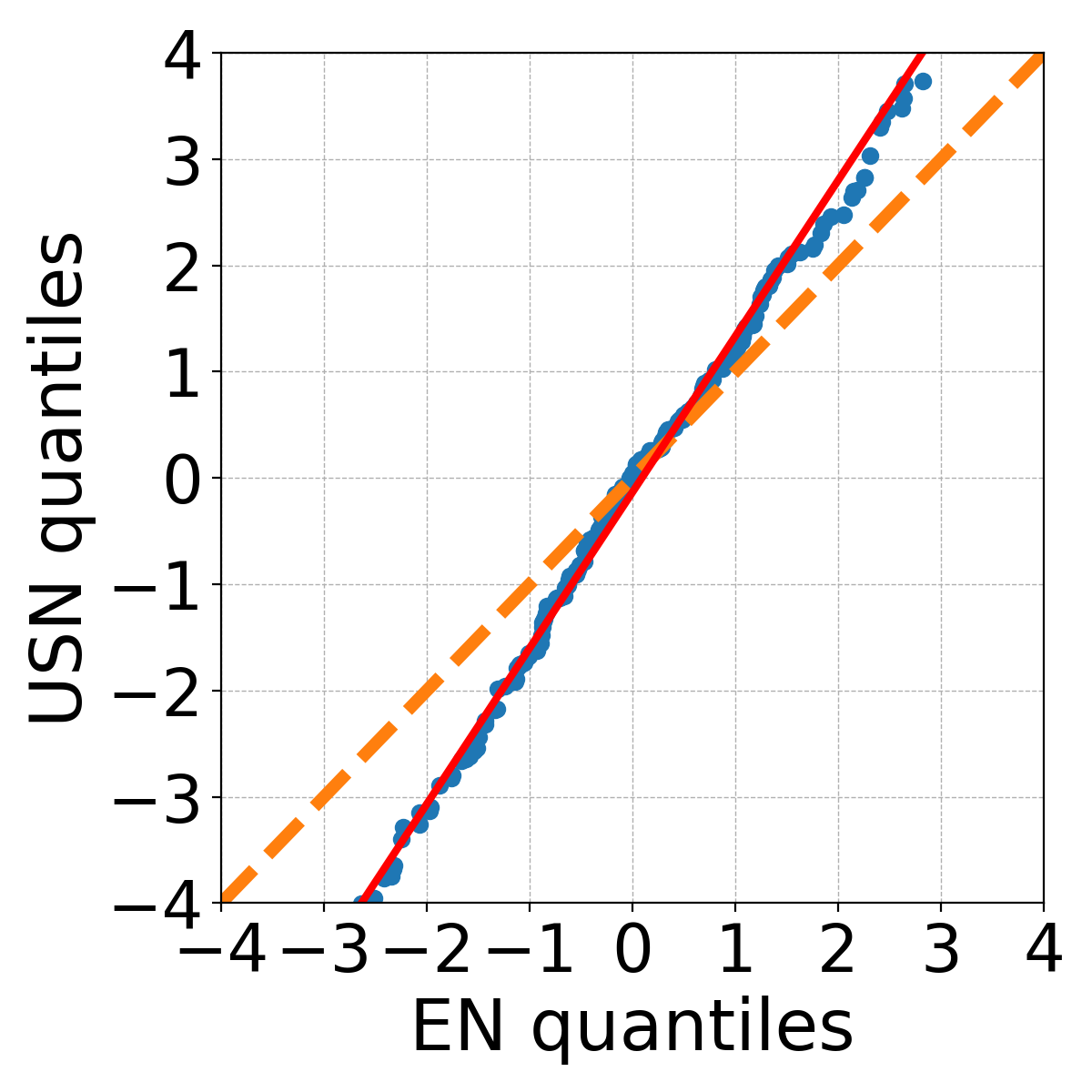}}
\subfigure{\includegraphics[width=0.16\textwidth]{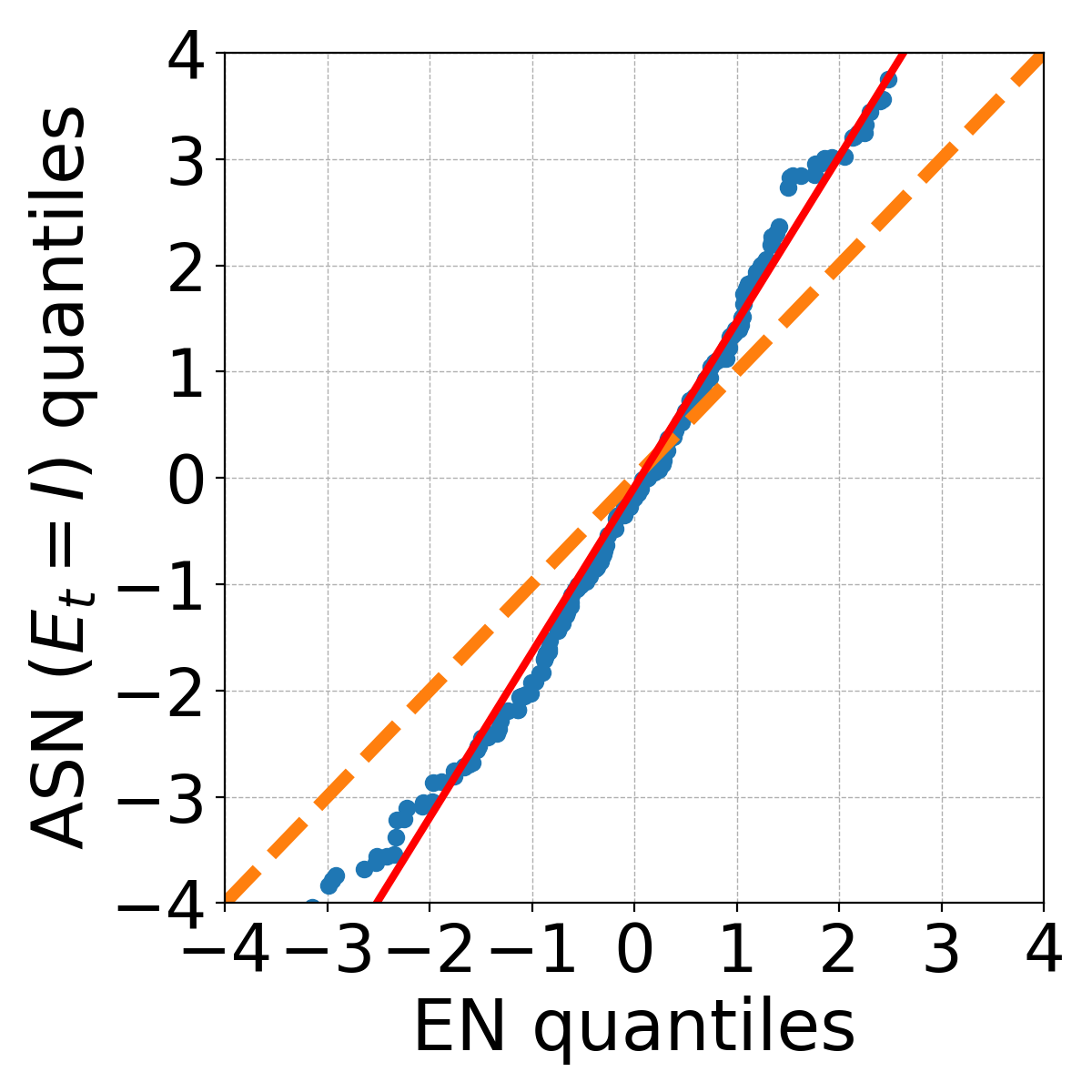}}
\subfigure{\includegraphics[width=0.16\textwidth]{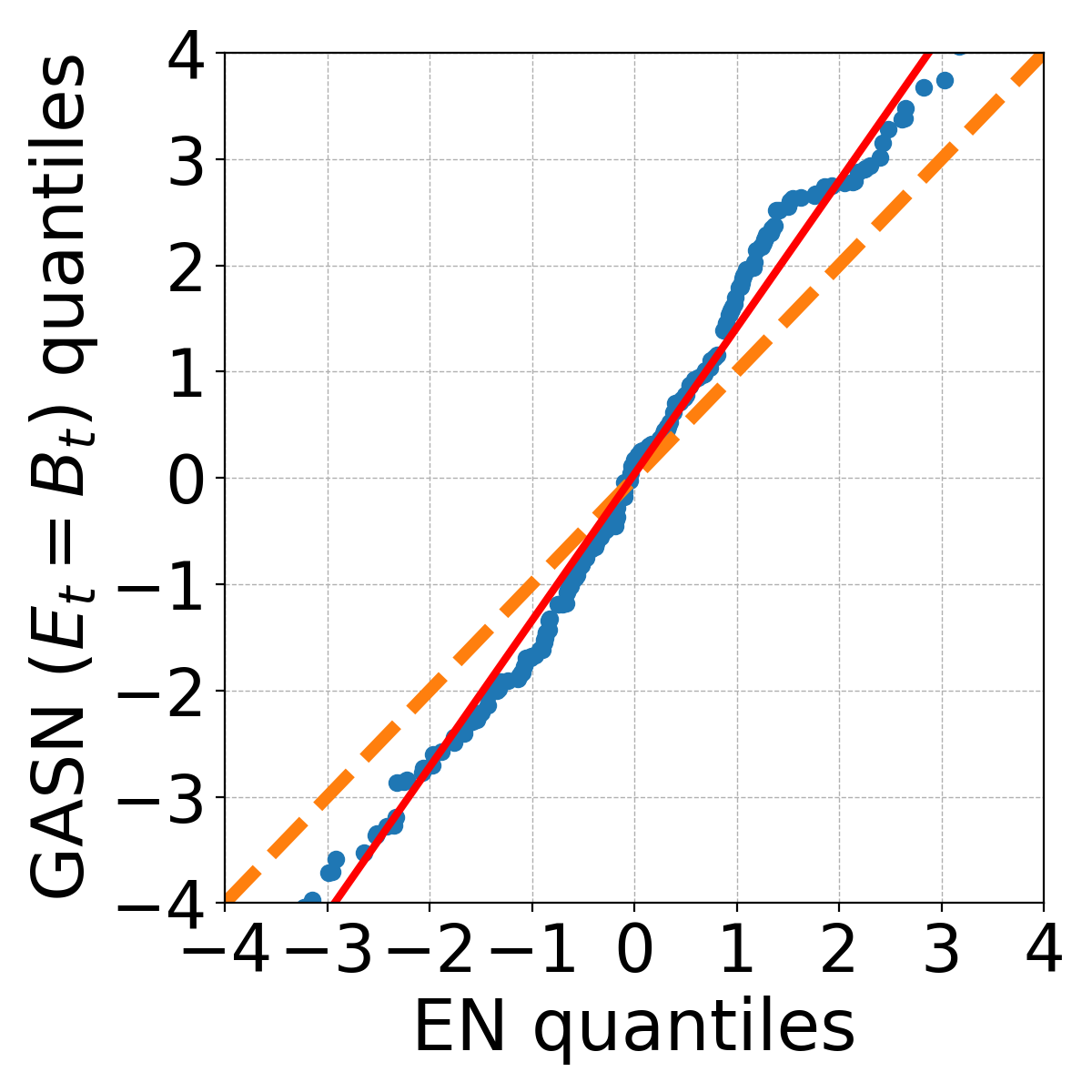}}
\subfigure{\includegraphics[width=0.16\textwidth]{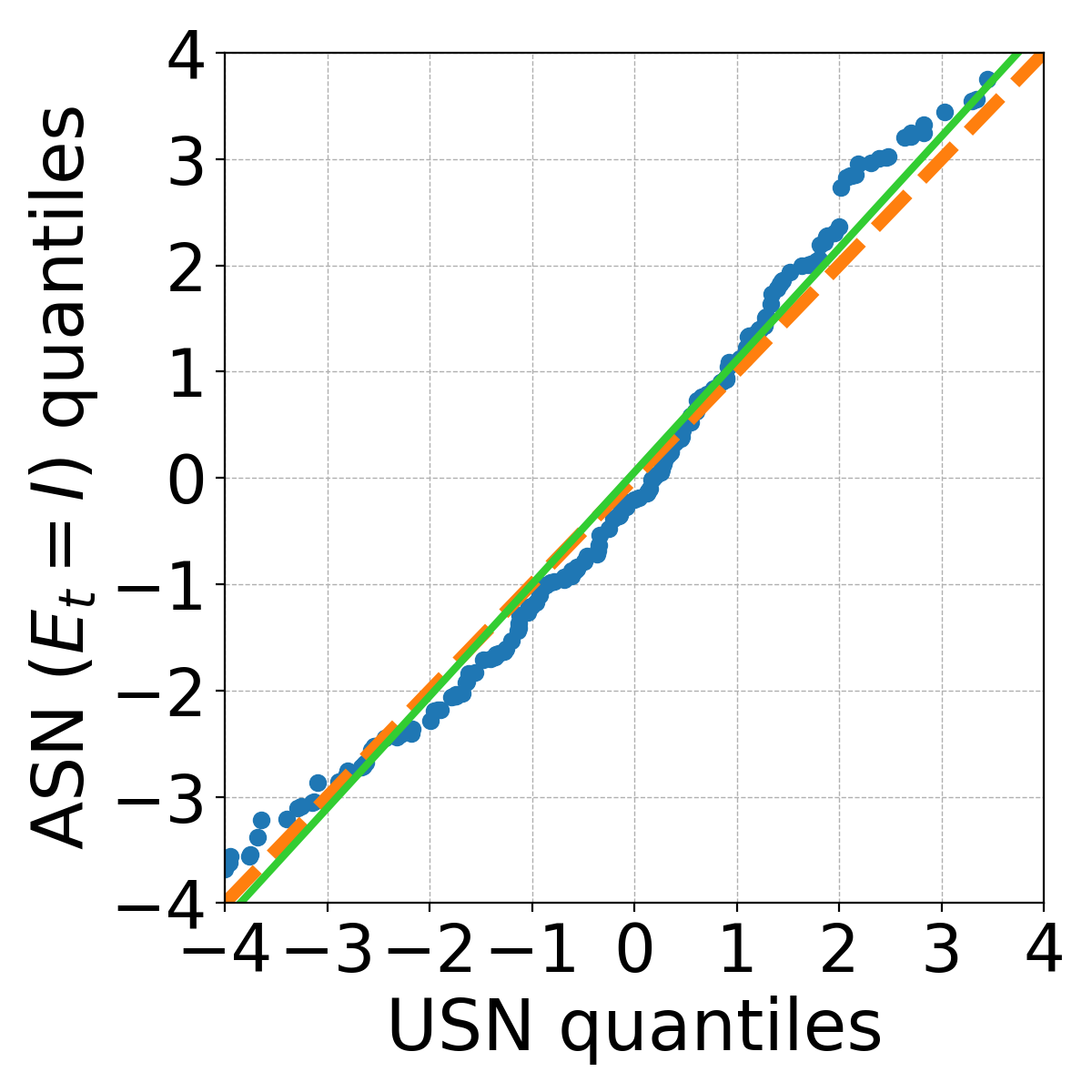}}
\subfigure{\includegraphics[width=0.16\textwidth]{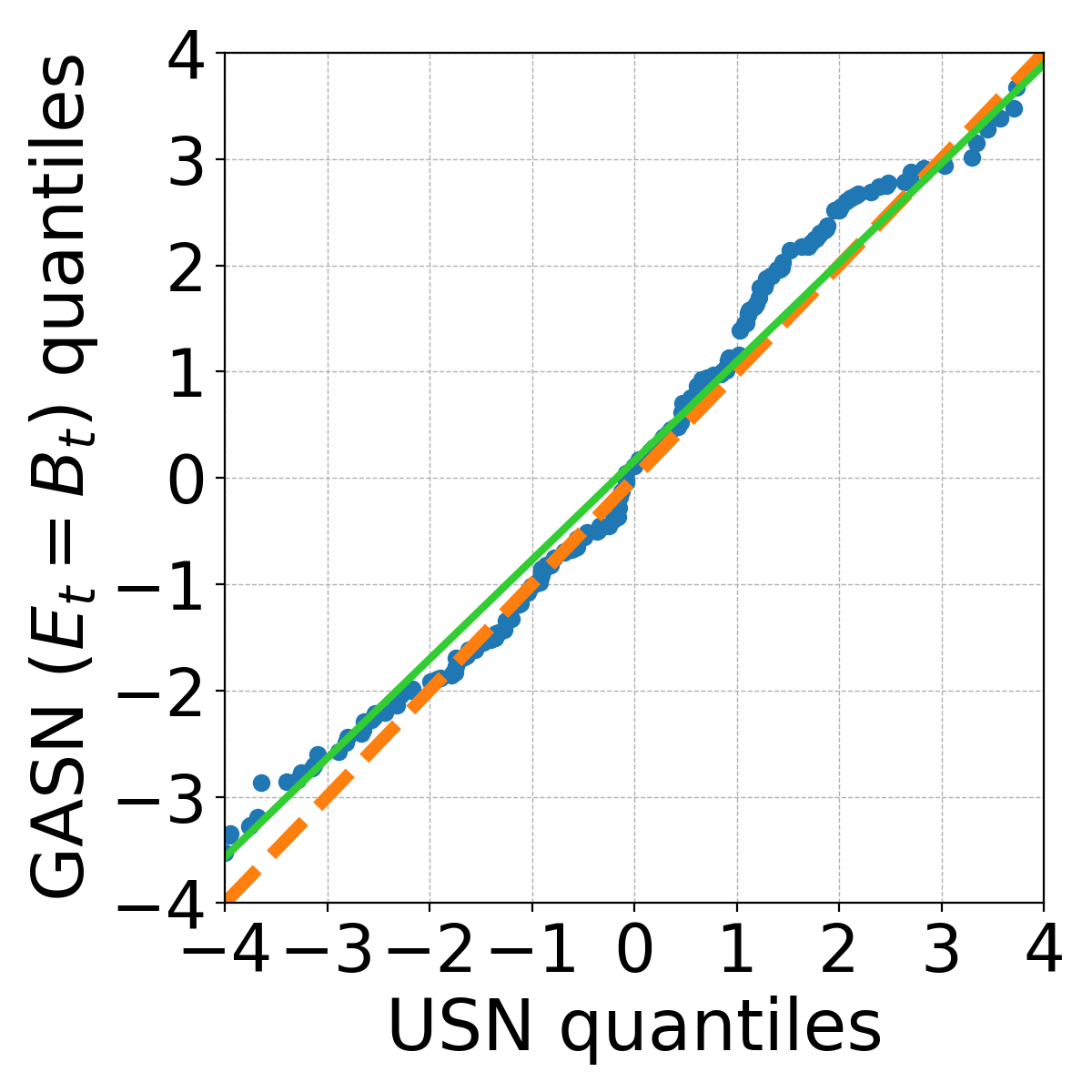}}
\subfigure{\includegraphics[width=0.16\textwidth]{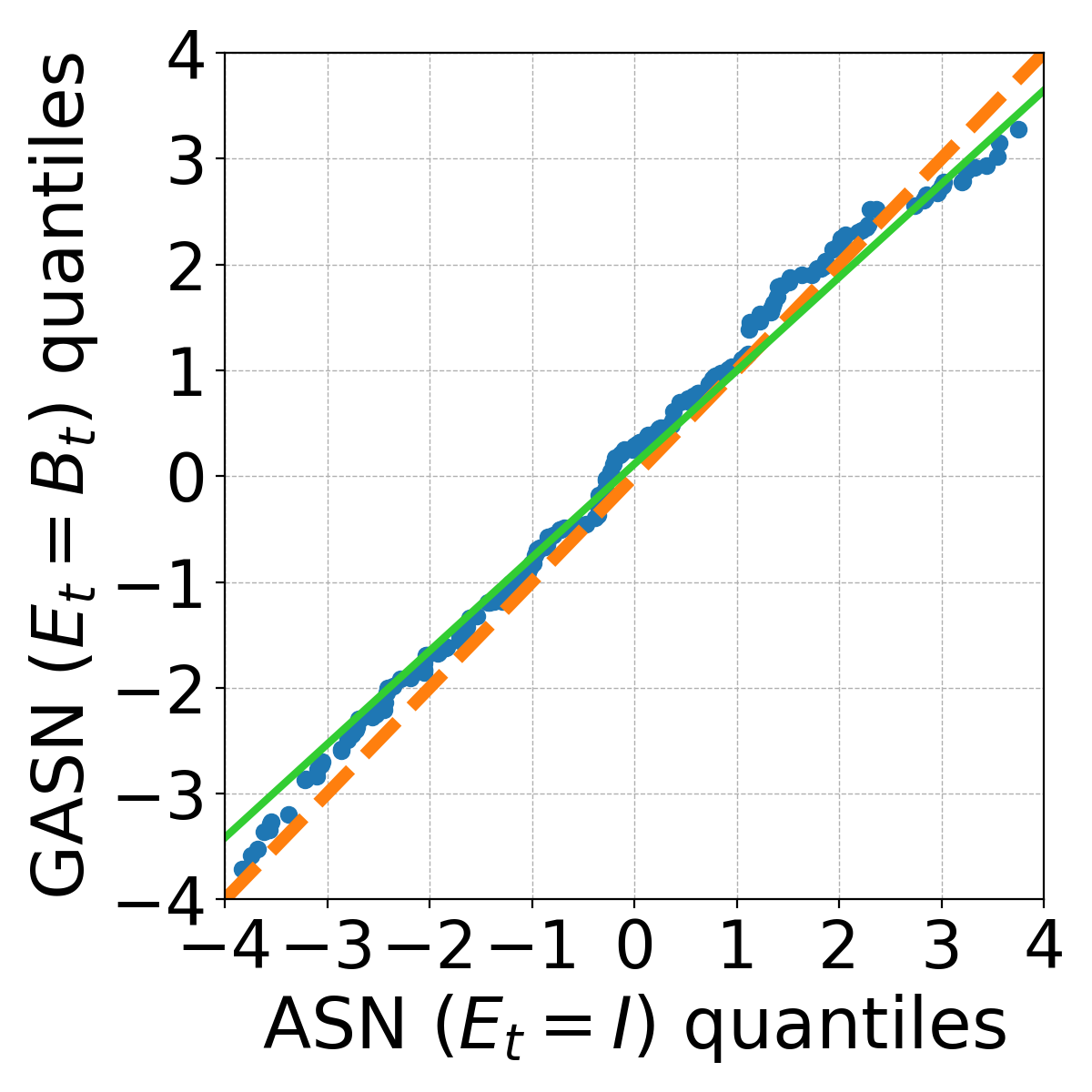}}
\vskip-0.15cm
\centering{\textbf{Linear Regression + Gaussian Sketches}}
\vskip-0.15cm
\caption{\textit{QQ plots for four stochastic Newton methods. The orange line denotes the reference line $y = x$. Red and green lines denote the fitted line through the blue data points.}}\label{fig:qq_linear}
\end{figure}

\noindent$\bullet$ \textbf{Covariance comparison.} Figure \ref{fig:qq_linear} reports QQ plots of the averaged-iterate covariances for EN, USN, ASN, and GASN (linear model with coordinate and Gaussian sketches). Two patterns emerge. First, the quantiles of all three sketched methods exceed those of EN, as the slopes of the red lines are larger than those of the orange lines in the first three columns, confirming that EN attains the smallest limiting variance (which is optimal). Second, the quantiles of ASN and GASN closely align with those of USN and are sometimes even smaller, as the slopes of the green lines fall below those of the orange lines in some (e.g., the fifth and sixth) plots. This indicates that the accelerated sketch-and-project with a better metric improves computational efficiency without inflating uncertainty or sacrificing statistical efficiency, consistent with the results in Section \ref{sec:4.1}. 

\vskip4pt

\noindent$\bullet$ \textbf{Asymptotic validity of inference.}
Table \ref{tab:d40_main_results} shows that the empirical coverage rates of the confidence intervals constructed by GASN remain close to the nominal $95\%$ level across all settings. This provides numerical support for our asymptotic normality and online inference theory (Theorems \ref{thm:asymptotic_normality} and \ref{thm:Random_Scaling}).

\input{table1}

\noindent $\bullet$ \textbf{Mean absolute error.} 
For both the last iterates, $\mathbb{E}[\|\bx_t - \bx^{\star}\|]$, and the averaged iterates, $\mathbb{E}[\|\bar{\bx}_t - \bx^{\star}\|]$, all three methods (EN, ASN, and GASN) achieve similar MAEs under the same combination of model, covariate covariance, and sketching distribution, regardless of the sketching steps. 
That being said, a consistent pattern across all blocks of Table~\ref{tab:d40_main_results} is that averaged iterates exhibit substantially smaller MAEs than last iterates, often by a factor of \textbf{$\boldsymbol{4}\boldsymbol{\times}$ to an order of magnitude}. Consequently,~\texttt{AveRS} provides more accurate point estimates than procedures based on last iterates.

\vskip4pt

\noindent $\bullet$ \textbf{Statistical efficiency of \texttt{AveRS} vs.\ \texttt{LastBF}.} Two patterns appear in Table~\ref{tab:d40_main_results}. First, both ASN and GASN produce confidence intervals that are slightly longer than those of EN, reflecting the additional randomness introduced by sketch-and-project. Second, \texttt{AveRS} produces confidence intervals that are roughly about an order of magnitude shorter than those of \texttt{LastBF} across all settings, without~any~loss in coverage. This suggests that averaged iterates achieve higher statistical efficiency.

In summary, the \texttt{AveRS} inference procedure built on GASN attains near-nominal coverage with smaller errors and shorter confidence intervals than \texttt{LastBF}. It preserves asymptotic validity while offering three key advantages over existing baselines: higher accuracy and statistical efficiency than \texttt{LastBF}, faster convergence than unaccelerated sketched Newton, and lower computational cost~than~Exact~Newton.

\subsection{Real-data experiment on UCI Covertype dataset}\label{sec:5.2}

To examine our inference theory in real-world data, we use the UCI Covertype dataset \citep{Blackard1998Covertype}, whose standard binary task ($\texttt{Cover\_Type}\in\{1,2\}$) is to predict the forest cover type of a land cell from cartographic features such as elevation, slope, and soil type. This yields $n = 495{,}141$ samples and $d = 50$ features. As the task is binary classification, we fit an $\ell_2$-regularized logistic regression model as following:
\begin{equation}\label{eq:UCI_model}
\bx^\star=\argmin_{\bx\in\mathbb{R}^{d}}\frac{1}{n}\sum_{i=1}^{n}F(\bx;\xi_i),
\qquad
F(\bx;\xi_i)=\log\!\big(1+\exp(-b_i\,\ba_i^{\top}\bx)\big)+\tfrac{\lambda}{2}\|\bx\|^{2}.
\end{equation}
Here, the vector $\bx$ is the logistic coefficient, and each observation is a pair $\xi_i=(\ba_i,b_i)$: the feature vector $\ba_i\in\mathbb{R}^{d}$ collects the $d=50$ standardized cartographic variables of cell $i$ and the label $b_i\in\{-1,+1\}$ encodes its cover type. Moreover, the $\ell_2$ term in \eqref{eq:UCI_model} is the standard ridge penalty. It shrinks the coefficients to reduce variance and guarantees a unique, well-conditioned solution even when the classes are nearly separable or the features are collinear, both of which are common for Covertype's binary indicators and correlated cartographic variables.

We evaluate three optimization algorithms, EN, ASN ($E_t=I$), and GASN ($E_t=B_t$), under both coordinate and Gaussian sketches with $\tau=10$ sketching steps, each paired with two inference procedures: \texttt{LastBF} and \texttt{AveRS}. The step size and all remaining settings match the synthetic experiments in Section \ref{sec:5.1}, and we report MAE, coverage (Ave Cov), and interval length (Ave Len) over $200$ independent runs.

The conclusions of our theory carry over to this real, large-scale problem (Table~\ref{tab:uci_results}). In every setting, \texttt{AveRS} attains empirical coverage at or above the nominal $95\%$ level while producing substantially smaller errors and shorter intervals than \texttt{LastBF}: its MAE is smaller by a factor of $4\times$ to an order of magnitude, and its confidence intervals are $2.5$--$6\times$ shorter, with no loss of coverage. Crucially, GASN matches ASN in both accuracy and coverage, e.g., MAE $1.28$ and length $0.16$ under the coordinate sketch for both, while relying on the cheaper projection metric $E_t=B_t$, corroborating our analysis in Section~\ref{sec:4.1} that the accelerated sketch-and-project with a better metric improves computational efficiency without inflating uncertainty.

\input{table2}

\section{Conclusion}

We studied online Newton methods with the generalized accelerated sketch-and-project solver (GASN). GASN enables faster and more computationally efficient Newton updates while preserving strong statistical guarantees: we proved global almost-sure and $L^2$ convergence, established asymptotic normality for the averaged GASN iterate, and characterized its limiting covariance. We show that, although GASN reduces computational cost at the expense of statistical efficiency when compared to exact Newton, it incurs \textit{no} such computational-statistical trade-off when compared to unaccelerated sketched Newton methods. For practical inference, we developed a random-scaling procedure based on a pivotal test statistic, which bypasses covariance estimation altogether.

\bibliographystyle{my-plainnat}
\bibliography{ref}

\appendix
\numberwithin{equation}{section}
\numberwithin{theorem}{section}
\input{Appendix}

\end{document}

%% file: table1.tex
\begin{table}[H]\vskip-0.3cm\centering\setlength{\tabcolsep}{4pt}\renewcommand{\arraystretch}{1.1}\resizebox{1.0\linewidth}{!}{\begin{tabular}{|c|c|c|c|c|ccc|ccc|ccc|ccc|}\hline\multirow{3}{*}{$d$} & \multirow{3}{*}{\shortstack{Design \\ Cov $\Sigma_a$}} & \multirow{3}{*}{\shortstack{Optimization \\ Algorithm}} & \multirow{3}{*}{$\tau$} & \multirow{3}{*}{\shortstack{Inference \\ Method}} & \multicolumn{3}{c|}{Linear Regression (Coordinate)} & \multicolumn{3}{c|}{Linear Regression (Gaussian)} & \multicolumn{3}{c|}{Logistic Regression (Coordinate)} & \multicolumn{3}{c|}{Logistic Regression (Gaussian)} \\\cline{6-17}& & & & & MAE & \multicolumn{1}{|c|}{Ave Cov} & Ave Len & MAE & \multicolumn{1}{|c|}{Ave Cov} & Ave Len & MAE & \multicolumn{1}{|c|}{Ave Cov} & Ave Len & MAE & \multicolumn{1}{|c|}{Ave Cov} & Ave Len \\[-4pt]& & & & & {\scriptsize $(10^{-2})$} & \multicolumn{1}{|c|}{{\scriptsize $(\%)$}} & {\scriptsize $(10^{-2})$} & {\scriptsize $(10^{-2})$} & \multicolumn{1}{|c|}{{\scriptsize $(\%)$}} & {\scriptsize $(10^{-2})$} & {\scriptsize $(10^{-2})$} & \multicolumn{1}{|c|}{{\scriptsize $(\%)$}} & {\scriptsize $(10^{-2})$} & {\scriptsize $(10^{-2})$} & \multicolumn{1}{|c|}{{\scriptsize $(\%)$}} & {\scriptsize $(10^{-2})$} \\\hline\multirow{18}{*}{40} & \multirow{6}{*}{\shortstack{Identity}} & \multirow{2}{*}{EN} & \multirow{2}{*}{--} & \multicolumn{1}{|c|}{\texttt{LastBF}} & 25.79 & \multicolumn{1}{|c|}{95.50} & 2.56 & 25.69 & \multicolumn{1}{|c|}{96.00} & 2.57 & 3.75 & \multicolumn{1}{|c|}{96.50} & 0.24 & 3.78 & \multicolumn{1}{|c|}{98.50} & 0.24 \\\cline{5-17}& & & & \gmc{\texttt{AveRS}} & \gc{2.48} & \gmc{98.00} & \gc{0.33} & \gc{2.45} & \gmc{95.50} & \gc{0.33} & \gc{0.37} & \gmc{95.00} & \gc{0.03} & \gc{0.37} & \gmc{94.50} & \gc{0.03} \\\cline{3-17}& & \multirow{2}{*}{\shortstack{ASN \\ ($E_t=I$)}} & \multirow{2}{*}{5} & \multicolumn{1}{|c|}{\texttt{LastBF}} & 26.04 & \multicolumn{1}{|c|}{92.50} & 2.46 & 26.61 & \multicolumn{1}{|c|}{94.50} & 2.58 & 3.75 & \multicolumn{1}{|c|}{93.00} & 0.35 & 3.91 & \multicolumn{1}{|c|}{93.50} & 0.35 \\\cline{5-17}& & & & \gmc{\texttt{AveRS}} & \gc{7.16} & \gmc{91.50} & \gc{0.82} & \gc{7.31} & \gmc{91.50} & \gc{0.88} & \gc{1.07} & \gmc{92.00} & \gc{0.13} & \gc{1.07} & \gmc{94.00} & \gc{0.12} \\\cline{3-17}& & \multirow{2}{*}{\shortstack{GASN \\ ($E_t=B_t$)}} & \multirow{2}{*}{5} & \multicolumn{1}{|c|}{\texttt{LastBF}} & 25.73 & \multicolumn{1}{|c|}{93.50} & 2.49 & 27.17 & \multicolumn{1}{|c|}{94.00} & 2.57 & 3.81 & \multicolumn{1}{|c|}{93.50} & 0.35 & 3.88 & \multicolumn{1}{|c|}{94.50} & 0.36 \\\cline{5-17}& & & & \gmc{\texttt{AveRS}} & \gc{7.19} & \gmc{95.00} & \gc{0.84} & \gc{7.13} & \gmc{91.50} & \gc{0.86} & \gc{1.05} & \gmc{92.00} & \gc{0.12} & \gc{1.05} & \gmc{95.50} & \gc{0.12} \\\cline{2-17}\cline{2-17}& \multirow{6}{*}{\shortstack{Toeplitz\\$r=0.4$}} & \multirow{2}{*}{EN} & \multirow{2}{*}{--} & \multicolumn{1}{|c|}{\texttt{LastBF}} & 29.99 & \multicolumn{1}{|c|}{95.00} & 1.70 & 30.17 & \multicolumn{1}{|c|}{95.50} & 1.70 & 3.12 & \multicolumn{1}{|c|}{94.00} & 0.27 & 3.09 & \multicolumn{1}{|c|}{95.00} & 0.27 \\\cline{5-17}& & & & \gmc{\texttt{AveRS}} & \gc{2.90} & \gmc{96.00} & \gc{0.22} & \gc{2.90} & \gmc{93.00} & \gc{0.23} & \gc{0.30} & \gmc{96.00} & \gc{0.04} & \gc{0.30} & \gmc{94.50} & \gc{0.03} \\\cline{3-17}& & \multirow{2}{*}{\shortstack{ASN \\ ($E_t=I$)}} & \multirow{2}{*}{5} & \multicolumn{1}{|c|}{\texttt{LastBF}} & 22.24 & \multicolumn{1}{|c|}{96.50} & 2.17 & 23.77 & \multicolumn{1}{|c|}{96.00} & 2.18 & 3.11 & \multicolumn{1}{|c|}{92.00} & 0.29 & 3.24 & \multicolumn{1}{|c|}{96.50} & 0.30 \\\cline{5-17}& & & & \gmc{\texttt{AveRS}} & \gc{11.33} & \gmc{93.00} & \gc{0.42} & \gc{11.31} & \gmc{92.00} & \gc{0.46} & \gc{0.87} & \gmc{95.00} & \gc{0.11} & \gc{0.87} & \gmc{94.50} & \gc{0.10} \\\cline{3-17}& & \multirow{2}{*}{\shortstack{GASN \\ ($E_t=B_t$)}} & \multirow{2}{*}{5} & \multicolumn{1}{|c|}{\texttt{LastBF}} & 29.89 & \multicolumn{1}{|c|}{97.00} & 1.73 & 31.39 & \multicolumn{1}{|c|}{93.50} & 1.75 & 3.15 & \multicolumn{1}{|c|}{95.00} & 0.30 & 3.27 & \multicolumn{1}{|c|}{94.50} & 0.30 \\\cline{5-17}& & & & \gmc{\texttt{AveRS}} & \gc{10.67} & \gmc{96.00} & \gc{0.41} & \gc{10.45} & \gmc{91.00} & \gc{0.42} & \gc{0.89} & \gmc{93.00} & \gc{0.11} & \gc{0.87} & \gmc{93.50} & \gc{0.11} \\\cline{2-17}\cline{2-17}& \multirow{6}{*}{\shortstack{Equi-Corr\\$r=0.4$}} & \multirow{2}{*}{EN} & \multirow{2}{*}{--} & \multicolumn{1}{|c|}{\texttt{LastBF}} & 20.13 & \multicolumn{1}{|c|}{95.00} & 0.61 & 20.08 & \multicolumn{1}{|c|}{94.00} & 0.61 & 2.66 & \multicolumn{1}{|c|}{95.00} & 0.19 & 2.63 & \multicolumn{1}{|c|}{94.50} & 0.19 \\\cline{5-17}& & & & \gmc{\texttt{AveRS}} & \gc{1.93} & \gmc{95.00} & \gc{0.08} & \gc{1.94} & \gmc{95.00} & \gc{0.08} & \gc{0.26} & \gmc{94.00} & \gc{0.02} & \gc{0.25} & \gmc{97.00} & \gc{0.03} \\\cline{3-17}& & \multirow{2}{*}{\shortstack{ASN \\ ($E_t=I$)}} & \multirow{2}{*}{5} & \multicolumn{1}{|c|}{\texttt{LastBF}} & 12.63 & \multicolumn{1}{|c|}{95.00} & 1.04 & 17.69 & \multicolumn{1}{|c|}{95.00} & 1.12 & 2.69 & \multicolumn{1}{|c|}{93.00} & 0.24 & 2.75 & \multicolumn{1}{|c|}{92.00} & 0.25 \\\cline{5-17}& & & & \gmc{\texttt{AveRS}} & \gc{6.61} & \gmc{97.00} & \gc{0.13} & \gc{6.85} & \gmc{95.00} & \gc{0.15} & \gc{0.73} & \gmc{94.50} & \gc{0.09} & \gc{0.74} & \gmc{92.50} & \gc{0.08} \\\cline{3-17}& & \multirow{2}{*}{\shortstack{GASN \\ ($E_t=B_t$)}} & \multirow{2}{*}{5} & \multicolumn{1}{|c|}{\texttt{LastBF}} & 19.77 & \multicolumn{1}{|c|}{94.00} & 0.68 & 20.31 & \multicolumn{1}{|c|}{97.50} & 0.67 & 2.64 & \multicolumn{1}{|c|}{92.00} & 0.25 & 2.73 & \multicolumn{1}{|c|}{94.00} & 0.25 \\\cline{5-17}& & & & \gmc{\texttt{AveRS}} & \gc{5.94} & \gmc{96.00} & \gc{0.10} & \gc{5.98} & \gmc{93.00} & \gc{0.10} & \gc{0.74} & \gmc{93.00} & \gc{0.09} & \gc{0.74} & \gmc{92.50} & \gc{0.09} \\\hline\end{tabular}}\vspace{0.1cm}\caption{\textit{Evaluation results for the inference procedures under different parameter settings with~$d = 40$. For Exact Newton (EN), sketching steps are not applicable and are hence denoted by ``--''.}}\label{tab:d40_main_results}\vspace{-0.2cm}\end{table}

%% file: table2.tex
\begin{table}[H]\vskip-0.2cm\centering\setlength{\tabcolsep}{4pt}\renewcommand{\arraystretch}{1.1}\resizebox{0.72\linewidth}{!}{\begin{tabular}{|c|c|c|c|ccc|ccc|}\hline
\multirow{3}{*}{$d$} & \multirow{3}{*}{\shortstack{Optimization \\ Algorithm}} & \multirow{3}{*}{$\tau$} & \multirow{3}{*}{\shortstack{Inference \\ Method}} & \multicolumn{3}{c|}{Coordinate Sketch} & \multicolumn{3}{c|}{Gaussian Sketch} \\\cline{5-10}
& & & & MAE & \multicolumn{1}{|c|}{Ave Cov} & Ave Len & MAE & \multicolumn{1}{|c|}{Ave Cov} & Ave Len \\[-4pt]
& & & & {\scriptsize $(10^{-2})$} & \multicolumn{1}{|c|}{{\scriptsize $(\%)$}} & {\scriptsize $(10^{-2})$} & {\scriptsize $(10^{-2})$} & \multicolumn{1}{|c|}{{\scriptsize $(\%)$}} & {\scriptsize $(10^{-2})$} \\\hline
\multirow{6}{*}{$50$} & \multirow{2}{*}{EN} & \multirow{2}{*}{--} & \multicolumn{1}{|c|}{\texttt{LastBF}} & 5.25 & \multicolumn{1}{|c|}{95.50} & 0.31 & 5.22 & \multicolumn{1}{|c|}{92.00} & 0.30 \\\cline{4-10}
& & & \gmc{\texttt{AveRS}} & \Gthree{0.52}{97.50}{0.05} & \Gthree{0.53}{99.00}{0.05} \\\cline{2-10}
& \multirow{2}{*}{\shortstack{ASN \\ ($E_t=I$)}} & \multirow{2}{*}{10} & \multicolumn{1}{|c|}{\texttt{LastBF}} & 5.43 & \multicolumn{1}{|c|}{94.50} & 0.41 & 5.56 & \multicolumn{1}{|c|}{98.00} & 0.37 \\\cline{4-10}
& & & \gmc{\texttt{AveRS}} & \Gthree{1.28}{96.00}{0.16} & \Gthree{0.65}{98.00}{0.08} \\\cline{2-10}
& \multirow{2}{*}{\shortstack{GASN \\ ($E_t=B_t$)}} & \multirow{2}{*}{10} & \multicolumn{1}{|c|}{\texttt{LastBF}} & 5.44 & \multicolumn{1}{|c|}{91.00} & 0.42 & 5.49 & \multicolumn{1}{|c|}{91.50} & 0.38 \\\cline{4-10}
& & & \gmc{\texttt{AveRS}} & \Gthree{1.28}{95.50}{0.16} & \Gthree{0.67}{95.50}{0.08} \\\hline
\end{tabular}}\vspace{0.05cm}\caption{\textit{Evaluation results for the inference procedures under the UCI Covertype dataset. For Exact Newton (EN), sketching steps are not applicable and are hence denoted by ``--''.}}\label{tab:uci_results}\vspace{-0.6cm}\end{table}

%% file: Appendix.tex
\clearpage

\begin{center}
\Large \bf Appendix: \papertitle
\end{center}

\section{Quantiles of the Limiting Distribution}\label{app:quantile}

Table \ref{table:fclt} presents the quantiles of the limiting distribution in \eqref{eq:RS_converge_in_distribution} adapted from \cite{Abadir1997Two}:

\begin{table}[H]
	\centering
	\begin{tabular}{c|cccc}
		\toprule
		$p$ & 90\% & 95\% & 97.5\% & 99\% \\
		\hline\\[-10pt]
		Quantile($p$) & 3.875 & 5.323 & 6.747 & 8.613 \\
		\bottomrule
	\end{tabular}
	\caption{Quantile table of the distribution $W_1(1)/\{\int_{0}^1  \left(W_1(r) - rW_1(1) \right)^2 dr\}^{1/2}$.}
	\vspace{-0.1cm}
	\label{table:fclt}
\end{table}

\section{Online Update of $V_t$}\label{app:online_update}

In this section, we show that the random scaling matrix~$V_t$~can be computed in an online fashion. To see this, we rewrite $V_t$ defined in \eqref{def:V_t} as follows:
\begin{align*}
V_{t} &= \frac{1}{t^2}\sum_{i = 1}^t i^2 (\bar {\boldsymbol{x}}_i - \bar {\boldsymbol{x}}_t)(\bar {\boldsymbol{x}}_i - \bar {\boldsymbol{x}}_t)^{\top}\\
    &= \frac{1}{t^2}\sum_{i = 1}^{t}i^2\bar{\boldsymbol{x}}_i\bar{\boldsymbol{x}}_i^{\top} - \frac{1}{t^2}\left(\sum_{i = 1}^{t}i^2\bar{\boldsymbol{x}}_i\right)\bar{\boldsymbol{x}}_t^{\top} - \frac{1}{t^2}\bar{\boldsymbol{x}}_t \left(\sum_{i = 1}^{t}i^2\bar{\boldsymbol{x}}_i^{\top}\right) + \frac{1}{t^2}\left( \sum_{i = 1}^{t}i^2\right) \bar{\boldsymbol{x}}_t\bar{\boldsymbol{x}}_t^{\top},
\end{align*}
Each of the four terms above can be easily computed recursively. We formalize the online calculation in the following algorithm.

\begin{algorithm*}
\caption{Online update of the random scaling matrix $V_{t}$}
\begin{algorithmic}[1]
\State \textbf{Initialize:} set initial values for \( \boldsymbol{s}_0 = \bx_0, \bar{\boldsymbol{s}}_1 = \boldsymbol{s}_0, P_1 = \bar{\boldsymbol{s}}_1\bar{\boldsymbol{s}}_1^{\top}, Q_1 = \bar{\boldsymbol{s}}_1, V_1 = 0\);
\For{\( t = 1, 2, \ldots \)}
\State Run GASN with the update \eqref{equ:update} to obtain $\boldsymbol{s}_{t} = \bx_t$;
\State Compute $\bar{\boldsymbol{s}}_{t+1} = \frac{1}{t+1}\boldsymbol{s}_t + \frac{t}{t+1}\bar{\boldsymbol{s}}_{t}$, $P_{t+1} = (t+1)^2\bar{\boldsymbol{s}}_{t+1}\bar{\boldsymbol{s}}_{t+1}^{\top} + P_{t}$, $Q_{t+1} = (t+1)^2\bar{\boldsymbol{s}}_{t+1} + Q_{t}$;
\State Output $V_{t+1} = \frac{1}{(t+1)^2}P_{t+1} - \frac{1}{(t+1)^2}Q_{t+1}\bar{\boldsymbol{s}}_{t+1}^{\top} - \frac{1}{(t+1)^2}\bar{\boldsymbol{s}}_{t+1} Q_{t+1}^{\top}+ \frac{(t+2)(2t+3)}{6(t+1)}\bar{\boldsymbol{s}}_{t+1}\bar{\boldsymbol{s}}_{t+1}^{\top}$;
\EndFor
\end{algorithmic}
\end{algorithm*}

\section{Preparation Lemmas}

\begin{lemma}[{\cite{Davis2024Asymptotic}, Lemma A.9}]\label{lem:phi_t_recursion_lemma}
Suppose $\{A_t\}_t$ is a non-negative sequence satisfying
\begin{equation*}
A_{t+1} \leq (1 - c_1 t^{-\beta}) A_t + c_2 t^{-2\beta}, \quad\quad \forall t\geq t_0,
\end{equation*}
for some constants $c_1, c_2>0$, $\beta\in(0.5,1)$, and a positive integer  $t_0$. Then, there exists a constant~$c_0>0$ such that $A_t \leq c_0 t^{-\beta}$ for all $t\geq 1$.
\end{lemma}

\begin{lemma}[Kronecker's lemma]\label{lem:kronecker}

Suppose $\{A_t\}_t$ is a sequence such that $\sum_{t=0}^{\infty}A_t$ exists and is finite. For any divergent positive non-decreasing sequence $\{a_t\}_t$, we have $\frac{1}{a_T}\sum_{t=1}^{T}a_tA_t \rightarrow 0$~as~$T\rightarrow\infty$.

\end{lemma}

\begin{lemma}[Stolz–Ces\`{a}ro theorem]\label{aux:lem:stolz-cesaro}
    Let $\{a_n\}_{n \geq 1}$ and $\{b_n\}_{n \geq 1}$ be two sequences of real numbers. Assume that $\{b_n\}_{n \geq 1}$ is a strictly monotone and divergent sequence and the following limit exists:
    \begin{equation*}
        \lim_{n \to \infty} \frac{a_{n+1} - a_n}{b_{n+1} - b_n} = l;
    \end{equation*}
    then, we have the limit:
    \begin{equation*}
        \lim_{n \to \infty} \frac{a_n}{b_n} = l.
    \end{equation*}
\end{lemma}

\section{Proofs in Section \ref{sec:3}}

We first introduce some useful notation. Let $K_t \coloneqq \mE[\tilde{K}_t \mid \mF_{t-1}]$, where $\tilde{K}_t$ is defined analogously to $\tilde{K}^\star$ in \eqref{def:tilde_H_star}, but evaluated at $(\alpha_t, \beta_t, \gamma_t)$, $B_t$, and $E_t$. We now express these quantities more explicitly. Define
\begin{equation}\label{def:tilde_C_t}
    \tilde{C}_t \coloneqq \prod_{j=0}^{\tau-1} \underbrace{\begin{pmatrix}
        (1 - \alpha_t)(I - Z_{t,j}) &  \alpha_t(I - Z_{t,j})\\
        (1 - \alpha_t)(1 - \beta_t) I - (1 - \alpha_t)\gamma_t Z_{t,j} & (\alpha_t + \beta_t - \alpha_t \beta_t) I - \alpha_t \gamma_t Z_{t,j}
    \end{pmatrix}}_{C_{t,j}} \in \mR^{2d \times 2d},
\end{equation}
where $Z_{t,j} = E_t^{-1/2} B_t S_{t,j} (S_{t,j}^\top B_t E_t^{-1} B_t S_{t,j})^\dagger S_{t,j}^\top B_t E_t^{-1/2} \in \mR^{d \times d}$. Let $[\tilde{C}_t]_{1,1}$ and $[\tilde{C}_t]_{1,2}$ denote the upper-left and upper-right $d \times d$ blocks of $\tilde{C}_t$, respectively. Then
\begin{equation}\label{def:tilde_H_t}
    \tilde{K}_t = E_t^{-1/2} \bigl([\tilde{C}_t]_{1,1} + [\tilde{C}_t]_{1,2}\bigr) E_t^{1/2} = E_t^{-1/2} \begin{pmatrix} I & \bm 0 \end{pmatrix} \tilde{C}_t \begin{pmatrix} I \\ I \end{pmatrix} E_t^{1/2} \in \mR^{d \times d},
\end{equation}
where $I \in \mR^{d \times d}$ is the identity matrix. We further define $C_t \coloneqq \mE[\tilde{C}_t \mid \mF_{t-1}]$. By the independence of $\{S_{t,j}\}_j$,
\begin{equation}\label{def:C_t}
    C_t = \underbrace{\begin{pmatrix}
        (1 - \alpha_t)(I - Z_t) &  \alpha_t(I - Z_t)\\
        (1 - \alpha_t)(1 - \beta_t) I - (1 - \alpha_t)\gamma_t Z_t & (\alpha_t + \beta_t - \alpha_t \beta_t) I - \alpha_t \gamma_t Z_t
    \end{pmatrix}^{\tau}}_{M_t^{\tau}} \in \mR^{2d \times 2d},
\end{equation}
where $\tilde{Z}_t = E_t^{-1/2} B_t S (S^\top B_t E_t^{-1} B_t S)^\dagger S^\top B_t E_t^{-1/2}$ and $Z_t = \mE\bigl[ \tilde{Z}_t \mid \mF_{t-1}\bigr]$. Letting $[C_t]_{1,1}$ and $[C_t]_{1,2}$ denote the corresponding $d \times d$ blocks of $C_t$, we obtain
\begin{equation}\label{def:H_t}
    K_t = E_t^{-1/2} \bigl([C_t]_{1,1} + [C_t]_{1,2}\bigr) E_t^{1/2} = E_t^{-1/2} \begin{pmatrix} I & \bm 0 \end{pmatrix} C_t \begin{pmatrix} I \\ I \end{pmatrix} E_t^{1/2} \in \mR^{d \times d}.
\end{equation}

\subsection{Proof of Lemma \ref{lem:4}}

    We define $T_t \coloneqq B_t S$, hence $Z_t = \mE[ E_t^{-1/2}T_t(T_t^{\top} E_t^{-1} T_t)^{\dagger}T_t^{\top} E_t^{-1/2} \mid \mF_{t-1}]$. Since $ (1/\gamma_E) \cdot T_t^{\top}T_t \succeq T_t^{\top}E_t^{-1}T_t$ by Assumption \ref{ass:4}, and $\text{Ker}(T_t^{\top}T_t) = \text{Ker}(T_t^{\top}E_t^{-1}T_t)$, we then have
    \begin{equation}
        (T_t^{\top}E_t^{-1} T_t)^{\dagger} \succeq \gamma_E \cdot (T_t^{\top}T_t)^{\dagger},
    \end{equation}
    which implies that 
    \begin{equation*}
        E_t^{-1/2}T_t(T_t^{\top} E_t^{-1} T_t)^{\dagger}T_t^{\top} E_t^{-1/2} \succeq \gamma_E \cdot E_t^{-1/2}T_t(T_t^{\top}  T_t)^{\dagger}T_t^{\top} E_t^{-1/2}.
    \end{equation*}
    We take conditional expectation $\mE[\cdot | \mF_{t-1}]$ on both sides, and using Assumption \ref{ass:4}, and we get
    \begin{equation}\label{equ:123}  
        Z_t \succeq \gamma_E \cdot \gamma_S \cdot E_t^{-1} \succeq \frac{\gamma_S \gamma_E}{\Upsilon_E} I,
    \end{equation}
    where the last inequality uses Assumption \ref{ass:4}.  We complete the proof of the lemma.

\subsection{Proof of Theorem \ref{thm:as_convergence}}

We first present the following lemma, which will be useful in our proof.

\begin{lemma}\label{lem:prop_sketch_solver}
    Under Assumptions \ref{ass:1}-\ref{ass:4}, the following results hold: \\
     (a) We have $\bz_{t, \tau} = (I - \tilde{K}_t) \Delta \bx_t = - (I - \tilde{K}_t) B_t^{-1} \bg_t$ for any $t \geq 0$.\\
     (b) For any  $t \geq 0$, we have $\mE[ \bz_{t, \tau} \mid \mF_{t-1}] = -(I - K_t)B_t^{-1}\nabla f_t$, and $ \bz_{t, \tau} - (I - K_t) \Delta \bx_t$ is a martingale difference. \\
     (c) For any $t \geq 0$, we have $\gamma_t \leq \sqrt{\frac{\Upsilon_E}{\gamma_E \gamma_S}}$, and $\|\tilde{K}_t\| \leq C_K$, where $C_K \coloneqq \sqrt{\frac{2\Upsilon_E}{\gamma_E}} \cdot \left( 2 + \sqrt{\frac{\Upsilon_E}{\gamma_E \gamma_S}} \right)^{\tau} > 0$ is a universal constant. \\
     (d) For any $ t\geq 0$, we have $\|K_t\| \leq \sqrt{\frac{\Upsilon_E}{\gamma_E}} \cdot 2\tau (1 - \sqrt{\mu_t/\nu_t})^{\tau - 2}$.
    
\end{lemma}

By Assumption \ref{ass:1}, $f(\bx)$ is strongly convex, then according to \cite{Nesterov2018Lectures}, we have 
\begin{equation}\label{equ:46}
    \frac{\gamma_H}{2} \|\bx_t - \bx^{\star}\|^2 \leq f_t - f^{\star} \leq \frac{1}{2\gamma_H} \|\nabla f_t\|^2.
\end{equation}
By \eqref{equ:46}; the $\Upsilon_H$-Lipschitz continuity of $\nabla f(\bx)$ implied in Assumption \ref{ass:1}; and \cite{Nesterov2018Lectures}, we have
\begin{equation}\label{equ:59}
    \frac{1}{2\Upsilon_H} \|\nabla f_t\|^2 \leq f_t - f^{\star} \leq \frac{\Upsilon_H}{2} \|\bx_t - \bx^{\star}\|^2 .
\end{equation}
By the $\Upsilon_H$-Lipschitz continuity of $\nabla f(\bx)$, we also have
\begin{equation}\label{equ:48}
    f_{t+1} - f^{\star} \leq f_t - f^{\star} + \varphi_t \nabla f_t^{\top} \bz_{t, \tau} + \frac{\Upsilon_H}{2} \varphi_t^2 \|\bz_{t, \tau}\|^2.
\end{equation}
We then take conditional expectation $\mE[\cdot \mid \mF_{t-1}]$ on both sides of \eqref{equ:48}, and obtain:
\begin{equation}\label{equ:44}
    \mE[f_{t+1} - f^{\star} \mid \mF_{t-1}] \leq f_t - f^{\star} + \varphi_t \mE[ \nabla f_t^{\top}\bz_{t, \tau} \mid \mF_{t-1}] + \frac{\Upsilon_H}{2} \varphi_t^2 \mE[\|\bz_{t, \tau}\|^2 \mid \mF_{t-1}].
\end{equation}
Also, by Assumption~\ref{ass:1} as well as the construction of the Hessian approximate $B_t$, we have
\begin{equation} \label{ass:2:c2}
\gamma_H \leq \lambda_{\min}(B_t) \leq \lambda_{\max}(B_t) \leq \Upsilon_H.
\end{equation}
For the second term in the right-hand side of \eqref{equ:44}, we obtain that
\begin{align}\label{equ:45}
    \varphi_t \mE[ \nabla f_t^{\top} \bz_{t, \tau} \mid \mF_{t-1}] &= \varphi_t \nabla f_t^{\top} \mE[ \bz_{t, \tau} \mid \mF_{t-1}] \nonumber \\
    &= -\varphi_t \nabla f_t^{\top} \mE[ (I - \tilde{K}_t) B_t^{-1} \bg_t \mid \mF_{t-1}]  \nonumber \\
    &= -\varphi_t \nabla f_t^{\top} (I - K_t)B_t^{-1} \nabla f_t \nonumber \\
    &= -\varphi_t \nabla f_t^{\top} B_t^{-1} \nabla f_t + \varphi_t \nabla f_t^{\top} K_t B_t^{-1} \nabla f_t \nonumber \\
    &\stackrel{\mathclap{\eqref{ass:2:c2}}}{\leq} -\frac{1}{\Upsilon_H} \varphi_t \|\nabla f_t\|^2 + \varphi_t \frac{1}{\gamma_H} \|\nabla f_t\|^2 \|K_t\| \nonumber \\
    &\leq -\frac{1}{2\Upsilon_H} \varphi_t \|\nabla f_t\|^2,
\end{align}
where the second equality uses Lemma \ref{lem:prop_sketch_solver}(a), the third equality uses Lemma \ref{lem:prop_sketch_solver}(b), the last inequality holds due to $\|K_t\| \leq \sqrt{\frac{\Upsilon_E}{\gamma_E}} \cdot 2\tau (1 - \sqrt{\mu_t/\nu_t})^{\tau - 2}$ for all $t$ in Lemma \ref{lem:prop_sketch_solver}(d), and the choice of $\tau$ such that $\|K_t\| \leq \sqrt{\frac{\Upsilon_E}{\gamma_E}} \cdot 2\tau (1 - \sqrt{\mu_t/\nu_t})^{\tau - 2} \leq \gamma_H / 2\Upsilon_H$. For the third term in the right-hand side of \eqref{equ:44}, we have
\begin{align}\label{equ:47} 
    \frac{\Upsilon_H}{2} \varphi_t^2 \mE[ & \|\bz_{t, \tau} \|^2  \mid \mF_{t-1}]  \leq \frac{\Upsilon_H}{2} \varphi_t^2 \mE[ \|I - \tilde{K}_t\|^2 \|B_t^{-1}\|^2 \|\bg_t\|^2 \mid \mF_{t-1}] \\
    &\leq \frac{\Upsilon_H}{2} \varphi_t^2  \frac{(1+C_K)^2 }{\gamma_H^2} \mE[\|\bg_t\|^2 \mid \mF_{t-1}] \\
    &\leq \frac{\Upsilon_{H}(1+C_K)^2}{\gamma_H^2} \varphi_t^2 \left( \|\nabla f_t\|^2 + \mE[\|\bg_t - \nabla f_t\|^2 \mid \mF_{t-1}] \right) \qquad \qquad (\text{Young's inequality}) \\
    &\stackrel{\mathclap{\eqref{ass:2:4th}}}{\leq} \frac{\Upsilon_{H}(1+C_K)^2}{\gamma_H^2} \varphi_t^2  \|\nabla f_t\|^2 + \frac{\Upsilon_{H}(1+C_K)^2}{\gamma_H^2} \varphi_t^2 \left(C_{g, 1}^{\frac{2}{2+\epsilon}}\|\bx_t - \bx^{\star}\|^2 + C_{g, 2}^{\frac{2}{2+\epsilon}} \right) \\
    & \stackrel{\mathclap{\eqref{equ:46}}}{\leq} \ \frac{\Upsilon_{H}(1+C_K)^2}{\gamma_H^2} \varphi_t^2  \|\nabla f_t\|^2 + \frac{\Upsilon_{H}(1+C_K)^2C_{g, 1}^{\frac{2}{2+\epsilon}}}{\gamma_H^4} \varphi_t^2  \|\nabla f_t\|^2 + \frac{\Upsilon_{H}(1+C_K)^2C_{g, 2}^{\frac{2}{2+\epsilon}}}{\gamma_H^2} \varphi_t^2 \\
    &= \frac{(\gamma_H^2+ C_{g, 1}^{\frac{2}{2+\epsilon}}) \Upsilon_{H}(1+C_K)^2}{\gamma_H^4}\varphi_t^2  \|\nabla f_t\|^2 + \frac{\Upsilon_{H}(1+C_K)^2C_{g, 2}^{\frac{2}{2+\epsilon}}}{\gamma_H^2} \varphi_t^2,
\end{align}
where the first inequality holds due to Lemma \ref{lem:prop_sketch_solver}(a), and the second inequality uses Assumption \ref{ass:1} and Lemma \ref{lem:prop_sketch_solver}(c), and the fourth inequality also uses Jensen's inequality.
Then we plug in \eqref{equ:45} and \eqref{equ:47} into \eqref{equ:44}, and get
\begin{multline}\label{equ:50}
    \mE[f_{t+1} - f^{\star} \mid \mF_{t-1}] \leq f_t - f^{\star}  -\frac{1}{2\Upsilon_H} \varphi_t \|\nabla f_t\|^2\\
    + \frac{\left(\gamma_H^2+ C_{g, 1}^{\frac{2}{2+\epsilon}}\right) \Upsilon_{H}(1+C_K)^2}{\gamma_H^4}\varphi_t^2  \|\nabla f_t\|^2 + \frac{\Upsilon_{H}(1+C_K)^2C_{g, 2}^{\frac{2}{2+\epsilon}}}{\gamma_H^2} \varphi_t^2.
\end{multline}
Since $\varphi_t = C_{\varphi} \cdot (t+1)^{-\varphi}$ where $\varphi \in (0.5, 1)$, then there exists $t_0 > 0$, such that for all $t \geq t_0$, we have $\frac{\left(\gamma_H^2+ C_{g, 1}^{\frac{2}{2+\epsilon}}\right) \Upsilon_{H}(1+C_K)^2}{\gamma_H^4} \varphi_t \leq \frac{1}{4\Upsilon_H}$. Then for all $t \geq t_0$, we have
\begin{equation}\label{equ:51}
    \mE[f_{t+1} - f^{\star} \mid \mF_{t-1}] \leq f_t - f^{\star} - \frac{1}{4\Upsilon_H}\varphi_t \|\nabla f_t\|^2 + \frac{\Upsilon_{H}(1+C_K)^2C_{g, 2}^{\frac{2}{2+\epsilon}}}{\gamma_H^2} \varphi_t^2.
\end{equation}
Noticing that $\sum_{t = t_0}^{\infty} \varphi_t^2 < \infty$, then we apply the Robbins-Siegmund theorem \citep[Theorem 1.3.12]{Duflo2013Random}, and conclude that $f_t - f^{\star}$ converges to a finite random variable, and $\sum_{t = t_0}^{\infty} \varphi_t \|\nabla f_t\|^2 < \infty$ almost surely. Furthermore, we have $\liminf_{t \to \infty} \|\nabla f_t\| = 0$ almost surely since $\sum_{t = t_0}^{\infty} \varphi_t = \infty$, which leads to $\liminf_{t \to \infty}(f_t - f^{\star}) = 0$ almost surely according to \eqref{equ:46}. Since we have already known that $f_t - f^{\star}$ converges almost surely, the conclusion can be strengthened to $\lim_{t \to \infty} f_t - f^{\star} = 0$. Again, we apply \eqref{equ:46} and obtain $\lim_{t \to \infty} \bx_t = \bx^{\star}$ almost surely.

We now establish the $L^2$-convergence of $\bx_t$. We take full expectation on both sides of \eqref{equ:51}, and obtain for all $t \geq t_0$,
\begin{align}\label{equ:52}
    \mE[f_{t+1} - f^{\star}] &\leq \mE[f_t - f^{\star}] - \frac{1}{4\Upsilon_H}\varphi_t \mE[\|\nabla f_t\|^2] + \frac{\Upsilon_{H}(1+C_K)^2C_{g, 2}^{\frac{2}{2+\epsilon}}}{\gamma_H^2} \varphi_t^2 \\
    &\stackrel{\mathclap{\eqref{equ:46}}}{\leq} \  \mE[f_t - f^{\star}] - \frac{\gamma_H}{2\Upsilon_H}\varphi_t \mE[f_t - f^{\star}] + \frac{\Upsilon_{H}(1+C_K)^2C_{g, 2}^{\frac{2}{2+\epsilon}}}{\gamma_H^2} \varphi_t^2 \\
    &= \left(1 - \frac{\gamma_H}{2\Upsilon_H}\varphi_t \right) \mE[f_t - f^{\star}] + \frac{\Upsilon_{H}(1+C_K)^2C_{g, 2}^{\frac{2}{2+\epsilon}}}{\gamma_H^2} \varphi_t^2.
\end{align}

Using Lemma \ref{lem:phi_t_recursion_lemma}, we know that $\mE[f_{t} - f^{\star}] \lesssim \varphi_t$ for all $t \geq 0$. We apply \eqref{equ:46} again, and then we have for all $t \geq 0$,
\begin{equation}\label{equ:53}
    \mE[\|\bx_t - \bx^{\star}\|^2] \leq \frac{2}{\gamma_H}\mE[f_{t} - f^{\star}] \lesssim \varphi_t.
\end{equation}
We now complete the proof of the theorem.

\section{Proofs in Section \ref{sec:4}}
We first introduce a lemma on the convergence of $B_t$, noting that neither the convergence nor the convergence rate of $B_t$ requires Assumption~\ref{ass:5}.
\begin{lemma}\label{lem:converge_rate_and_Bt_converge}
    Let $B^\star = \nabla^2 f^\star$. Under the conditions of Theorem~\ref{thm:as_convergence}, $B_t \to B^\star$ almost surely as $t \to \infty$, and
    \begin{equation}\label{equ:Bt_rate}
        \mE[\|B_t - B^\star\|^2] \lesssim \varphi_t \quad \text{for all } t \geq 0.
    \end{equation}
\end{lemma}
We next express $C^\star$ and $K^\star$ more explicitly. Recall that $C^\star = \mE[\tilde{C}^\star]$. Since $S_0, S_1, \ldots, S_{\tau-1} \stackrel{iid}{\sim} S$, we have
\begin{equation}\label{def:C_star}
    C^\star = \underbrace{\begin{pmatrix}
        (1 - \alpha^\star)(I - Z^\star) & \alpha^\star(I - Z^\star) \\
        (1 - \alpha^\star)(1 - \beta^\star) I - (1 - \alpha^\star)\gamma^\star Z^\star & (\alpha^\star + \beta^\star - \alpha^\star \beta^\star) I - \alpha^\star \gamma^\star Z^\star
    \end{pmatrix}^{\tau}}_{(M^\star)^\tau} \in \mR^{2d \times 2d}.
\end{equation}
Let $[C^\star]_{1,1}, [C^\star]_{1,2} \in \mR^{d \times d}$ denote the upper-left and upper-right $d \times d$ blocks of $C^\star$, respectively. Since $K^\star = \mE[\tilde{K}^\star]$,
\begin{equation}\label{def:H_star}
    K^\star = (E^\star)^{-1/2} \bigl([C^\star]_{1,1} + [C^\star]_{1,2}\bigr) (E^\star)^{1/2} = (E^\star)^{-1/2} \begin{pmatrix} I & \bm 0 \end{pmatrix} C^\star \begin{pmatrix} I \\ I \end{pmatrix} (E^\star)^{1/2} \in \mR^{d \times d},
\end{equation}
where $I \in \mR^{d \times d}$ is the identity matrix. The following lemma establishes the convergence of $K_t$.
\begin{lemma}\label{lem:bound_nu_t_mu_t}
    Under the conditions of Theorem~\ref{thm:as_convergence} and Assumption~\ref{ass:5}, $\|K_t - K^\star\| \lesssim \|B_t - B^\star\| + \|E_t - E^\star\|$ for all $t \geq 0$. Furthermore, $K_t \to K^\star$ almost surely as $t \to \infty$.
\end{lemma}
We now turn to the proof of the main theorem.

\subsection{Proof of Theorem \ref{thm:asymptotic_normality}}

    We first have the following lemma regarding the error decomposition of $\bx_t - \bx^{\star}$.

\begin{lemma}\label{lem:error_recursion}
        Under the conditions of Theorem \ref{thm:asymptotic_normality}, the error recursion of the algorithm can be decomposed as 
        \begin{equation}\label{equ:xt_recursion}
            \bx_{t+1} - \bx^{\star} = \{I - \varphi_t(I - K^{\star})\}(\bx_t - \bx^{\star}) + \varphi_t \boldsymbol{\delta}_t + \varphi_t \boldsymbol{\theta}_t,
        \end{equation}
        where $\btheta_t$ is a martingale difference and $\boldsymbol{\delta}_t$ is the error term, which are defined by
        \begin{subequations}\label{rec:def}
\begin{flalign}
\boldsymbol{\theta}_t &= -  (I - K_t) B_t^{-1} (\bg_t - \nabla f_t ) + (\bz_{t, \tau} - (I - K_t)\Delta \bx_t), \label{rec:def:b}\\ 
\bdelta_{t} & = (K_t - K^{\star}) (\bx_t - \bx^{\star}) \\
& \quad \ \ - (I - K_t)[(B^{\star})^{-1}(\nabla f_t - B^{\star}(\bx_t - \bx^{\star}))] - (I - K_t) [B_t^{-1} - (B^{\star})^{-1}] \nabla f_t. \label{rec:def:c}
\end{flalign}
\end{subequations}	
\end{lemma}

Since $K_t \stackrel{a.s.}{\to} K^{\star}$ (cf. Lemma \ref{lem:bound_nu_t_mu_t}) and $\|K_t\|\le 1/2$ for all $t \geq 0$ \footnote{By Lemma \ref{lem:prop_sketch_solver}(d), we know that $\|K_t\|  \leq \sqrt{\frac{\Upsilon_E}{\gamma_E}} \cdot 2\tau (1 - \sqrt{\mu_t/\nu_t})^{\tau - 2}$. Then under the condition of Theorem \ref{thm:as_convergence}, we have $\|K_t\| \leq \frac{\gamma_H}{2\Upsilon_H} \leq 1/2$.}, the continuity of the spectral norm implies $\|K^\star\|\le 1/2$. Hence, for every eigenvalue $\lambda_i(I-K^\star)$ of $I-K^\star$, we have
$\text{Re}\bigl(\lambda_i(I-K^\star)\bigr) \ge 1/2 >0$, where $\text{Re}(\cdot)$ denotes the real part. Let us define two matrices
\begin{equation}\label{eq:A_B_recursion_matrices}
R_i^t \coloneqq \varphi_i \sum_{k=i}^{t}\prod_{j= i+1}^{k}(I - \varphi_j (I - K^{\star}))\qquad \text{  and  } \qquad A_i^t \coloneqq R_i^t - (I - K^{\star})^{-1}.
\end{equation}
We apply \eqref{equ:xt_recursion} in Lemma \ref{lem:error_recursion} recursively and obtain
\begin{equation*}
\bx_{k+1} - \bx^{\star} = \prod_{j=0}^{k}\rbr{I - \varphi_j(I-K^\star)}(\bx_0 - \bx^{\star}) + \sum_{i=0}^{k}\prod_{j=i+1}^{k}\rbr{I - \varphi_j(I-K^\star)}\varphi_i\rbr{\btheta_i + \bdelta_i }.
\end{equation*}
Therefore, we obtain
\begin{align}\label{eq:polyakrecursion}
& \frac{1}{\sqrt{t}} \sum_{k = 1}^{t} (\bx_{k} - \bx^{\star}) \nonumber \\
& = \frac{1}{\sqrt{t}}\sum_{k=0}^{t-1}\prod_{j=0}^{k}\rbr{I - \varphi_j(I-K^\star)}(\bx_0 - \bx^{\star}) + \frac{1}{\sqrt{t}}\sum_{k=0}^{t-1}\sum_{i=0}^{k}\prod_{j=i+1}^{k}\rbr{I - \varphi_j(I-K^\star)}\varphi_i\rbr{\btheta_i + \bdelta_i} \nonumber\\
& \stackrel{\mathclap{\eqref{eq:A_B_recursion_matrices}}}{=} \frac{1}{\sqrt{t} \varphi_0} R_{0}^{t-1} (I - \varphi_0 (I - K^{\star})) (\bx_0 - \bx^{\star}) + \frac{1}{\sqrt{t}}\sum_{i=0}^{t-1}\sum_{k=i}^{t-1}\prod_{j=i+1}^{k}\rbr{I - \varphi_j(I - K^\star)}\varphi_i\rbr{\btheta_i + \bdelta_i} \nonumber\\
& \stackrel{\mathclap{\eqref{eq:A_B_recursion_matrices}}}{=} \frac{1}{\sqrt{t} \varphi_0} R_{0}^{t-1} (I - \varphi_0 (I - K^{\star})) (\bx_0 - \bx^{\star}) + \frac{1}{\sqrt{t}}\sum_{i=0}^{t-1}R_{i}^{t-1}\rbr{\btheta_i + \bdelta_i} \nonumber\\
& \stackrel{\mathclap{\eqref{eq:A_B_recursion_matrices}}}{=} \frac{1}{\sqrt{t} \varphi_0} R_{0}^{t-1} (I - \varphi_0 (I - K^{\star})) (\bx_0 - \bx^{\star}) + \frac{1}{\sqrt{t}}\sum_{i=0}^{t-1}R_{i}^{t-1}\bdelta_i + \frac{1}{\sqrt{t}}\sum_{i=0}^{t-1}A_i^{t-1}\btheta_i  \nonumber \\
& \quad + \frac{1}{\sqrt{t}} \sum_{i = 0}^{t-1}(I - K^{\star})^{-1}\btheta_i.
\end{align}
In what follows, we analyze each right-hand-side term above and establish the following claims:
\begin{enumerate}[topsep=2pt,parsep=3pt,label=(\alph*):]
\setlength\itemsep{0.0em}
\item $\frac{1}{\sqrt{t}} \sum_{i = 0}^{t-1}(I - K^{\star})^{-1}\btheta_i \stackrel{d}{\longrightarrow} \mathcal{N}(0, \Xi^{\star}_{\text{gasn, ave}})$, where $\Xi^\star_{\text{gasn, ave}} = (I - K^{\star})^{-1} \mE[(I - \tilde{K}^{\star}) \Omega^{\star} (I - \tilde{K}^{\star})^{\top}] (I - K^{\star})^{-\top}$.
\item $\frac{1}{\sqrt{t}}\sum_{i=0}^{t-1}A_i^{t-1} \btheta_i = o_p(1)$.
\item $\frac{1}{\sqrt{t}}\sum_{i=0}^{t-1} R_i^{t-1} \bdelta_i = o(1)$.
\item $\frac{1}{\sqrt{t} \varphi_0} R_{0}^{t-1} [I - \varphi_0 (I - K^{\star})] (\bx_0 - \bx^{\star}) = o(1)$.	
\end{enumerate}
Combining these claims together, we obtain
\begin{equation*}
\frac{1}{\sqrt{t}} \sum_{k = 1}^{t} (\bx_k - \bx^{\star}) \stackrel{d}{\longrightarrow} \mathcal{N}(0, \Xi^{\star}_{\text{gasn, ave}}).
\end{equation*}
This completes the proof by noting that $1/\sqrt{t}\sum_{k=0}^{t-1} (\bx_k - \bx^{\star}) = 1/\sqrt{t} \sum_{k = 1}^{t} (\bx_k - \bx^{\star}) - (\bx_t - \bx^{\star})/\sqrt{t} + (\bx_0 - \bx^{\star})/\sqrt{t}$ and~$(\bx_t - \bx^{\star})\rightarrow \0$ as $t\rightarrow \infty$ almost surely.

\vskip4pt
\noindent $\bullet$ {\textbf{Proof of (a).}} We need the following lemma to proceed with the proof.

\begin{lemma}[Adapted Lemma B.6 in \cite{Wang2026Inference}]\label{lem:3}
    Under the conditions of Theorem \ref{thm:asymptotic_normality}. We have the following convergence as $t \to \infty$:
    \begin{equation*}
        \mE[\boldsymbol{\theta}_t \boldsymbol{\theta}_t^{\top} \mid \mF_{t-1}] \stackrel{a.s.}{\to} \mE[(I - \tilde{K}^{\star}) \Omega^{\star} (I - \tilde{K}^{\star})^{\top}].
    \end{equation*}
\end{lemma}
By Lemma \ref{lem:3}, we know $\btheta_i$ is a martingale difference with the limiting covariance $\mE[(I-\tilde{K}^\star)\Omega^{\star}(I-\tilde{K}^\star)^{\top}]$. Hence, by Stolz–Ces\`{a}ro theorem (cf. Lemma \ref{aux:lem:stolz-cesaro}), we have
\begin{equation*}
    \frac1t\sum_{i=0}^{t-1}
    \mathbb E[\boldsymbol\theta_i\boldsymbol\theta_i^\top
    \mid \mathcal F_{i-1}] \stackrel{a.s.}{\to} \mE[(I - \tilde{K}^{\star}) \Omega^{\star} (I - \tilde{K}^{\star})^{\top}].
\end{equation*}
Consequently,
\begin{equation*}
    \frac1t\sum_{i=0}^{t-1}
    \mathbb E[
    (I-K^\star)^{-1}\boldsymbol\theta_i
    \boldsymbol\theta_i^\top
    (I-K^\star)^{-\top}
    \mid \mathcal F_{i-1}]
    \stackrel{a.s.}{\longrightarrow}
    \Xi^\star_{\text{gasn, ave}}.
\end{equation*}
Thus, it suffices to verify the following Lindeberg condition.

\begin{lemma}\label{lem:B1}
Under the conditions of Theorem \ref{thm:asymptotic_normality}, we have for any $\delta>0$, 
\begin{equation*}
\frac{1}{t} \sum_{i = 0}^{t-1} \mathbb{E}\left[ \|\boldsymbol{\theta}_i\|^2 \cdot \mathbf{1}_{\|\boldsymbol{\theta}_i\| > \delta \sqrt{t}}  \mid \mathcal{F}_{i-1} \right] \stackrel{a.s.}{\to} 0 \quad \text{ as }\quad t\rightarrow\infty.
\end{equation*}
\end{lemma}

With the above Lindeberg condition, we apply the martingale central limit theorem in \citet[Proposition 2.1.9]{Duflo2013Random} and obtain the result in \textbf{(a)}.

\vskip4pt
\noindent $\bullet$ {\textbf{Proof of (b).}} We need the following lemma to provide the uniform bounds for matrices $R_i^t$ and $A_i^t$ in \eqref{eq:A_B_recursion_matrices}.

\begin{lemma}[Adapted Lemma 1 in \cite{Polyak1992Acceleration}]\label{lem:B2}
Given $-(I - K^\star)$ is Hurwitz, i.e. $\text{Re} \; \lambda_i(I - K^{\star}) > 0$ for all $i = 1, 2, ..., d$, then there is a constant $\Upsilon_{AR}>0$ such that $\max\{\|R_i^t\|, \|A_i^t\|\}\leq \Upsilon_{AR}$, $\forall i, t\geq 0$. Furthermore, we have $\lim\limits_{t \rightarrow \infty} \frac{1}{t}\sum_{i=0}^{t} \norm{A_i^{t}} =0$.
\end{lemma}

With the above lemma, we have
\begin{align}\label{equ:104} 
\mE \left[ \left\| \frac{1}{\sqrt{t}}\sum_{i=0}^{t-1}A_i^{t-1} \btheta_i \right\|^2 \right] &= \frac{1}{t} \sum_{i = 0}^{t-1} \mE \left[ \|A_i^{t-1} \btheta_i \|^2 \right] \leq \Upsilon_{AR}\cdot \left(\sup_{i \geq 0} \mE[\|\btheta_i\|^2] \right)\cdot\frac{1}{t} \sum_{i = 0}^{t-1} \|A_i^{t-1}\| \nonumber \\
&\leq \Upsilon_{AR} \cdot \sup_{i \geq 0}\rbr{(C_q^{(1)})^{2/q} \cdot \mE [\|\bx_i - \bx^{\star}\|^2] + (C_q^{(2)})^{2/q}} \cdot \frac{1}{t} \sum_{i = 0}^{t-1} \|A_i^{t-1}\| \nonumber \\
&\lesssim \  \frac{1}{t} \sum_{i = 0}^{t-1} \|A_i^{t-1}\| \to 0
\end{align}
as $t\rightarrow\infty$, where the second inequality uses \eqref{eq: bounded_q_moment} in proving Lemma \ref{lem:B1}, and the last inequality uses \eqref{equ:xt_rate} in Theorem \ref{thm:as_convergence}. Thus, $\frac{1}{\sqrt{t}}\sum_{i=0}^{t-1}A_i^{t-1} \btheta_i = o_p(1)$ and this shows \textbf{(b)}.

\vskip4pt
\noindent $\bullet$ {\textbf{Proof of (c).}} By Lemma \ref{lem:B2}, we know it suffices to show that $1/\sqrt{t}\sum_{i=0}^{t-1} \left\| \bdelta_i \right\| \stackrel{a.s.}{\longrightarrow} 0$ as~$t\rightarrow\infty$. By the definition of $\boldsymbol{\delta}_i$ in \eqref{rec:def:c}, we have
\begin{align}\label{nequ:4}
\|\bdelta_i\| \; & \leq \;   \frac{3}{2} \cbr{\|(B^\star)^{-1}\|\cdot\|\nabla f_i - B^{\star}(\bx_i - \bx^{\star})\| + \|B_i^{-1} - (B^\star)^{-1}\|\cdot\|\nabla f_i\|} + \|K_i - K^{\star}\|\cdot \|\bx_i - \bx^{\star}\| \nonumber\\
& \lesssim \frac{3 \Upsilon_{L}}{2 \gamma_H}\|\bx_i - \bx^{\star}\|^2 + \frac{3}{2 \gamma_H^2} \|B_i - B^{\star}\|\cdot\|\nabla f_i\| +  (\|B_i- B^{\star}\| + \|E_i - E^{\star}\|)\cdot\|\bx_i - \bx^{\star}\| \nonumber\\
& \lesssim \|\bx_i - \bx^{\star}\|^2 +  (\|B_i- B^{\star}\| + \|E_i - E^{\star}\|)\cdot\|\bx_i - \bx^{\star}\| \nonumber \\
& \leq \frac{1}{2} \|B_i - B^\star\|^2 + \frac{1}{2} \|E_i - E^\star\|^2 + 2\|\bx_i - \bx^{\star}\|^2,
\end{align}
where the first inequality also uses $\|K_t\| \leq 1/2$; the second inequality uses Assumption \ref{ass:1}, and Lemma \ref{lem:bound_nu_t_mu_t}; the third inequality uses the $\Upsilon_H$-Lipschitz continuity of $\nabla f(\bx)$; and the last inequality uses $2ab \leq a^2 + b^2$ for $a, b\geq 0$.

Using Theorem \ref{thm:as_convergence} and Lemma \ref{lem:converge_rate_and_Bt_converge}, we have
\begin{align}\label{equ:84}
    \mE\left[ \sum_{i = 0}^{\infty} \frac{1}{\sqrt{i+1}} \|\boldsymbol{\delta}_i\| \right] & = \sum_{i = 0}^{\infty} \frac{1}{\sqrt{i+1}} \mE[\|\boldsymbol{\delta}_i\|] \\
    &\stackrel{\mathclap{\eqref{nequ:4}}}{\lesssim} \  \sum_{i = 0}^{\infty} \frac{1}{\sqrt{i+1}} \left( \mE[\|B_i - B^{\star}\|^2] + \mE[\|E_i - E^{\star}\|^2] + \mE[\|\bx_i - \bx^{\star}\|^2] \right) \\
    &\lesssim \ \sum_{i = 0}^{\infty} \frac{1}{\sqrt{i+1}} \cdot \varphi_i < \infty,
\end{align}
where the second last inequality also uses Assumption \ref{ass:5}, and the last inequality uses $\varphi_i = C_{\varphi} / (i+1)^{\varphi}$ with $\varphi \in (0.5, 1)$. This suggests that $\sum_{i = 0}^{\infty}  \|\boldsymbol{\delta}_i\|/\sqrt{i+1} < \infty$ almost surely. Then by Kronecker's lemma (cf. Lemma \ref{lem:kronecker}), we have 
\begin{equation}\label{equ:114}  
    \frac{1}{\sqrt{t}}  \sum_{i = 0}^{t-1}  \|\boldsymbol{\delta}_i\| \stackrel{a.s.}{\to} 0.
\end{equation}
We complete the proof of $\textbf{(c)}$.

\vskip4pt
\noindent $\bullet$ {\textbf{Proof of (d).}} The result in \textbf{(d)} is trivial due to Lemma \ref{lem:B2}.

\subsection{Proof of Proposition \ref{prop:comp_stats_tradeoff}}

    \noindent$\bullet$ \textbf{Proof of (a).} When we solve $B_t \Delta \bx_t = - \bg_t$ exactly,  we have $\tilde{K}^{\star} = K^{\star} = \boldsymbol{0}$. Therefore, we finish the proof of (a) by using the fact that $\Xi^{\star}_{\text{gasn, ave}} = (I - K^{\star})^{-1} \mE[(I - \tilde{K}^{\star}) \Omega^{\star} (I - \tilde{K}^{\star})^{\top}] (I - K^{\star})^{-\top}$. \\

    \noindent$\bullet$ \textbf{Proof of (b).} By the definitions of ${\Xi}^{\star}_{\text{gasn, ave}}$ in Theorem \ref{thm:asymptotic_normality} and $\Omega^{\star}$ in \eqref{nsequ:sgd_normal}, we have 
\begin{align} \label{equ:149}  
{\Xi}^{\star}_{\text{gasn, ave}} - \Omega^{\star}  &= (I - K^{\star})^{-1}\mathbb{E}\left[ (I - \widetilde{K}^{\star})\Omega^{\star}(I - \widetilde{K}^{\star})^{\top}\right](I - K^{\star})^{-\top} - \Omega^{\star} \nonumber\\
&= (I - K^{\star})^{-1}\mE[(\widetilde{K}^{\star} - K^{\star}) \Omega^{\star} (\widetilde{K}^{\star} - K^{\star})^{\top}](I - K^{\star})^{-\top}.
\end{align}
Since $\Omega^{\star} \succeq 0$, we know $(I - K^{\star})^{-1}\mE[(\widetilde{K}^{\star} - K^{\star}) \Omega^{\star} (\widetilde{K}^{\star} - K^{\star})^{\top}](I - K^{\star})^{-\top} \succeq 0$, then we have ${\Xi}^{\star}_{\text{gasn, ave}} \succeq \Omega^{\star}$ by \eqref{equ:149}, which proves the first part of the result.

    For the second part of the result, we have
    \begin{align}\label{equ:133}
        \|\Xi^{\star}_{\text{gasn, ave}} - \Omega^{\star}\| \ &  \stackrel{\mathclap{\eqref{equ:149}}}{=}  \ \left\|(I - K^{\star})^{-1}\mE[(\widetilde{K}^{\star} - K^{\star}) \Omega^{\star} (\widetilde{K}^{\star} - K^{\star})^{\top}](I - K^{\star})^{-\top} \right\| \\
        &= \left\|(I - K^{\star})^{-1}  \left(\mE[\tilde{K}^{\star} \Omega^{\star} (\tilde{K}^{\star})^{\top}] - K^{\star} \Omega^{\star} (K^{\star})^{\top} \right)  (I - K^{\star})^{-\top}  \right\| \nonumber \\
        &\leq \|(I - K^{\star})^{-1}\|^2 \cdot \|\Omega^{\star}\| \cdot \left(\|\mE[\tilde{K}^{\star} (\tilde{K}^{\star})^{\top}]\| + \|K^{\star}\|^2 \right) \nonumber \\
        &\leq 4 \cdot \|\Omega^{\star}\| \cdot \left(\|\mE[\tilde{K}^{\star} (\tilde{K}^{\star})^{\top}]\| + \frac{4\Upsilon_E}{\gamma_E} \cdot \tau^2 (1 - \sqrt{\mu^{\star}/\nu^{\star}})^{2\tau - 4} \right),
    \end{align}
    where the first inequality holds due to the fact that $\tilde{K}^{\star} \Omega^{\star} (\tilde{K}^{\star})^{\top} \preceq \|\Omega^{\star}\| \cdot \tilde{K}^{\star} (\tilde{K}^{\star})^{\top}$, and the last inequality holds since the fact that $\|K_t\| \leq \sqrt{\frac{\Upsilon_E}{\gamma_E}} \cdot 2\tau (1 - \sqrt{\mu_t/\nu_t})^{\tau - 2}$  (Lemma \ref{lem:prop_sketch_solver}(d)), and $\|K^{\star}\| \leq 0.5$. Now we focus on the upper bound of $\mE[\tilde{K}^{\star} (\tilde{K}^{\star})^{\top}]$.

        We now consider solving a special linear system $B^{\star} \Delta \bx^{\star} = -\nabla f(\bx^{\star})$, and we solve it by running $\texttt{GAS}(B^{\star}, -\nabla f(\bx^{\star}); E^{\star}, \alpha^{\star}, \beta^{\star}, \gamma^{\star}, \tau)$ in Algorithm \ref{alg:generalized}, with the initial values $\bz_0, \boldsymbol{v}_0 \in \mR^d$ be set deterministically and arbitrarily \footnote{It's worth mentioning that we are solving a different linear system compared those in GASN method, so the initialization we choose in the proof of Proposition \ref{prop:comp_stats_tradeoff} is a little bit different.}. In this case, for $j = 0, 1, ..., \tau-1$, the co-state pair $(\bz_j, \bv_j)$ has the following equivalent update:
\begin{equation}\label{equ:150}  
    \begin{cases}
        (E^{\star})^{1/2}  \by_j & = \alpha^{\star} (E^{\star})^{1/2}  \boldsymbol{v}_j + (1 - \alpha^{\star}) (E^{\star})^{1/2}  \bz_j; \\
        (E^{\star})^{1/2} \bz_{j+1} &= (E^{\star})^{1/2} \by_j - \tilde{Z}_j^{\star} \left((E^{\star})^{1/2}  \by_j - (E^{\star})^{1/2} \Delta \bx^{\star} \right); \\
        (E^{\star})^{1/2}  \boldsymbol{v}_{j+1} &= \beta^{\star} (E^{\star})^{1/2}  \boldsymbol{v}_j + (1 - \beta^{\star}) (E^{\star})^{1/2}  \by_j - \gamma^{\star} \tilde{Z}_j^{\star} \left((E^{\star})^{1/2}  \by_j - (E^{\star})^{1/2}  \Delta \bx^{\star} \right),
    \end{cases}
\end{equation}
where we recall that $\tilde{Z}_j^{\star} = (E^{\star})^{-1/2}B^{\star} S_j (S_j^{\top}B^{\star} (E^{\star})^{-1} B^{\star}S_j)^{\dagger}S_j^{\top}B^{\star} (E^{\star})^{-1/2}$. Direct calculation of \eqref{equ:150} gives, for all $j = 0, 1, ..., \tau-1$,
    \begin{multline}\label{equ:124}  
        \underbrace{\begin{pmatrix}
            (E^{\star})^{1/2} ( \bz_{j+1} - \Delta \bx^{\star}) \\
            (E^{\star})^{1/2} (\boldsymbol{v}_{j+1} - \Delta \bx^{\star})
        \end{pmatrix}}_{\boldsymbol{e}^{\star}_{j+1}} \\
        = \underbrace{\begin{pmatrix}
        (1 - \alpha^{\star})(I - \tilde{Z}_j^{\star}) & \alpha^{\star} (I - \tilde{Z}_j^{\star}) \\
        (1 - \alpha^{\star})(1 - \beta^{\star})I - (1 - \alpha^{\star}) \gamma^{\star} \tilde{Z}_j^{\star} & (\alpha^{\star} + \beta^{\star} - \alpha^{\star} \beta^{\star})I - \alpha^{\star} \gamma^{\star} \tilde{Z}_j^{\star}
    \end{pmatrix}}_{\tilde{M}^{\star}_j}  \underbrace{\begin{pmatrix}
        (E^{\star})^{1/2} (\bz_{j} - \Delta \bx^{\star}) \\
            (E^{\star})^{1/2} ( \boldsymbol{v}_{j} - \Delta \bx^{\star})
        \end{pmatrix}}_{\boldsymbol{e}^{\star}_{j}}.
    \end{multline}
    Here, we treat the transformed variable $(E^{\star})^{1/2} \Delta \bx^{\star}$ as the solution to the same linear system $\left( B^{\star} (E^{\star})^{-1/2} \right) \boldsymbol{x}  = -\nabla f(\bx^{\star})$, and the transformed variables $(E^{\star})^{1/2}  \bz_{j}$, $(E^{\star})^{1/2} \boldsymbol{v}_{j}$ are the new co-state pairs to solve the linear system. Since $B^\star\succ0$ and $E^\star\succ0$, the matrix $B^\star(E^\star)^{-1/2}$ is nonsingular. According to the formulation of \eqref{equ:150}, we apply the contraction result \citep[Eq. (41)]{Gower2018Accelerated} \footnote{In applying \cite{Gower2018Accelerated}, we identify the linear operator and right-hand side in their formulation with $A = B^{\star}(E^{\star})^{-1/2}$ and $b = -\nabla f(\bx^{\star})$. Although \cite{Gower2018Accelerated} initialize with
$\bz_0 = \bv_0$, this equality is not needed in our setting. Indeed, since
$B^{\star}\succ 0$ and $E^{\star}\succ 0$, the matrix $A=B^{\star}(E^{\star})^{-1/2}$ is nonsingular, and hence $
\operatorname{Range}(A^{\top})=\mathbb R^d$.
Moreover, $Z^{\star}\succ 0$, so $(Z^{\star})^\dagger=(Z^{\star})^{-1}$. Therefore, the range-invariance condition (cf. Lemma 6 in \cite{Gower2018Accelerated}) used in their proof is automatically satisfied for arbitrary deterministic initial values
$\bz_0, \bv_0 \in\mathbb R^d$. Consequently, their contraction argument applies verbatim and yields the same contraction inequality.} to the linear system $\left( B^{\star} (E^{\star})^{-1/2} \right) \boldsymbol{x}  = -\nabla f(\bx^{\star})$, and get that, for all $j = 0, 1, ..., \tau-1$,
    \begin{multline}\label{equ:125}
        \mE \left[ \left\|(E^{\star})^{1/2} \left( \boldsymbol{v}_{j+1} - \Delta \bx^{\star} \right) \right\|^2_{(Z^{\star})^{-1}} + \frac{1}{\mu^{\star}} \left\| (E^{\star})^{1/2} \left(\bz_{j+1} - \Delta \bx^{\star} \right) \right\|^2 \mid  \by_{j},  \boldsymbol{v}_{j}, \bz_{j} \right] \\
        \leq (1 - \sqrt{\mu^{\star}/ \nu^{\star}}) \cdot \left( \left\|(E^{\star})^{1/2} \left( \boldsymbol{v}_{j} - \Delta \bx^{\star} \right) \right\|^2_{(Z^{\star})^{-1}} + \frac{1}{\mu^{\star}} \left\| (E^{\star})^{1/2} \left(  \bz_{j} - \Delta \bx^{\star} \right) \right\|^2 \right),
    \end{multline}  
    which implies, for all $j = 0, 1, ..., \tau-1$,
    \begin{equation}\label{equ:126}
        \mE \left[ (\boldsymbol{e}^{\star}_{j+1})^{\top} \underbrace{\begin{pmatrix}
            \frac{1}{\mu^{\star}} I  & \boldsymbol{0} \\
            \boldsymbol{0} & (Z^{\star})^{-1}
        \end{pmatrix}}_{P^{\star}} \boldsymbol{e}^{\star}_{j+1} \mid  \by_{j},   \boldsymbol{v}_{j},  \bz_{j} \right] \leq (1 - \sqrt{\mu^{\star}/ \nu^{\star}}) (\boldsymbol{e}^{\star}_{j})^{\top} \underbrace{\begin{pmatrix}
             \frac{1}{\mu^{\star}} I  & \boldsymbol{0} \\
            \boldsymbol{0} & (Z^{\star})^{-1}
        \end{pmatrix}}_{P^{\star}} \boldsymbol{e}^{\star}_{j}.
    \end{equation}
    According to \eqref{equ:124}, we have for all $j = 0, 1, ..., \tau-1$,
    \begin{equation}\label{equ:127}
        \mE \left[ (\boldsymbol{e}^{\star}_{j})^{\top} {(\tilde{M}^{\star}_j)}^{\top} P^{\star} \tilde{M}^{\star}_j \boldsymbol{e}^{\star}_{j} \mid  \by_{j},  \boldsymbol{v}_{j},  \bz_{j} \right] \leq (1 - \sqrt{\mu^{\star}/ \nu^{\star}}) (\boldsymbol{e}^{\star}_{j})^{\top} P^{\star} \boldsymbol{e}^{\star}_{j}.
    \end{equation}
    Using \eqref{equ:127} recursively, and using the tower property, we obtain
    \begin{align}\label{equ:128}
        \mE[(\boldsymbol{e}^{\star}_{0})^{\top} (\tilde{C}^{\star})^{\top} P^{\star} \tilde{C}^{\star} \boldsymbol{e}^{\star}_{0}] & \stackrel{\mathclap{\eqref{equ:77}}}{=} \mE \left[(\boldsymbol{e}^{\star}_{0})^{\top} \left( \prod_{j = 0}^{\tau - 1} \tilde{M}^{\star}_j \right)^{\top} P^{\star} \left( \prod_{j = 0}^{\tau - 1} \tilde{M}^{\star}_j \right) \boldsymbol{e}^{\star}_{0} \right] \ \stackrel{\mathclap{\eqref{equ:124}}}{=} \ \mE \left[ (\boldsymbol{e}^{\star}_{\tau})^{\top}  P^{\star} \boldsymbol{e}^{\star}_{\tau} \right] \nonumber \\
        & \stackrel{\mathclap{\eqref{equ:124}}}{=} \ \  \mE\left[\mE\left[ (\boldsymbol{e}^{\star}_{\tau-1})^{\top}  {(\tilde{M}^{\star}_{\tau - 1})}^{\top} P^{\star} \tilde{M}^{\star}_{\tau - 1}  \boldsymbol{e}^{\star}_{\tau-1} \mid  \by_{\tau-1},  \boldsymbol{v}_{\tau-1},  \bz_{\tau-1} \right]\right] \nonumber \\
        & \stackrel{\mathclap{\eqref{equ:127}}}{\leq} (1 - \sqrt{\mu^{\star}/ \nu^{\star}}) \mE \left[(\boldsymbol{e}^{\star}_{\tau-1})^{\top} P^{\star} \boldsymbol{e}^{\star}_{\tau-1}\right] \nonumber \\
        & \stackrel{\mathclap{\eqref{equ:124}}}{=} \ \ (1 - \sqrt{\mu^{\star}/ \nu^{\star}}) \mE \left[ \mE \left[ (\boldsymbol{e}^{\star}_{\tau-2})^{\top}  {(\tilde{M}^{\star}_{\tau - 2})}^{\top} P^{\star} \tilde{M}^{\star}_{\tau - 2}  \boldsymbol{e}^{\star}_{\tau-2} \mid  \by_{\tau-2},  \boldsymbol{v}_{\tau-2}, \bz_{\tau - 2} \right] \right] \nonumber \\
        & \stackrel{\mathclap{\eqref{equ:127}}}{\leq} (1 - \sqrt{\mu^{\star}/ \nu^{\star}})^2 \mE \left[(\boldsymbol{e}^{\star}_{\tau-2})^{\top} P^{\star} \boldsymbol{e}^{\star}_{\tau-2} \right] \nonumber \\
        &\leq \cdots \nonumber \leq (1 - \sqrt{\mu^{\star}/ \nu^{\star}})^{\tau} \cdot (\boldsymbol{e}^{\star}_{0})^{\top} P^{\star} \boldsymbol{e}^{\star}_{0}. \\
    \end{align}
    Since $\boldsymbol{e}_{0}^{\star} \in \mR^{2d}$ is deterministic and arbitrary, then \eqref{equ:128} gives
    \begin{equation}\label{equ:129}
        \mE[ (\tilde{C}^{\star})^{\top} P^{\star} \tilde{C}^{\star} ] \preceq (1 - \sqrt{\mu^{\star}/ \nu^{\star}})^{\tau} \cdot P^{\star} .
    \end{equation}
    According to the definition of $P^{\star}$ and Lemma \ref{lem:4}, we know that $I \preceq P^{\star} \preceq  (\Upsilon_E/\gamma_E\gamma_S) \cdot I$, hence we have
    \begin{equation}\label{equ:130}
        \mE[(\tilde{C}^{\star})^{\top} \tilde{C}^{\star}] \preceq \mE[(\tilde{C}^{\star})^{\top} P^{\star} \tilde{C}^{\star}] \preceq (\Upsilon_E/\gamma_E\gamma_S) \cdot (1 - \sqrt{\mu^{\star}/ \nu^{\star}})^{\tau} \cdot I.
    \end{equation}
    Since $\tilde{K}^{\star} = (E^{\star})^{-1/2} \begin{pmatrix}
        I & \boldsymbol{0}
    \end{pmatrix} \tilde{C} ^{\star} \begin{pmatrix}
        I \\
        I
    \end{pmatrix} (E^{\star})^{1/2}$, we have
    \begin{multline}\label{equ:131}
        \|\mE [(\tilde{K}^{\star})^{\top} \tilde{K}^{\star}]\| \leq 2 \cdot \|(E^{\star})^{1/2}\|^2 \cdot \left\|\mE \left[(\tilde{C}^{\star})^{\top} \begin{pmatrix}
            I  \\
            \boldsymbol{0} 
        \end{pmatrix} (E^{\star})^{-1} \begin{pmatrix}
        I & \boldsymbol{0}
    \end{pmatrix} \tilde{C}^{\star} \right] \right\| \\
    \leq (2 \Upsilon_E /\gamma_E)  \cdot \|\mE[(\tilde{C}^{\star})^{\top} \tilde{C}^{\star}]\| \stackrel{\eqref{equ:130}}{\leq} (2 \Upsilon_E^2 /\gamma_E^2 \gamma_S) \cdot (1 - \sqrt{\mu^{\star}/ \nu^{\star}})^{\tau},
    \end{multline}
    where the second last inequality uses Assumption \ref{ass:4} and  \ref{ass:5}. Then we obtain
    \begin{align}\label{equ:132}
        \|\mE[\tilde{K}^{\star} (\tilde{K}^{\star})^{\top}]\| &\leq \text{Trace}\left( \mE[\tilde{K}^{\star} (\tilde{K}^{\star})^{\top}]\right) = \text{Trace}\left( \mE[ (\tilde{K}^{\star})^{\top}\tilde{K}^{\star}]\right) \nonumber \\
        &\leq d \cdot \|\mE[(\tilde{K}^{\star})^{\top} \tilde{K}^{\star}]\| \nonumber \\
        &\stackrel{\mathclap{\eqref{equ:131}}}{\leq} (2d \Upsilon_E^2 /\gamma_E^2 \gamma_S) \cdot (1 - \sqrt{\mu^{\star}/ \nu^{\star}})^{\tau}.
    \end{align}
    Combining \eqref{equ:133} and \eqref{equ:132}, we finally obtain
    \begin{equation*}
    \frac{\|\Xi^{\star}_{\text{gasn, ave}} - \Omega^{\star}\|}{\|\Omega^{\star}\|} = O \left( (1 - \sqrt{\mu^{\star}/ \nu^{\star}})^{\tau} \right).
    \end{equation*}
    We then complete the proof of the proposition.

\subsection{Proof of Proposition \ref{thm:covariance_comparison}}

Recall the definition of $\tilde{C}^{\star}$ in \eqref{equ:77} and  $\tilde{C}^{\star} = \prod_{j = 0}^{\tau - 1} \tilde{M}^{\star}_j$. Under the case where $E_t = I$, we have $\tilde{Z}_j^{\star} = B^{\star} S_j (S_j^{\top} (B^{\star})^2 S_j)^{\dagger} S_j^{\top} B^{\star}$, and $Z^{\star} = \mE[B^{\star} S (S^{\top} (B^{\star})^2 S)^{\dagger} S^{\top} B^{\star}]$. We then use Kronecker product to rewrite the matrix $\tilde{M}^{\star}_j$:
\begin{align*}
    \tilde{M}^{\star}_j &\stackrel{\mathclap{\eqref{equ:77}}}{=} \begin{pmatrix}
        (1 - \alpha^{\star})(I - \tilde{Z}_j^{\star}) & \alpha^{\star} (I - \tilde{Z}_j^{\star}) \\
        (1 - \alpha^{\star})(1 - \beta^{\star})I - (1 - \alpha^{\star}) \gamma^{\star} \tilde{Z}_j^{\star} & (\alpha^{\star} + \beta^{\star} - \alpha^{\star} \beta^{\star})I - \alpha^{\star} \gamma^{\star} \tilde{Z}_j^{\star}
    \end{pmatrix} \\
    &= \begin{pmatrix}
        (1 - \alpha^{\star})(I - \tilde{Z}_j^{\star}) & \alpha^{\star} (I - \tilde{Z}_j^{\star}) \\
        (1 - \alpha^{\star})(1 - \beta^{\star} - \gamma^{\star})I + (1 - \alpha^{\star}) \gamma^{\star} (I - \tilde{Z}_j^{\star}) & (\alpha^{\star} + \beta^{\star} - \alpha^{\star} \beta^{\star} - \alpha^{\star}\gamma^{\star})I + \alpha^{\star} \gamma^{\star} (I - \tilde{Z}_j^{\star}) 
    \end{pmatrix} \\
    &= \underbrace{\begin{pmatrix}
        0 & 0 \\
        (1 - \alpha^{\star})(1 - \beta^{\star} - \gamma^{\star}) & \alpha^{\star}+\beta^{\star}-\alpha^{\star}\beta^{\star}-\alpha^{\star}\gamma^{\star}
        \end{pmatrix}}_{\coloneqq  \Theta} \otimes I + \underbrace{\begin{pmatrix}
            1 - \alpha^{\star} & \alpha^{\star} \\
            (1-\alpha^{\star}) \gamma^{\star} & \alpha^{\star}\gamma^{\star}
        \end{pmatrix}}_{ \coloneqq  \Gamma} \otimes (I - \tilde{Z}_j^{\star}) ,
\end{align*}
where $``\otimes"$ denotes the Kronecker product of matrices. When $\mu_t \nu_t = \gamma_t = 1$, considering that $\gamma_t \stackrel{a.s.}{\to} \gamma^{\star}$ by Lemma \ref{lem:7} (in the proof of Lemma \ref{lem:bound_nu_t_mu_t}), then we have $\gamma^{\star} = 1$.  Hence we have
\begin{equation}\label{equ:100} 
    \Theta \begin{pmatrix}
    1 \\
    1
\end{pmatrix} = \begin{pmatrix}
    0 \\
    1-\gamma^{\star}
\end{pmatrix} =  \begin{pmatrix}
    0 \\
    0
\end{pmatrix}, \qquad \Gamma \begin{pmatrix}
    1 \\
    1
\end{pmatrix} = \begin{pmatrix}
    1 \\
    \gamma^{\star}
\end{pmatrix} = \begin{pmatrix}
    1 \\
    1
\end{pmatrix}.
\end{equation}
We then prove the following claim: For any matrix $Y \in \mR^{d \times d}$, and any $j \geq 0$, we have
\begin{equation}\label{equ:101}
    \tilde{M}_j^{\star} \left( \begin{pmatrix}
        1 \\
        1
    \end{pmatrix} \otimes Y \right) =  \begin{pmatrix}
        1 \\
        1
    \end{pmatrix} \otimes \left((I - \tilde{Z}_j^{\star})Y \right).
\end{equation}
The proof of the claim proceeds as follows:
\begin{align*}
    \tilde{M}_j^{\star} \left( \begin{pmatrix}
        1 \\
        1
    \end{pmatrix} \otimes Y \right) &= \left(\Theta \otimes I + \Gamma \otimes (I - \tilde{Z}_j^{\star}) \right) \left( \begin{pmatrix}
        1 \\
        1
    \end{pmatrix} \otimes Y \right)\\
    &= (\Theta \otimes I) \left( \begin{pmatrix}
        1 \\
        1
    \end{pmatrix} \otimes Y \right) + \left( \Gamma \otimes (I - \tilde{Z}_j^{\star}) \right) \left( \begin{pmatrix}
        1 \\
        1
    \end{pmatrix} \otimes Y \right) \\
    &= \boldsymbol{0} + \begin{pmatrix}
        1 \\
        1
    \end{pmatrix} \otimes \rbr{(I - \tilde{Z}_j^{\star}) Y},
\end{align*}
where the third equality uses the property of Kronecker product that $(A \otimes B)(C \otimes D) = (AC) \otimes (BD)$ and \eqref{equ:100}. Hence we finish the proof of \eqref{equ:101}. Thus, we then obtain 
\begin{align}\label{equ:102}
    \tilde{C}^{\star} \begin{pmatrix}
        I \\
        I
    \end{pmatrix} &= \tilde{M}^{\star}_{\tau - 1} \tilde{M}^{\star}_{\tau-2} \cdots \tilde{M}^{\star}_0 \rbr{\begin{pmatrix}
        1 \\
        1
    \end{pmatrix} \otimes I}  \stackrel{\eqref{equ:101}}{=} \begin{pmatrix}
        1 \\
        1
    \end{pmatrix} \otimes \rbr{\prod_{j = 0}^{\tau - 1} (I - \tilde{Z}_j^{\star})}.
\end{align}
We then further obtain
\begin{align*}
    \begin{pmatrix}
        I & \boldsymbol{0}
    \end{pmatrix} \tilde{C}^{\star} \begin{pmatrix}
        I \\
        I
    \end{pmatrix} \ &\stackrel{\mathclap{\eqref{equ:102}}}{=}  \ \begin{pmatrix}
        I & \boldsymbol{0}
    \end{pmatrix} \rbr{\begin{pmatrix}
        1 \\
        1
    \end{pmatrix} \otimes \rbr{\prod_{j = 0}^{\tau - 1} (I - \tilde{Z}_j^{\star})}}  =  \begin{pmatrix}
        I & \boldsymbol{0}
    \end{pmatrix} \begin{pmatrix}
        \prod_{j = 0}^{\tau - 1} (I - \tilde{Z}_j^{\star}) \\
        \prod_{j = 0}^{\tau - 1} (I - \tilde{Z}_j^{\star})
    \end{pmatrix} = \prod_{j = 0}^{\tau - 1} (I - \tilde{Z}_j^{\star}).
\end{align*}
Hence $\tilde{K}^{\star}$ reduces to $\prod_{j = 0}^{\tau - 1} (I - \tilde{Z}_j^{\star})$, and $K^{\star}$ reduces to $(I - Z^{\star})^{\tau}$. Comparing with the definition of $\Xi^{\star}_{\text{usn, ave}}$ in \citet[Eq. (17)]{Du2025Online}, we complete the proof of the proposition.

\subsection{Proof of Proposition \ref{prop:1}}
Note that when $E_t = I$, the matrix $K^{\star}$ is symmetric.
Under $C_{\varphi} = 1$ and  $\varphi = 1$, the Lyapunov equation \eqref{equ:lyp}  simplifies to
\begin{equation*}
(0.5I - K^{\star}){\Xi}^\star_{\text{asn, last}}  + {\Xi}^\star_{\text{asn, last}} (0.5 I - K^{\star}) = \mathbb{E}[ (I - \widetilde{K}^{\star})\Omega^{\star}(I - \widetilde{K}^{\star})^{\top}].
\end{equation*}
Recall that ${\Xi}^\star_{\text{gasn, ave}} = (I - K^{\star})^{-1}\mathbb{E}[(I - \widetilde{K}^{\star})\Omega^{\star}(I-\widetilde{K}^{\star})^{\top}](I- K^{\star})^{-1}$, then we obtain
\begin{align*}
(0.5I - K^{\star}) & ({\Xi}^\star_{\text{asn, last}} - {\Xi}^\star_{\text{gasn, ave}})  + ({\Xi}^\star_{\text{asn, last}} - {\Xi}^\star_{\text{gasn, ave}})(0.5I - K^{\star}) \\
& = \mathbb{E}[ (I - \widetilde{K}^{\star})\Omega^{\star}(I - \widetilde{K}^{\star})^{\top}] - {\Xi}^\star_{\text{gasn, ave}} + K^{\star} {\Xi}^\star_{\text{gasn, ave}} + {\Xi}^\star_{\text{gasn, ave}} K^{\star} \\
& = \mathbb{E}[ (I - \widetilde{K}^{\star})\Omega^{\star}(I - \widetilde{K}^{\star})^{\top}] - (I - K^{\star}){\Xi}^\star_{\text{gasn, ave}} - {\Xi}^\star_{\text{gasn, ave}}(I - K^{\star}) + {\Xi}^\star_{\text{gasn, ave}} \\
& = (I - (I - K^{\star})^{-1}) \mathbb{E}[ (I - \widetilde{K}^{\star})\Omega^{\star}(I - \widetilde{K}^{\star})^{\top}] (I - (I - K^{\star})^{-1}) \\
& \succeq \0.
\end{align*}
By Lemma~\ref{lem:prop_sketch_solver}(d) with $E_t = I$, the condition $\tau (1 - \sqrt{\mu_t / \nu_t})^{\tau - 2} \leq \gamma_H / (4 \Upsilon_H)$ in Theorem \ref{thm:as_convergence}, and $K_t \stackrel{a.s.}{\to} K^\star$ (Lemma~\ref{lem:bound_nu_t_mu_t}), we have $\|K^\star\| \leq \gamma_H / (2 \Upsilon_H) < 1/2$. The Lyapunov theorem \citep[Theorem~4.6]{Khalil2002Nonlinear} then yields $\Xi^\star_{\text{asn, last}} \succeq \Xi^\star_{\text{gasn, ave}}$, which completes the proof.

\subsection{Proof of Theorem \ref{thm:FCLT}}

Following the same decomposition in \eqref{eq:polyakrecursion}, the partial sum process can be decomposed as 
\begin{align*}
\frac{1}{\sqrt{t}} \sum_{k = 1}^{\lfloor rt \rfloor} (\bx_k - \bx^{\star}) \ & \stackrel{\mathclap{\eqref{eq:polyakrecursion}}}{=} \ \frac{1}{\sqrt{t}} \sum_{i = 0}^{\lfloor rt \rfloor -1}(I - K^{\star})^{-1} \btheta_i + \frac{1}{\sqrt{t}}\sum_{i=0}^{\lfloor rt \rfloor -1}A_i^{\lfloor rt \rfloor-1} \btheta_i + \frac{1}{\sqrt{t}}\sum_{i=0}^{\lfloor rt \rfloor -1}R_i^{\lfloor rt \rfloor -1}   \bdelta_i \\
& \quad + \frac{1}{\sqrt{t} \varphi_0} R_{0}^{\lfloor rt \rfloor -1} [I - \varphi_0 (I - K^{\star})] (\bx_0 - \bx^{\star}) \\
& = \frac{1}{\sqrt{t}} \sum_{i = 0}^{\lfloor rt \rfloor -1}(I - K^{\star})^{-1} \btheta_i + \frac{1}{\sqrt{t}}\sum_{i=0}^{\lfloor rt \rfloor -1} A_i^{t-1} \btheta_i + \frac{1}{\sqrt{t}}\sum_{i=0}^{\lfloor rt \rfloor -1}\left(A_i^{\lfloor rt \rfloor-1} - A_i^{t-1}\right) \btheta_i  \\
& \quad + \frac{1}{\sqrt{t}}\sum_{i=0}^{\lfloor rt \rfloor -1}R_i^{\lfloor rt \rfloor -1}  \bdelta_i  + \frac{1}{\sqrt{t} \varphi_0} R_{0}^{\lfloor rt \rfloor -1} [I - \varphi_0 (I - K^{\star})] (\bx_0 - \bx^{\star}),
\end{align*}
where $r \in [0, 1]$, and $A_i^{\lfloor rt \rfloor}$ and $R_i^{\lfloor rt \rfloor}$ are defined in \eqref{eq:A_B_recursion_matrices}. In what follows, we analyze each~right-hand-side term above and establish the following claims: 
\begin{enumerate}[topsep=2pt,parsep=3pt,label=(\alph*):]
\setlength\itemsep{0.0em}
\item $\frac{1}{\sqrt{t}} \sum_{i = 0}^{\lfloor rt \rfloor - 1} (I - K^{\star})^{-1} \btheta_i \Longrightarrow  (\Xi^{\star}_{\text{gasn, ave}})^{1/2} W_{d}(r)$.
\item $\sup_{r \in [0, 1]} \| \frac{1}{\sqrt{t}}\sum_{i=0}^{\lfloor rt \rfloor - 1}A_i^{t - 1} \btheta_i \| = o_p(1)$.
\item $\sup_{r \in [0, 1]} \|\frac{1}{\sqrt{t}}\sum_{i=0}^{\lfloor rt \rfloor -1}(A_i^{\lfloor rt \rfloor-1} - A_i^{t-1}) \btheta_i\| = o_p(1)$.
\item $\sup_{r \in [0, 1]} \|\frac{1}{\sqrt{t}}\sum_{i=0}^{\lfloor rt \rfloor-1} R_i^{\lfloor rt \rfloor-1} \bdelta_i \| = o(1)$.
\item $\sup_{r \in [0, 1]}\|\frac{1}{\sqrt{t} \varphi_0} R_{0}^{\lfloor rt \rfloor-1} (I - \varphi_0 (I - K^{\star})) (\bx_0 - \bx^{\star})\| = o(1)$.	
\end{enumerate}

Combining these claims together, we obtain
\begin{equation*}
\frac{1}{\sqrt{t}} \sum_{k = 1}^{\lfloor rt \rfloor } (\bx_k - \bx^{\star}) \Longrightarrow (\Xi^{\star}_{\text{gasn, ave}})^{1/2} W_{d}(r).
\end{equation*}
This completes the proof of the theorem by noting that $\frac{1}{\sqrt{t}}\sum_{k=0}^{\lfloor rt \rfloor -1} (\bx_k - \bx^{\star}) = \frac{1}{\sqrt{t}} \sum_{k = 1}^{\lfloor rt \rfloor } (\bx_k - \bx^{\star}) - \frac{1}{\sqrt{t}} (\bx_{\lfloor rt \rfloor } - \bx^{\star}) + \frac{1}{\sqrt{t}} (\bx_0 - \bx^{\star})$ and the fact that $\bx_t \stackrel{a.s.}{\to} \bx^{\star}$. 

\vskip4pt

\noindent $\bullet$ {\textbf{Proof of (a).}} By Lemma \ref{lem:error_recursion} and Lemma \ref{lem:3}, we know that $\btheta_i$ is a martingale difference and the conditional~covariance of $\btheta_i$ converges almost surely to $\mE[(I-\tilde{K}^\star)\Omega^{\star}(I-\tilde{K}^\star)^{\top}]$. Let $\Gamma_i :=
    \mathbb E[\boldsymbol\theta_i\boldsymbol\theta_i^\top
    \mid \mathcal F_{i-1}]$, and $\Gamma^\star :=
    \mathbb E[(I-\widetilde K^\star)\Omega^\star
    (I-\widetilde K^\star)^\top]$. By Stolz–Ces\`{a}ro theorem (cf. Lemma \ref{aux:lem:stolz-cesaro}), we have
\begin{align*}
     \sup_{r\in[0,1]}  \left\| \frac1t\sum_{i=0}^{\lfloor rt\rfloor-1}
    \Gamma_i - r \Gamma^\star \right\| & \leq \sup_{r\in[0,1]}
    \left\| \frac1t\sum_{i=0}^{\lfloor rt\rfloor-1} \left( \Gamma_i - \Gamma^\star \right)\right\| + \sup_{r\in[0,1]} \left\| \frac{\lfloor rt\rfloor}{t}\Gamma^\star-r\Gamma^\star \right\| \\
    &= \sup_{r\in[0,1]}
    \left\| \frac1t\sum_{i=0}^{\lfloor rt\rfloor-1} \left( \Gamma_i - \Gamma^\star \right)\right\| + \sup_{r\in[0,1]} \left| \frac{\lfloor rt\rfloor}{t}-r\right| \|\Gamma^\star\|\\
    & \leq \frac1t\sum_{i=0}^{t-1}\|\Gamma_i-\Gamma^\star\| + \frac{\|\Gamma^\star\|}{t} \stackrel{a.s.}{\to} 0.
\end{align*}
Hence, we have the following uniform convergence:
\begin{equation*}
     \sup_{r\in[0,1]}
    \left\|
    \frac1t\sum_{i=0}^{\lfloor rt\rfloor-1} (I-K^\star)^{-1}
    \mathbb E[\boldsymbol\theta_i\boldsymbol\theta_i^\top
    \mid \mathcal F_{i-1}] (I-K^\star)^{-\top} - r \Xi^\star_{\text{gasn, ave}} \right\| \stackrel{a.s.}{\to} 0.
\end{equation*}
By Lemma \ref{lem:B1}, we have verified the Lindeberg condition. Then, we apply the martingale Functional Central Limit Theorem (FCLT)  \citep[Theorem 4.2]{Hall2014Martingale}. We have as $t \rightarrow \infty$
\begin{equation*}
\frac{1}{\sqrt{t}} \sum_{i = 0}^{\lfloor rt \rfloor - 1} (I - K^{\star})^{-1} \btheta_i \Longrightarrow  (\Xi^{\star}_{\text{gasn, ave}})^{1/2} W_{d}(r).
\end{equation*}
This completes the proof of \textbf{(a)}.

\vskip4pt
\noindent $\bullet$ \textbf{Proof of \textbf{(b)}.} Note that $\sum_{i=0}^{\lfloor rt \rfloor -1} A_i^{t-1} \btheta_i$ is a martingale indexed by $r \in [0, 1]$. By Doob's~inequality \citep[Theorem 2.2]{Hall2014Martingale}, we have
\begin{align*}
\mE \left[\sup_{r \in [0, 1]} \left\| \frac{1}{\sqrt{t}}\sum_{i=0}^{\lfloor rt \rfloor - 1}A_i^{t - 1} \btheta_i \right\|^2 \right] 
&\leq \frac{4}{t} \mE \left[ \left\|  \sum_{i=0}^{t-1}A_i^{t-1} \btheta_i \right\|^2 \right] = \frac{4}{t} \sum_{i = 0}^{t-1} \mE \left[ \left\|  A_i^{t-1} \btheta_i \right\|^2 \right] \\
&\leq \Upsilon_{AR} \cdot \left(\sup_{i \geq 0} \mE[\|\btheta_i\|^2] \right) \cdot \frac{4}{t} \sum_{i = 0}^{t-1} \|A_i^{t-1}\| \to 0,
\end{align*}
where the last inequality is due to Lemma \ref{lem:B2}, and the convergence holds due to \eqref{equ:104} in proving Theorem \ref{thm:asymptotic_normality}. Thus, we have \mbox{$\sup_{r \in [0, 1]} \| \frac{1}{\sqrt{t}}\sum_{i=0}^{\lfloor rt \rfloor - 1}A_i^{t - 1} \btheta_i \| = o_p(1)$}~and~\mbox{complete}~the proof~of~\textbf{(b)}.

\vskip4pt
\noindent $\bullet$ \textbf{Proof of (c).} For any fixed integer $t_0>0$ and a constant $\eta > 0$, we define a stopping time as 
\begin{equation}\label{def:stopping_time}
    T_{t_0, \eta} \coloneqq \inf\{t\geq t_0: \|\bx_t - \bx^\star\|> \eta \}.
\end{equation}
We let $\boldsymbol{1}_{\{\cdot\}}$ denote the indicator function, which takes the value 1 if the condition inside the brackets $\{\}$ is true, and 0 otherwise.
Hence, we know that, for all $i \geq t_0$,
\begin{align}\label{equ:105} 
    \mE\left[\|\btheta_i \mathbf{1}_{\{T_{t_0, \eta} > i\}}\|^{2+\epsilon}\right] &= \mE\left[\|\btheta_i\|^{2+\epsilon} \mathbf{1}_{\{T_{t_0, \eta} > i\}}\right] = \mE \left[\mE \left[\|\btheta_i\|^{2+\epsilon} \mid \mF_{i-1} \right] \mathbf{1}_{\{T_{t_0, \eta} > i\}} \right] \\
    \ & \stackrel{\mathclap{\eqref{eq: bounded_q_moment}}}{\leq} \ C_q^{(1)} \mE\left[\|\bx_i - \bx^{\star}\|^{2+\epsilon} \cdot \mathbf{1}_{\{T_{t_0, \eta} > i\}}\right] + C_q^{(2)} \\
    &\leq C_q^{(1)} \eta^{2+\epsilon} + C_q^{(2)},
\end{align}
where the second equality holds since $\mathbf{1}_{\{T_{t_0, \eta} > i\}}$ is $\mF_{i-1}$-measurable, and the last inequality holds due to the definition of the stopping time $T_{t_0, \eta}$ in \eqref{def:stopping_time}. This indicates that $\btheta_i \mathbf{1}_{\{T_{t_0, \eta} > i\}}$ is a martingale difference with bounded $(2+\epsilon)$-th moment for all $i \geq t_0$ (the constant $\epsilon$ is defined in Assumption \ref{ass:2}). We notice that for any integer $n \in \{t_0, t_0 + 1, ..., t-1\}$, 
\begin{align}\label{eq:17}
\sum_{i = t_0}^{n} &\left(A_i^{t-1} - A_i^n \right) \btheta_i \mathbf{1}_{\{T_{t_0, \eta} > i\}}  \ \stackrel{\mathclap{\eqref{eq:A_B_recursion_matrices}}}{=} \ \sum_{i = t_0}^{n} \left(R_i^{t-1} - R_i^n \right) \btheta_i \mathbf{1}_{\{T_{t_0, \eta} > i\}}  \\
& \stackrel{\eqref{eq:A_B_recursion_matrices}}{=} \sum_{i = t_0}^{n} \sum_{k = n+1}^{t-1} \prod_{j= i+1}^{k}(I - \varphi_j (I - K^{\star})) \varphi_i \btheta_i \mathbf{1}_{\{T_{t_0, \eta} > i\}}  \nonumber\\
&\; =\;  \sum_{k = n+1}^{t-1}  \sum_{i = t_0}^{n} \prod_{j= i+1}^{k}(I - \varphi_j (I - K^{\star})) \varphi_i \btheta_i \mathbf{1}_{\{T_{t_0, \eta} > i\}} \nonumber\\
& \; =\;  \frac{1}{\varphi_{n+1}} \cdot \varphi_{n+1} \sum_{k = n+1}^{t-1} \prod_{j= n+1}^{k}(I - \varphi_j (I - K^{\star})) \left[  \sum_{i = t_0}^n \prod_{j= i+1}^{n}(I - \varphi_j (I - K^{\star})) \varphi_i \btheta_i \mathbf{1}_{\{T_{t_0, \eta} > i\}}  \right] \nonumber\\
&\; \stackrel{\mathclap{\eqref{eq:A_B_recursion_matrices}}}{=}\; \frac{1}{\varphi_{n+1}} \cdot R_{n+1}^{t-1} (I - \varphi_{n+1} (I - K^{\star})) \left[  \sum_{i = t_0}^n \prod_{j= i+1}^{n}(I - \varphi_j (I - K^{\star})) \varphi_i \btheta_i \mathbf{1}_{\{T_{t_0, \eta} > i\}}  \right].
\end{align}
By Lemma \ref{lem:B2}, we know that $\|R_{n+1}^{t-1} (I - \varphi_{n+1} (I - K^{\star}))\| \leq \Upsilon_{AR} (1+C_{\varphi}(1+\|K^{\star}\|))$ for any $n, t \geq 0$. Thus, we obtain
\begin{multline}\label{eq:16}
\sup_{r \in [0, 1]} \left\|\frac{1}{\sqrt{t}}\sum_{i=t_0}^{\lfloor rt \rfloor -1} \left(A_i^{\lfloor rt \rfloor-1} - A_i^{t-1}\right) \btheta_i \mathbf{1}_{\{T_{t_0, \eta} > i\}}\right\| = \sup_{n \in [t_0, t-1]} \left\|\frac{1}{\sqrt{t}}\sum_{i=t_0}^{n}\left(A_i^{n} - A_i^{t-1}\right) \btheta_i \mathbf{1}_{\{T_{t_0, \eta} > i\}}\right\| \\
\stackrel{\mathclap{\eqref{eq:17}}}{\leq} \Upsilon_{AR} (1+C_{\varphi}(1+\|K^{\star}\|)) \cdot \sup_{n \in [t_0, t-1]} \left\|\frac{1}{\sqrt{t}} \frac{1}{\varphi_{n+1}} \sum_{i = t_0}^n \prod_{j= i+1}^{n}(I - \varphi_j (I - K^{\star})) \varphi_i \btheta_i \mathbf{1}_{\{T_{t_0, \eta} > i\}}\right\|.
\end{multline}
We need a technical lemma to proceed with the proof.

\begin{lemma}[Adapted Lemma 3.3 in \cite{Li2023Online}]  \label{lem:phi_2_error}
Let $\{\boldsymbol{\zeta}_n\}_{n \geq 0}$ be a martingale difference sequence. Given the matrix $-(I - K^\star)$ is Hurwitz (i.e. $\text{Re} \; \lambda_i(I - K^{\star}) > 0$ for all $i = 1, 2, ..., d$), and $t_0 > 0$ is a fixed integer, we define an auxiliary sequence $\{\boldsymbol{y}_n\}_{n \geq t_0}$ as follows: $\boldsymbol{y}_{t_0} = \boldsymbol{0}$, and for $n \geq t_0$,
\begin{equation*}
\boldsymbol{y}_{n+1} = \sum_{i = t_0}^{n} \left( \prod_{j = i+1}^{n}(I - \varphi_j (I - K^{\star}))\right) \varphi_i \boldsymbol{\zeta}_i.
\end{equation*}	
Let the step size $\varphi_i = C_{\varphi}/(i+1)^{\varphi}$ for some $C_{\varphi} > 0$ and $\varphi \in (0.5, 1)$. If $\sup_{i \geq t_0} \mE \|\boldsymbol{\zeta}_i\|^{2+\epsilon} < \infty$ for some $\epsilon > 0$, then we have that as $t \rightarrow \infty$,
\begin{equation*}
\sup_{n \in [t_0, t-1]} \left\|\frac{1}{\sqrt{t}} \frac{1}{\varphi_{n+1}} \sum_{i = t_0}^n \prod_{j= i+1}^{n}(I - \varphi_j (I - K^{\star})) \varphi_i \boldsymbol{\zeta}_i \right\| = o_p(1).
\end{equation*}	
\end{lemma}

We recall that $\btheta_i \mathbf{1}_{\{T_{t_0, \eta} > i\}}$ is a martingale difference with $(2+\epsilon)$-th bounded moment for all $i \geq t_0$ \eqref{equ:105}. Thus, we directly apply Lemma \ref{lem:phi_2_error} to \eqref{eq:16}, and get that 
\begin{equation}\label{equ:106}
    \sup_{r \in [0, 1]} \left\|\frac{1}{\sqrt{t}}\sum_{i=t_0}^{\lfloor rt \rfloor -1} \left(A_i^{\lfloor rt \rfloor-1} - A_i^{t-1}\right) \btheta_i \mathbf{1}_{\{T_{t_0, \eta} > i\}}\right\| = o_p(1).
\end{equation}
Therefore, we obtain that
\begin{align}\label{equ:107} 
    \sup_{r \in [0, 1]} & \left\|\frac{1}{\sqrt{t}}\sum_{i=0}^{\lfloor rt \rfloor -1}(A_i^{\lfloor rt \rfloor-1} - A_i^{t-1}) \btheta_i \mathbf{1}_{\{T_{t_0, \eta} > i\}}\right\| \nonumber \\
    & \leq \frac{1}{\sqrt{t}}\sum_{i=0}^{t_0}2\Upsilon_{AR} \left\|\btheta_i \right\| + \sup_{r \in [0, 1]} \left\|\frac{1}{\sqrt{t}}\sum_{i=t_0}^{\lfloor rt \rfloor -1} \left(A_i^{\lfloor rt \rfloor-1} - A_i^{t-1}\right) \btheta_i \mathbf{1}_{\{T_{t_0, \eta} > i\}}\right\|,
\end{align}
where the inequality uses Lemma \ref{lem:B2}. For the first term in right-hand side of \eqref{equ:107}, we get
\begin{equation}\label{equ:108}
    \mE\left( \frac{1}{\sqrt{t}}\sum_{i=0}^{t_0} 2\Upsilon_{AR} \left\|\btheta_i \right\| \right) = \frac{2\Upsilon_{AR}}{\sqrt{t}} \sum_{i=0}^{t_0} \mE[\|\btheta_i \|] \leq \frac{2\Upsilon_{AR}}{\sqrt{t}} \sum_{i=0}^{t_0} \sqrt{\mE[\|\btheta_i \|^2]} \lesssim \frac{1}{\sqrt{t}} \rightarrow 0,
\end{equation}
where the last inequality holds due to $t_0$ is a fixed integer and $\sup_{i \geq 0} \mE[\|\btheta_i\|^2]$ is uniformly bounded according to \eqref{equ:104}. Therefore, we combine \eqref{equ:106} and \eqref{equ:108} with \eqref{equ:107}, and obtain, for any fixed $t_0 > 0$,
\begin{equation}\label{equ:111}
    \sup_{r \in [0, 1]} \left\|\frac{1}{\sqrt{t}}\sum_{i=0}^{\lfloor rt \rfloor -1}(A_i^{\lfloor rt \rfloor-1} - A_i^{t-1}) \btheta_i \mathbf{1}_{\{T_{t_0, \eta} > i\}}\right\| = o_p(1).
\end{equation}

We now show that the indicator $\mathbf{1}_{\{T_{t_0, \eta} > i\}}$ is negligible. First of all, since $\bx_t \stackrel{a.s.}{\rightarrow} \bx^{\star}$, we know that $\mathbf{1}_{\{T_{t_0, \eta} = \infty\}} \stackrel{a.s.}{\rightarrow} 1$ as $t_0 \rightarrow \infty$. This indicates that for any constant $K > 0$, 
\begin{equation}\label{equ:113}
    \lim_{t_0 \to \infty} \mathbb{P}( |\mathbf{1}_{\{T_{t_0, \eta} = \infty\}} - 1| < K ) = 1.
\end{equation}
Since $\mathbb{P}( |\mathbf{1}_{\{T_{t_0, \eta} = \infty\}} - 1| < K ) = \mathbb{P}( T_{t_0, \eta} = \infty)$ for any $0 < K < 1$, combining with \eqref{equ:113} we get $\lim_{t_0 \to \infty} \mathbb{P}( T_{t_0, \eta} = \infty ) = 1$. This suggests that $\lim_{t_0 \to \infty} \mathbb{P}( T_{t_0, \eta} < \infty ) = 0$, which means that for any $\delta' > 0$, there exists $t_0 = t_0(\delta')$, such that 
\begin{equation}\label{equ:111_1}
    \mathbb{P}(T_{t_0, \eta} < \infty) \leq \delta'/2.
\end{equation}

Secondly, we define 
\begin{equation}\label{equ:109}
    \mL_t \coloneqq \sup_{r \in [0, 1]} \left\|\frac{1}{\sqrt{t}}\sum_{i=0}^{\lfloor rt \rfloor -1}(A_i^{\lfloor rt \rfloor-1} - A_i^{t-1}) \btheta_i \right\|; \  \mL_{t, t_0} \coloneqq \sup_{r \in [0, 1]} \left\|\frac{1}{\sqrt{t}}\sum_{i=0}^{\lfloor rt \rfloor -1}(A_i^{\lfloor rt \rfloor-1} - A_i^{t-1}) \btheta_i \mathbf{1}_{\{T_{t_0, \eta} > i\}}\right\|.
\end{equation}
Then we know that for any $t \geq 0$ and $\delta > 0$,
\begin{align*}
    \mathbb{P} (T_{t_0, \eta} = \infty) \stackrel{\eqref{equ:109}}{\leq} \mathbb{P}\left(\left|\mL_{t, t_0} - \mL_t\right| = 0 \right) 
    &\leq \mathbb{P}\left(\left|\mL_{t, t_0} - \mL_t\right| < \frac{\delta}{2} \right).
\end{align*}
Hence we have, for any $t \geq 0$ and $\delta > 0$,
\begin{equation}\label{equ:110}
    \mathbb{P}\left(\left|\mL_{t, t_0} - \mL_t\right| \geq \frac{\delta}{2} \right) < \mathbb{P} (T_{t_0, \eta} < \infty).
\end{equation}

Thirdly, according to the fact that $\mL_{t, t_0} = o_p(1)$ in \eqref{equ:111}, we know that for any $\delta' > 0$, there exists $N = N(t_0, \delta, \delta')$, such that for all $t \geq N$, we have
\begin{equation}\label{equ:111_2}
    \mathbb{P}\left(|\mL_{t, t_0}| \geq \frac{\delta}{2}\right) \leq \frac{\delta'}{2}.
\end{equation}

In the end, for any $\delta, \delta' > 0$, we choose $N = N(t_0(\delta'), \delta, \delta')$, such that for all $t \geq N$, we have
\begin{align*}
    \mathbb{P}(|\mL_t| \geq \delta) &\leq \mathbb{P}\left(\left|\mL_{t, t_0} - \mL_t\right| \geq \frac{\delta}{2} \right) + \mathbb{P}\left(|\mL_{t, t_0}| \geq \frac{\delta}{2}\right) \\
    &\stackrel{\mathclap{\eqref{equ:110}}}{\leq} \ \mathbb{P} (T_{t_0, \eta} < \infty) + \mathbb{P}\left(|\mL_{t, t_0}| \geq \frac{\delta}{2}\right) \\
    &\stackrel{\mathclap{\eqref{equ:111_1}, \eqref{equ:111_2}}}{\leq} \ \; \quad \frac{\delta'}{2} + \frac{\delta'}{2} = \delta',
\end{align*}
which means that $\mL_t = \sup_{r \in [0, 1]} \left\|\frac{1}{\sqrt{t}}\sum_{i=0}^{\lfloor rt \rfloor -1}(A_i^{\lfloor rt \rfloor-1} - A_i^{t-1}) \btheta_i \right\| = o_p(1)$. We complete the proof of \textbf{(c)}.

\vskip4pt
\noindent $\bullet$ \textbf{Proof of (d).} By Lemma \ref{lem:B2}, we know that $\|R_i^{t}\| \leq \Upsilon_{AR}$ for any $i, t \geq 0$, then as $t \to \infty$,
\begin{equation*}
\sup_{r \in [0, 1]} \left\|\frac{1}{\sqrt{t}} \sum_{i=0}^{\lfloor rt \rfloor - 1} R_i^{\lfloor rt \rfloor - 1}  \bdelta_i \right\| \leq \frac{\Upsilon_{AR}}{\sqrt{t}} \sum_{i=0}^{t-1} \|\bdelta_i\| \stackrel{a.s.}{\longrightarrow} 0,
\end{equation*}
where the convergence holds due to \eqref{equ:114}. This completes the proof of \textbf{(d)}.

\vskip4pt
\noindent $\bullet$ \textbf{Proof of (e).} By Lemma \ref{lem:B2}, we know that $\|R_i^{t}\| \leq \Upsilon_{AR}$ for any $i, t \geq 0$, then as $t \to \infty$,
\begin{equation*}
\sup_{r \in [0, 1]}\left\|\frac{1}{\sqrt{t} \varphi_0} R_{0}^{\lfloor rt \rfloor-1} (I - \varphi_0 (I - K^{\star})) (\bx_0 - \bx^{\star})\right\| \leq  \frac{\Upsilon_{AR}}{\sqrt{t} \varphi_0} \left\|(I - \varphi_0 (I - K^{\star})) (\bx_0 - \bx^{\star}) \right\| \rightarrow 0.
\end{equation*}
This completes the proof of \textbf{(e)}.

\subsection{Proof of Theorem \ref{thm:Random_Scaling}}

By Proposition \ref{prop:comp_stats_tradeoff} (b), we have $\Xi^{\star}_{\text{gasn,ave}} \succeq \Omega^{\star}$. Furthermore, by the definition of $\Omega^{\star}$ in \eqref{nsequ:sgd_normal} and the assumption in Theorem \ref{thm:Random_Scaling}, we have $\Omega^{\star} \succ \boldsymbol{0}$. Hence $\Xi^{\star}_{\text{gasn,ave}} \succeq \Omega^{\star} \succ \boldsymbol{0}$.

    We now consider the random function $C_t(r) \coloneqq \frac{1}{\sqrt{t}} \sum_{i = 0}^{\lfloor rt \rfloor - 1} \boldsymbol{w}^{\top} (\bx_i - \bx^{\star})$. Based on Theorem \ref{thm:FCLT}, the one-dimensional random function $C_t(r)$ satisfies:
\begin{equation*} 
    C_t(r) \Longrightarrow \left(\boldsymbol{w}^{\top} \Xi^{\star}_{\text{gasn,ave}} \boldsymbol{w}\right)^{1/2}W_1(r), \qquad r \in [0, 1],
\end{equation*}
where $W_{1}(\cdot)$ stands for the standard one-dimensional Brownian motion. Recall that $\bar {\bx}_t = \frac{1}{t} \sum_{i = 0}^{t-1} \bx_i$. Then, we have
\begin{align*}
    \frac{\sqrt{t} \boldsymbol{w}^{\top} (\bar{\bx}_t - \bx^{\star})}{\sqrt{\boldsymbol{w}^{\top} V_t \boldsymbol{w}}} &= \frac{C_t(1)}{\sqrt{\frac{1}{t} \sum_{i = 1}^{t} \frac{i^2}{t} \left( \boldsymbol{w}^{\top} \left(\bar{\bx}_i - \bar{\bx}_t \right)\right)^2 }} \\
    &= \frac{C_t(1)}{\sqrt{\frac{1}{t} \sum_{i = 1}^{t} \left( \frac{i}{\sqrt{t}} \boldsymbol{w}^{\top} \left[ \frac{1}{i} \sum_{k = 0}^{i-1} (\bx_k - \bx^{\star}) - \frac{1}{t} \sum_{k = 0}^{t-1} (\bx_k - \bx^{\star}) \right] \right)^2}} \\
    &= \frac{C_t(1)}{\sqrt{\frac{1}{t} \sum_{i = 1}^{t} \left( \boldsymbol{w}^{\top} \left[ \frac{1}{\sqrt{t}} \sum_{k = 0}^{i-1} (\bx_k - \bx^{\star}) - \frac{i}{t} \cdot \frac{1}{\sqrt{t}} \sum_{k = 0}^{t-1} (\bx_k - \bx^{\star}) \right] \right)^2}} \\
    &= \frac{C_t(1)}{\sqrt{\frac{1}{t} \sum_{i = 1}^{t} \left( C_t(\frac{i}{t}) - \frac{i}{t} C_t(1) \right)^2}}.
\end{align*}
Since $\frac{C_t(1)}{\sqrt{\int_0^1 (C_t(r) - rC_t(1))^2 \text{d}r}}$ is a continuous functional of $C_t(\cdot)$, together with the extended continuous mapping theorem \cite[Theorem 1.11.1]{Vaart1996Weak} and the fact that $\int_0^1 (W_1(r)-rW_1(1))^2\,dr>0$ almost surely, we have
\begin{align*}
\frac{C_t(1)}{\sqrt{\frac{1}{t}\sum_{i = 1}^{t}\left( C_t\left(\frac{i}{t}\right) - \frac{i}{t}C_t(1)\right)^2}} &  \stackrel{d}{\longrightarrow}   \frac{(\boldsymbol{w}^{\top} \Xi^{\star}_{\text{gasn,ave}} \boldsymbol{w})^{1/2} \cdot W_1(1)}{(\boldsymbol{w}^{\top} \Xi^{\star}_{\text{gasn,ave}} \boldsymbol{w})^{1/2} \cdot \sqrt{\int_{0}^1  \left(W_1(r) - rW_1(1) \right)^2 dr}} \\
& = \frac{W_1(1)}{\sqrt{\int_{0}^1  \left(W_1(r) - rW_1(1) \right)^2 dr}},
\end{align*}
where the convergence also uses the Riemann integral approximation, and the last equality follows since $\Xi^{\star}_{\text{gasn,ave}}  \succ \boldsymbol{0}$ and $\boldsymbol{w} \neq \boldsymbol{0}$. Therefore, we obtain
\begin{equation*}
\frac{\sqrt{t} \boldsymbol{w}^{\top} (\bar{\bx}_t - \bx^{\star})}{\sqrt{\boldsymbol{w}^{\top} V_t \boldsymbol{w}}}  \stackrel{d}{\longrightarrow}  \frac{W_1(1)}{\sqrt{\int_{0}^1  \left(W_1(r) - rW_1(1) \right)^2 dr}},
\end{equation*}
which completes the proof.

\section{Proofs of Technical Lemmas}

\subsection{Proof of Lemma \ref{lem:prop_sketch_solver}}

\noindent $\bullet$ {\textbf{Proof of (a).}} First, we derive the recursion between $(\bz_{t, j+1} - \Delta \bx_t, \bv_{t, j+1} - \Delta \bx_t)$ and $(\bz_{t, j} - \Delta \bx_t, \bv_{t, j} - \Delta \bx_t)$ for any outer loop iteration $t$. By the \texttt{GAS} solver in Algorithm \ref{alg:generalized}, and noticing that $B_t \Delta \bx_t = -\bg_t$, then we have
\begin{align*}
    E_t^{1/2} (\bz_{t, j+1} - \Delta \bx_t) & = E_t^{1/2} (\by_{t, j} - \Delta \bx_t - \boldsymbol{\omega}_{t, j})\\
    & = E_t^{1/2} (\alpha_t (\bv_{t, j} - \Delta\bx_{t}) + (1 - \alpha_t)(\bz_{t, j} - \Delta\bx_{t}) \\
    &  \qquad \qquad - E_t^{-1} B_t S_{t, j}(S_{t, j}^{\top}B_t E_t^{-1} B_t S_{t, j})^{\dagger}S_{t, j}^{\top} B_t ( \by_{t, j} - \Delta \bx_t) ) \\
    &= E_t^{1/2} (I - E_t^{-1} B_t S_{t, j}(S_{t, j}^{\top}B_t E_t^{-1} B_t S_{t, j})^{\dagger}S_{t, j}^{\top} B_t) ( \alpha_t ( \boldsymbol{v}_{t, j} - \Delta\bx_{t}) \\
    & \qquad \qquad + (1 - \alpha_t)(\bz_{t, j} - \Delta\bx_{t}) ) \\
    &= (I - E_t^{-1/2}B_t S_{t, j}(S_{t, j}^{\top}B_t E_t^{-1} B_t S_{t, j})^{\dagger}S_{t, j}^{\top} B_t E_t^{-1/2}) [\alpha_t E_t^{1/2}(\bv_{t, j} - \Delta\bx_{t}) \\
    & \qquad \qquad + (1 - \alpha_t) E_t^{1/2} (\bz_{t, j} - \Delta\bx_{t})]
\end{align*}
Recall that $Z_{t, j} = E_t^{-1/2} B_t S_{t,j} (S_{t,j}^\top B_t E_t^{-1} B_t S_{t,j})^\dagger S_{t,j}^\top B_t E_t^{-1/2}$, hence we have
\begin{align}\label{equ:117} 
    \bz_{t, j+1} - \Delta \bx_t &= E_t^{-1/2} \left(I - Z_{t, j} \right) E_t^{1/2} [\alpha_t (\boldsymbol{v}_{t, j} - \Delta\bx_{t}) + (1 - \alpha_t) (\bz_{t, j} - \Delta\bx_{t})] \nonumber \\
    &= \left(I - E_t^{-1/2} Z_{t, j} E_t^{1/2} \right) [\alpha_t (\boldsymbol{v}_{t, j} - \Delta\bx_{t}) + (1 - \alpha_t) (\bz_{t, j} - \Delta\bx_{t})].
\end{align}
On the other hand, we have
\begin{align*}
    E_t^{1/2} &(\boldsymbol{v}_{t, j+1}  - \Delta \bx_t) = \beta_t E_t^{1/2} ( \boldsymbol{v}_{t, j} - \Delta \bx_t) + (1-\beta_t) E_t^{1/2}( \boldsymbol{y}_{t, j} - \Delta \bx_{t}) \\
    & \qquad \qquad \qquad \qquad - \gamma_t E_t^{-1/2} B_t S_{t, j}(S_{t, j}^{\top}B_tE_t^{-1}B_tS_{t, j})^{\dagger}S_{t, j}^{\top} B_t ( \by_{t, j} - \Delta \bx_t) \\
    &= \beta_t E_t^{1/2} (\boldsymbol{v}_{t, j} - \Delta \bx_t)  + \left((1 - \beta_t) I - \gamma_t Z_{t, j} \right) E_t^{1/2}( \by_{t, j} - \Delta \bx_t) \\
    &= \beta_t E_t^{1/2} ( \boldsymbol{v}_{t, j} - \Delta \bx_t) + \left((1 - \beta_t) I - \gamma_t Z_{t, j} \right) \cdot \left(\alpha_t E_t^{1/2}(\boldsymbol{v}_{t, j} - \Delta \bx_t) + E_t^{1/2}(1 - \alpha_t) (\bz_{t, j} - \Delta \bx_t) \right) \\
    &= \left[(\alpha_t + \beta_t - \alpha_t \beta_t)I - \alpha_t \gamma_t Z_{t, j} \right] E_t^{1/2}( \boldsymbol{v}_{t, j} - \Delta \bx_t) \\
    & \qquad \qquad \qquad \qquad + \left[(1-\alpha_t)(1-\beta_t)I - (1-\alpha_t)\gamma_t Z_{t, j} \right] E_t^{1/2}(\bz_{t, j} - \Delta \bx_t).
\end{align*}
Hence we have
\begin{multline}\label{equ:119}
    (\boldsymbol{v}_{t, j+1} - \Delta \bx_t) = \left[(\alpha_t + \beta_t - \alpha_t \beta_t)I - \alpha_t \gamma_t E_t^{-1/2} Z_{t, j} E_t^{1/2} \right] 
    ( \boldsymbol{v}_{t, j} - \Delta \bx_t) \\
    + \left[(1-\alpha_t)(1-\beta_t)I - (1-\alpha_t)\gamma_t E_t^{-1/2} Z_{t, j} E_t^{1/2} \right] ( \bz_{t, j} - \Delta \bx_t).
\end{multline}
Combining \eqref{equ:117} and \eqref{equ:119}, and recalling the definition of $C_{t, j}$ in  \eqref{def:tilde_C_t}, we have the following recursion formula:
\begin{align}\label{equ:118} 
    &\begin{pmatrix}
        \bz_{t, j+1} - \Delta \bx_t \\
        \boldsymbol{v}_{t, j+1} - \Delta \bx_t
    \end{pmatrix} \nonumber \\
    &= \begin{pmatrix}
        (1 - \alpha_t)(I - E_t^{-1/2}Z_{t, j}E_t^{1/2}) &  \alpha_t(I - E_t^{-1/2}Z_{t, j}E_t^{1/2})\\
        (1 - \alpha_t)(1 - \beta_t)I - (1 - \alpha_t)\gamma_t E_t^{-1/2}Z_{t, j}E_t^{1/2} & (\alpha_t+\beta_t - \alpha_t\beta_t)I - \alpha_t \gamma_t E_t^{-1/2}Z_{t, j}E_t^{1/2}
    \end{pmatrix} \begin{pmatrix}
        \bz_{t, j} - \Delta \bx_t \\
        \boldsymbol{v}_{t, j} - \Delta \bx_t
    \end{pmatrix} \nonumber \\
    &= \begin{pmatrix}
        E_t^{-1/2} & 0\\
        0 & E_t^{-1/2}
    \end{pmatrix}  C_{t, j} \begin{pmatrix}
        E_t^{1/2} & 0\\
        0 & E_t^{1/2}
    \end{pmatrix} \begin{pmatrix}
     \bz_{t, j} - \Delta \bx_t \\
        \boldsymbol{v}_{t, j} - \Delta \bx_t
    \end{pmatrix}.
\end{align}
Applying the above formula \eqref{equ:118} recursively, we have the closed form of $(\bz_{t, \tau} - \Delta \bx_{t}, \boldsymbol{v}_{t, \tau} - \Delta \bx_{t})$:
\begin{align*}
    & \begin{pmatrix}
        \bz_{t, \tau} - \Delta \bx_{t} \\
        \boldsymbol{v}_{t, \tau} - \Delta \bx_{t}
    \end{pmatrix}  = - \begin{pmatrix}
        E_t^{-1/2} & 0\\
        0 & E_t^{-1/2}
    \end{pmatrix} \tilde{C}_t \begin{pmatrix}
        E_t^{1/2} & 0\\
        0 & E_t^{1/2}
    \end{pmatrix}\begin{pmatrix}
        \Delta \bx_{t} \\
        \Delta \bx_{t} 
    \end{pmatrix},
\end{align*}
where the equality holds also uses $\bz_{t, 0} = \boldsymbol{v}_{t, 0} = \boldsymbol{0}$.
Hence we have the closed form for the sketched solver for a fixed $\tau > 0$:
\begin{equation}\label{equ:122} 
    \begin{pmatrix}
        \bz_{t, \tau}  \\
        \boldsymbol{v}_{t, \tau}
    \end{pmatrix} = \left(I -  \begin{pmatrix}
        E_t^{-1/2} & 0\\
        0 & E_t^{-1/2}
    \end{pmatrix} \tilde{C}_t  \begin{pmatrix}
        E_t^{1/2} & 0\\
        0 & E_t^{1/2}
    \end{pmatrix} \right) \begin{pmatrix}
        \Delta \bx_{t} \\
        \Delta \bx_{t} 
    \end{pmatrix}.
\end{equation}
Recall the definition \eqref{def:tilde_H_t} that $\tilde{K}_t = E_t^{-1/2}([\tilde{C}_t]_{1, 1} + [\tilde{C}_t]_{1, 2})E_t^{1/2}$. Then from \eqref{equ:122}, we get
\begin{equation}\label{equ:120}  
     \bz_{t, \tau} = (I - \tilde{K}_t) \Delta \bx_t = - (I - \tilde{K}_t) B_t^{-1} \bg_t,
\end{equation}
where the last equality uses the fact that $B_t\Delta\bx_t = -\bg_t$. Hence we complete the proof of (a).

\vskip4pt

\noindent $\bullet$ {\textbf{Proof of (b).}} From \eqref{equ:120}, we obtain that
\begin{multline}\label{equ:134}  
    \mE[ \bz_{t, \tau} \mid \mF_{t-1}] = \mE[\mE[- (I - \tilde{K}_t) B_t^{-1} \bg_t| \mF_{t - 0.5}]| \mF_{t-1}]\\
    = -(I - K_t)B_t^{-1} \mE[\bg_t| \mF_{t-1}] = -(I - K_t)B_t^{-1} \nabla f_t,
\end{multline}
where the first equality uses the tower property of conditional expectation; the second equality uses the fact that $\mE[\tilde{K}_t | \mF_{t - 0.5}] = K_t$, and $\{S_{t, j}\}_j$ and $\xi_t$ are independent; and the third equality uses $\mE[\bg_t | \mF_{t-1} ] = \nabla f_t$ in Assumption \ref{ass:2}. Therefore, we have
\begin{equation*}
    \mE[\bz_{t, \tau} - (I - K_t)\Delta \bx_t | \mF_{t-1}] \stackrel{\eqref{equ:134}}{=} -(I - K_t)B_t^{-1} \nabla f_t + (I - K_t) B_t^{-1}  \mE[\bg_t | \mF_{t-1} ] = 0,
\end{equation*}
where the first equality also uses $B_t\Delta\bx_t = -\bg_t$ in Algorithm \ref{alg:generalized}, and the second equality uses $\mE[\bg_t | \mF_{t-1} ] = \nabla f_t$ in Assumption \ref{ass:2}. We complete the proof of (b).

\vskip4pt

\noindent $\bullet$ {\textbf{Proof of (c).}} By Lemma \ref{lem:4} and the definition of $\mu_t$ in \eqref{def:mu_t_nu_t}, we obtain $\mu_t \geq \frac{\gamma_E \gamma_S}{\Upsilon_E}$ for all $t \geq 0$, and we also know from \eqref{prop_mu_nu} that $\nu_t \geq 1$. Therefore, we obtain that for all $t \geq 0$,
\begin{equation}\label{equ:135}
    \gamma_t \stackrel{\eqref{def:alpha_beta_gamma_t}}{=} \sqrt{\frac{1}{\mu_t \nu_t}} \leq \sqrt{\frac{\Upsilon_E}{\gamma_E \gamma_S}}.
\end{equation}
Moreover, by the definition of $\tilde{K}_t$, we obtain that
\begin{align*}
    \|\tilde{K}_t\| &\stackrel{\mathclap{\eqref{def:tilde_H_t}}}{\leq} \sqrt{2} \cdot \sqrt{\frac{\Upsilon_E}{\gamma_E}} \cdot \|\tilde{C}_t\|\\
    &\stackrel{\mathclap{\eqref{def:tilde_C_t}}}{\leq} \sqrt{\frac{2\Upsilon_E}{\gamma_E}} \cdot \left((1 - \alpha_t) + \alpha_t + (1-\alpha_t)(1-\beta_t) + (1-\alpha_t)\gamma_t + \alpha_t+\beta_t - \alpha_t\beta_t + \alpha_t \gamma_t \right)^{\tau} \\
    &= \sqrt{\frac{2\Upsilon_E}{\gamma_E}} \cdot (1 + 1 + \gamma_t)^{\tau} \\
    &\stackrel{\mathclap{\eqref{equ:135}}}{\leq} \sqrt{\frac{2\Upsilon_E}{\gamma_E}} \cdot \left( 2 + \sqrt{\frac{\Upsilon_E}{\gamma_E \gamma_S}} \right)^{\tau},
\end{align*}
where the first inequality also uses $\gamma_E I \preceq E_t \preceq \Upsilon_E I$ in Assumption \ref{ass:4}. Hence, we complete the proof of (c).

\vskip4pt

\noindent $\bullet$ \textbf{Proof of (d).} The matrices $\tilde{C}_t$ and $C_t$ defined in \eqref{def:tilde_C_t} and \eqref{def:C_t} share the same structural form, in terms of $(\alpha_t, \beta_t, \gamma_t)$, $Z_{t,j}$, and $Z_t$, as their counterparts defined in \citet[Section 3.2]{Wang2026Inference}. The only differences lie in the definitions of these underlying quantities (i.e., $\alpha_t, \beta_t, \gamma_t, Z_{t,j}, Z_t$), which depend on $E_t$ in our setting. Nevertheless, the relevant properties used in the proof are preserved: $Z_{t,j}$ remains a projection matrix, $Z_t$ remains lower bounded by a constant (cf.\ Lemma~\ref{lem:4}), and $(\alpha_t, \beta_t, \gamma_t)$ satisfies \eqref{prop_mu_nu}. Therefore, we directly follow the same spirit as \citet[Lemma 3.7]{Wang2026Inference}, and obtain that
\begin{equation}\label{equ:136}  
    \|[C_t]_{1,1} + [C_t]_{1,2}\| \leq 2\tau \, (1 - \sqrt{\mu_t / \nu_t})^{\tau - 2}.
\end{equation}
By the definition of $K_t$ in \eqref{def:H_t} and the submultiplicativity of the spectral norm,
\begin{equation*}
    \|K_t\| \stackrel{\eqref{def:H_t}}{\leq} \|E_t^{-1/2}\| \cdot \|[C_t]_{1,1} + [C_t]_{1,2}\| \cdot \|E_t^{1/2}\| \stackrel{\eqref{equ:136}}{\leq} \sqrt{\frac{\Upsilon_E}{\gamma_E}} \cdot 2\tau \, (1 - \sqrt{\mu_t / \nu_t})^{\tau - 2},
\end{equation*}
where the final inequality uses $\gamma_E I \preceq E_t \preceq \Upsilon_E I$ from Assumption~\ref{ass:4}. This completes the proof of (d).

\subsection{Proof of Lemma \ref{lem:converge_rate_and_Bt_converge}}

We first establish the almost sure convergence of $B_t$. By the definition of $B_t$ given immediately before \eqref{equ:update}, and recall that $B^{\star} = \nabla^2 f^{\star}$ , we obtain
\begin{align}\label{equ:42}
    \|B_t - B^{\star}\| &= \left\|\frac{1}{t}\sum_{i = 0}^{t-1}(H_i - \nabla^2 f_i) + \frac{1}{t}\sum_{i = 0}^{t-1} (\nabla^2 f_i - \nabla^2f^{\star})\right\| \nonumber \\
    &\leq \left\|\frac{1}{t}\sum_{i = 0}^{t-1}(H_i - \nabla^2 f_i)\right\| + \left\| \frac{1}{t}\sum_{i = 0}^{t-1} (\nabla^2 f_i - \nabla^2f^{\star})\right\|  \nonumber\\
    &\leq \left\|\frac{1}{t}\sum_{i = 0}^{t-1}(H_i - \nabla^2 f_i)\right\| + \frac{\Upsilon_L}{t} \sum_{i = 0}^{t-1} \|\bx_i - \bx^{\star}\|,
\end{align}
where the last inequality holds due to Assumption \ref{ass:1}. The second term in the right-hand side of \eqref{equ:42} converges to $0$ almost surely due to the fact that $\bx_t \stackrel{a.s.}{\to} \bx^{\star}$ and Stolz--Ces\`aro theorem (cf. Lemma \ref{aux:lem:stolz-cesaro}). For the first term in the right-hand side of \eqref{equ:42}, we know that $H_i - \nabla^2 f_i$ is a martingale difference. Moreover, by Assumption \ref{ass:2}, we have
\begin{equation}\label{equ:54}
    \mE[\|H_i - \nabla^2f_i\|^2] \leq C_{H, 1}\mE[\|\bx_i - \bx^{\star}\|^2] + C_{H, 2} \stackrel{\eqref{equ:53}}{\leq} C_{H, 3},
\end{equation}
for some constant $C_{H, 3} > 0$. Therefore, we obtain that, for all $t$,
\begin{equation*}
    \mE\left[ \left\|\sum_{i = 0}^{t-1} \frac{1}{i+1}(H_i - \nabla^2 f_i)\right\|_F^2 \right] = \sum_{i = 0}^{t-1} \mE\left[ \frac{1}{(i+1)^2} \|H_i - \nabla^2 f_i\|_F^2 \right] \stackrel{\eqref{equ:54}}{\leq} \sum_{i = 0}^{\infty} \frac{d \cdot C_{H, 3}}{(i+1)^2} < \infty,
\end{equation*}
hence the martingale $\sum_{i = 0}^{t-1} \frac{1}{i+1}(H_i - \nabla^2 f_i)$ is $L^2$ bounded. Then by Doob's Martingale convergence theorem \citep[Corollary 2.2]{Hall2014Martingale}, $\sum_{i = 0}^{t-1} \frac{1}{i+1}(H_i - \nabla^2 f_i)$ converges almost surely. Then by Kronecker's lemma, we have $\frac{1}{t}\sum_{i = 0}^{t-1}(H_i - \nabla^2 f_i) \stackrel{a.s.}{\to} 0$. Therefore, we get $B_t \stackrel{a.s.}{\to} B^{\star}$.

For the $L^2$-convergence,  by \eqref{equ:42}, we further obtain that
\begin{align}\label{equ:81}  
    \mE[\|B_t - B^\star\|^2] &\leq 2\mE\left\|\frac{1}{t}\sum_{i = 0}^{t-1}(H_i - \nabla^2 f_i)\right\|^2 + 2 \mE \left( \frac{\Upsilon_L}{t} \sum_{i = 0}^{t-1} \|\bx_i - \bx^{\star}\| \right)^2 \nonumber \\
    &\leq 2\mE\left\|\frac{1}{t}\sum_{i = 0}^{t-1}(H_i - \nabla^2 f_i)\right\|^2 + \frac{2\Upsilon_L^2}{t^2} \mE \left(  \sum_{i = 0}^{t-1} \|\bx_i - \bx^{\star}\| \right)^2 \nonumber \\
    &\leq 2\mE\left\|\frac{1}{t}\sum_{i = 0}^{t-1}(H_i - \nabla^2 f_i)\right\|^2 + \frac{2\Upsilon_L^2}{t^2} \left(  \sum_{i = 0}^{t-1} \left(\mE\|\bx_i - \bx^{\star}\|^2 \right)^{1/2} \right)^2,
\end{align}
where the last inequality uses Minkowski's inequality. For the first term in right-hand side of \eqref{equ:81}, we have
\begin{align}\label{equ:82}
     \mathbb{E}\left\|\frac{1}{t}\sum_{i = 0}^{t-1} (H_i - \nabla^2 f_i) \right\|^2 &\leq \mathbb{E}\left\|\frac{1}{t}\sum_{i = 0}^{t-1} (H_i - \nabla^2 f_i) \right\|^2_{F} = \frac{1}{t^2} \text{Tr}\left( \mathbb{E}\left( \sum_{i = 0}^{t-1}(H_i - \nabla^2 f_i) \right)^2 \right) \nonumber\\
     &= \frac{1}{t^2} \text{Tr} \left( \sum_{i = 0}^{t-1} \mathbb{E} (H_i - \nabla^2 f_i)^2 \right)  \qquad ( H_i - \nabla^2 f_i \ \text{is a martingale difference.}) \nonumber \nonumber\\
     &= \frac{1}{t^2}   \sum_{i = 0}^{t-1} \text{Tr} \left(\mathbb{E} (H_i - \nabla^2 f_i)^2 \right) = \frac{1}{t^2} \sum_{i = 0}^{t-1} \mE[\|H_i - \nabla^2 f_i\|^2_F] \nonumber\\
     & \leq \frac{d}{t^2} \sum_{i = 0}^{t-1} \mE[\|H_i - \nabla^2 f_i\|^2_2] = \frac{d}{t^2} \sum_{i = 0}^{t-1} \mE[\mE[\|H_i - \nabla^2 f_i\|^2_2 \mid \mF_{i-1}]] \nonumber\\
     &\stackrel{\mathclap{\eqref{ass:3:4th}}}{\leq} \frac{d}{t^2} \sum_{i = 0}^{t-1} \mE[C_{H, 1}\|\bx_i - \bx^{\star}\|^2 + C_{H, 2}] \nonumber\\
     &\lesssim \ \frac{d}{t},
\end{align}
where the last inequality holds due to Theorem \ref{thm:as_convergence}. For the second term right-hand side of \eqref{equ:81}, we have
\begin{multline}\label{equ:83}
    \frac{1}{t^2} \left(  \sum_{i = 0}^{t-1} \left(\mE\|\bx_i - \bx^{\star}\|^2 \right)^{1/2} \right)^2 \lesssim \frac{1}{t^2} \cdot \left( \sum_{i = 0}^{t-1} \sqrt{\varphi_i} \right)^2 = \frac{C_{\varphi}}{t^2} \rbr{\sum_{i=0}^{t-1}\frac{1}{(i+1)^{0.5\varphi}}}^2 \\
     \leq \frac{C_{\varphi}}{t^2} \rbr{1 + \int_{0}^{t}\frac{1}{(i+1)^{0.5\varphi}}\;di }^2 \lesssim \frac{C_{\varphi}}{(t+1)^\varphi} = \varphi_t,
\end{multline}
where the first inequality also uses Theorem \ref{thm:as_convergence}. Combining \eqref{equ:81}, \eqref{equ:82}, and \eqref{equ:83}, we obtain that $\mE[\|B_t - B^{\star}\|^2] \lesssim \varphi_t$. Hence, we complete the proof of the lemma.

\subsection{Proof of Lemma \ref{lem:bound_nu_t_mu_t}}

We now present the following lemma regarding the Lipschitz property of the two projection matrices 
\begin{equation}\label{nsequ:11}
\begin{aligned}
\tilde{Z}_t & = E_t^{-1/2}B_t S (S^{\top}B_t E_t^{-1} B_tS)^{\dagger}S^{\top}B_t E_t^{-1/2},\\
\tilde{Z}^{\star} & = (E^{\star})^{-1/2}B^{\star} S (S^{\top}B^{\star} (E^{\star})^{-1} B^{\star} S)^{\dagger}S^{\top}B^{\star} (E^{\star})^{-1/2}.
\end{aligned}    
\end{equation}

\begin{lemma}[Adapted Lemma 5.2 in \cite{Na2025Statistical}]\label{lem:continuity_projection_matrix}
Suppose $B_t, E_t, B^{\star}, E^{\star} \in \mR^{d \times d}$ are positive definite, and satisfy $\gamma_E I \preceq E_t, E^{\star} \preceq \Upsilon_E I$, $\gamma_H I \preceq B_t, B^{\star} \preceq \Upsilon_H I$. Consider $\tilde{Z}_t$ and $\tilde{Z}^{\star}$ defined as in \eqref{nsequ:11}, for any $S \in \mR^{d \times s}$, we have
\begin{equation*}
\left\|\tilde{Z}_t-\tilde{Z}^{\star}\right\| \leq \frac{2 \sqrt{\Upsilon_E}}{\gamma_H} \left( \frac{1}{\sqrt{\gamma_E}} \|B_t - B^{\star}\| + \frac{\Upsilon_H}{2\gamma_E^{3/2}}\|E_t - E^{\star}\|\right) \cdot \|S\|\|S^{\dagger}\|.
\end{equation*}            
\end{lemma}

By Lemma~\ref{lem:continuity_projection_matrix}, the discrepancy between the projection matrices can be bounded as
\begin{equation} \label{equ:B_bound_projection}
\mathbb{E} [\|\tilde Z_t - \tilde Z^{\star}\|\mid \mathcal F_{t-1}] \lesssim \left( \|B_t - B^{\star}\| + \|E_t - E^{\star}\| \right) \mathbb{E}\!\left[\|S\|\,\|S^{\dagger}\|\right]\;\lesssim\; \|B_t - B^{\star}\| + \|E_t - E^{\star}\|.
\end{equation}
Here, the first inequality follows from Lemma \ref{lem:continuity_projection_matrix} together with the independence between $S$ and $\mathcal F_{t-1}$ implied by Assumption \ref{ass:4}, while the second one is a direct consequence of Assumption \ref{ass:4}.

For $M_t$ defined in \eqref{def:C_t}, we find the following uniform upper bound:
\begin{align}\label{equ:142}  
\|M_t\| & \; \stackrel{\mathclap{\eqref{def:C_t}}}{\leq} \; \|I - Z_{t}\| + (1-\alpha_t)(1-\beta_t) + (1-\alpha_t)\gamma_t \|Z_{t}\| + (\alpha_t + \beta_t - \alpha_t\beta_t) + \alpha_t \gamma_t \|Z_{t}\| \nonumber\\
& \; \leq\; \|I - Z_{t}\| + 1 + \gamma_t \|Z_{t}\| \leq  1 + 1 + \gamma_t \leq 2+ \sqrt{\frac{\Upsilon_E}{\gamma_E \gamma_S}},
\end{align}
where the last inequality uses Lemma \ref{lem:prop_sketch_solver}(c). Therefore, we find the uniform upper bound for $C_t$:
\begin{equation}\label{equ:143} 
    \|C_t\| \stackrel{\eqref{def:C_t}}{\leq} \|M_t\|^{\tau} \stackrel{\eqref{equ:142}}{\leq} \left( 2+ \sqrt{\frac{\Upsilon_E}{\gamma_E \gamma_S}} \right)^{\tau}.
\end{equation}

We now need the following lemma to proceed with the proof.

\begin{lemma}\label{lem:7}  
    Under the conditions of Lemma \ref{lem:bound_nu_t_mu_t}, we have $\max \{|\alpha_t - \alpha^{\star}|, |\beta_t-\beta^\star|, |\gamma_t-\gamma^\star| \} \lesssim \mathbb{E} [\|\tilde Z_t - \tilde Z^{\star}\|\mid \mathcal F_{t-1}]$. Furthermore, we have $(\alpha_t,\beta_t,\gamma_t) \stackrel{a.s.}{\rightarrow} (\alpha^{\star}, \beta^{\star}, \gamma^{\star})$.
\end{lemma}

By the almost-sure convergence $(\alpha_t,\beta_t,\gamma_t)$ (Lemma \ref{lem:7}) and the fact that $\|Z^{\star}\| \leq 1$, the same bound holds for $\|M^{\star}\|$ (see the definition of $M^{\star}$ in \eqref{def:C_star}). Therefore, using the fact that $C^{\star} = (M^{\star})^{\tau}$, we have
\begin{equation}\label{equ:145}    
    \max\{\|C_t\|, \|C^{\star}\|\} \leq \left( 2+ \sqrt{\frac{\Upsilon_E}{\gamma_E \gamma_S}} \right)^{\tau}.
\end{equation}

By the definition of $K_t$ and $K^{\star}$ in \eqref{def:H_t} and \eqref{def:H_star}, respectively, we have
\begin{align}\label{equ:146}
    \|K_t - K^{\star}\| \lesssim \|E_t - E^{\star}\| + \|C_t - C^{\star}\|,
\end{align}
where the inequality uses the fact that $\max \{\|E_t^{1/2} - (E^{\star})^{1/2}\|,  \|E_t^{-1/2} - (E^{\star})^{-1/2}\|\} \lesssim \|E_t - E^{\star}\|$, and the uniformly boundedness of $\|E_t\|$ and $\|E^{\star}\|$ (Assumption \ref{ass:5}), as well as $\|C_t\|$ and $\|C^{\star}\|$ (cf. \eqref{equ:145}). For the second term in the right-hand side of \eqref{equ:146}, we have
\begin{align}\label{equ:147} 
    \|C_t - C^{\star}\| &= \|(M_t)^{\tau} - (M^{\star})^{\tau}\| = \left\| \sum_{k = 0}^{\tau - 1} M_t^{\tau - 1 - k}(M_t - M^{\star}) (M^{\star})^k  \right\| \\
    & \leq \left(\sum_{k = 0}^{\tau - 1} \|M_t\|^{\tau -1 -k} \|M^{\star}\|^{k}\right) \|M_t - M^{\star}\| \stackrel{\eqref{equ:142}}{\lesssim} \|M_t - M^{\star}\|,
\end{align}
where the last inequality also uses the fact that $\|M^{\star}\|$ has the same bound as $\|M_t\|$. Now we bound $\|M_t - M^{\star}\|$ as follows:
\begin{align} \label{equ:M_bound}
\|M_t - M^{\star}\| &\leq \|(1 - \alpha_t)(I - Z_t) - (1 - \alpha^{\star})(I - Z^{\star})\| + \|\alpha_t(I - Z_t) - \alpha^{\star} (I - Z^{\star})\| \nonumber\\
& \quad + \|(1 - \alpha_t)(1 - \beta_t)I - (1 - \alpha_t)\gamma_t Z_t - (1 - \alpha^{\star})(1 - \beta^{\star})I + (1 - \alpha^{\star}) \gamma^{\star} Z^{\star}\| \nonumber\\
& \quad + \|(\alpha_t+\beta_t - \alpha_t\beta_t)I - \alpha_t \gamma_t Z_t - (\alpha^{\star}  + \beta^{\star} - \alpha^{\star} \beta^{\star})I + \alpha^{\star} \gamma^{\star} Z^{\star}\| \nonumber\\
& \lesssim |\alpha_t - \alpha^{\star}| + |\beta_t - \beta^{\star}| + |\gamma_t - \gamma^{\star}| + \|Z_t - Z^{\star}\| \nonumber\\
& \lesssim \; \mathbb{E} [\|\tilde Z_t - \tilde Z^{\star}\|\mid \mathcal F_{t-1}] \stackrel{\eqref{equ:B_bound_projection}}{\lesssim} \|B_t - B^{\star}\| + \|E_t - E^{\star}\|.
\end{align}
Here, the second inequality uses the bounds $\alpha_t < 1$, $\beta_t < 1$, $\gamma_t \leq \sqrt{\frac{\Upsilon_E}{\gamma_E \gamma_S}}$, and $\|Z_t\| \leq 1$, and the third inequality uses Lemma \ref{lem:7} and the fact that $\|Z_t - Z^{\star}\| = \|\mathbb{E}[\tilde{Z}_t - \tilde{Z}^{\star} \mid \mathcal{F}_{t-1}] \| \leq \mathbb{E}[\|\tilde{Z}_t - \tilde{Z}^{\star}\| \mid \mathcal{F}_{t-1}]$. Therefore, we apply \eqref{equ:147} and \eqref{equ:M_bound} into \eqref{equ:146}, and obtain that
\begin{equation*}
    \|K_t - K^{\star}\| \lesssim \|B_t - B^{\star}\| + \|E_t - E^{\star}\|.
\end{equation*}
Since $B_t \stackrel{a.s.}{\to} B^{\star}$ in Lemma \ref{lem:converge_rate_and_Bt_converge} and $E_t \stackrel{a.s.}{\to} E^{\star}$ in Assumption \ref{ass:5}, we further have $K_t \stackrel{a.s.}{\to} K^{\star}$. Hence, we complete the proof of the lemma.

\subsection{Proof of Lemma \ref{lem:error_recursion}}
       
 From Lemma \ref{lem:prop_sketch_solver}(b), we know that $ \bz_{t, \tau} - (I - K_t) \Delta \bx_t$ is a martingale difference. Recall that we define $B^{\star} = \nabla^2f^{\star}$, we now decompose the error $\bx_{t+1} - \bx^{\star}$ as follows:
\begin{align*}
    \bx_{t+1} - \bx^{\star} &= \bx_t - \bx^{\star} + \varphi_t  \bz_{t, \tau} \\
    &= \bx_t - \bx^{\star} + \varphi_t (I - K_t) \Delta \bx_t + \varphi_t (\bz_{t, \tau} - (I - K_t)\Delta \bx_t) \\
    &= \bx_t - \bx^{\star} - \varphi_t (I - K_t) B_t^{-1} \bg_t + \varphi_t ( \bz_{t, \tau} - (I - K_t)\Delta \bx_t) \\
    &= \bx_t - \bx^{\star} - \varphi_t (I - K_t) B_t^{-1} \nabla f_t - \varphi_t (I - K_t) B_t^{-1} (\bg_t - \nabla f_t ) + \varphi_t (\bz_{t, \tau} - (I - K_t)\Delta \bx_t) \\
    &=  \bx_t - \bx^{\star} - \varphi_t (I - K_t)(B^{\star})^{-1} \nabla f_t - \varphi_t (I - K_t) [B_t^{-1} - (B^{\star})^{-1}] \nabla f_t +\varphi_t\boldsymbol{\theta}_t \\
    &= \{I - \varphi_t(I - K_t)\}(\bx_t - \bx^{\star}) - \varphi_t (I - K_t)[(B^{\star})^{-1}(\nabla f_t - B^{\star}(\bx_t - \bx^{\star}))] \\
    & \quad \  - \varphi_t (I - K_t) [B_t^{-1} - (B^{\star})^{-1}] \nabla f_t +\varphi_t\boldsymbol{\theta}_t \\
    &= \{I - \varphi_t(I - K^{\star})\}(\bx_t - \bx^{\star}) + \varphi_t (K_t - K^{\star}) (\bx_t - \bx^{\star}) \\
    & \quad \  - \varphi_t (I - K_t)[(B^{\star})^{-1}(\nabla f_t - B^{\star}(\bx_t - \bx^{\star}))] - \varphi_t (I - K_t) [B_t^{-1} - (B^{\star})^{-1}] \nabla f_t +\varphi_t\boldsymbol{\theta}_t \\
    &= \{I - \varphi_t(I - K^{\star})\}(\bx_t - \bx^{\star}) + \varphi_t \boldsymbol{\delta}_t + \varphi_t \boldsymbol{\theta}_t,
\end{align*}
where $\boldsymbol{\theta}_t = -  (I - K_t) B_t^{-1} (\bg_t - \nabla f_t ) + (\bz_{t, \tau} - (I - K_t)\Delta \bx_t)$ is a martingale difference, and $\boldsymbol{\delta}_t = (K_t - K^{\star}) (\bx_t - \bx^{\star}) - (I - K_t)[(B^{\star})^{-1}(\nabla f_t - B^{\star}(\bx_t - \bx^{\star}))] - (I - K_t) [B_t^{-1} - (B^{\star})^{-1}] \nabla f_t$ denotes the approximation error. We complete the proof of the lemma.

\subsection{Proof of Lemma \ref{lem:3}}

Recall from \eqref{rec:def:b} that $\boldsymbol{\theta}_t = -  (I - K_t) B_t^{-1} (\boldsymbol{g}_t - \nabla f_t ) + (\bz_{t, \tau} - (I - K_t)\Delta \bx_t)$ is a martingale difference. We now investigate the limiting covariance of $\mE[\boldsymbol{\theta}_t \boldsymbol{\theta}_t^{\top} \mid \mF_{t-1}]$. We note that
\begin{align}\label{equ:66}  
    \mE[\boldsymbol{\theta}_t \boldsymbol{\theta}_t^{\top} \mid \mF_{t-1}] &=  \mE[   \{(I - K_t) B_t^{-1} (\bg_t - \nabla f_t ) - (\bz_{t, \tau} - (I - K_t)\Delta \bx_t)\} \nonumber\\
    & \qquad \{(I - K_t) B_t^{-1} (\bg_t - \nabla f_t ) - (\bz_{t, \tau} - (I - K_t)\Delta \bx_t)\}^{\top} \mid \mF_{t-1}] \nonumber\\
    &= (I - K_t)B_t^{-1}\mE[ (\bg_t - \nabla f_t) (\bg_t - \nabla f_t)^{\top}  \mid \mF_{t-1}]B_t^{-1}(I - K_t)^{\top} \nonumber\\
    & \qquad + \mE[ (\bz_{t, \tau} - (I - K_t)\Delta \bx_t)(\bz_{t, \tau} - (I - K_t)\Delta \bx_t)^{\top} \mid \mF_{t-1}] \nonumber\\
    &\coloneqq \mJ_{1, t} + \mJ_{2, t},
\end{align}
where the second equality holds due to the tower property and the fact that $\bz_{t, \tau} - (I - K_t)\Delta \bx_t$ is a martingale difference by Lemma \ref{lem:prop_sketch_solver}(b), i.e.
\begin{align*}
    \mE[(\bg_t - \nabla f_t )(\bz_{t, \tau} - (I - K_t)\Delta \bx_t)^{\top} & \mid \mF_{t-1}] =  \mE[\mE[(\bg_t - \nabla f_t )(\bz_{t, \tau} - (I - K_t)\Delta \bx_t)^{\top} \mid \mF_{t-0.5} ]\mid \mF_{t-1}] \\
    &= \mE[(\bg_t - \nabla f_t )\mE[(\bz_{t, \tau} - (I - K_t)\Delta \bx_t)^{\top} \mid \mF_{t-0.5} ]\mid \mF_{t-1}] = \boldsymbol{0}.
\end{align*}
\noindent $\bullet$ {\textbf{Convergence of the $\mJ_{1,t}$. }}For the term $\mJ_{1, t}$, we have $\mE[ (\bg_t - \nabla f_t) (\bg_t - \nabla f_t)^{\top}  \mid \mF_{t-1}] = \mE [\bg_t \bg^{\top}_t \mid \mF_{t-1}] - \nabla f_t \nabla f_t^{\top}$. We also note that
\begin{align}\label{equ:67}
    & \left\|\mathbb{E}[\bg_t \bg^{\top}_t \mid \mathcal{F}_{t-1}] - \mathbb{E}[\nabla F(\bx^{\star}; \xi) (\nabla F(\bx^{\star}; \xi))^{\top}]\right\| = \left\|\mathbb{E}[\bg_t \bg^{\top}_t - \nabla F(\bx^{\star}; \xi_t) (\nabla F(\bx^{\star}; \xi_t))^{\top} \mid \mathcal{F}_{t-1}] \right\| \nonumber\\
& \leq 2\mathbb{E}[\|\bg_t - \nabla F(\bx^{\star}; \xi_t)\|\cdot\|\bg_t\| \mid \mathcal{F}_{t-1}] + \mathbb{E}[\|\bg_t - \nabla F(\bx^{\star}; \xi_t)\|^2 \mid \mathcal{F}_{t-1}] \nonumber\\
& \leq 2\sqrt{\mathbb{E}[\|\bg_t - \nabla F(\bx^{\star}; \xi_t)\|^2 \mid \mathcal{F}_{t-1}]} \sqrt{\mathbb{E}[\|\bg_t\|^2 \mid \mathcal{F}_{t-1}]} + \mathbb{E}[\|\bg_t - \nabla F(\bx^{\star}; \xi_t)\|^2 \mid \mathcal{F}_{t-1}],
\end{align}
where the first inequality uses $\|aa^{\top} - bb^{\top}\| \leq 2\|a\|\|a-b\| + \|a-b\|^2$, and the last inequality follows from the Cauchy--Schwarz inequality. For the first term in \eqref{equ:67}, by the $\Upsilon_H$-Lipschitz continuity of $\nabla f(\bx)$ in Assumption \ref{ass:1} and Assumption \ref{ass:2} with $2+\epsilon > 2$, we have
\begin{equation}\label{equ:68}
    \mathbb{E}[\|\bg_t - \nabla F(\bx^{\star}; \xi_t)\|^2 \mid \mathcal{F}_{t-1}] \leq \mathbb{E}[\sup_{\bx}\|\nabla^2 F(\bx; \xi)\|^2] \cdot \|\bx_t - \bx^{\star}\|^2 \leq \Upsilon_H^2 \|\bx_t - \bx^{\star}\|^2,
\end{equation}
and
\begin{align}\label{equ:69}
    \mE[\|\bg_t\|^2 \mid \mF_{t-1}] &\leq 2 \|\nabla f_t\|^2 + 2 \mE [ \|\bg_t - \nabla f_t\|^2\mid \mF_{t-1}] \\
    &\stackrel{\mathclap{\eqref{ass:2:4th}}}{\leq} 2\Upsilon_H^2 \|\bx_t - \bx^{\star}\|^2 + 2 C_{g, 1}^{\frac{2}{2+\epsilon}} \|\bx_t - \bx^{\star}\|^2 + 2C_{g, 2}^{\frac{2}{2+\epsilon}}.
\end{align}
Plugging \eqref{equ:68} and \eqref{equ:69} into \eqref{equ:67} gives us
\begin{equation*}
\left\|\mathbb{E}[\bg_t \bg^{\top}_t \mid \mathcal{F}_{t-1}] - \mathbb{E}[\nabla F(\bx^{\star}; \xi_t) (\nabla F(\bx^{\star}; \xi_t))^{\top}]\right\| \stackrel{a.s.}{\rightarrow} 0,
\end{equation*}
which further implies
\begin{equation} \label{equ:70}
\mathbb{E}[(\bg_t - \nabla f_t) (\bg_t - \nabla f_t)^{\top}  \mid \mathcal{F}_{t-1}] \stackrel{a.s.}{\rightarrow} \mathbb{E}[\nabla F(\bx^{\star}; \xi) (\nabla F(\bx^{\star}; \xi))^{\top}].
\end{equation}
 Since $K_t \stackrel{a.s.}{\to} K^{\star}$ from Lemma \ref{lem:bound_nu_t_mu_t} and $B_t \stackrel{a.s.}{\to} B^{\star}$ from Lemma \ref{lem:converge_rate_and_Bt_converge} , we finally have
\begin{equation}\label{equ:71}
    \mJ_{1, t} \stackrel{a.s.}{\to} (I - K^{\star}) ( B^{\star})^{-1} \mE[\nabla F(\bx^{\star}; \xi) \nabla^{\top} F(\bx^{\star}; \xi)] ( B^{\star} )^{-1} (I - K^{\star})^{\top} \stackrel{\eqref{nsequ:sgd_normal}}{=} (I - K^{\star}) \Omega^{\star} (I - K^{\star})^{\top}.
\end{equation}

\noindent $\bullet$ {\textbf{Convergence of the $\mJ_{2,t}$.}}
For the term $\mJ_{2, t}$, we have
\begin{align} \label{equ:72}
\mJ_{2, t} &= \mathbb{E}[(\bz_{t, \tau} - (I - K_t)\Delta \bx_t)(\bz_{t, \tau} - (I - K_t)\Delta \bx_t)^{\top} \mid \mathcal{F}_{t-1}] \nonumber\\
&= \mathbb{E}[(\tilde{K}_t - K_t)\Delta \bx_t \Delta \bx_t^{\top} (\tilde{K}_t - K_t)^{\top} \mid \mathcal{F}_{t-1}] \nonumber\\
&= \mathbb{E}[(\tilde{K}_t - K_t) B_t^{-1} \bg_t \bg^{\top}_t B_t^{-1} (\tilde{K}_t - K_t)^{\top} \mid \mathcal{F}_{t-1}] \nonumber\\
&= \mathbb{E}[(\tilde{K}_t - K_t) B_t^{-1} (\bg_t - \nabla f_t) (\bg_t - \nabla f_t)^{\top} B_t^{-1} (\tilde{K}_t - K_t)^{\top} \mid \mathcal{F}_{t-1}] \nonumber\\
& \quad + \mathbb{E}[(\tilde{K}_t - K_t) B_t^{-1} \nabla f_t \nabla f_t^{\top} B_t^{-1} (\tilde{K}_t - K_t)^{\top} \mid \mathcal{F}_{t-1}] \nonumber\\
& \quad + \mathbb{E}[(\tilde{K}_t - K_t) B_t^{-1} (\bg_t - \nabla f_t) \nabla f_t^{\top} B_t^{-1} (\tilde{K}_t - K_t)^{\top} \mid \mathcal{F}_{t-1}] \nonumber\\
& \quad + \mathbb{E}[(\tilde{K}_t - K_t) B_t^{-1} \nabla f_t (\bg_t - \nabla f_t)^{\top} B_t^{-1} (\tilde{K}_t - K_t)^{\top} \mid \mathcal{F}_{t-1}].	
\end{align}
For the last two terms in \eqref{equ:72}, we apply the tower property. In particular, for the last term, we have
\begin{align} \label{equ:73}
& \mathbb{E}[(\tilde{K}_t - K_t) B_t^{-1} \nabla f_t (\bg_t - \nabla f_t)^{\top} B_t^{-1} (\tilde{K}_t - K_t)^{\top} \mid \mathcal{F}_{t-1}] \nonumber\\
&= \mathbb{E}[(\tilde{K}_t - K_t) B_t^{-1} \nabla f_t \mathbb{E}[(\bg_t - \nabla f_t)^{\top} \mid \mathcal{F}_{t-1} \cup \sigma(\{S_{t, j}\}_j)] B_t^{-1} (\tilde{K}_t - K_t)^{\top} \mid \mathcal{F}_{t-1}] \nonumber\\
&= \mathbb{E}[(\tilde{K}_t - K_t) B_t^{-1} \nabla f_t \mathbb{E}[(\bg_t - \nabla f_t)^{\top} \mid \mathcal{F}_{t-1}] B_t^{-1} (\tilde{K}_t - K_t)^{\top} \mid \mathcal{F}_{t-1}] = \boldsymbol{0},	
\end{align}
where the last equality holds since $\mE[\bg_t - \nabla f_t \mid \mathcal{F}_{t-1}] = 0$. Similarly, for the second last term in \eqref{equ:72}, we have
\begin{align} \label{equ:74}
& \mathbb{E}[(\tilde{K}_t - K_t) B_t^{-1} (\bg_t - \nabla f_t) \nabla f_t^{\top} B_t^{-1} (\tilde{K}_t - K_t)^{\top} \mid \mathcal{F}_{t-1}] \nonumber\\
&= \mathbb{E}[(\tilde{K}_t - K_t) B_t^{-1} \mathbb{E}[(\bg_t - \nabla f_t) \mid \mathcal{F}_{t-1} \cup \sigma(\{S_{t, j}\}_j)] \nabla f_t^{\top}  B_t^{-1} (\tilde{K}_t - K_t)^{\top} \mid \mathcal{F}_{t-1}] \nonumber\\
& = \mathbb{E}[(\tilde{K}_t - K_t) B_t^{-1} \mathbb{E}[(\bg_t - \nabla f_t) \mid \mathcal{F}_{t-1}] \nabla f_t^{\top} B_t^{-1} (\tilde{K}_t - K_t)^{\top} \mid \mathcal{F}_{t-1}] = \boldsymbol{0}.	
\end{align}
For the second term in \eqref{equ:72}, we have
\begin{align} \label{equ:75}
& \left\|\mathbb{E}[(\tilde{K}_t - K_t)B_t^{-1} \nabla f_t \nabla f_t^{\top} B_t^{-1}(\tilde{K}_t - K_t)^{\top} \mid \mathcal{F}_{t-1}]\right\|  \leq \mathbb{E}[\|(\tilde{K}_t - K_t) B_t^{-1} \nabla f_t \nabla f_t^{\top} B_t^{-1} (\tilde{K}_t - K_t)^{\top} \|\mid \mathcal{F}_{t-1}]  \nonumber\\
& \leq \mathbb{E}[\|(\tilde{K}_t - K_t)\|^2 \|B_t^{-1}\|^2 \|\nabla f_t\|^2 \mid \mathcal{F}_{t-1}] \leq (C_K + 0.5)^2 \cdot 1/\gamma_H^2 \cdot \mathbb{E}[\|\nabla f_t\|^2 \mid \mathcal{F}_{t-1}] \nonumber\\
& = (C_K + 0.5)^2 \cdot 1/\gamma_H^2 \cdot \|\nabla f_t\|^2 \stackrel{a.s.}{\rightarrow} 0.
\end{align}
Here, the third inequality follows from the uniform bounds
$\|\tilde K_t\|\le C_K$ in~Lemma \ref{lem:prop_sketch_solver},
$\|K_t\|<0.5$ by Lemma~\ref{lem:prop_sketch_solver} together with the condition on $\tau$ that $\tau \cdot (1 - \sqrt{\mu_t/\nu_t})^{\tau - 2} \leq \gamma_H\sqrt{\gamma_E}/(4 \Upsilon_{H}\sqrt{\Upsilon_E})$, and $\|B_t^{-1}\|\le 1/\gamma_H$ in~\eqref{ass:2:c2}. For the first term in \eqref{equ:72}, we need the following technical lemma to proceed.

\begin{lemma}\label{lem:2}  
    Define $\Omega_t \coloneqq B_t^{-1} \mE[(\bg_t - \nabla f_t) (\bg_t - \nabla f_t)^{\top} \mid \mF_{t-1} ]B_t^{-1}$. Under the conditions of Theorem~\ref{thm:asymptotic_normality}, we have the following convergence as $t \to \infty$:
    \begin{equation*}
\mathbb{E}[(\tilde{K}_t - K_t) \Omega_t (\tilde{K}_t - K_t)^{\top} \mid \mathcal{F}_{t-1}] \stackrel{a.s.}{\rightarrow} \mathbb{E}[(\tilde{K}^{\star} - K^{\star}) \Omega^{\star} (\tilde{K}^{\star} - K^{\star})^{\top}].	
\end{equation*}
\end{lemma}

Therefore, by Lemma \ref{lem:2}, we have    
\begin{align} \label{equ:76}
& \mathbb{E}[(\tilde{K}_t - K_t) B_t^{-1} (\bg_t - \nabla f_t) (\bg_t - \nabla f_t)^{\top} B_t^{-1} (\tilde{K}_t - K_t)^{\top} \mid \mathcal{F}_{t-1}] \nonumber\\
&= \mathbb{E}[(\tilde{K}_t - K_t) B_t^{-1} \mathbb{E}[(\bg_t - \nabla f_t) (\bg_t - \nabla f_t)^{\top} \mid \mathcal{F}_{t-1} \cup \sigma(\{S_{t, j}\}_{j})] B_t^{-1} (\tilde{K}_t - K_t)^{\top} \mid \mathcal{F}_{t-1}] \nonumber\\
&= \mathbb{E}[(\tilde{K}_t - K_t) B_t^{-1} \mathbb{E}[(\bg_t - \nabla f_t) (\bg_t - \nabla f_t)^{\top} \mid \mathcal{F}_{t-1}] B_t^{-1} (\tilde{K}_t - K_t)^{\top}) \mid \mathcal{F}_{t-1}] \nonumber\\
&\stackrel{a.s.}{\rightarrow} \mathbb{E}[(\tilde{K}^{\star} - K^{\star}) \Omega^{\star} (\tilde{K}^{\star} - K^{\star})^{\top}] = \mathbb{E}[\tilde{K}^{\star} \Omega^{\star} (\tilde{K}^{\star})^{\top}] - K^{\star} \Omega^{\star} (K^{\star})^\top.	
\end{align}
Here $\tilde{K}^{\star}$ is defined in \eqref{def:tilde_H_star}. As a result, substituting \eqref{equ:73}, \eqref{equ:74}, \eqref{equ:75} and \eqref{equ:76} into \eqref{equ:72} yields
\begin{equation} \label{equ:79}  
\mJ_{2, t} \stackrel{a.s.}{\rightarrow} \mathbb{E}[\tilde{K}^{\star} \Omega^{\star} (\tilde{K}^{\star})^{\top}] - K^{\star} \Omega^{\star} (K^{\star})^\top.
\end{equation}
Combining~\eqref{equ:71} and~\eqref{equ:79} with~\eqref{equ:66}, we obtain
\begin{equation} \label{equ:80}
\mathbb{E}[\boldsymbol{\theta}_t \boldsymbol{\theta}_t^{\top} \mid \mathcal{F}_{t-1}] \stackrel{a.s.}{\rightarrow} \mathbb{E}[(I - \tilde{K}^{\star}) \Omega^{\star} (I - \tilde{K}^{\star})^{\top}],
\end{equation}
where we use the definition $K^{\star} \coloneqq \mE[\tilde{K}^{\star}]$. This completes the proof.

\subsection{Proof of Lemma \ref{lem:B1}}

    For $\epsilon > 0$ chosen in Assumption \ref{ass:2}, we let $q = 2 + \epsilon$ and have
    \begin{align} \label{eq: bounded_q_moment}
        & \mE[\|\btheta_i\|^q \mid \mF_{i-1}]\;\; \stackrel{\mathclap{\eqref{rec:def:b}}}{\leq}\;\; 2^{q-1}\cbr{\mE[\|(I-K_i)B_i^{-1}(\bg_i - \nabla f_i)\|^q \mid \mF_{i-1}] + \mE[\| \bz_{i, \tau} - (I - K_i)\Delta \bx_i\|^q\mid \mF_{i-1}]} \nonumber\\
        & \leq 2^{q-1}\rbr{\mE[\|(I - K_i)B_i^{-1}\|^q\|\bg_i - \nabla f_i\|^q\mid \mF_{i-1}] + \mE[\|(\tilde{K}_i - K_i)\Delta \bx_i\|^q\mid \mF_{i-1}] } \qquad \  (\text{Lemma \ref{lem:prop_sketch_solver}(a)})  \nonumber\\
        & \leq 2^{q-1}\cbr{\left(\frac{3}{2\gamma_H} \right)^{q} \mE[\|\bg_i - \nabla f_i\|^q\mid \mF_{i-1}] + \left(C_K + \frac{1}{2}\right)^q \mE[ \|\Delta \bx_i\|^q \mid \mF_{i-1}] } \quad \left(\|K_i\|\leq \frac{1}{2},  \|\tilde{K}_i\|\leq C_K \right) \nonumber\\
        & \stackrel{\mathclap{\eqref{ass:2:4th}}}{\leq} 2^{q-1}\cbr{\left(\frac{3}{2\gamma_H} \right)^{q} \left(C_{g, 1}\|\bx_i - \bx^{\star}\|^q + C_{g, 2} \right) + \left(C_K + \frac{1}{2}\right)^q \mE[ \|\Delta \bx_i\|^q \mid \mF_{i-1}] }  \nonumber \\
        & \leq 2^{q-1}\cbr{\left(\frac{3}{2\gamma_H} \right)^{q} \left(C_{g, 1}\|\bx_i - \bx^{\star}\|^q + C_{g, 2} \right) + \left(C_K + \frac{1}{2}\right)^q \mE[ \|B_i^{-1}\bg_i\|^q \mid \mF_{i-1}] }  \nonumber  \\
        & \leq \frac{1}{2 \gamma_H^q} \cbr{3^{q} \left(C_{g, 1}\|\bx_i - \bx^{\star}\|^q + C_{g, 2} \right) + \left(2C_K + 1\right)^q \mE[ \|\bg_i\|^q \mid \mF_{i-1}] } \qquad (\|B_i^{-1}\| \leq 1/\gamma_H)  \nonumber  \\
        & \leq \frac{1}{2\gamma_H^q} \cbr{3^{q} \left(C_{g, 1}\|\bx_i - \bx^{\star}\|^q + C_{g, 2} \right) + \left(2C_K + 1\right)^q \cdot 2^{q-1} (\|\nabla f_i\|^q + \mE[ \|\bg_i - \nabla f_i\|^q \mid \mF_{i-1}] )}  \nonumber  \\
        & \leq \frac{1}{2 \gamma_H^q} \cbr{3^{q} \left(C_{g, 1}\|\bx_i - \bx^{\star}\|^q + C_{g, 2} \right) + \frac{1}{2}\left(4C_K + 2\right)^q  ( \Upsilon_H^q \|\bx_i - \bx^{\star}\|^q + C_{g, 1}\|\bx_i - \bx^{\star}\|^q + C_{g, 2} )}  \nonumber  \\
        & \leq C_q^{(1)}\|\bx_i - \bx^{\star}\|^q + C_q^{(2)},
    \end{align}
    where the first and the third last inequality are due to Jensen's inequality; the second last inequality is due to the $\Upsilon_H$-Lipschitz continuity of $\nabla f(\bx)$ in Assumption \ref{ass:1} and the moment bound \eqref{ass:2:4th} in Assumption \ref{ass:2}. With the above display, we have for any $\delta>0$, as $t \to \infty$,
    \begin{align}\label{eq:Lindeberg_CLT}
        \frac{1}{t} \sum_{i = 0}^{t-1} \mathbb{E}\left[ \|\boldsymbol{\theta}_i\|^2 \cdot \mathbf{1}_{\|\boldsymbol{\theta}_i\| > \delta \sqrt{t}}  \mid \mathcal{F}_{i-1} \right] &\leq \frac{1}{\delta^{\epsilon} t^{1+\epsilon/2}} \sum_{i = 0}^{t-1} \mathbb{E}\left[ \|\boldsymbol{\theta}_i\|^{2+\epsilon} | \mathcal{F}_{i-1} \right] \\
        &\stackrel{\mathclap{\eqref{eq: bounded_q_moment}}}{\leq} \ \frac{C^{(2)}_{2+\epsilon}}{\delta^{\epsilon}t^{\epsilon/2}} + \frac{C_{2+\epsilon}^{(1)} \sum_{i = 0}^{t-1} \|\bx_i - \bx^{\star}\|^{2+\epsilon}}{\delta^{\epsilon} t^{1+\epsilon/2}} \stackrel{a.s.}{\rightarrow} 0,
    \end{align}
    where the convergence holds due to $\bx_i \stackrel{a.s.}{\rightarrow} \bx^{\star}$ as $i \rightarrow \infty$ and Stolz--Ces\`aro theorem (cf. Lemma \ref{aux:lem:stolz-cesaro}). This completes the proof of the lemma.

\subsection{Proof of Lemma \ref{lem:continuity_projection_matrix}}

Since $B_t, E_t, B^{\star}, E^{\star} \in \mR^{d \times d}$ are positive definite, $E_t^{-1/2}B_t$ and $(E^{\star})^{-1/2}B^{\star}$ are non-singular. Then we directly use \citet[Lemma 5.2]{Na2025Statistical}\footnote{Indeed, the argument and the proof only use singular-value bounds, Wedin's $\sin(\Theta)$ theorem, and properties of the Moore--Penrose pseudoinverse. They do not require matrix symmetry.}, and obtain 
     \begin{multline}\label{equ:137}  
        \left\|E_t^{-1/2}B_t S (S^{\top}B_t E_t^{-1} B_tS)^{\dagger}S^{\top}B_t E_t^{-1/2} - (E^{\star})^{-1/2}B^{\star} S (S^{\top}B^{\star} (E^{\star})^{-1} B^{\star} S)^{\dagger}S^{\top}B^{\star} (E^{\star})^{-1/2} \right\| \\
        \leq \frac{2 \left\| E_t^{-1/2}B_t  - (E^{\star})^{-1/2} B^{\star}  \right\|}{\sigma_{\min}( (E^{\star})^{-1/2} B^{\star})} \cdot \|S\|\|S^{\dagger}\|,
    \end{multline}
where $\sigma_{\min}(\cdot)$ denotes the least singular value. We first have
\begin{align}\label{equ:138} 
\|E_t^{-1/2}B_t -  (E^\star)^{-1/2} B^\star\|
&\le
\|E_t^{-1/2} (B_t-B^\star)\|
+
\|(E_t^{-1/2}-(E^\star)^{-1/2}) B^\star\| \nonumber \\
&\le
\|B_t-B^\star\|\,\|E_t^{-1/2}\|
+
\|B^\star\|\,\|E_t^{-1/2}-(E^\star)^{-1/2}\| \nonumber \\
&\le  \frac{1}{\sqrt{\gamma_E}}\|B_t-B^\star\|
+
\Upsilon_H\,\|E_t^{-1/2}-(E^\star)^{-1/2}\|,
\end{align}
where the first inequality uses the triangle inequality, the second inequality uses the Cauchy-Schwarz inequality, and the last inequality uses the fact that $E_t \succeq \gamma_E I$, and $B^{\star} \preceq \Upsilon_H I$. We follow the same spirit as \citet[Lemma B.2]{Wang2026Inference}, and get
\begin{equation}\label{equ:139} 
    \|E_t^{-1/2}-(E^\star)^{-1/2}\| \leq \frac{1}{2\gamma_E^{3/2}}\|E_t - E^{\star}\|.
\end{equation}
Hence,  we obtain that
\begin{equation}\label{equ:140} 
    \|E_t^{-1/2}B_t -  (E^\star)^{-1/2} B^\star\| \stackrel{\eqref{equ:138}, \; \eqref{equ:139}}{\leq} \frac{1}{\sqrt{\gamma_E}}\|B_t-B^\star\|
+ \frac{\Upsilon_H}{2\gamma_E^{3/2}}\|E_t - E^{\star}\|.
\end{equation}
Second, since $E_t \stackrel{a.s.}{\to} E^{\star}$ in Assumption \ref{ass:5}, and $E_t \preceq \Upsilon_E I$ for all $t \geq 0$ in Assumption \ref{ass:4}, then we have $E^{\star} \preceq \Upsilon_E I$, and furthermore, by Assumption \ref{ass:1}, we have
\begin{equation}\label{equ:141} 
    \sigma_{\min} ( (E^{\star})^{-1/2} B^{\star} ) \geq  \sigma_{\min} ((E^{\star})^{-1/2}) \sigma_{\min} (B^{\star}) \geq \frac{\gamma_H}{\sqrt{\Upsilon_E}}.
\end{equation}
We apply \eqref{equ:140} and \eqref{equ:141} into \eqref{equ:137}, and finally obtain
\begin{multline*}
        \left\|E_t^{-1/2}B_t S (S^{\top}B_t E_t^{-1} B_tS)^{\dagger}S^{\top}B_t E_t^{-1/2} - (E^{\star})^{-1/2}B^{\star} S (S^{\top}B^{\star} (E^{\star})^{-1} B^{\star} S)^{\dagger}S^{\top}B^{\star} (E^{\star})^{-1/2} \right\| \\
        \leq \frac{2 \sqrt{\Upsilon_E}}{\gamma_H} \left( \frac{1}{\sqrt{\gamma_E}} \|B_t - B^{\star}\| + \frac{\Upsilon_H}{2\gamma_E^{3/2}}\|E_t - E^{\star}\|\right) \cdot \|S\|\|S^{\dagger}\|.
    \end{multline*}
We complete the proof of the lemma.

\subsection{Proof of Lemma \ref{lem:7}}

Recalling the definitions of $\alpha, \beta, \gamma$ in \eqref{def:alpha_beta_gamma_t}, we view them as functions of $(\mu, \nu)$ on $(0, \infty)^2$:
\begin{equation*}
\alpha(\mu, \nu) = \frac{1}{1 + \gamma \nu}, \quad \beta(\mu, \nu) = 1 - \sqrt{\frac{\mu}{\nu}}, \quad \gamma(\mu, \nu) = \frac{1}{\sqrt{\mu \nu}},
\end{equation*}
where the only distinction between $(\alpha_t, \beta_t, \gamma_t)$ and $(\alpha^\star, \beta^\star, \gamma^\star)$ lies in their evaluation at $(\mu_t, \nu_t)$ versus $(\mu^\star, \nu^\star)$. Each of $\alpha, \beta, \gamma$ is continuously differentiable on $(0, \infty)^2$. By \eqref{prop_mu_nu} and Lemma~\ref{lem:4},
\begin{equation*}
(\mu_t, \nu_t) \in [\gamma_E \gamma_S / \Upsilon_E,\, 1] \times [1,\, \Upsilon_E / (\gamma_E \gamma_S)] \quad \text{for all } t \geq 0.
\end{equation*}
We further bound the deviation of $Z_t$ from $Z^\star$:
\begin{equation}\label{equ:151}  
    \|Z_t - Z^\star\| = \bigl\|\mE[\tilde{Z}_t - \tilde{Z}^\star \mid \mF_{t-1}]\bigr\| \leq \mE[\|\tilde{Z}_t - \tilde{Z}^\star\| \,\big|\, \mF_{t-1}] \stackrel{\eqref{equ:B_bound_projection}}{\lesssim} \|B_t - B^\star\| + \|E_t - E^\star\| \stackrel{\text{a.s.}}{\longrightarrow} 0,
\end{equation}
where the first inequality uses Jensen's inequality and the almost-sure convergence follows from Lemma~\ref{lem:converge_rate_and_Bt_converge} and Assumption~\ref{ass:5}. Since $Z_t \succeq (\gamma_E \gamma_S / \Upsilon_E) I$ by Lemma~\ref{lem:4}, the limit $Z_t \stackrel{\text{a.s.}}{\to} Z^\star$ from \eqref{equ:151} yields $Z^\star \succeq (\gamma_E \gamma_S / \Upsilon_E) I$. Applying \eqref{prop_mu_nu} to $Z^\star$, we obtain
\begin{equation*}
(\mu^\star, \nu^\star) \in [\gamma_E \gamma_S / \Upsilon_E,\, 1] \times [1,\, \Upsilon_E / (\gamma_E \gamma_S)].
\end{equation*}
The gradients $\nabla\alpha, \nabla\beta, \nabla\gamma$ are continuous and therefore bounded on the compact rectangle $[\gamma_E \gamma_S / \Upsilon_E,\, 1] \times [1,\, \Upsilon_E / (\gamma_E \gamma_S)]$. Since this rectangle is convex and contains both $(\mu_t, \nu_t)$ and $(\mu^\star, \nu^\star)$, the mean value theorem yields
\begin{equation}\label{equ:152}
\max\bigl\{|\alpha_t - \alpha^\star|,\, |\beta_t - \beta^\star|,\, |\gamma_t - \gamma^\star|\bigr\} \lesssim |\mu_t - \mu^\star| + |\nu_t - \nu^\star|.
\end{equation}
Therefore, it is sufficient to investigate the bounds $|\mu_t - \mu^\star|$ and $|\nu_t - \nu^\star|$.

\noindent $\bullet$ {\textbf{Bound for $|\mu_t - \mu^\star|$.}}
By the definitions of $\mu_t$ in \eqref{def:mu_t_nu_t} and $\mu^{\star}$ in Section~\ref{sec:4}, we have $\mu_t = \lambda_{\min}(Z_t)$ and $\mu^{\star} = \lambda_{\min}(Z^{\star})$. Applying Weyl’s perturbation theorem \citep[Corollary~III.2.6]{Bhatia1997Matrix}, we obtain
\begin{equation} \label{equ:mu_bound}
|\mu_t - \mu^{\star}| \leq \|Z_t - Z^{\star}\| \stackrel{\eqref{equ:151}}{\leq} \mE[\|\tilde{Z}_t - \tilde{Z}^\star\| \,\big|\, \mF_{t-1}] \stackrel{\eqref{equ:151}}{\lesssim} \|B_t - B^\star\| + \|E_t - E^\star\| \stackrel{\text{a.s.}}{\longrightarrow} 0.
\end{equation}

\noindent $\bullet$ \textbf{Bound for $|\nu_t - \nu^\star|$.} The expressions for $\nu_t$ in terms of $(Z_t, \tilde Z_t)$ and $\nu^\star$ in terms of $(Z^\star, \tilde Z^\star)$ share the same structural form as their counterparts in \citet[Eq. (2.5)]{Wang2026Inference}. The only difference is that the matrices $Z_t,\tilde Z_t, Z^\star,\tilde Z^\star$ here depend on the general metric $E_t$. Nevertheless, the ingredients used in the perturbation argument of \citet[Lemma~4.1, Eq.~(B.10)]{Wang2026Inference} remain valid in our setting:
$\tilde Z_t$ and $\tilde Z^\star$ are projection matrices, and Lemma~\ref{lem:4} gives the uniform lower bound $Z_t, Z^\star \succeq (\gamma_E \gamma_S / \Upsilon_E) I$. Thus, the same perturbation argument applies verbatim and yields
\begin{equation}\label{equ:nu_bound}
        |\nu_t - \nu^{\star}| \lesssim \mathbb{E} [\|\tilde Z_t - \tilde Z^{\star}\|\mid \mathcal F_{t-1}] \stackrel{\eqref{equ:151}}{\lesssim} \|B_t - B^\star\| + \|E_t - E^\star\| \stackrel{\text{a.s.}}{\longrightarrow} 0.
\end{equation}
Hence, combining  \eqref{equ:152}, \eqref{equ:mu_bound}, \eqref{equ:nu_bound}, we complete the proof of the lemma.

\subsection{Proof of Lemma \ref{lem:2}}

By \eqref{equ:70} and Lemma \ref{lem:converge_rate_and_Bt_converge}, we know $\Omega_t\to\Omega^\star$ almost surely. Therefore, we have
\begin{align} \label{equ:85}
& \left\| \mathbb{E}[(\tilde{K}_t - K_t) \Omega_t (\tilde{K}_t - K_t)^{\top} \mid \mathcal{F}_{t-1}] - \mathbb{E}[(\tilde{K}^{\star} - K^{\star}) \Omega^{\star} (\tilde{K}^{\star} - K^{\star})^{\top}] \right\| \nonumber\\
&\leq \left\| \mathbb{E}[(\tilde{K}_t - K_t) \Omega_t (\tilde{K}_t - K_t)^{\top} \mid \mathcal{F}_{t-1}] - \mathbb{E}[(\tilde{K}_t - K_t) \Omega^{\star} (\tilde{K}_t - K_t)^{\top} \mid \mathcal{F}_{t-1}]\right\| \nonumber\\
& \quad + \left\| \mathbb{E}[(\tilde{K}_t - K_t) \Omega^{\star} (\tilde{K}_t - K_t)^{\top} \mid \mathcal{F}_{t-1}]- \mathbb{E}[(\tilde{K}^{\star} - K^{\star}) \Omega^{\star} (\tilde{K}^{\star} - K^{\star})^{\top}] \right\|.
\end{align}
The expectation in \eqref{equ:85} is taken with respect to the sketching distribution $S$ only. Throughout this proof, we explicitly realize the sketching matrices as independent draws $S_0,S_1,\ldots,S_{\tau-1}\stackrel{i.i.d.}{\sim} S$. Recalling the definition of $\tilde C_t$ in \eqref{def:tilde_C_t}, we write, for this proof only,
\begin{equation}\label{equ:86}
    \tilde{C}_t =
    \prod_{j = 0}^{\tau-1}
    \begin{pmatrix}
        (1 - \alpha_t)(I - Z_{t,j}) 
        & \alpha_t (I - Z_{t,j}) \\
        (1 - \alpha_t)(1 - \beta_t)I - (1 - \alpha_t)\gamma_t Z_{t,j} 
        & (\alpha_t + \beta_t - \alpha_t\beta_t)I - \alpha_t\gamma_t Z_{t,j}
    \end{pmatrix}
    \in \mathbb{R}^{2d \times 2d},
\end{equation}
where $Z_{t,j} = E_t^{-1/2} B_t S_{j} (S_{j}^\top B_t E_t^{-1} B_t S_{j})^\dagger S_{j}^\top B_t E_t^{-1/2}$.
In this representation, the same sketching realization
\(S_0,S_1,\ldots,S_{\tau-1}\) is used to define both \(\tilde C_t\) and \(\tilde C^\star\); see \eqref{equ:77}. The definition of \(\tilde K_t\) remains unchanged from \eqref{def:tilde_H_t}, namely $\tilde K_t = E_t^{-1/2} \bigl([\tilde C_t]_{1,1}+[\tilde C_t]_{1,2}\bigr) E_t^{1/2}$.
This is only a coupling convention for the proof and does not change the quantity in \eqref{equ:85}. For the first term on the right-hand side of~\eqref{equ:85}, we have
\begin{align} \label{equ:87}
& \left\| \mathbb{E}[(\tilde{K}_t - K_t) \Omega_t (\tilde{K}_t - K_t)^{\top} \mid \mathcal{F}_{t-1}] - \mathbb{E}[(\tilde{K}_t - K_t) \Omega^{\star} (\tilde{K}_t - K_t)^{\top} \mid \mathcal{F}_{t-1}]\right\| \nonumber\\
&= \left\|\mathbb{E}[(\tilde{K}_t - K_t) (\Omega_t - \Omega^{\star}) (\tilde{K}_t - K_t)^{\top} \mid \mathcal{F}_{t-1}]\right\| \nonumber\\
&\leq \mathbb{E}\left[\|(\tilde{K}_t - K_t) (\Omega_t - \Omega^{\star}) (\tilde{K}_t - K_t)^{\top} \| \mid \mathcal{F}_{t-1}\right] \nonumber\\
&\leq (C_K + 0.5)^2 \cdot \mathbb{E}\left[\|\Omega_t - \Omega^{\star} \| \mid \mathcal{F}_{t-1}\right] = (C_K + 0.5)^2  \| \Omega_t - \Omega^{\star} \| \stackrel{a.s.}{\rightarrow} 0.	
\end{align}
Here, the second inequality follows from $\|\tilde{K}_t\|\leq C_K$ and $\|K_t\|<0.5$ (under $\tau \cdot (1 - \sqrt{\mu_t/\nu_t})^{\tau - 2} \leq \gamma_H\sqrt{\gamma_E}/(4 \Upsilon_{H}\sqrt{\Upsilon_E})$). For the second term on the right-hand side of~\eqref{equ:85}, we have
\begin{align*}
& \left\|\mathbb{E}[(\tilde{K}_t - K_t) \Omega^{\star} (\tilde{K}_t - K_t)^{\top} \mid \mathcal{F}_{t-1}] - \mathbb{E}[(\tilde{K}^{\star} - K^{\star}) \Omega^{\star}(\tilde{K}^{\star} - K^{\star})^{\top}]\right\| \nonumber\\
& \leq \left\|\mathbb{E}[(\tilde{K}_t - K_t) \Omega^{\star} (\tilde{K}_t - K_t)^{\top} \mid \mathcal{F}_{t-1}] - \mathbb{E}[(\tilde{K}^{\star} - K^{\star}) \Omega^{\star}(\tilde{K}_t - K_t)^{\top} \mid \mathcal{F}_{t-1}]  \right\| \nonumber\\
& \quad + \left\| \mathbb{E}[(\tilde{K}^{\star} - K^{\star}) \Omega^{\star}(\tilde{K}_t - K_t)^{\top} \mid \mathcal{F}_{t-1}] - \mathbb{E}[(\tilde{K}^{\star} - K^{\star}) \Omega^{\star}(\tilde{K}^{\star} - K^{\star})^{\top}] \right\| \nonumber\\
& \leq \|\Omega^{\star}\| \cdot (C_K + 0.5) \cdot (\mathbb{E}[ \|\tilde{K}_t - \tilde{K}^{\star} \| \mid \mathcal{F}_{t-1}] + \|K_t - K^{\star}\|) \nonumber\\
& \quad + \|\Omega^{\star}\| \cdot (C_K + 0.5) \cdot (\mathbb{E}[ \|\tilde{K}_t - \tilde{K}^{\star} \| \mid \mathcal{F}_{t-1}] + \|K_t - K^{\star}\|).	
\end{align*}

By the definition of $\tilde{K}_t$ and $\tilde{K}^{\star}$ in \eqref{def:tilde_H_t} and \eqref{def:tilde_H_star},  we follow the same spirit of \eqref{equ:146}, and have
\begin{align}\label{appequ:1}
    \|\tilde{K}_t - \tilde{K}^{\star}\| \lesssim \|E_t - E^{\star}\| +     \|\tilde{C}_t - \tilde{C}^{\star}\|.
\end{align}
By the definition of $C_{t, j}$ in \eqref{def:tilde_C_t} and $\tilde{M}_j^{\star}$ in \eqref{equ:124}, we know that their spectral norms are bounded, since $(\alpha^\star, \beta^\star, \gamma^\star)$ and $(\alpha_t, \beta_t, \gamma_t)$ are bounded (cf. \eqref{prop_mu_nu}), and both $Z_{t, j}$ and $\tilde{Z}_j^\star$ are projection matrices. Therefore, we follow the same spirit of \eqref{equ:142}, and obtain
\begin{equation}\label{appequ:2}
    \|\tilde{M}_j^{\star}\| \leq 2+ \sqrt{\frac{\Upsilon_E}{\gamma_E \gamma_S}}; \qquad \|C_{t, j}\| \leq 2+ \sqrt{\frac{\Upsilon_E}{\gamma_E \gamma_S}}.
\end{equation}
Hence, for the term $\|\tilde{C}_t - \tilde{C}^{\star}\|$, we have
\begin{align}\label{equ:89}
    \|\tilde{C}_t - \tilde{C}^{\star}\| \quad &\stackrel{\mathclap{\eqref{def:tilde_C_t}, \eqref{equ:77}}}{\lesssim} \quad \  \sum_{j = 0}^{\tau - 1} (2 + \Upsilon_E/\sqrt{\gamma_E \gamma_S})^{j} \cdot ( |\alpha_t - \alpha^{\star}| + |\beta_t - \beta^{\star}| + |\gamma_t - \gamma^{\star}| \nonumber \\
    & \qquad \qquad \qquad \qquad \qquad \qquad \qquad  + \|Z_{t, j} - \tilde{Z}^{\star}_j\| )  \cdot (2 + \Upsilon_E/\sqrt{\gamma_E \gamma_S})^{\tau - j - 1} \nonumber\\
    & \lesssim \sum_{j = 0}^{\tau - 1} (2+\Upsilon_E/\sqrt{\gamma_E \gamma_S})^{\tau - 1} \cdot ( \mE[\|\tilde{Z}_t - \tilde{Z}^{\star}\| \mid \mF_{t-1}] + \|Z_{t, j} - \tilde{Z}^{\star}_j\| ) \nonumber\\
    & =  (2+\Upsilon_E/\sqrt{\gamma_E \gamma_S})^{\tau - 1} \sum_{j = 0}^{\tau - 1} \rbr{\mE[\|\tilde{Z}_t - \tilde{Z}^{\star}\| \mid \mF_{t-1}] + \|Z_{t, j} - \tilde{Z}^{\star}_j\|} \quad (\text{by Lemma \ref{lem:continuity_projection_matrix}}) \nonumber\\
    & \stackrel{\mathclap{\eqref{equ:B_bound_projection}}}{\lesssim} (2 + \Upsilon_E/\sqrt{\gamma_E \gamma_S})^{\tau - 1} \sum_{j = 0}^{\tau - 1} \rbr{ (\|B_t - B^{\star}\| + \|E_t - E^\star\|) \cdot \mE[\|S\|\|S^{\dagger}\|] + \|Z_{t, j} - \tilde{Z}^{\star}_j\|} \nonumber \\
    &\lesssim  \sum_{j = 0}^{\tau - 1} \rbr{ \|B_t - B^{\star}\| + \|E_t - E^\star\| + \|Z_{t, j} - \tilde{Z}^{\star}_j\|},
\end{align}
where the first inequality also uses the fact that 
\begin{equation*}
    \prod_{i=1}^{n} A_i \;-\; \prod_{i=1}^{n} B_i
\;=\;
\sum_{k=1}^{n}
\left( \prod_{i=1}^{k-1} A_i \right)
\,(A_k - B_k)\,
\left( \prod_{j=k+1}^{n} B_j \right),
\end{equation*}
 the second inequality uses Lemma \ref{lem:7}, the third inequality uses Lemma \ref{lem:continuity_projection_matrix}, and the last inequality uses Assumption \ref{ass:4}. Taking conditional expectation on both sides of \eqref{equ:89}, altogether with Lemma \ref{lem:continuity_projection_matrix}, we have
 \begin{equation}\label{appequ:3}
     \mE[\|\tilde{C}_t - \tilde{C}^\star \|| \mF_{t-1}] \lesssim \|B_t - B^\star\| + \|E_t - E^\star\| \stackrel{a.s.}{\to} 0,
 \end{equation}
 where the convergence uses Lemma \ref{lem:converge_rate_and_Bt_converge} and Assumption \ref{ass:5}.
 Thus, the above display leads to 
\begin{equation} \label{equ:91}
\left\|\mathbb{E}[(\tilde{K}_t - K_t) \Omega^{\star} (\tilde{K}_t - K_t)^{\top} \mid \mathcal{F}_{t-1}] - \mathbb{E}[(\tilde{K}^{\star} - K^{\star}) \Omega^{\star}(\tilde{K}^\star - K^{\star})^{\top}]\right\| \stackrel{a.s.}{\rightarrow} 0.
\end{equation}
Substituting \eqref{equ:87} and \eqref{equ:91} into \eqref{equ:85} yields
\begin{equation*}
\mathbb{E}[(\tilde{K}_t - K_t) \Omega_t (\tilde{K}_t - K_t)^{\top} \mid \mathcal{F}_{t-1}] \stackrel{a.s.}{\rightarrow} \mathbb{E}[(\tilde{K}^{\star} - K^{\star}) \Omega^{\star} (\tilde{K}^{\star} - K^{\star})^{\top}] \quad \text{as } t \rightarrow \infty.
\end{equation*}
This completes the proof.

\section{Additional Experimental Details} \label{app:exp}

This section provides the full experimental setup and additional results that supplement the experiments reported in Section~\ref{sec:5.1}.

\subsection{Detailed experimental setup}

For the linear regression problem, we consider the model $\xi_{b} = \xi_{a}^\top \boldsymbol{x}^\star + \varepsilon$, where $\xi = (\xi_a, \xi_b) \in \mathbb{R}^{d+1}$ is the covariate-response pair, $\varepsilon \sim \mathcal{N}(0, \sigma^2)$ is Gaussian noise, and $\bx^{\star} \in \mathbb{R}^{d}$ is the model parameter. We adopt the squared loss $F(\bx; \xi) = \tfrac{1}{2}(\xi_b - \xi_a^\top \boldsymbol{x})^2$ in this case.

For the logistic regression problem, we consider the model $\mathbb{P}(\xi_{b} \mid \xi_a) = \exp(\xi_b \cdot \xi_a^\top \boldsymbol{x}^\star) / \{1 + \exp(\xi_b \cdot \xi_a^\top \boldsymbol{x}^\star)\}$, where $(\xi_a, \xi_b)\in \mR^d\times \{-1,1\}$ is the covariate-response pair and $\boldsymbol{x}^{\star} \in \mathbb{R}^d$ is the model parameter. We adopt the log loss $F(\bx; \xi) = \log( 1 + \exp(-\xi_b \cdot \xi_a^\top \boldsymbol{x}))$ in this case. 

For both models, we apply Hessian regularization from the above and below whenever it becomes~close to singular. However, since $B_t\rightarrow B^\star = \nabla^2 f^\star\succ \0$, the regularization always vanishes in~finite~steps.

\vskip2pt

\noindent $\bullet$ \textbf{Model parameters.}
We vary the dimension over $d \in \{20, 40\}$ for both problems and set the true parameter $\bx^{\star} \in \mR^d$ to be linearly spaced between $0$ and $1$. For the linear model, we fix the noise variance at $\sigma^2 = 1$. Following previous studies \citep{Chen2020Statistical, Zhu2021Online, Chen2024Online, Na2025Statistical, Du2025Online}, we draw the covariate $\xi_a \sim \mN(\0, \Sigma_a)$ under three covariance designs: (i) identity, $\Sigma_a = I$; (ii) Toeplitz, $[\Sigma_a]_{i, j} = r^{|i - j|}$ with $r = 0.4$; and (iii) equi-correlation, $[\Sigma_a]_{i, i} = 1$ and $[\Sigma_a]_{i, j} = r$ for $i \neq j$ with $r = 0.4$. Given $\xi_a$, the response $\xi_b$ is generated from the corresponding linear or logistic model.

\vskip2pt

\noindent $\bullet$ \textbf{Algorithm parameters.}
Both inference procedures operate on the iteration sequence produced by the three optimization methods (EN, ASN, and GASN) under a common algorithmic setup. To be specific, we vary the number of sketching steps over $\tau \in \{5, 10\}$ and consider two sketching distributions: Coordinate sketches \citep{Strohmer2008Randomized}, with sketching vectors drawn from $S \sim \text{Uniform}(\{\be_i\}_{i=1}^{d})$; and Gaussian sketches, with $S \sim \mathcal{N}(0, I)$ (i.e., we set $s=1$ and use sketching vectors). Following \cite{Kuang2025Online, Du2025Online, Wang2026Inference}, we set the decaying step size to $\varphi_t = 1/(t+1)^{0.501}$. For every inference method, we initialize $\bx_0$ as the vector of all ones and run $10^5$ iterations. The nominal coverage probability is fixed at $1 - p = 95\%$, and we conduct inference on the coordinate-wise average of the model parameter, $\sum_{i=1}^d \bx_i^\star / d$. Coverage rates and interval lengths are reported as averages over $200$ independent~runs~per~configuration.

\subsection{Additional results and discussions}

\begin{figure}[t!]
\centering     
\subfigure[EN]{\includegraphics[width=0.325\textwidth]{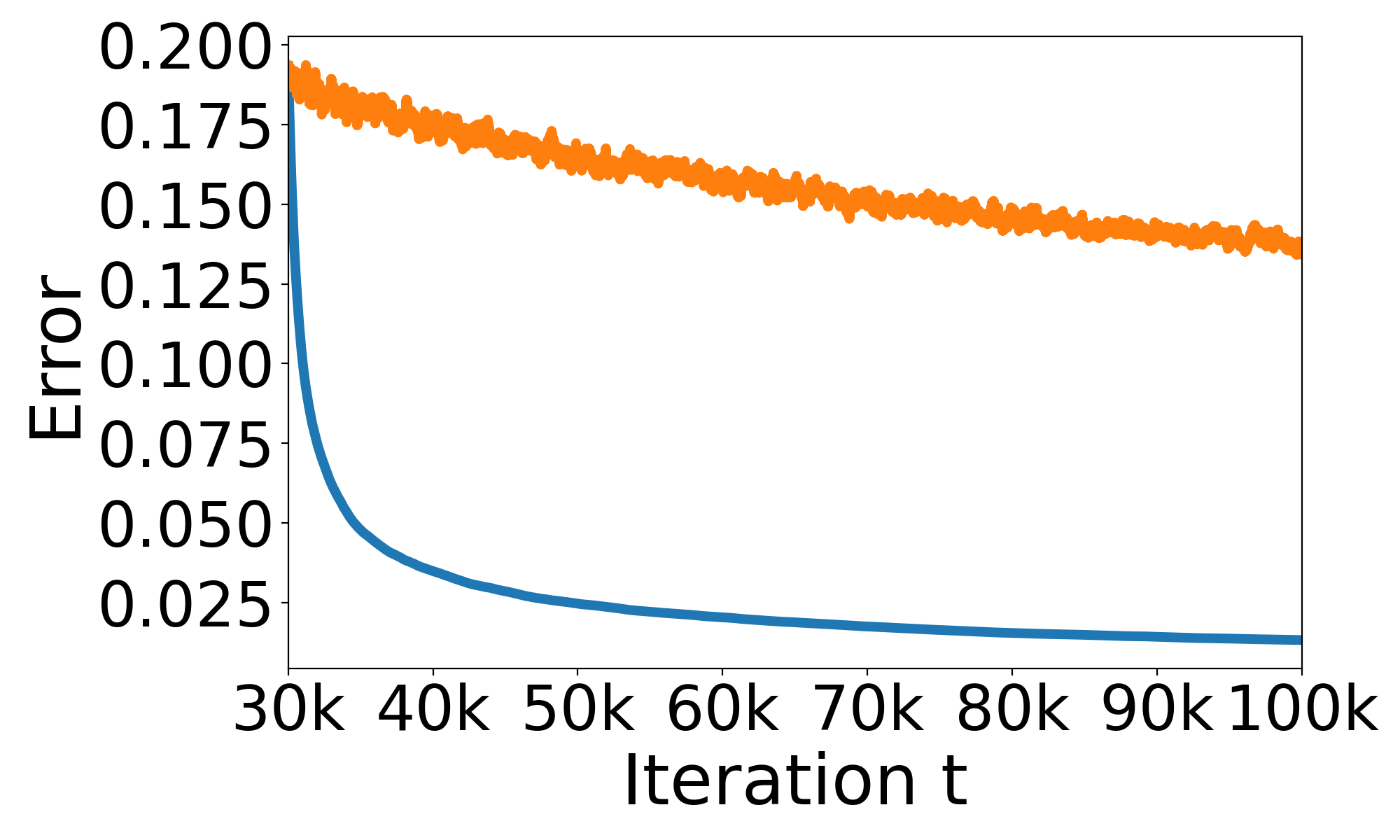}}
\subfigure[ASN$(E_t = I)$]{\includegraphics[width=0.325\textwidth]{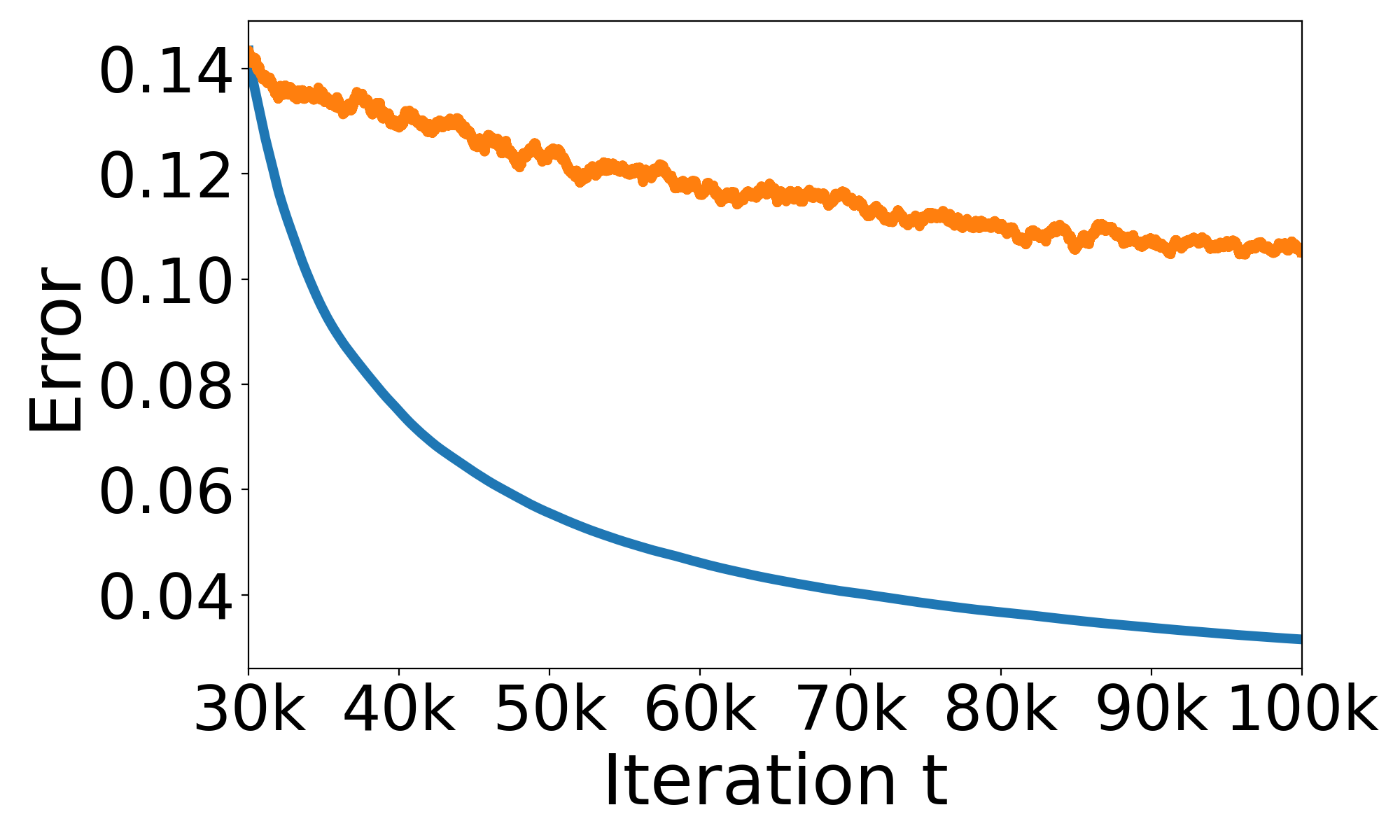}}
\subfigure[GASN$(E_t = B_t)$]{\includegraphics[width=0.325\textwidth]{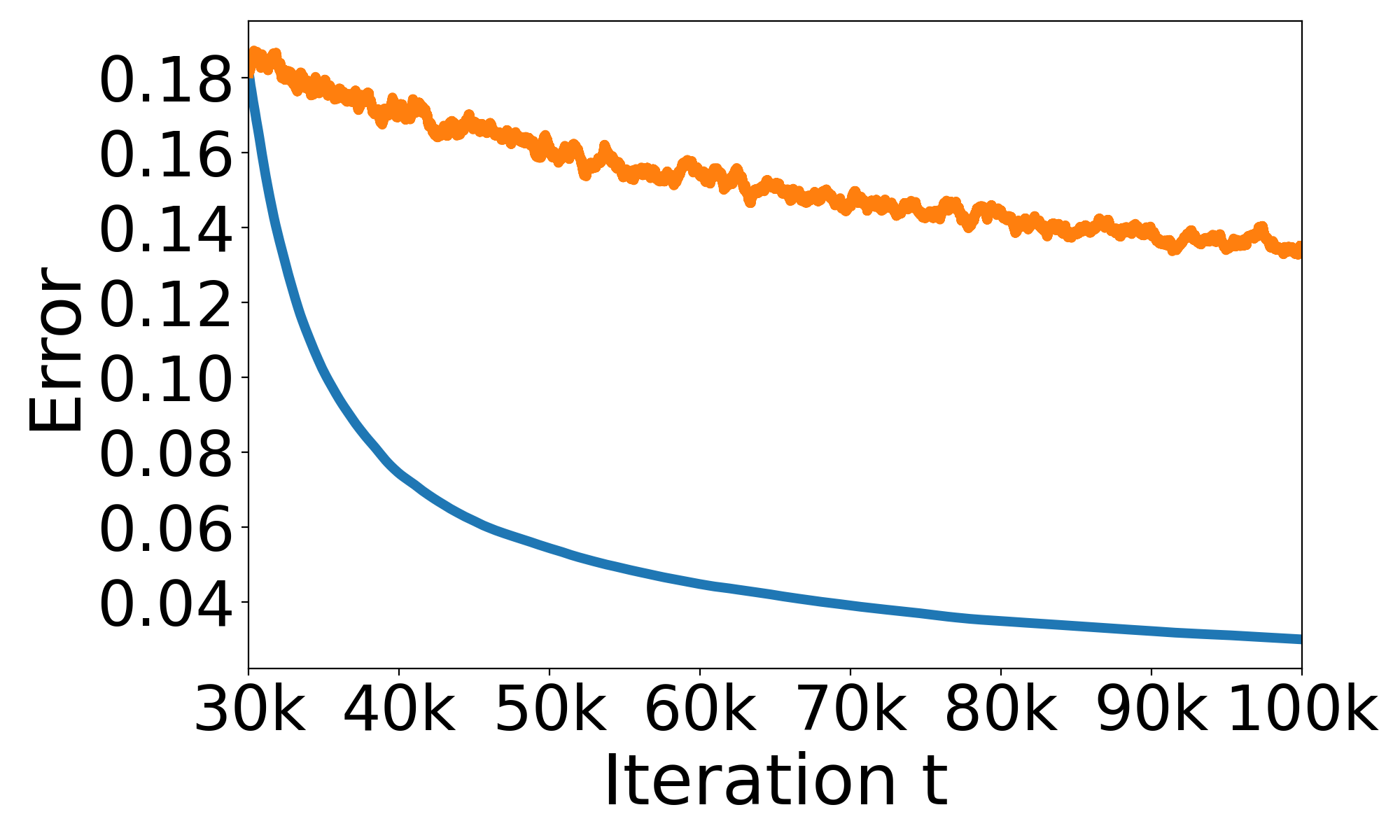}}
\centering{Linear Regression}

\subfigure[EN]{\includegraphics[width=0.325\textwidth]{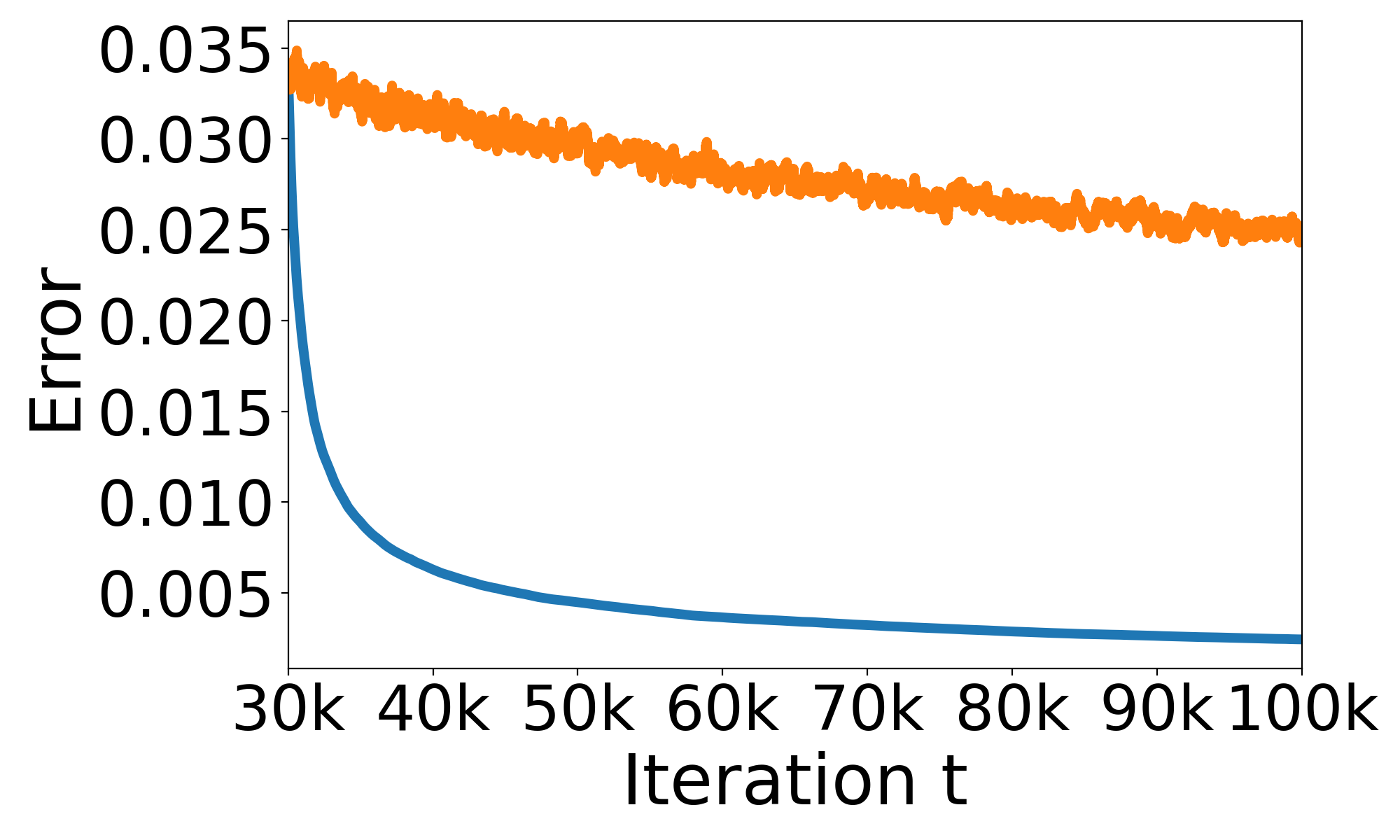}}
\subfigure[ASN$(E_t = I)$]{\includegraphics[width=0.325\textwidth]{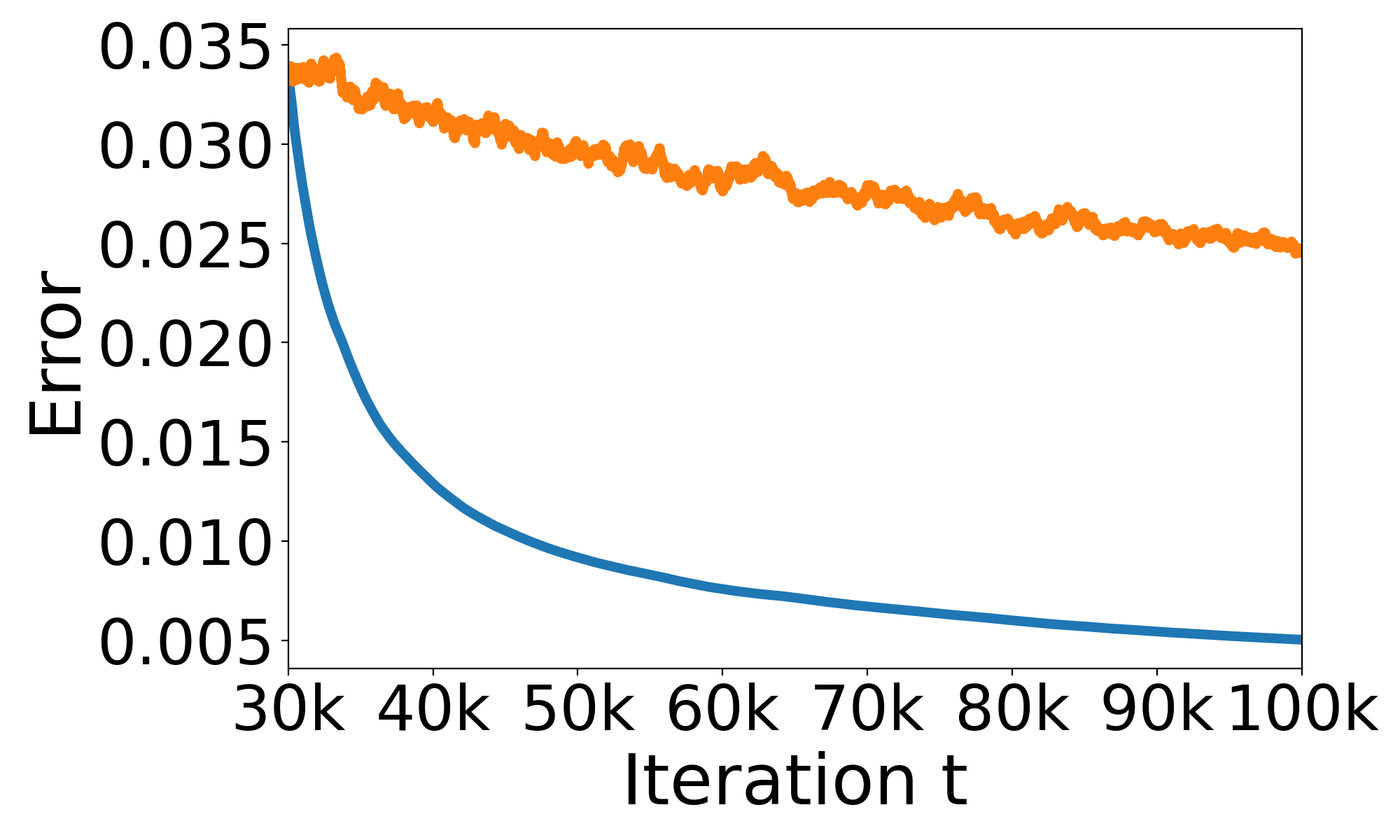}}
\subfigure[GASN$(E_t = B_t)$]{\includegraphics[width=0.325\textwidth]{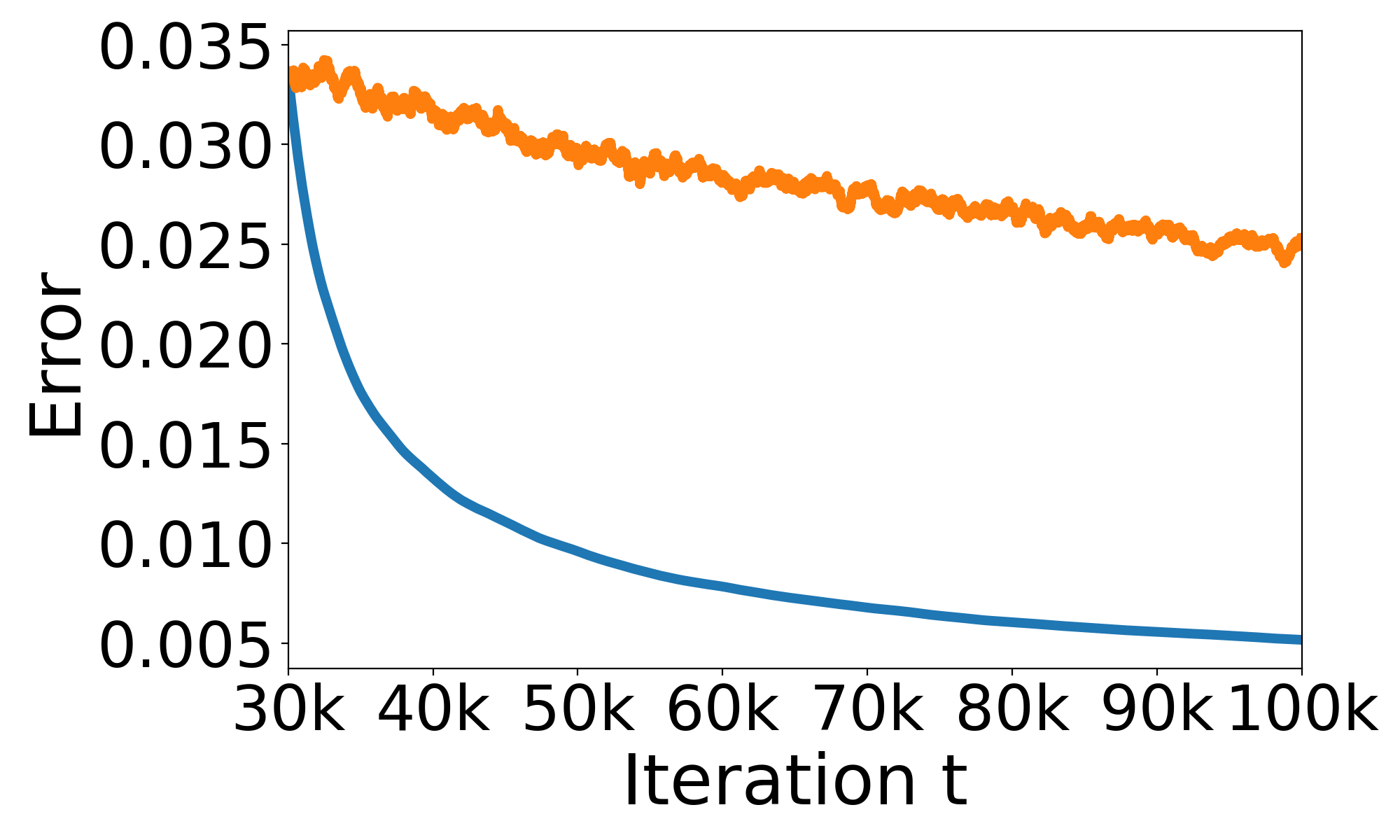}}
\centering{Logistic Regression}
\vskip4pt
\includegraphics[width=0.5\textwidth]{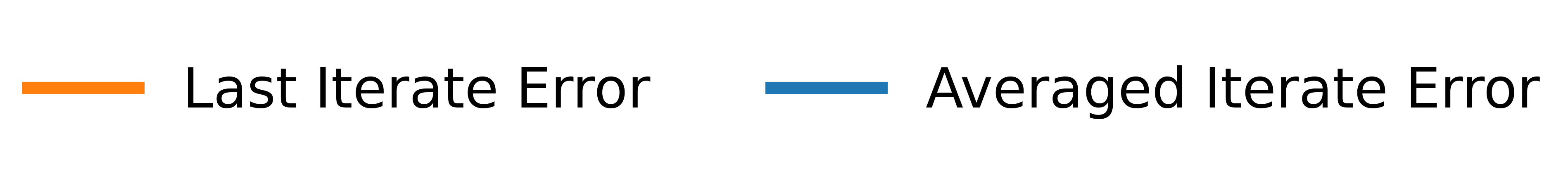}
\caption{\textit{
Comparison of the expected averaged-iterate error $\mE \|\bar{\bx}_t - \bx^{\star}\|$ with the expected last-iterate error $\mE \|\bx_t - \bx^{\star}\|$ for linear and logistic regression under three stochastic Newton methods; both sketched methods are run with $\tau = 5$ sketching steps.
}}\label{fig:error}
\end{figure}
\vspace{-0.2cm}
\noindent$\bullet$ \textbf{Algorithm consistency.} 
In this section, we examine the consistency of stochastic Newton iterates in both linear and logistic regression. Specifically, we compare the expected estimation error of the averaged iterate, $\mE\|\bar{\bx}_t - \bx^{\star}\|$, with that of the last iterate, $\mE\|\bx_t - \bx^{\star}\|$. We approximate these expectations by averaging the iterate errors over 200 independent runs. Figure \ref{fig:error} reports a representative example with $d = 20$ under the equi-correlation design with $r = 0.4$ for three stochastic Newton methods, namely EN, ASN, and GASN; the same qualitative pattern appears in all other settings. The figure shows that the error of the averaged iterate decreases substantially faster than that of the last iterate. This behavior is consistent with Theorem \ref{thm:asymptotic_normality}, which shows that the averaged iterate achieves $\sqrt{t}$-consistency, while the last iterate attains only $\sqrt{1/\varphi_t} \approx O(t^{1/4})$-consistency (see the GASN-vs-ASN comparison in Section~\ref{sec:4.1}). The averaged-iterate curves are also smoother, while the last-iterate curves exhibit larger fluctuations, which suggests lower variability and hence greater statistical efficiency for the averaged iterate.

\vskip4pt

\noindent$\bullet$ \textbf{Complete inference results.}
We provide the complete inference results corresponding to the experiments in Section~\ref{sec:5.1}. Table~\ref{tab:d20_d40_complete_results} reports results for both dimensions $d=20$ and $d = 40$, and sketching steps $\tau\in\{5,10\}$. The green-highlighted rows correspond to the proposed \texttt{AveRS} procedure. Across all these settings, \texttt{AveRS} consistently delivers substantially smaller mean absolute errors and much shorter confidence intervals than the last-iterate covariance-based baseline \texttt{LastBF}, while maintaining comparable empirical coverage. These complete results reinforce the main conclusion in Section~\ref{sec:5.1}: random scaling applied to the averaged GASN iterates provides more accurate and statistically efficient online inference than inference based on last iterates.

The table also illustrates the effect of the inner sketching accuracy. Compared with $\tau=5$, increasing the number of sketching steps to $\tau=10$ generally reduces the mean absolute error and shortens the confidence intervals for the sketched Newton methods. This trend is consistent with Proposition~\ref{prop:comp_stats_tradeoff}, since a larger sketching step $\tau$ gives a more accurate approximation to the Newton direction in~\eqref{equ:Newton_system} and moves the limiting covariance closer to the minimax-optimal covariance $\Omega^{\star}$. Overall, the appendix results show that the advantage of \texttt{AveRS} is robust across dimensions, models, covariance structures, and sketching distributions.

\input{tableAPP}

\input{figureAPP}

\vskip4pt

\noindent$\bullet$ \textbf{Additional covariance comparison.}
We provide additional QQ-plot comparisons for the logistic regression experiments, complementing the covariance comparison in Section~\ref{sec:5.1}. Figure~\ref{fig:qq_logistic} compares the averaged-iterate covariance structures of four stochastic Newton methods: EN, USN, ASN, and GASN, under both coordinate sketches and Gaussian sketches. As a representative configuration, we set $d=20$ with the equi-correlation design $r=0.4$, and use $\tau=5$ sketching steps for the three stochastic Newton methods.

The logistic regression results in Figure \ref{fig:qq_logistic} exhibit the same qualitative patterns as the linear regression results in Figure~\ref{fig:qq_linear}. In the comparisons against EN, the fitted QQ lines for USN, ASN, and GASN are steeper than the reference line $y=x$, indicating that the sketched Newton methods (USN, ASN, GASN) have larger limiting variances than Exact Newton. This agrees with the expected statistical efficiency cost introduced by sketch-and-project. In contrast, the QQ plots comparing USN, ASN, and GASN are closely aligned with the reference line, showing that Nesterov acceleration and the generalized metric choice do not materially increase the uncertainty beyond that already caused by sketch-and-project. These additional results confirm that the covariance behavior observed in~Section~\ref{sec:5.1} is robust across both linear and logistic models.

%% file: tableAPP.tex
\begin{table}[t]
\vskip-0.3cm
\centering
\setlength{\tabcolsep}{4pt}
\renewcommand{\arraystretch}{1.1}
\resizebox{1.0\linewidth}{!}{
\begin{tabular}{|c|c|c|c|c|ccc|ccc|ccc|ccc|}
\hline
\multirow{3}{*}{$d$} & \multirow{3}{*}{\shortstack{Design \\ Cov $\Sigma_a$}} & \multirow{3}{*}{\shortstack{Optimization \\ Algorithm}} & \multirow{3}{*}{$\tau$} & \multirow{3}{*}{\shortstack{Inference \\ Method}} & \multicolumn{3}{c|}{Linear Regression (Coordinate)} & \multicolumn{3}{c|}{Linear Regression (Gaussian)} & \multicolumn{3}{c|}{Logistic Regression (Coordinate)} & \multicolumn{3}{c|}{Logistic Regression (Gaussian)} \\
\cline{6-17}
& & & & & MAE & \multicolumn{1}{|c|}{Ave Cov} & Ave Len & MAE & \multicolumn{1}{|c|}{Ave Cov} & Ave Len & MAE & \multicolumn{1}{|c|}{Ave Cov} & Ave Len & MAE & \multicolumn{1}{|c|}{Ave Cov} & Ave Len \\[-4pt]
& & & & & {\scriptsize $(10^{-2})$} & \multicolumn{1}{|c|}{{\scriptsize $(\%)$}} & {\scriptsize $(10^{-2})$} & {\scriptsize $(10^{-2})$} & \multicolumn{1}{|c|}{{\scriptsize $(\%)$}} & {\scriptsize $(10^{-2})$} & {\scriptsize $(10^{-2})$} & \multicolumn{1}{|c|}{{\scriptsize $(\%)$}} & {\scriptsize $(10^{-2})$} & {\scriptsize $(10^{-2})$} & \multicolumn{1}{|c|}{{\scriptsize $(\%)$}} & {\scriptsize $(10^{-2})$} \\
\hline

\multirow{30}{*}{20} 
& \multirow{10}{*}{\shortstack{Identity}} & \multirow{2}{*}{EN} & \multirow{2}{*}{--} & \multicolumn{1}{|c|}{\texttt{LastBF}} & 17.89 & \multicolumn{1}{|c|}{96.50} & 3.54 & 17.71 & \multicolumn{1}{|c|}{96.00} & 3.55 & 2.95 & \multicolumn{1}{|c|}{95.50} & 0.42 & 2.96 & \multicolumn{1}{|c|}{94.00} & 0.42 \\
\cline{5-17}
& & & & \gmc{\texttt{AveRS}} & \Gthree{1.69}{95.00}{0.46} & \Gthree{1.72}{94.00}{0.44} & \Gthree{0.29}{96.50}{0.06} & \Gthree{0.29}{92.00}{0.05} \\
\cline{3-17}
& & \multirow{4}{*}{\shortstack{ASN \\ ($E_t=I$)}} & \multirow{2}{*}{5} & \multicolumn{1}{|c|}{\texttt{LastBF}} & 17.86 & \multicolumn{1}{|c|}{95.50} & 3.50 & 18.20 & \multicolumn{1}{|c|}{95.00} & 3.61 & 3.01 & \multicolumn{1}{|c|}{92.50} & 0.54 & 2.94 & \multicolumn{1}{|c|}{94.50} & 0.55 \\
\cline{5-17}
& & & & \gmc{\texttt{AveRS}} & \Gthree{3.62}{94.00}{0.89} & \Gthree{3.60}{92.50}{0.93} & \Gthree{0.62}{94.00}{0.15} & \Gthree{0.61}{91.50}{0.14} \\
\cline{4-17}
& & & \multirow{2}{*}{10} & \multicolumn{1}{|c|}{\texttt{LastBF}} & 17.98 & \multicolumn{1}{|c|}{94.00} & 3.54 & 18.68 & \multicolumn{1}{|c|}{98.00} & 3.76 & 3.02 & \multicolumn{1}{|c|}{92.50} & 0.52 & 3.11 & \multicolumn{1}{|c|}{97.00} & 0.54 \\
\cline{5-17}
& & & & \gmc{\texttt{AveRS}} & \Gthree{2.74}{92.50}{0.71} & \Gthree{2.65}{95.50}{0.70} & \Gthree{0.47}{94.00}{0.11} & \Gthree{0.46}{94.00}{0.11} \\
\cline{3-17}
& & \multirow{4}{*}{\shortstack{GASN \\ ($E_t=B_t$)}} & \multirow{2}{*}{5} & \multicolumn{1}{|c|}{\texttt{LastBF}} & 17.65 & \multicolumn{1}{|c|}{96.00} & 3.51 & 18.02 & \multicolumn{1}{|c|}{93.00} & 3.60 & 2.98 & \multicolumn{1}{|c|}{93.00} & 0.55 & 3.03 & \multicolumn{1}{|c|}{96.50} & 0.55 \\
\cline{5-17}
& & & & \gmc{\texttt{AveRS}} & \Gthree{3.61}{94.00}{0.94} & \Gthree{3.66}{94.00}{0.91} & \Gthree{0.61}{93.00}{0.15} & \Gthree{0.61}{95.50}{0.15} \\
\cline{4-17}
& & & \multirow{2}{*}{10} & \multicolumn{1}{|c|}{\texttt{LastBF}} & 17.78 & \multicolumn{1}{|c|}{93.00} & 3.52 & 18.71 & \multicolumn{1}{|c|}{93.00} & 3.71 & 2.95 & \multicolumn{1}{|c|}{95.50} & 0.53 & 3.17 & \multicolumn{1}{|c|}{93.50} & 0.55 \\
\cline{5-17}
& & & & \gmc{\texttt{AveRS}} & \Gthree{2.70}{94.50}{0.70} & \Gthree{2.69}{95.50}{0.69} & \Gthree{0.46}{95.50}{0.11} & \Gthree{0.46}{94.50}{0.11} \\
\cline{2-17}\cline{2-17}

& \multirow{10}{*}{\shortstack{Toeplitz\\$r=0.4$}} & \multirow{2}{*}{EN} & \multirow{2}{*}{--} & \multicolumn{1}{|c|}{\texttt{LastBF}} & 20.50 & \multicolumn{1}{|c|}{97.50} & 2.38 & 20.46 & \multicolumn{1}{|c|}{95.00} & 2.38 & 2.42 & \multicolumn{1}{|c|}{95.00} & 0.47 & 2.45 & \multicolumn{1}{|c|}{96.00} & 0.47 \\
\cline{5-17}
& & & & \gmc{\texttt{AveRS}} & \Gthree{1.91}{94.50}{0.30} & \Gthree{1.97}{96.50}{0.31} & \Gthree{0.24}{95.00}{0.06} & \Gthree{0.25}{96.50}{0.06} \\
\cline{3-17}
& & \multirow{4}{*}{\shortstack{ASN \\ ($E_t=I$)}} & \multirow{2}{*}{5} & \multicolumn{1}{|c|}{\texttt{LastBF}} & 15.89 & \multicolumn{1}{|c|}{95.00} & 3.05 & 16.73 & \multicolumn{1}{|c|}{95.50} & 3.05 & 2.50 & \multicolumn{1}{|c|}{94.50} & 0.48 & 2.57 & \multicolumn{1}{|c|}{93.50} & 0.49 \\
\cline{5-17}
& & & & \gmc{\texttt{AveRS}} & \Gthree{5.10}{93.50}{0.46} & \Gthree{5.12}{93.50}{0.51} & \Gthree{0.52}{91.00}{0.13} & \Gthree{0.52}{92.50}{0.13} \\
\cline{4-17}
& & & \multirow{2}{*}{10} & \multicolumn{1}{|c|}{\texttt{LastBF}} & 16.58 & \multicolumn{1}{|c|}{97.00} & 3.02 & 17.36 & \multicolumn{1}{|c|}{97.00} & 3.04 & 2.56 & \multicolumn{1}{|c|}{92.50} & 0.48 & 2.61 & \multicolumn{1}{|c|}{94.00} & 0.51 \\
\cline{5-17}
& & & & \gmc{\texttt{AveRS}} & \Gthree{3.73}{95.00}{0.39} & \Gthree{3.76}{95.00}{0.41} & \Gthree{0.39}{93.50}{0.10} & \Gthree{0.38}{96.50}{0.10} \\
\cline{3-17}
& & \multirow{4}{*}{\shortstack{GASN \\ ($E_t=B_t$)}} & \multirow{2}{*}{5} & \multicolumn{1}{|c|}{\texttt{LastBF}} & 20.46 & \multicolumn{1}{|c|}{94.00} & 2.50 & 21.35 & \multicolumn{1}{|c|}{96.00} & 2.51 & 2.49 & \multicolumn{1}{|c|}{95.50} & 0.48 & 2.55 & \multicolumn{1}{|c|}{94.00} & 0.50 \\
\cline{5-17}
& & & & \gmc{\texttt{AveRS}} & \Gthree{5.03}{95.00}{0.47} & \Gthree{4.92}{95.00}{0.48} & \Gthree{0.52}{95.00}{0.13} & \Gthree{0.51}{93.50}{0.13} \\
\cline{4-17}
& & & \multirow{2}{*}{10} & \multicolumn{1}{|c|}{\texttt{LastBF}} & 20.01 & \multicolumn{1}{|c|}{97.00} & 2.58 & 21.35 & \multicolumn{1}{|c|}{97.50} & 2.60 & 2.55 & \multicolumn{1}{|c|}{92.50} & 0.48 & 2.66 & \multicolumn{1}{|c|}{95.50} & 0.51 \\
\cline{5-17}
& & & & \gmc{\texttt{AveRS}} & \Gthree{3.66}{95.00}{0.42} & \Gthree{3.66}{95.00}{0.39} & \Gthree{0.38}{96.50}{0.10} & \Gthree{0.38}{97.00}{0.10} \\
\cline{2-17}\cline{2-17}

& \multirow{10}{*}{\shortstack{Equi-Corr\\$r=0.4$}} & \multirow{2}{*}{EN} & \multirow{2}{*}{--} & \multicolumn{1}{|c|}{\texttt{LastBF}} & 13.66 & \multicolumn{1}{|c|}{94.00} & 1.13 & 13.81 & \multicolumn{1}{|c|}{96.00} & 1.14 & 2.54 & \multicolumn{1}{|c|}{94.50} & 0.41 & 2.53 & \multicolumn{1}{|c|}{95.50} & 0.41 \\
\cline{5-17}
& & & & \gmc{\texttt{AveRS}} & \Gthree{1.32}{95.50}{0.14} & \Gthree{1.30}{96.50}{0.15} & \Gthree{0.25}{96.00}{0.05} & \Gthree{0.24}{95.00}{0.05} \\
\cline{3-17}
& & \multirow{4}{*}{\shortstack{ASN \\ ($E_t=I$)}} & \multirow{2}{*}{5} & \multicolumn{1}{|c|}{\texttt{LastBF}} & 10.50 & \multicolumn{1}{|c|}{94.50} & 1.74 & 12.82 & \multicolumn{1}{|c|}{95.50} & 1.77 & 2.48 & \multicolumn{1}{|c|}{95.50} & 0.47 & 2.60 & \multicolumn{1}{|c|}{95.50} & 0.48 \\
\cline{5-17}
& & & & \gmc{\texttt{AveRS}} & \Gthree{3.22}{96.50}{0.22} & \Gthree{3.18}{94.00}{0.23} & \Gthree{0.50}{92.50}{0.12} & \Gthree{0.50}{94.00}{0.12} \\
\cline{4-17}
& & & \multirow{2}{*}{10} & \multicolumn{1}{|c|}{\texttt{LastBF}} & 10.71 & \multicolumn{1}{|c|}{94.00} & 1.69 & 13.17 & \multicolumn{1}{|c|}{95.50} & 1.69 & 2.49 & \multicolumn{1}{|c|}{95.50} & 0.46 & 2.58 & \multicolumn{1}{|c|}{94.50} & 0.48 \\
\cline{5-17}
& & & & \gmc{\texttt{AveRS}} & \Gthree{2.40}{94.50}{0.22} & \Gthree{2.35}{96.50}{0.22} & \Gthree{0.39}{97.50}{0.09} & \Gthree{0.38}{97.50}{0.10} \\
\cline{3-17}
& & \multirow{4}{*}{\shortstack{GASN \\ ($E_t=B_t$)}} & \multirow{2}{*}{5} & \multicolumn{1}{|c|}{\texttt{LastBF}} & 13.31 & \multicolumn{1}{|c|}{96.00} & 1.26 & 13.78 & \multicolumn{1}{|c|}{96.00} & 1.25 & 2.55 & \multicolumn{1}{|c|}{92.50} & 0.47 & 2.62 & \multicolumn{1}{|c|}{96.50} & 0.48 \\
\cline{5-17}
& & & & \gmc{\texttt{AveRS}} & \Gthree{3.03}{95.00}{0.19} & \Gthree{2.91}{97.00}{0.19} & \Gthree{0.50}{92.50}{0.13} & \Gthree{0.50}{94.50}{0.13} \\
\cline{4-17}
& & & \multirow{2}{*}{10} & \multicolumn{1}{|c|}{\texttt{LastBF}} & 13.39 & \multicolumn{1}{|c|}{95.50} & 1.28 & 14.23 & \multicolumn{1}{|c|}{93.50} & 1.30 & 2.54 & \multicolumn{1}{|c|}{94.50} & 0.46 & 2.63 & \multicolumn{1}{|c|}{95.00} & 0.49 \\
\cline{5-17}
& & & & \gmc{\texttt{AveRS}} & \Gthree{2.19}{95.00}{0.16} & \Gthree{2.21}{96.00}{0.17} & \Gthree{0.39}{97.00}{0.09} & \Gthree{0.39}{98.00}{0.09} \\
\hline

\multirow{30}{*}{40} 
& \multirow{10}{*}{\shortstack{Identity}} & \multirow{2}{*}{EN} & \multirow{2}{*}{--} & \multicolumn{1}{|c|}{\texttt{LastBF}} & 25.79 & \multicolumn{1}{|c|}{95.50} & 2.56 & 25.69 & \multicolumn{1}{|c|}{96.00} & 2.57 & 3.75 & \multicolumn{1}{|c|}{96.50} & 0.24 & 3.78 & \multicolumn{1}{|c|}{98.50} & 0.24 \\
\cline{5-17}
& & & & \gmc{\texttt{AveRS}} & \Gthree{2.48}{98.00}{0.33} & \Gthree{2.45}{95.50}{0.33} & \Gthree{0.37}{95.00}{0.03} & \Gthree{0.37}{94.50}{0.03} \\
\cline{3-17}
& & \multirow{4}{*}{\shortstack{ASN \\ ($E_t=I$)}} & \multirow{2}{*}{5} & \multicolumn{1}{|c|}{\texttt{LastBF}} & 26.04 & \multicolumn{1}{|c|}{92.50} & 2.46 & 26.61 & \multicolumn{1}{|c|}{94.50} & 2.58 & 3.75 & \multicolumn{1}{|c|}{93.00} & 0.35 & 3.91 & \multicolumn{1}{|c|}{93.50} & 0.35 \\
\cline{5-17}
& & & & \gmc{\texttt{AveRS}} & \Gthree{7.16}{91.50}{0.82} & \Gthree{7.31}{91.50}{0.88} & \Gthree{1.07}{92.00}{0.13} & \Gthree{1.07}{94.00}{0.12} \\
\cline{4-17}
& & & \multirow{2}{*}{10} & \multicolumn{1}{|c|}{\texttt{LastBF}} & 25.65 & \multicolumn{1}{|c|}{93.50} & 2.54 & 27.79 & \multicolumn{1}{|c|}{97.00} & 2.72 & 3.71 & \multicolumn{1}{|c|}{93.00} & 0.34 & 4.03 & \multicolumn{1}{|c|}{94.00} & 0.36 \\
\cline{5-17}
& & & & \gmc{\texttt{AveRS}} & \Gthree{5.21}{92.50}{0.64} & \Gthree{5.23}{94.00}{0.66} & \Gthree{0.78}{94.00}{0.09} & \Gthree{0.79}{90.50}{0.09} \\
\cline{3-17}
& & \multirow{4}{*}{\shortstack{GASN \\ ($E_t=B_t$)}} & \multirow{2}{*}{5} & \multicolumn{1}{|c|}{\texttt{LastBF}} & 25.73 & \multicolumn{1}{|c|}{93.50} & 2.49 & 27.17 & \multicolumn{1}{|c|}{94.00} & 2.57 & 3.81 & \multicolumn{1}{|c|}{93.50} & 0.35 & 3.88 & \multicolumn{1}{|c|}{94.50} & 0.36 \\
\cline{5-17}
& & & & \gmc{\texttt{AveRS}} & \Gthree{7.19}{95.00}{0.84} & \Gthree{7.13}{91.50}{0.86} & \Gthree{1.05}{92.00}{0.12} & \Gthree{1.05}{95.50}{0.12} \\
\cline{4-17}
& & & \multirow{2}{*}{10} & \multicolumn{1}{|c|}{\texttt{LastBF}} & 25.82 & \multicolumn{1}{|c|}{94.00} & 2.52 & 27.42 & \multicolumn{1}{|c|}{91.00} & 2.73 & 3.72 & \multicolumn{1}{|c|}{95.00} & 0.34 & 4.05 & \multicolumn{1}{|c|}{94.00} & 0.36 \\
\cline{5-17}
& & & & \gmc{\texttt{AveRS}} & \Gthree{5.18}{95.50}{0.70} & \Gthree{5.33}{90.50}{0.66} & \Gthree{0.77}{92.00}{0.09} & \Gthree{0.79}{94.00}{0.09} \\
\cline{2-17}\cline{2-17}

& \multirow{10}{*}{\shortstack{Toeplitz\\$r=0.4$}} & \multirow{2}{*}{EN} & \multirow{2}{*}{--} & \multicolumn{1}{|c|}{\texttt{LastBF}} & 29.99 & \multicolumn{1}{|c|}{95.00} & 1.70 & 30.17 & \multicolumn{1}{|c|}{95.50} & 1.70 & 3.12 & \multicolumn{1}{|c|}{94.00} & 0.27 & 3.09 & \multicolumn{1}{|c|}{95.00} & 0.27 \\
\cline{5-17}
& & & & \gmc{\texttt{AveRS}} & \Gthree{2.90}{96.00}{0.22} & \Gthree{2.90}{93.00}{0.23} & \Gthree{0.30}{96.00}{0.04} & \Gthree{0.30}{94.50}{0.03} \\
\cline{3-17}
& & \multirow{4}{*}{\shortstack{ASN \\ ($E_t=I$)}} & \multirow{2}{*}{5} & \multicolumn{1}{|c|}{\texttt{LastBF}} & 22.24 & \multicolumn{1}{|c|}{96.50} & 2.17 & 23.77 & \multicolumn{1}{|c|}{96.00} & 2.18 & 3.11 & \multicolumn{1}{|c|}{92.00} & 0.29 & 3.24 & \multicolumn{1}{|c|}{96.50} & 0.30 \\
\cline{5-17}
& & & & \gmc{\texttt{AveRS}} & \Gthree{11.33}{93.00}{0.42} & \Gthree{11.31}{92.00}{0.46} & \Gthree{0.87}{95.00}{0.11} & \Gthree{0.87}{94.50}{0.10} \\
\cline{4-17}
& & & \multirow{2}{*}{10} & \multicolumn{1}{|c|}{\texttt{LastBF}} & 22.72 & \multicolumn{1}{|c|}{96.50} & 2.16 & 24.59 & \multicolumn{1}{|c|}{95.00} & 2.24 & 3.12 & \multicolumn{1}{|c|}{96.00} & 0.29 & 3.37 & \multicolumn{1}{|c|}{95.50} & 0.32 \\
\cline{5-17}
& & & & \gmc{\texttt{AveRS}} & \Gthree{7.43}{96.00}{0.33} & \Gthree{7.52}{95.00}{0.35} & \Gthree{0.64}{91.50}{0.08} & \Gthree{0.64}{92.00}{0.08} \\
\cline{3-17}
& & \multirow{4}{*}{\shortstack{GASN \\ ($E_t=B_t$)}} & \multirow{2}{*}{5} & \multicolumn{1}{|c|}{\texttt{LastBF}} & 29.89 & \multicolumn{1}{|c|}{97.00} & 1.73 & 31.39 & \multicolumn{1}{|c|}{93.50} & 1.75 & 3.15 & \multicolumn{1}{|c|}{95.00} & 0.30 & 3.27 & \multicolumn{1}{|c|}{94.50} & 0.30 \\
\cline{5-17}
& & & & \gmc{\texttt{AveRS}} & \Gthree{10.67}{96.00}{0.41} & \Gthree{10.45}{91.00}{0.42} & \Gthree{0.89}{93.00}{0.11} & \Gthree{0.87}{93.50}{0.11} \\
\cline{4-17}
& & & \multirow{2}{*}{10} & \multicolumn{1}{|c|}{\texttt{LastBF}} & 30.03 & \multicolumn{1}{|c|}{95.00} & 1.78 & 31.73 & \multicolumn{1}{|c|}{97.00} & 1.86 & 3.13 & \multicolumn{1}{|c|}{95.50} & 0.29 & 3.31 & \multicolumn{1}{|c|}{94.50} & 0.32 \\
\cline{5-17}
& & & & \gmc{\texttt{AveRS}} & \Gthree{7.42}{94.00}{0.33} & \Gthree{7.59}{97.00}{0.33} & \Gthree{0.65}{95.00}{0.08} & \Gthree{0.65}{94.50}{0.08} \\
\cline{2-17}\cline{2-17}

& \multirow{10}{*}{\shortstack{Equi-Corr\\$r=0.4$}} & \multirow{2}{*}{EN} & \multirow{2}{*}{--} & \multicolumn{1}{|c|}{\texttt{LastBF}} & 20.13 & \multicolumn{1}{|c|}{95.00} & 0.61 & 20.08 & \multicolumn{1}{|c|}{94.00} & 0.61 & 2.66 & \multicolumn{1}{|c|}{95.00} & 0.19 & 2.63 & \multicolumn{1}{|c|}{94.50} & 0.19 \\
\cline{5-17}
& & & & \gmc{\texttt{AveRS}} & \Gthree{1.93}{95.00}{0.08} & \Gthree{1.94}{95.00}{0.08} & \Gthree{0.26}{94.00}{0.02} & \Gthree{0.25}{97.00}{0.03} \\
\cline{3-17}
& & \multirow{4}{*}{\shortstack{ASN \\ ($E_t=I$)}} & \multirow{2}{*}{5} & \multicolumn{1}{|c|}{\texttt{LastBF}} & 12.63 & \multicolumn{1}{|c|}{95.00} & 1.04 & 17.69 & \multicolumn{1}{|c|}{95.00} & 1.12 & 2.69 & \multicolumn{1}{|c|}{93.00} & 0.24 & 2.75 & \multicolumn{1}{|c|}{92.00} & 0.25 \\
\cline{5-17}
& & & & \gmc{\texttt{AveRS}} & \Gthree{6.61}{97.00}{0.13} & \Gthree{6.85}{95.00}{0.15} & \Gthree{0.73}{94.50}{0.09} & \Gthree{0.74}{92.50}{0.08} \\
\cline{4-17}
& & & \multirow{2}{*}{10} & \multicolumn{1}{|c|}{\texttt{LastBF}} & 13.20 & \multicolumn{1}{|c|}{97.00} & 1.02 & 18.86 & \multicolumn{1}{|c|}{96.00} & 1.09 & 2.66 & \multicolumn{1}{|c|}{94.50} & 0.24 & 2.83 & \multicolumn{1}{|c|}{93.50} & 0.26 \\
\cline{5-17}
& & & & \gmc{\texttt{AveRS}} & \Gthree{4.88}{96.50}{0.13} & \Gthree{4.98}{96.50}{0.13} & \Gthree{0.55}{94.00}{0.07} & \Gthree{0.54}{95.50}{0.07} \\
\cline{3-17}
& & \multirow{4}{*}{\shortstack{GASN \\ ($E_t=B_t$)}} & \multirow{2}{*}{5} & \multicolumn{1}{|c|}{\texttt{LastBF}} & 19.77 & \multicolumn{1}{|c|}{94.00} & 0.68 & 20.31 & \multicolumn{1}{|c|}{97.50} & 0.67 & 2.64 & \multicolumn{1}{|c|}{92.00} & 0.25 & 2.73 & \multicolumn{1}{|c|}{94.00} & 0.25 \\
\cline{5-17}
& & & & \gmc{\texttt{AveRS}} & \Gthree{5.94}{96.00}{0.10} & \Gthree{5.98}{93.00}{0.10} & \Gthree{0.74}{93.00}{0.09} & \Gthree{0.74}{92.50}{0.09} \\
\cline{4-17}
& & & \multirow{2}{*}{10} & \multicolumn{1}{|c|}{\texttt{LastBF}} & 19.32 & \multicolumn{1}{|c|}{94.00} & 0.70 & 21.49 & \multicolumn{1}{|c|}{93.50} & 0.71 & 2.64 & \multicolumn{1}{|c|}{96.00} & 0.24 & 2.86 & \multicolumn{1}{|c|}{93.00} & 0.26 \\
\cline{5-17}
& & & & \gmc{\texttt{AveRS}} & \Gthree{4.37}{92.50}{0.09} & \Gthree{4.34}{96.50}{0.09} & \Gthree{0.54}{92.50}{0.06} & \Gthree{0.54}{94.50}{0.07} \\
\hline
\end{tabular}
}
\vspace{0.1cm}
\caption{\textit{Evaluation results for the inference procedures under different parameter settings. For Exact Newton (EN), sketching steps are not applicable, so we denote it by ``--''.}}
\label{tab:d20_d40_complete_results}
\vspace{-0.3cm}
\end{table}

%% file: figureAPP.tex
\begin{figure}[t!]
\centering     
\subfigure{\includegraphics[width=0.16\textwidth]{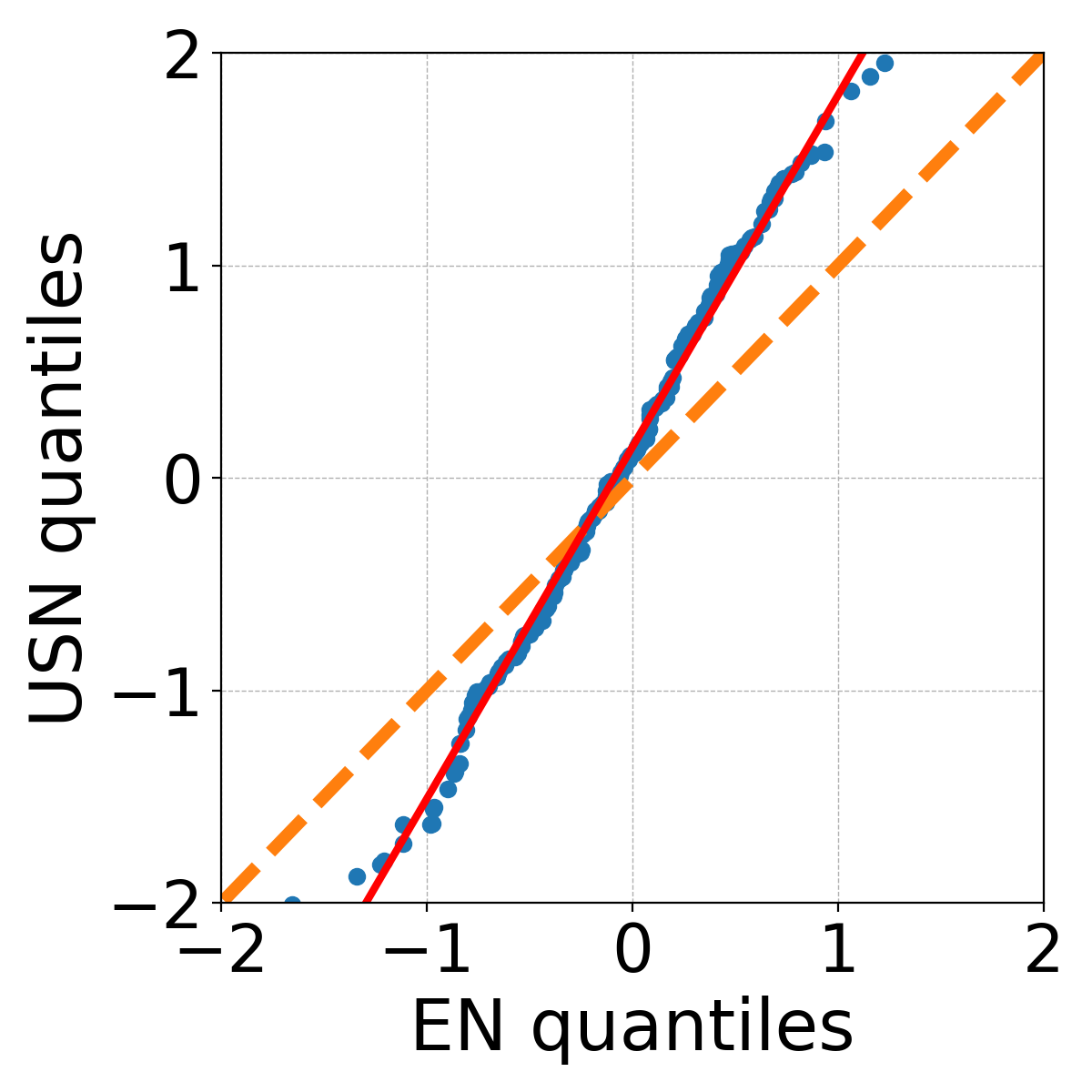}}
\subfigure{\includegraphics[width=0.16\textwidth]{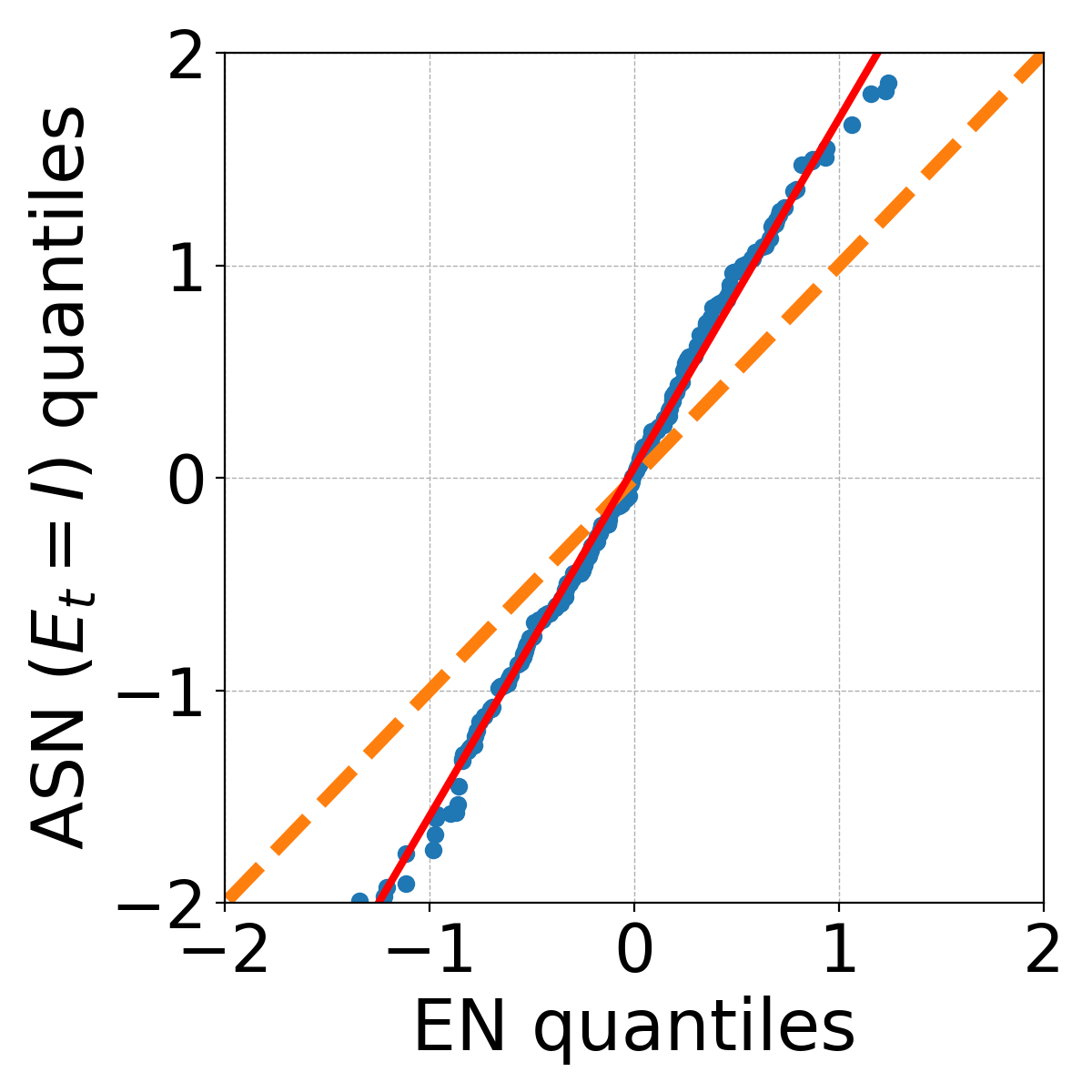}}
\subfigure{\includegraphics[width=0.16\textwidth]{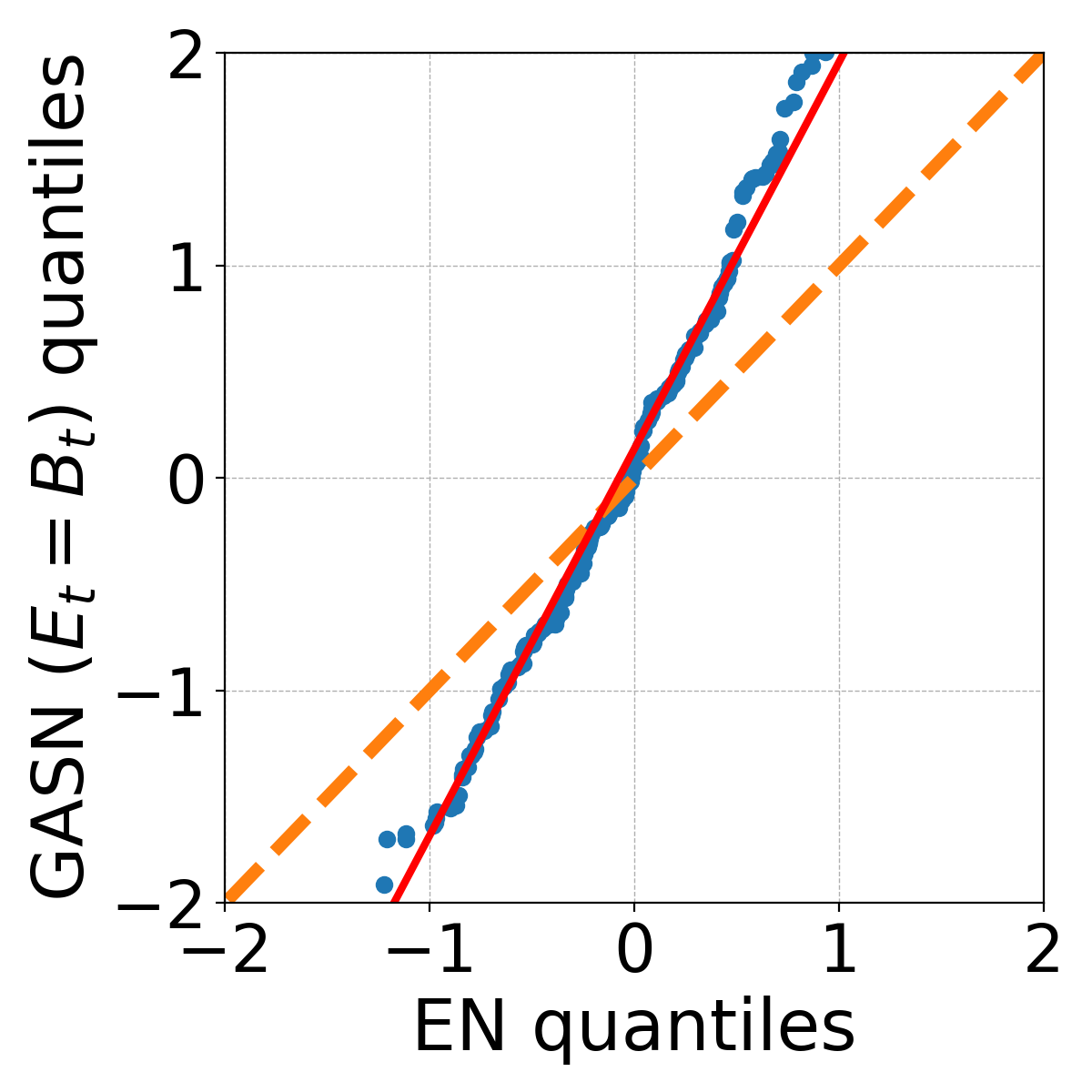}}
\subfigure{\includegraphics[width=0.16\textwidth]{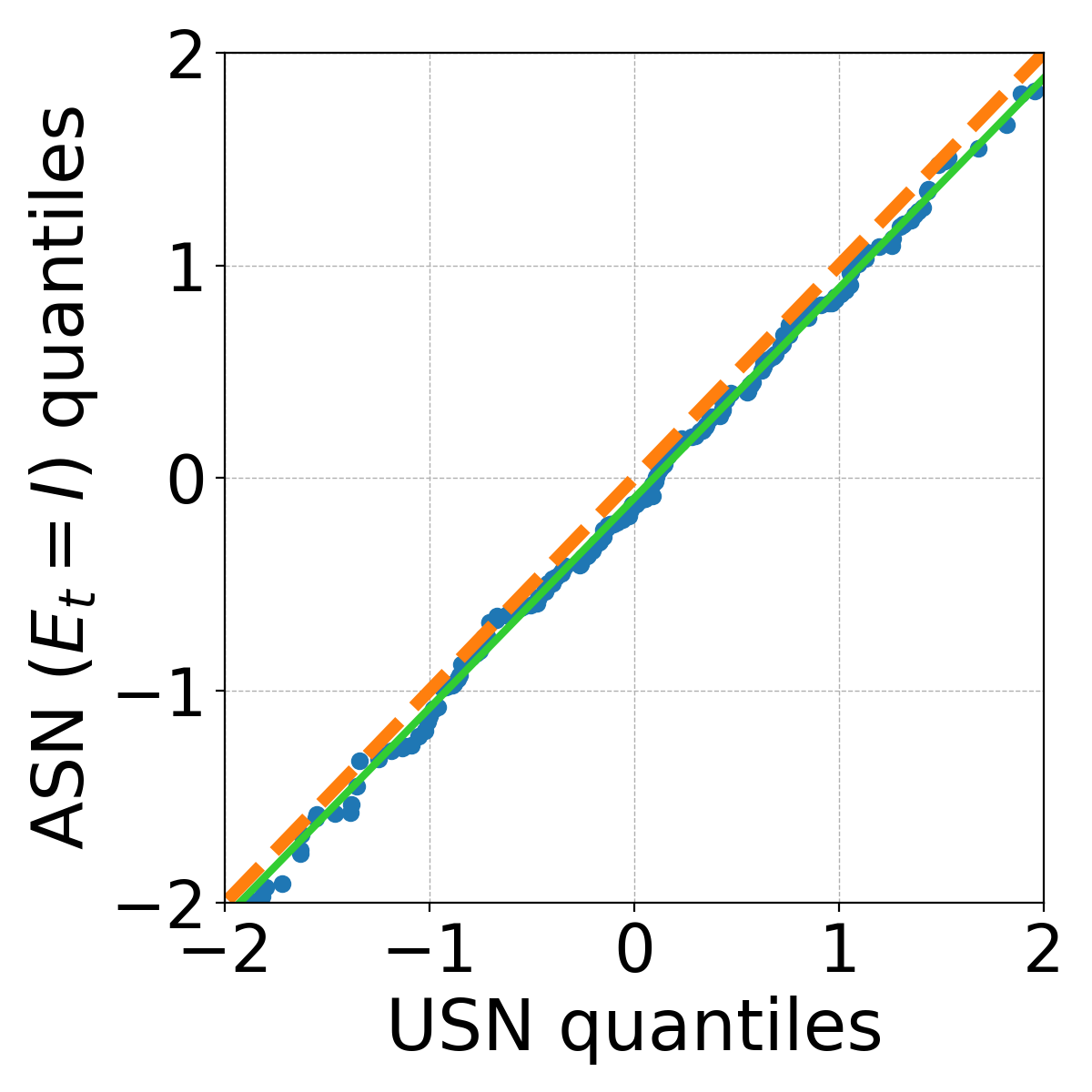}}
\subfigure{\includegraphics[width=0.16\textwidth]{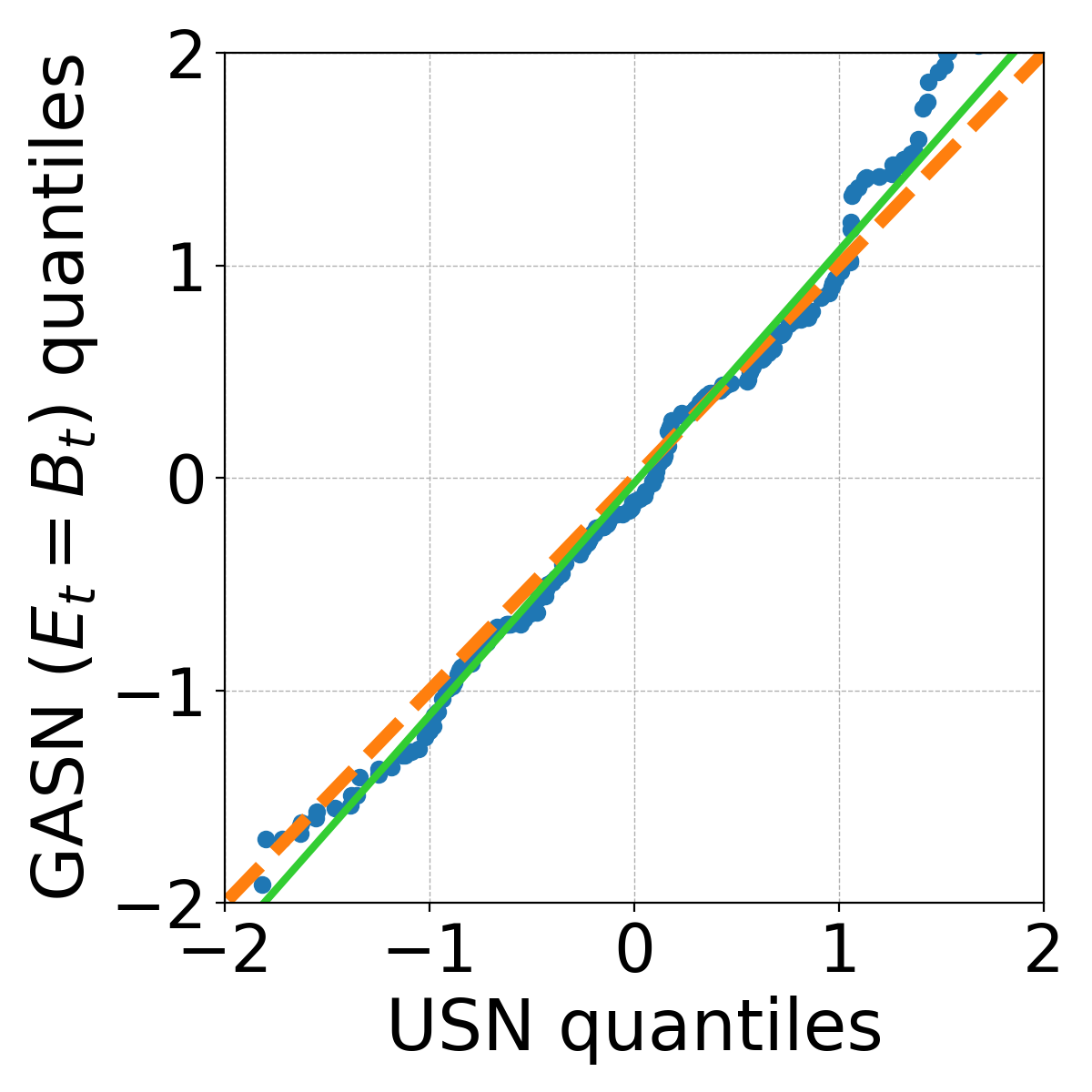}}
\subfigure{\includegraphics[width=0.16\textwidth]{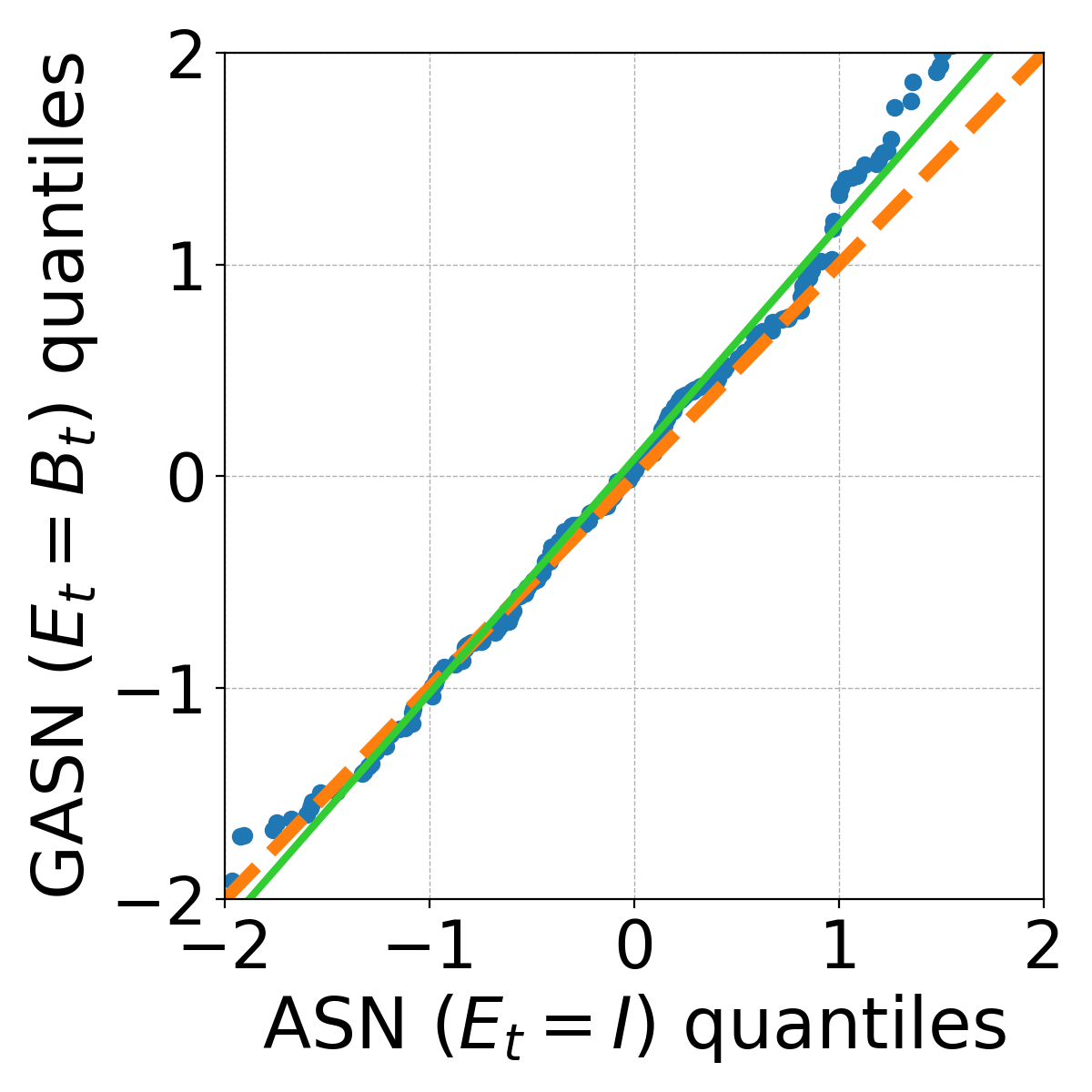}}
\centering{\textbf{Logistic Regression + Coordinate Sketches}}

\subfigure{\includegraphics[width=0.16\textwidth]{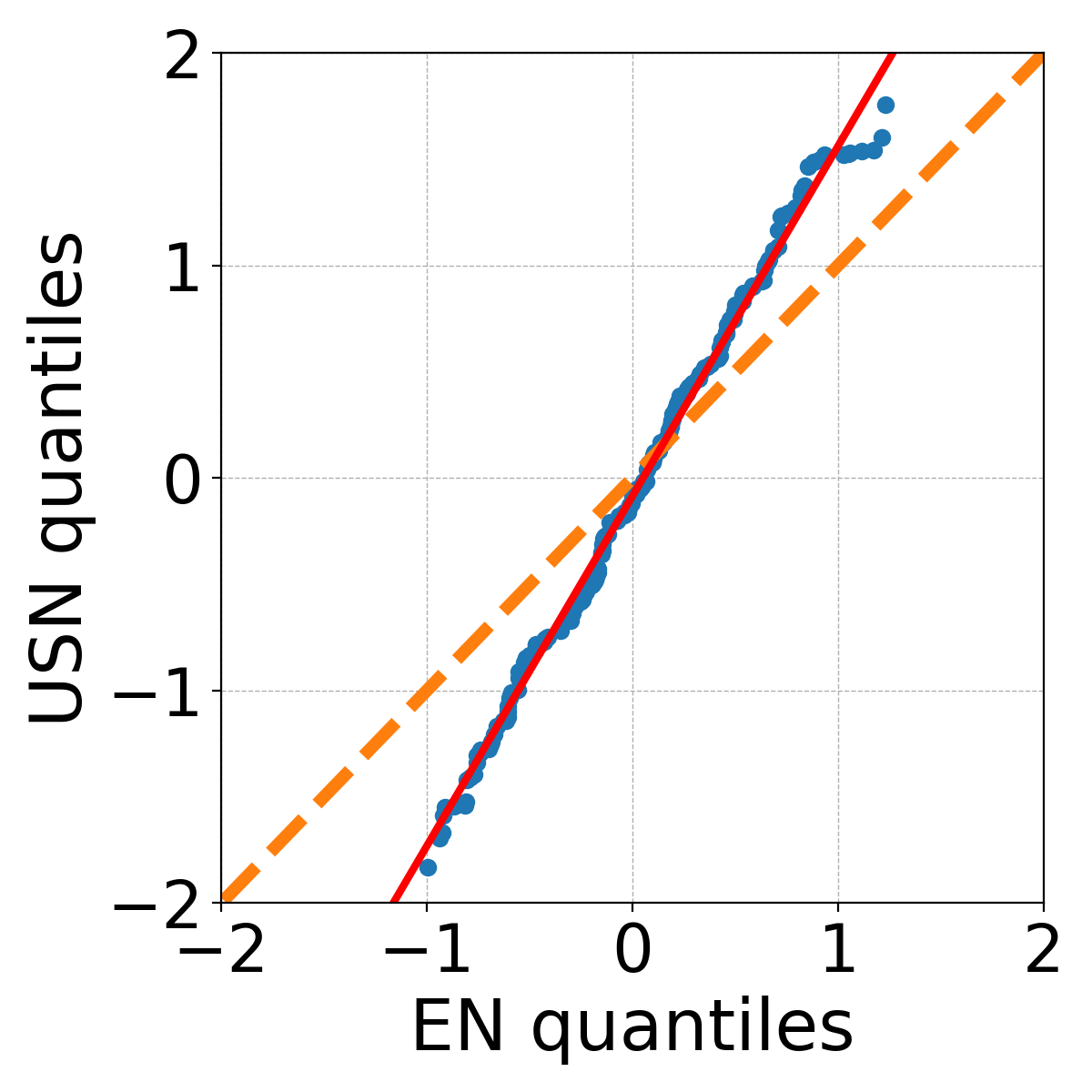}}
\subfigure{\includegraphics[width=0.16\textwidth]{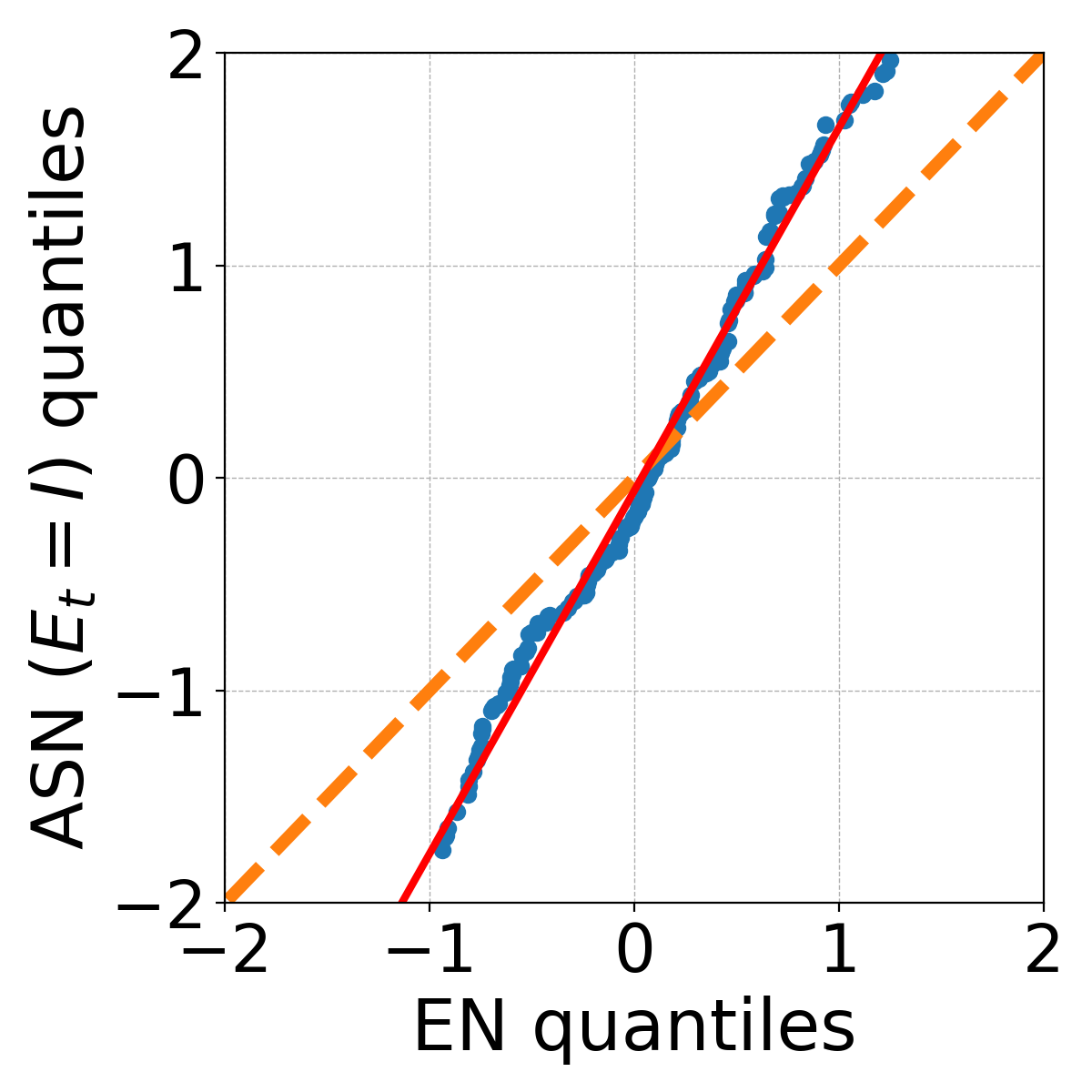}}
\subfigure{\includegraphics[width=0.16\textwidth]{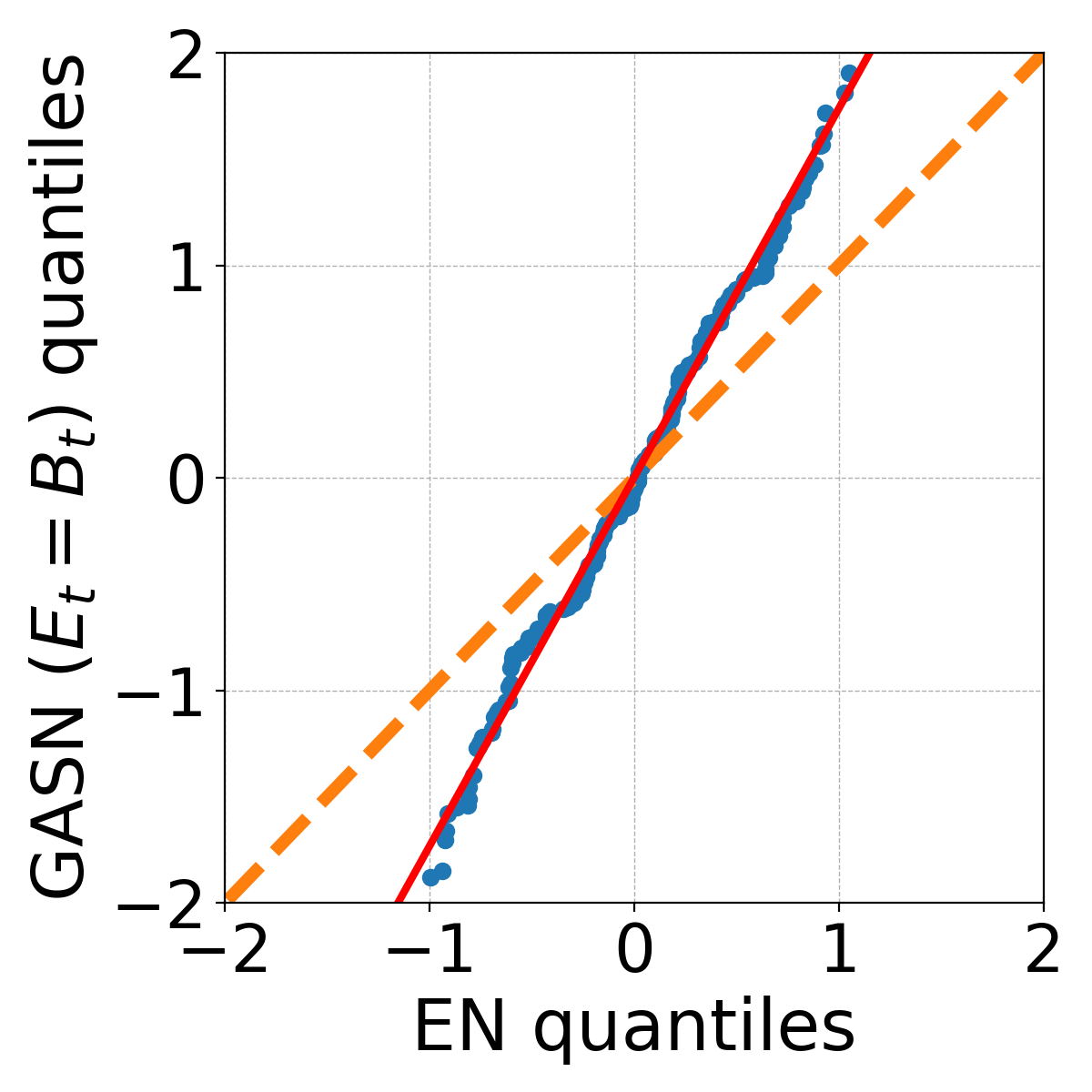}}
\subfigure{\includegraphics[width=0.16\textwidth]{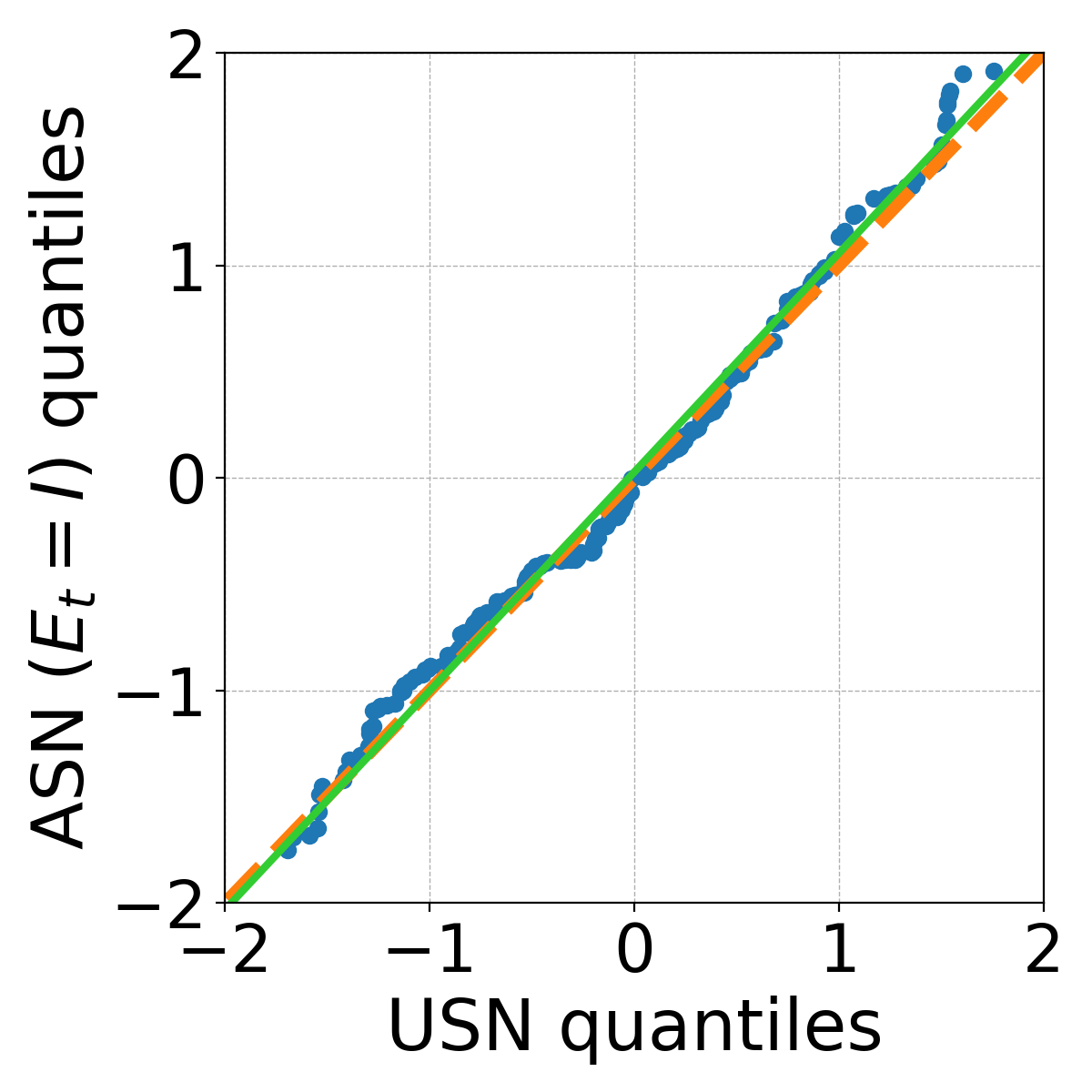}}
\subfigure{\includegraphics[width=0.16\textwidth]{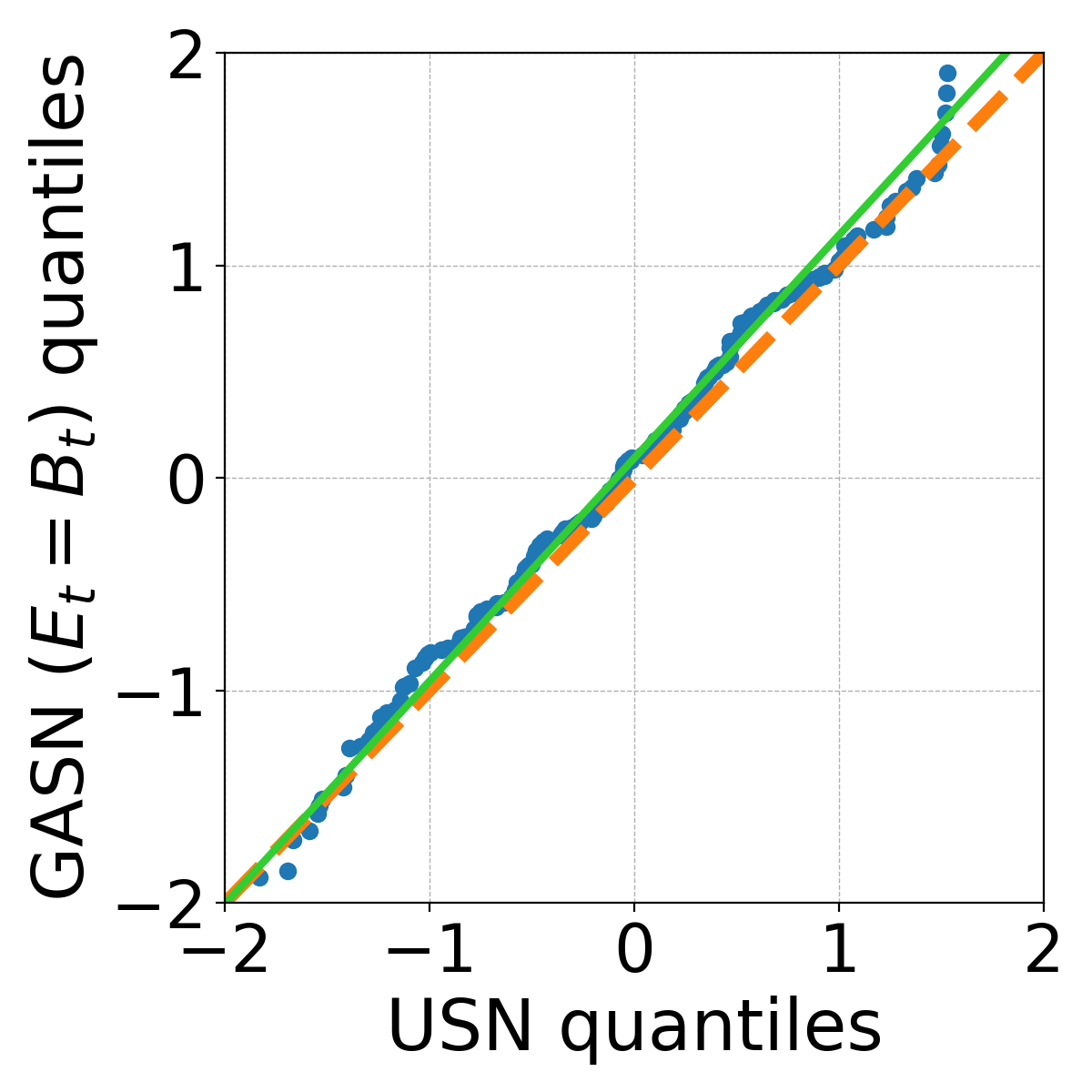}}
\subfigure{\includegraphics[width=0.16\textwidth]{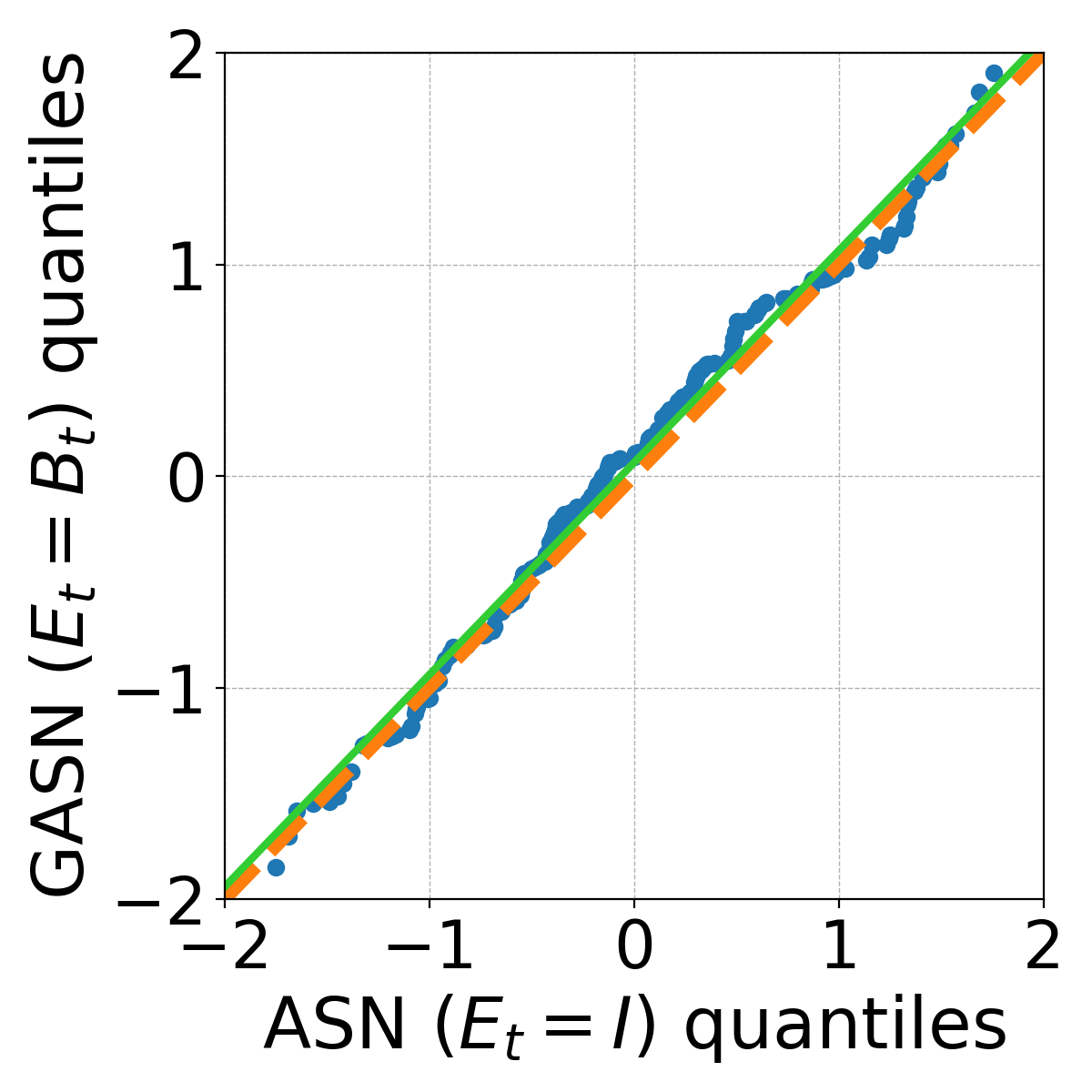}}
\centering{\textbf{Logistic Regression + Gaussian Sketches}}
\vskip5pt
\caption{\textit{QQ plots for four stochastic Newton methods. The orange line denotes the reference line $y = x$. The red line denotes the fitted line through the blue data points.}}\label{fig:qq_logistic}
\end{figure}

%% file: ref.bib
@Article{Abadir1997Two,
  author    = {Abadir, Karim M. and Paruolo, Paolo},
  journal   = {Econometrica},
  title     = {Two Mixed Normal Densities from Cointegration Analysis},
  year      = {1997},
  issn      = {0012-9682},
  month     = may,
  number    = {3},
  pages     = {671--680},
  volume    = {65},
  doi       = {10.2307/2171758},
  publisher = {JSTOR},
}

@Article{Bottou2018Optimization,
  author     = {L{\'{e}}on Bottou and Frank E. Curtis and Jorge Nocedal},
  journal    = {{SIAM} Review},
  title      = {Optimization Methods for Large-Scale Machine Learning},
  year       = {2018},
  issn       = {0036-1445},
  month      = {jan},
  number     = {2},
  pages      = {223--311},
  volume     = {60},
  doi        = {10.1137/16m1080173},
  fjournal   = {SIAM Review},
  mrclass    = {65K05 (68Q25 68T05 90C06 90C30 90C90)},
  mrnumber   = {3797719},
  mrreviewer = {Boualem Alleche},
  publisher  = {Society for Industrial {\&} Applied Mathematics ({SIAM})},
}

@Article{Bercu2020Efficient,
  author    = {Bernard Bercu and Antoine Godichon and Bruno Portier},
  journal   = {{SIAM} Journal on Control and Optimization},
  title     = {An Efficient Stochastic {N}ewton Algorithm for Parameter Estimation in Logistic Regressions},
  year      = {2020},
  month     = {jan},
  number    = {1},
  pages     = {348--367},
  volume    = {58},
  doi       = {10.1137/19m1261717},
  publisher = {Society for Industrial {\&} Applied Mathematics ({SIAM})},
}

@Article{Davis2024Asymptotic,
  author    = {Davis, Damek and Drusvyatskiy, Dmitriy and Jiang, Liwei},
  journal   = {The Annals of Statistics},
  title     = {Asymptotic normality and optimality in nonsmooth stochastic approximation},
  year      = {2024},
  issn      = {0090-5364},
  month     = aug,
  number    = {4},
  volume    = {52},
  doi       = {10.1214/24-aos2401},
  publisher = {Institute of Mathematical Statistics},
  pages     = {1485--1508},
}

@Article{Na2022Hessian,
  author    = {Sen Na and Micha{\l} Derezi{\'{n}}ski and Michael W. Mahoney},
  journal   = {Mathematical Programming},
  title     = {Hessian averaging in stochastic {N}ewton methods achieves superlinear convergence},
  year      = {2023},
  month     = {sep},
  doi       = {10.1007/s10107-022-01913-5},
  publisher = {Springer Science and Business Media {LLC}},
  volume    = {201},
  pages     = {473--520},
}

@Article{Toulis2017Asymptotic,
  author    = {Panos Toulis and Edoardo M. Airoldi},
  journal   = {The Annals of Statistics},
  title     = {Asymptotic and finite-sample properties of estimators based on stochastic gradients},
  year      = {2017},
  month     = {aug},
  number    = {4},
  pages     = {1694--1727},
  volume    = {45},
  doi       = {10.1214/16-aos1506},
  publisher = {Institute of Mathematical Statistics},
}

@Article{Liang2019Statistical,
  author    = {Tengyuan Liang and Weijie J. Su},
  journal   = {Journal of the Royal Statistical Society: Series B (Statistical Methodology)},
  title     = {Statistical inference for the population landscape via moment-adjusted stochastic gradients},
  year      = {2019},
  month     = {feb},
  number    = {2},
  pages     = {431--456},
  volume    = {81},
  doi       = {10.1111/rssb.12313},
  publisher = {Wiley},
}

@TechReport{Ruppert1988Efficient,
  author = {Ruppert, David},
  title     = {Efficient estimations from a slowly convergent {Robbins-Monro} process},
  year   = {1988},
  institution = {Cornell University Operations Research and Industrial Engineering},
}

@Article{Duchi2021Asymptotic,
  author    = {Duchi, John C. and Ruan, Feng},
  journal   = {The Annals of Statistics},
  title     = {Asymptotic optimality in stochastic optimization},
  year      = {2021},
  issn      = {0090-5364},
  month     = feb,
  number    = {1},
  volume    = {49},
  doi       = {10.1214/19-aos1831},
  publisher = {Institute of Mathematical Statistics},
  pages     = {21--48},
}

@Article{Karniadakis2021Physics,
  author    = {George Em Karniadakis and Ioannis G. Kevrekidis and Lu Lu and Paris Perdikaris and Sifan Wang and Liu Yang},
  journal   = {Nature Reviews Physics},
  title     = {Physics-informed machine learning},
  year      = {2021},
  month     = {may},
  number    = {6},
  pages     = {422--440},
  volume    = {3},
  doi       = {10.1038/s42254-021-00314-5},
  publisher = {Springer Science and Business Media {LLC}},
}

@Article{Du2022High,
  author    = {Du, Jin-Hong and Guo, Yifeng and Wang, Xueqin},
  journal   = {Journal of the American Statistical Association},
  title     = {High-Dimensional Portfolio Selection with Cardinality Constraints},
  year      = {2023},
  issn      = {1537-274X},
  number    = {542},
  pages     = {779--791},
  volume    = {118},
  doi       = {10.1080/01621459.2022.2133718},
  publisher = {Informa UK Limited},
}

@Article{Abadir2002Simple,
  author    = {Abadir, Karim M. and Paruolo, Paolo},
  journal   = {Econometrica},
  title     = {Simple Robust Testing of Regression Hypotheses: A Comment},
  year      = {2002},
  issn      = {1468-0262},
  month     = sep,
  number    = {5},
  pages     = {2097--2099},
  volume    = {70},
  doi       = {10.1111/1468-0262.00367},
  publisher = {The Econometric Society},
}

@Article{Chen2020Statistical,
  author    = {Chen, Xi and Lee, Jason D. and Tong, Xin T. and Zhang, Yichen},
  journal   = {The Annals of Statistics},
  title     = {Statistical inference for model parameters in stochastic gradient descent},
  year      = {2020},
  issn      = {0090-5364},
  month     = feb,
  number    = {1},
  volume    = {48},
  doi       = {10.1214/18-aos1801},
  publisher = {Institute of Mathematical Statistics},
  pages     = {251--273},
}

@Book{Duflo2013Random,
  author    = {Duflo, Marie},
  publisher = {Springer Science \& Business Media},
  title     = {Random Iterative Models},
  year      = {2013},
  address   = {Berlin New York},
  isbn      = {9783540571001},
  volume    = {34},
}

@Article{Gower2015Randomized,
  author    = {Gower, Robert M. and Richtárik, Peter},
  journal   = {SIAM Journal on Matrix Analysis and Applications},
  title     = {Randomized Iterative Methods for Linear Systems},
  year      = {2015},
  issn      = {1095-7162},
  month     = jan,
  number    = {4},
  pages     = {1660--1690},
  volume    = {36},
  doi       = {10.1137/15m1025487},
  publisher = {Society for Industrial & Applied Mathematics (SIAM)},
}

@Book{Khalil2002Nonlinear,
  author    = {Khalil, H.K.},
  publisher = {Prentice Hall},
  title     = {Nonlinear Systems},
  year      = {2002},
  isbn      = {9780130673893},
  series    = {Pearson Education},
  lccn      = {95045804},
}

@Article{Kiefer1952Stochastic,
  author    = {Kiefer, J. and Wolfowitz, J.},
  journal   = {The Annals of Mathematical Statistics},
  title     = {Stochastic Estimation of the Maximum of a Regression Function},
  year      = {1952},
  issn      = {0003-4851},
  month     = sep,
  number    = {3},
  pages     = {462--466},
  volume    = {23},
  doi       = {10.1214/aoms/1177729392},
  publisher = {Institute of Mathematical Statistics},
}

@Article{Kiefer2000Simple,
  author    = {Nicholas M. Kiefer and Timothy J. Vogelsang and Helle Bunzel},
  journal   = {Econometrica},
  title     = {Simple Robust Testing of Regression Hypotheses},
  year      = {2000},
  issn      = {00129682, 14680262},
  number    = {3},
  pages     = {695--714},
  volume    = {68},
  publisher = {[Wiley, Econometric Society]},
  urldate   = {2025-02-01},
}

@InProceedings{Lee2022Fast,
  author    = {Lee, Sokbae and Liao, Yuan and Seo, Myung Hwan and Shin, Youngki},
  booktitle = {Proceedings of the AAAI Conference on Artificial Intelligence},
  title     = {Fast and robust online inference with stochastic gradient descent via random scaling},
  year      = {2022},
  pages     = {7381--7389},
  volume    = {36},
  doi       = {10.1609/aaai.v36i7.20701},
}

@Article{Li2023Online,
  author  = {Li, Xiang and Liang, Jiadong and Zhang, Zhihua},
  journal = {arXiv preprint arXiv:2302.07690},
  title     = {Online statistical inference for nonlinear stochastic approximation with {Markovian} data},
  year    = {2023},
}

@Article{Luo2022Covariance,
  author  = {Luo, Yiling and Huo, Xiaoming and Mei, Yajun},
  journal = {arXiv preprint arXiv:2212.01259},
  title     = {Covariance estimators for the {ROOT-SGD} algorithm in online learning},
  year    = {2022},
}

@Article{Polyak1992Acceleration,
  author    = {Polyak, B. T. and Juditsky, A. B.},
  journal   = {SIAM Journal on Control and Optimization},
  title     = {Acceleration of Stochastic Approximation by Averaging},
  year      = {1992},
  issn      = {1095-7138},
  month     = jul,
  number    = {4},
  pages     = {838--855},
  volume    = {30},
  doi       = {10.1137/0330046},
  publisher = {Society for Industrial & Applied Mathematics (SIAM)},
}

@Article{Robbins1951Stochastic,
  author    = {Robbins, Herbert and Monro, Sutton},
  journal   = {The Annals of Mathematical Statistics},
  title     = {A Stochastic Approximation Method},
  year      = {1951},
  issn      = {0003-4851},
  month     = sep,
  number    = {3},
  pages     = {400--407},
  volume    = {22},
  doi       = {10.1214/aoms/1177729586},
  publisher = {Institute of Mathematical Statistics},
}

@InBook{Robbins1971convergence,
  author    = {Robbins, H. and Siegmund, D.},
  pages     = {233--257},
  publisher = {Elsevier},
  title     = {A convergence theorem for non negative almost supermartingales and some applications},
  year      = {1971},
  isbn      = {9780126045505},
  booktitle = {Optimizing Methods in Statistics},
  doi       = {10.1016/b978-0-12-604550-5.50015-8},
}

@Article{Roy2023Online,
  author  = {Roy, Abhishek and Balasubramanian, Krishnakumar},
  journal = {arXiv preprint arXiv:2308.01481},
  title     = {Online covariance estimation for stochastic gradient descent under {Markovian} sampling},
  year    = {2023},
}

@Article{Strohmer2008Randomized,
  author    = {Strohmer, Thomas and Vershynin, Roman},
  journal   = {Journal of Fourier Analysis and Applications},
  title     = {A Randomized {Kaczmarz} Algorithm with Exponential Convergence},
  year      = {2009},
  issn      = {1531-5851},
  month     = apr,
  number    = {2},
  pages     = {262--278},
  volume    = {15},
  doi       = {10.1007/s00041-008-9030-4},
  publisher = {Springer Science and Business Media LLC},
}

@Article{Wei2023Weighted,
  author  = {Wei, Ziyang and Zhu, Wanrong and Wu, Wei Biao},
  journal = {arXiv preprint arXiv:2307.06915},
  title   = {Weighted averaged stochastic gradient descent: Asymptotic normality and optimality},
  year    = {2023},
}

@Article{Chen2024Online,
  author    = {Chen, Xi and Lai, Zehua and Li, He and Zhang, Yichen},
  journal   = {Journal of the American Statistical Association},
  title     = {Online Statistical Inference for Stochastic Optimization via {Kiefer-Wolfowitz} Methods},
  year      = {2024},
  issn      = {1537-274X},
  month     = jan,
  number    = {548},
  pages     = {2972--2982},
  volume    = {119},
  doi       = {10.1080/01621459.2023.2296703},
  publisher = {Informa UK Limited},
}

@Article{Zhu2021Online,
  author    = {Zhu, Wanrong and Chen, Xi and Wu, Wei Biao},
  journal   = {Journal of the American Statistical Association},
  title     = {Online Covariance Matrix Estimation in Stochastic Gradient Descent},
  year      = {2023},
  issn      = {1537-274X},
  number    = {541},
  pages     = {393--404},
  volume    = {118},
  doi       = {10.1080/01621459.2021.1933498},
  publisher = {Informa UK Limited},
}

@Book{Hall2014Martingale,
  author    = {Hall, Peter and Heyde, Christopher C},
  publisher = {Academic Press},
  title     = {Martingale Limit Theory and Its Application},
  year      = {2014},
}

@Article{Na2025Statistical,
  author  = {Na, Sen and Mahoney, Michael},
  journal = {Journal of Machine Learning Research},
  title   = {Statistical inference of constrained stochastic optimization via sketched sequential quadratic programming},
  year    = {2025},
  number  = {33},
  pages   = {1--75},
  volume  = {26},
}

@InProceedings{Jiang2025Online,
  title = 	 {Online Covariance Estimation in Nonsmooth Stochastic Approximation},
  author =       {Jiang, Liwei and Roy, Abhishek and Balasubramanian, Krishnakumar and Davis, Damek and Drusvyatskiy, Dmitriy and Na, Sen},
  booktitle = 	 {Proceedings of Thirty Eighth Conference on Learning Theory},
  pages = 	 {3079--3123},
  year = 	 {2025},
  publisher =    {PMLR}
}

@Article{Kuang2025Online,
  author    = {Kuang, Wei and Anitescu, Mihai and Na, Sen},
  journal   = {Information and Inference: A Journal of the IMA},
  title     = {Online covariance matrix estimation in sketched {Newton} methods},
  year      = {2026},
  issn      = {2049-8772},
  month     = {jun},
  number    = {2},
  volume    = {15},
  doi       = {10.1093/imaiai/iaag012},
  publisher = {Oxford University Press (OUP)},
  pages     = {iaag012},
}

@Article{Leluc2023Asymptotic,
  author  = {Leluc, R{\'e}mi and Portier, Fran{\c{c}}ois},
  journal = {Transactions on Machine Learning Research},
  title   = {Asymptotic analysis of conditioned stochastic gradient descent},
  year    = {2023},
}

@Article{Cenac2025efficient,
  author    = {Cénac, Peggy and Godichon-Baggioni, Antoine and Portier, Bruno},
  journal   = {Bernoulli},
  title     = {An efficient averaged stochastic {Gauss-Newton} algorithm for estimating parameters of nonlinear regressions models},
  year      = {2025},
  issn      = {1350-7265},
  month     = feb,
  number    = {1},
  volume    = {31},
  doi       = {10.3150/23-bej1637},
  publisher = {Bernoulli Society for Mathematical Statistics and Probability},
  pages     = {1--29},
}

@InProceedings{Gower2018Accelerated,
  title     = {Accelerated stochastic matrix inversion: general theory and speeding up {BFGS} rules for faster second-order optimization},
  author={Gower, Robert and Hanzely, Filip and Richt{\'a}rik, Peter and Stich, Sebastian U},
  booktitle={Advances in Neural Information Processing Systems},
  volume={31},
  year={2018}
}

@Book{Nesterov2018Lectures,
  author    = {Nesterov, Yurii},
  publisher = {Springer International Publishing},
  title     = {Lectures on Convex Optimization},
  year      = {2018},
  isbn      = {9783319915784},
  doi       = {10.1007/978-3-319-91578-4},
  issn      = {1931-6836},
  journal   = {Springer Optimization and Its Applications},
}

@Book{Bhatia1997Matrix,
  author    = {Bhatia, Rajendra},
  publisher = {Springer New York},
  title     = {Matrix Analysis},
  year      = {1997},
  isbn      = {9781461206538},
  doi       = {10.1007/978-1-4612-0653-8},
  issn      = {0072-5285},
  journal   = {Graduate Texts in Mathematics},
}

@Article{Du2025Online,
  author  = {Du, Xinchen and Zhu, Wanrong and Wu, Wei Biao and Na, Sen},
  journal = {arXiv preprint arXiv:2505.18327},
  title   = {Online Statistical Inference of Constrained Stochastic Optimization via Random Scaling},
  year    = {2025},
}

@Article{Derezinski2025Fine,
  author  = {Derezi{\'n}ski, Michal and LeJeune, Daniel and Needell, Deanna and Rebrova, Elizaveta},
  journal = {Journal of Machine Learning Research},
  title   = {Fine-grained analysis and faster algorithms for iteratively solving linear systems},
  year    = {2025},
  number  = {144},
  pages   = {1--49},
  volume  = {26},
}

@InProceedings{McMahan2013Ad,
  author    = {McMahan, H Brendan and Holt, Gary and Sculley, David and Young, Michael and Ebner, Dietmar and Grady, Julian and Nie, Lan and Phillips, Todd and Davydov, Eugene and Golovin, Daniel and others},
  booktitle = {Proceedings of the 19th ACM SIGKDD International Conference on Knowledge Discovery and Data Mining},
  title     = {Ad click prediction: a view from the trenches},
  year      = {2013},
  pages     = {1222--1230},
}

@InProceedings{Li2010Contextual,
  author    = {Li, Lihong and Chu, Wei and Langford, John and Schapire, Robert E},
  booktitle = {Proceedings of the 19th International Conference on World Wide Web},
  title     = {A contextual-bandit approach to personalized news article recommendation},
  year      = {2010},
  pages     = {661--670},
}

@Article{Bertsekas2000Gradient,
  author    = {Bertsekas, Dimitri P and Tsitsiklis, John N},
  journal   = {SIAM Journal on Optimization},
  title     = {Gradient convergence in gradient methods with errors},
  year      = {2000},
  number    = {3},
  pages     = {627--642},
  volume    = {10},
  publisher = {SIAM},
}

@InProceedings{Moulines2011Non,
  author    = {Moulines, Eric and Bach, Francis R.},
  booktitle = {Advances in Neural Information Processing Systems},
  title   = {Non-asymptotic analysis of stochastic approximation algorithms for machine learning},
  year    = {2011},
  volume  = {24},
}

@Article{Athey2021Policy,
  author    = {Athey, Susan and Wager, Stefan},
  journal   = {Econometrica},
  title     = {Policy Learning With Observational Data},
  year      = {2021},
  issn      = {0012-9682},
  number    = {1},
  pages     = {133--161},
  volume    = {89},
  doi       = {10.3982/ecta15732},
  publisher = {The Econometric Society},
}

@Article{Kleinberg2017Human,
  author    = {Kleinberg, Jon and Lakkaraju, Himabindu and Leskovec, Jure and Ludwig, Jens and Mullainathan, Sendhil},
  journal   = {The Quarterly Journal of Economics},
  title     = {Human Decisions and Machine Predictions},
  year      = {2018},
  issn      = {1531-4650},
  month     = {feb},
  number    = {1},
  pages     = {237--293},
  volume    = {133},
  doi       = {10.1093/qje/qjx032},
  publisher = {Oxford University Press (OUP)},
}

@Article{Psaros2023Uncertainty,
  author    = {Psaros, Apostolos F. and Meng, Xuhui and Zou, Zongren and Guo, Ling and Karniadakis, George Em},
  journal   = {Journal of Computational Physics},
  title     = {Uncertainty quantification in scientific machine learning: Methods, metrics, and comparisons},
  year      = {2023},
  issn      = {0021-9991},
  month     = Mar,
  pages     = {111902},
  volume    = {477},
  doi       = {10.1016/j.jcp.2022.111902},
  publisher = {Elsevier BV},
}

@Article{Garcia2015Comprehensive,
  author    = {Garc{\'i}a, Javier and Fern{\'a}ndez, Fernando},
  journal = {Journal of Machine Learning Research},
  title   = {A comprehensive survey on safe reinforcement learning},
  year    = {2015},
  number  = {42},
  pages   = {1437--1480},
  volume  = {16},
}

@Article{Brunke2022Safe,
  author    = {Brunke, Lukas and Greeff, Melissa and Hall, Adam W. and Yuan, Zhaocong and Zhou, Siqi and Panerati, Jacopo and Schoellig, Angela P.},
  journal   = {Annual Review of Control, Robotics, and Autonomous Systems},
  title     = {Safe Learning in Robotics: From Learning-Based Control to Safe Reinforcement Learning},
  year      = {2022},
  issn      = {2573-5144},
  month     = May,
  number    = {1},
  pages     = {411--444},
  volume    = {5},
  doi       = {10.1146/annurev-control-042920-020211},
  publisher = {Annual Reviews},
}

@Article{Gottesman2019Guidelines,
  author    = {Gottesman, Omer and Johansson, Fredrik and Komorowski, Matthieu and Faisal, Aldo and Sontag, David and Doshi-Velez, Finale and Celi, Leo Anthony},
  journal   = {Nature Medicine},
  title     = {Guidelines for reinforcement learning in healthcare},
  year      = {2019},
  issn      = {1546-170X},
  month     = Jan,
  number    = {1},
  pages     = {16--18},
  volume    = {25},
  doi       = {10.1038/s41591-018-0310-5},
  publisher = {Springer Science and Business Media LLC},
}

@article{Ruder2016Overview,
  title={An overview of gradient descent optimization algorithms},
  author={Ruder, Sebastian},
  journal={arXiv preprint arXiv:1609.04747},
  year={2016}
}

@Article{GodichonBaggioni2025Adaptive,
  author  = {Godichon-Baggioni, Antoine and Werge, Nicklas},
  journal = {Journal of Machine Learning Research},
  title     = {On Adaptive Stochastic Optimization for Streaming Data: A {Newton}'s Method with {$O(dN)$} Operations},
  year    = {2025},
  number  = {59},
  pages   = {1--49},
  volume  = {26},
}

@InProceedings{ShalevShwartz2009Stochastic,
  author    = {Shalev-Shwartz, Shai and Shamir, Ohad and Srebro, Nathan and Sridharan, Karthik},
  booktitle = {Proceedings of the 22nd Annual Conference on Learning Theory},
  title     = {Stochastic Convex Optimization},
  year      = {2009},
  url       = {https://www.learningtheory.org/colt2009/papers/018.pdf},
}

@InProceedings{Hajek1972Local,
  author    = {H{\'a}jek, Jaroslav},
  booktitle = {Proceedings of the Sixth Berkeley Symposium on Mathematical Statistics and Probability},
  title     = {Local asymptotic minimax and admissibility in estimation},
  year      = {1972},
  pages     = {175--194},
  volume    = {1},
}

@InProceedings{LeCam1972Limits,
  author       = {Le Cam, Lucien},
  booktitle    = {Proceedings of the Sixth Berkeley Symposium on Mathematical Statistics and Probability},
  title        = {Limits of experiments},
  year         = {1972},
  organization = {University of California Press Berkeley-Los Angeles},
  pages        = {245--261},
  volume       = {1},
}

@Article{Wang2026Inference,
  author  = {Wang, Haoxuan and Du, Xinchen and Na, Sen},
  journal = {arXiv preprint arXiv:2604.23436},
  title     = {Inference of Online {Newton} Methods with {Nesterov}'s Accelerated Sketching},
  year    = {2026},
}

@Article{Gao2025Online,
  author  = {Gao, Yihang and Ng, Michael K and Mahoney, Michael W and Na, Sen},
  journal = {arXiv preprint arXiv:2512.08948},
  title   = {Online Inference of Constrained Optimization: Primal-Dual Optimality and Sequential Quadratic Programming},
  year    = {2025},
}

@Book{Vaart1996Weak,
author = {van der Vaart, Aad W. and Wellner, Jon A.},
publisher = {Springer New York},
title = {Weak Convergence and Empirical Processes},
year = {1996},
isbn = {9781475725452},
doi = {10.1007/978-1-4757-2545-2},
issn = {0172-7397},
journal = {Springer Series in Statistics},
}

@Misc{Blackard1998Covertype,
  author       = {Blackard, Jock},
  title        = {{Covertype}},
  year         = {1998},
  howpublished = {UCI Machine Learning Repository},
}
